\documentclass{article}

\usepackage[margin=1in]{geometry}

\usepackage{graphicx}
\usepackage[round,authoryear]{natbib}

\usepackage{amsmath,amssymb,amsthm,mathtools,bm}
\usepackage{booktabs}
\usepackage{enumitem}
\usepackage{microtype}
\usepackage{placeins}

\usepackage{etoc}

\usepackage[
    colorlinks=true,
    citecolor=blue,
    linkcolor=black,
    urlcolor=blue
]{hyperref}

\mathtoolsset{showonlyrefs}

\newtheorem{theorem}{Theorem}[section]
\newtheorem{proposition}[theorem]{Proposition}
\newtheorem{lemma}[theorem]{Lemma}

\newtheorem{corollary}[theorem]{Corollary}
\newtheorem{remark}[theorem]{Remark}

\newtheorem{assumption}[theorem]{Assumption}

\newcommand{\one}{\mathbf{1}}

\title{From Geometry to Generalization: Why Row Normalization Can Beat Adam and Muon}

\author{
Jihwan Kim\textsuperscript{1,2}
\qquad
Dogyoon Song\textsuperscript{3}
\qquad
Chulhee Yun\textsuperscript{4,2}
\\[0.7em]
\textsuperscript{1}Seoul National University
\qquad
\textsuperscript{2}KAIST InnoCORE LLM
\\
\textsuperscript{3}University of California, Davis
\qquad
\textsuperscript{4}KAIST
\\[0.6em]
}

\date{}

\begin{document}

\maketitle

\vspace{-1.1cm}
\begin{center}
\small
\texttt{aqua4689@snu.ac.kr}
\qquad
\texttt{dgsong@ucdavis.edu}
\qquad
\texttt{chulhee.yun@kaist.ac.kr}
\end{center}
\vspace{1em}

\begin{abstract}
Different optimizers can fit the same training data while selecting classifiers with substantially different geometries, but whether this difference provably affects population performance remains unclear. We show that row-wise normalization can achieve strictly higher population accuracy than full-batch Adam, a proxy for random-reshuffling Adam, and exact-SVD Muon in high-dimensional multiclass classification. Under an isotropic Gaussian-cloud data model, this advantage arises because row normalization's class-wise Euclidean geometry asymptotically preserves the population decision-boundary directions, whereas Adam's coordinate-wise geometry and Muon's spectral geometry introduce nonvanishing distortions. Beyond isotropy, the advantage persists for full-batch training on class means with independently oriented class-mean and test-noise covariances. It holds for power-law spectra with class-mean exponent below one, even under heavily anisotropic test noise. When both covariances are diagonal and sufficiently close, the advantage over Adam can reverse, while applying the same random rotation to both restores it by changing only their alignment with Adam's coordinate axes. Synthetic and last-layer language-model experiments support the predicted advantage.
\end{abstract}

\section{Introduction}

Row-wise normalization has recently emerged as a simple and effective ingredient in optimizers for large language model training.\footnote{We use \(W\in\mathbb{R}^{d_{\mathrm{out}}\times d_{\mathrm{in}}}\); under the transposed convention, row-wise becomes column-wise normalization.}
\textsc{Scale} \citep{glentis2025memory} uses row-normalized SGD as its base method and adds momentum only to the final layer, reporting performance comparable to or better than Adam \citep{kingma2015adam} and Muon \citep{jordan2024muon} while using substantially less optimizer-state memory.\footnote{\textsc{Scale} uses the transposed weight convention and thus calls the same operation column-wise normalization.} Related row-level normalization operations also appear in recent methods such as SinkGD, NorMuon, RMNP, and MOGA \citep{scetbon2026gradient,li2025normuon,deng2026rmnp,xu2026width}. Together, these results provide growing empirical evidence that normalizing updates separately across output units can be beneficial for large-scale optimization.
Yet what drives the strong empirical performance of row-normalized methods remains poorly understood.

A natural hypothesis is that this behavior is tied to the geometry induced by the optimizer.
Recent implicit-bias results make this connection concrete: normalized steepest-descent methods can converge toward max-margin classifiers associated with their underlying norm geometry \citep{soudry2018implicit, gunasekar2018characterizing, fan2026implicit,li2026implicit,baek2026implicit}. 
Under the implicit-bias characterizations used in this work, row-normalized methods are associated with the row norm\footnote{In our paper, row norm means \(\|W\|_{\mathrm{row}}:= \max_{i}\|W_{i,:}\|_2\).}, Muon with the spectral norm, and full-batch Adam with the entrywise max-norm.
These distinct geometries can therefore lead the methods toward classifiers with substantially different decision boundaries, even when they all fit the training data well.
What remains unclear is whether these geometric differences translate into systematic differences in population performance. This motivates our central question: \emph{can row-wise normalization provably achieve higher population accuracy than competing optimizers, and if so, what geometric mechanism explains the advantage?}

Existing theory does not answer this question. The closest results analyze row-normalized methods through nonconvex convergence or width-scaling properties, but do not establish that they select statistically better solutions than competing optimizers \citep{deng2026rmnp,xu2026width}.

To the best of our knowledge, we provide the first such result. We analyze the statistical quality of the limiting classifiers selected by different optimizer geometries and establish a strict population-accuracy advantage for row-wise normalization in high-dimensional Gaussian data models. Our analysis also identifies the geometric mechanism behind this advantage: the class-wise Euclidean geometry of row-wise normalization better preserves the population decision-boundary directions, whereas coordinate-wise and spectral geometries introduce systematic directional distortions.

\begin{figure}[t]
    \centering
    \includegraphics[width=\linewidth]{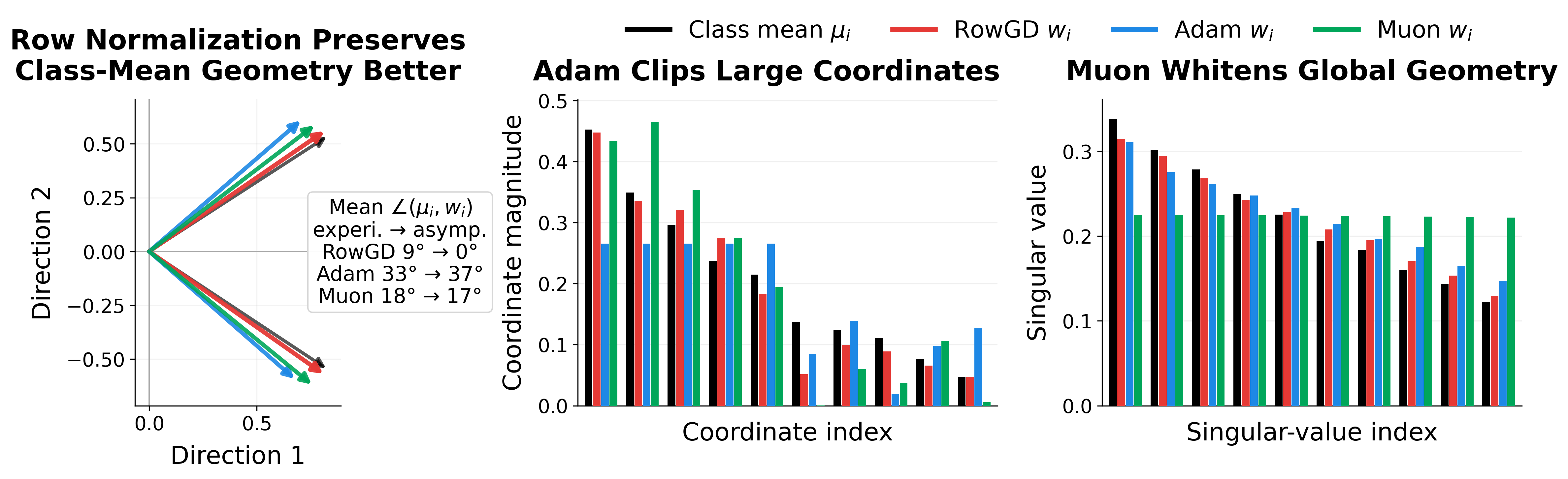}
    \caption{\textbf{Geometry of the learned classifiers.}
    Class means are compared with classifier rows learned by row-normalized GD, Adam, and Muon for 60 classes in 20 dimensions: a 2D projection of two classes, with mean angles over all classes and their proportional-limit predictions from Proposition~\ref{prop:limiting-geometry} (left); coordinate magnitudes at ranks \(1,3,\ldots,19\) after sorting by class-mean magnitude (middle); and singular values at indices \(1,3,\ldots,19\) of the full matrices (right).}
    \vspace{-4pt}
    \label{fig:optimizer-geometry-summary}
\end{figure}

\subsection{Our Contributions}

We study how optimizer geometry affects population accuracy in high-dimensional multiclass classification. Our main contributions are the following.

\begin{itemize}[leftmargin=*]
\item \textbf{Population-accuracy ordering under an isotropic Gaussian-cloud data model.} We first establish a strict population-accuracy advantage of row-normalized optimizers over representative spectral, coordinate-wise, and adaptive alternatives under our isotropic Gaussian-cloud data model. As an ingredient, we extend existing normalized steepest-descent implicit-bias analyses \citep{fan2026implicit, li2026implicit} to the row norm. Specifically, we show the advantage in three settings as follows.
\begin{itemize}
    \item \textbf{Deterministic full-batch setting.} Row-normalized GD and row-normalized GD with momentum outperform Adam, SignGD, Signum, exact-SVD Muon, and Spectral-GD (Theorem~\ref{thm:det-optimizer-accuracy-order-clean}).
    \item \textbf{Stochastic random-reshuffling setting.} Under random reshuffling, row-normalized SGD with momentum outperforms stochastic Signum and stochastic exact-SVD Muon (Proposition~\ref{thm:stoch-optimizer-accuracy-order-clean}).
    \item \textbf{Random-reshuffling Adam via AdamProxy.} We extend the AdamProxy framework of \citet{baek2026implicit} from incremental Adam to random-reshuffling mini-batch Adam. We prove that the resulting proxy approximates the one-epoch displacement of random-reshuffling Adam in the $\beta_1,\beta_2\to1$ regime (Proposition~\ref{thm:rr-adamproxy-epoch-high-memory}) and show that row-normalized SGD with momentum achieves strictly higher population accuracy than AdamProxy (Proposition~\ref{thm:adamproxy-optimizer-accuracy-order-clean}).
\end{itemize}

\item \textbf{Extension beyond isotropy.} We show that the advantage of row-normalized methods is not limited to isotropic data distributions.
For deterministic full-batch training on the class means, we consider class-mean and test-noise covariances with independent random eigenbases and power-law spectra, allowing any class-mean exponent below one and any test-noise exponent. Under these conditions, we prove that the row-normalized methods continue to outperform Adam and Muon (Theorem~\ref{thm:aniso-dd-ordering}).
We further exhibit a coordinate-aligned anisotropic regime in which row-normalized methods are worse than Adam, and show that jointly rotating the class means and test-noise covariance by an independently sampled Haar matrix restores the advantage with high probability (Proposition~\ref{prop:main-coordinate-reversal}).
\end{itemize}

\subsection{Related Work}
\label{rel}

Here we discuss only the works most directly related to our problem formulation and proof strategy; additional related work is deferred to Appendix~\ref{app:add-rel}.

\paragraph{Row-normalized optimizers in practice.}
Recent LLM optimizers use output-wise row normalization as a principal matrix update in \textsc{Scale}, RMNP, and MOGA; as one step of alternating row--column normalization in SinkGD; and as row-wise adaptive rescaling after Muon orthogonalization in NorMuon \citep{glentis2025memory,scetbon2026gradient,li2025normuon,deng2026rmnp,xu2026width}. These works establish row normalization as a practical optimizer component, whereas we theoretically compare the quality of the solutions it selects against competing optimizer geometries.

\paragraph{Theoretical comparisons involving row normalization.}
RMNP proves nonconvex convergence guarantees for row-normalized momentum methods that match existing guarantees for Muon \citep{deng2026rmnp}. MOGA compares the width dependence of smoothness and learning-rate transfer across matrix geometries, showing favorable scaling for row-normalized methods \citep{xu2026width}.

\paragraph{Implicit bias and optimizer-selected solutions.}
Prior work characterizes the norm-induced max-margin directions of full-batch SignGD, Spectral-GD, Muon, and Adam in separable multiclass linear classification \citep{fan2026implicit}. This framework has been extended to random-reshuffling stochastic optimization, where batch size and momentum govern approximation to the full-batch margin \citep{li2026implicit}. 
Separately, \citet{baek2026implicit} show that incremental Adam can have a distinct stochastic implicit bias and introduce an AdamProxy characterized by a data-dependent Mahalanobis margin problem.
We add the row norm and extend AdamProxy to multiclass random reshuffling.

\section{Preliminaries}
\label{sec:preliminaries}

\paragraph{Notation.}
For \(k\ge2\), let \([k]:=\{1,\ldots,k\}\) and \(e_1,\ldots,e_k\) be the standard basis of
\(\mathbb R^k\). 
For positive integers \(b,N\), we write \(b\mid N\) if \(N\) is divisible by \(b\).
For matrices, \(\langle A,B\rangle:=\operatorname{tr}(A^\top B)\),
\(\|A\|_{\max}:=\max_{r,j}|A_{r,j}|\), and
\(\|A\|_{\mathrm{sum}}:=\sum_{r,j}|A_{r,j}|\); we use
\(\|A\|_{\mathrm{op}}\) and \(\|A\|_{\mathrm F}\) for the spectral and Frobenius norms.
For \(W\in\mathbb R^{k\times d}\) with rows \(w_r^\top\), define
\(\|W\|_{\mathrm{row}}:=\|W\|_{2,\infty}:=\max_{r \in [k]}\|w_r\|_2\) as a row norm.
For \(a\in\mathbb R^k\) and \(c\in[k]\), define \(\mathbb{S}_c(a):=e^{a_c}/\sum_{\ell=1}^k e^{a_\ell}\), the \(c\)-th component of the softmax map.
We write \(A^{\odot p}\) for entrywise powers and \(\oslash\) for entrywise division, with \(0/0:=0\).
For nonzero vectors \(u,v\), let \(\angle(u,v)\in[0,\pi]\) denote the angle between \(u\) and \(v\).
Write \(\widehat v:=v/\|v\|_2\) for \(v\neq{\mathbf 0}\).
For random quantities indexed by \(d\), \(o_{\mathbb P}(1)\) denotes a term that converges to zero in probability as \(d\to\infty\).

\paragraph{Data model.}
We collect the structural assumptions on the data distribution in the following Gaussian-cloud data model; asymptotic and optimization conditions appear separately below.
Let \(d\) be the input dimension, \(k\) the number of classes, and \(n_i\) the class-\(i\) sample size, \(N:=\sum_{i=1}^k n_i\) and \(n_{\min}:=\min_i n_i\).

\begin{assumption}[Isotropic Gaussian-cloud data model]
\label{ass:gaussian-cloud}
The training dataset consists of \((x_{i,s},i)\) for \(i\in[k]\) and \(s\in[n_i]\) such that
\begin{equation}
x_{i,s}:=\mu_i+r z_{i,s},
    \qquad\text{where}\qquad \mu_i\overset{\mathrm{iid}}{\sim}\mathcal N({\mathbf 0},I_d),\qquad z_{i,s}\overset{\mathrm{iid}}{\sim}\mathcal N({\mathbf 0},I_d).
\label{eq:gaussian-cloud}
\end{equation}
A test example \(X^{\mathrm{te}}\) with label \(Y\) is
\begin{equation}
X^{\mathrm{te}}:=\mu_Y+r_{\mathrm{te}}z^{\mathrm{te}},\qquad \mathbb P(Y=i)=\pi_{i,d},\qquad z^{\mathrm{te}}\sim\mathcal N({\mathbf 0},I_d),
\label{eq:population-test-model-clean}
\end{equation}
where \((\mu_i)_{i=1}^k\), \((z_{i,s})_{i,s}\), \(Y\), and \(z^{\mathrm{te}}\) are mutually independent. 
Here, \(r=r(d)\ge0\), \(r_{\mathrm{te}}=r_{\mathrm{te}}(d)>0\), \(\pi_{i,d}>0\), and \(\sum_{i=1}^k\pi_{i,d}=1\).
\end{assumption}

We also use the flattened notation \(\{(x_n,y_n)\}_{n=1}^N\) for the same training set, where \(x_n=x_{i,s}\) and \(y_n=i\) under some fixed bijection between \(\{(i,s):i\in[k],\,s\in[n_i]\}\) and \([N]\).

\begin{assumption}[High-dimensional scaling]
\label{ass:asymptotic-regime}
Along a sequence of models indexed by \(d\), with \(k=k(d)\), \(n_i=n_i(d)\), and \(\pi_{\min,d}:=\min_i\pi_{i,d}\),
\begin{equation}
k/d\to\alpha\in(0,\infty),\qquad n_{\min}\to\infty,\qquad \log N=o(d),\qquad \log(1/\pi_{\min,d})=o(d).
\label{eq:asymptotic-regime-clean}
\end{equation}
Furthermore, \(r_0:=\sup_d r(d)<\infty\) and \(r_{\mathrm{te},0}:=\sup_d r_{\mathrm{te}}(d)<\infty\).
\end{assumption}

We focus on a proportional regime in which \(k\) and \(d\) remain on the same scale, excluding settings where one asymptotically dominates the other; this is qualitatively compatible with language-model last layers. The condition \(\log N=o(d)\) rules out datasets exponential in \(d\), so that the geometry is not dominated by rare deviations accumulated over exponentially many samples. Meanwhile, \(\log(1/\pi_{\min,d})=o(d)\) allows class frequencies following Zipf's law, as observed in language data.

\paragraph{Classifier and loss.}
A linear multiclass classifier is \(W\in\mathbb R^{k\times d}\), with score vector \(Wx\) at \(x\in\mathbb R^d\). We train with multiclass cross-entropy
\begin{equation}
L(W):=-\frac1N\sum_{i=1}^k\sum_{s=1}^{n_i}\log \mathbb{S}_i(Wx_{i,s}).
\label{eq:ce-loss-clean}
\end{equation}
This standard setting \citep{fan2026implicit,li2026implicit} abstracts a language model's final linear layer.

\paragraph{Population accuracy.}
We measure population performance by the probability that the true class attains the strictly largest linear score on a test example:
\begin{equation}
\operatorname{Acc}_{\pi,d}(W):=\sum_{i=1}^k\pi_{i,d}\,\mathbb P\!\left(w_i^\top X^{\mathrm{te}}> w_j^\top X^{\mathrm{te}}\ \text{for all }j\in[k]\setminus\{i\}\mid Y=i\right).
\label{eq:population-accuracy-clean}
\end{equation}

\paragraph{Norm-based update maps.}
We use three norm-based update maps. For the row norm \(\psi_{\mathrm{row}}:=\|\cdot\|_{\mathrm{row}}\), set \((\Delta_{\mathrm{row}}(G))_{r,:}:=G_{r,:}/\|G_{r,:}\|_2\) when \(G_{r,:}\neq{\mathbf 0}\), and \((\Delta_{\mathrm{row}}(G))_{r,:}:={\mathbf 0}\) otherwise. For the entrywise max-norm \(\psi_{\max}:=\|\cdot\|_{\max}\), let \(\Delta_{\max}(G):=\operatorname{sign}(G)\) entrywise, with \(\operatorname{sign}(0):=0\). For the spectral norm \(\psi_{\mathrm{sp}}:=\|\cdot\|_{\mathrm{op}}\), if \(G\neq{\mathbf 0}\), let \(G=U\Sigma V^\top\) be a compact SVD and define \(\Delta_{\mathrm{sp}}(G):=UV^\top\); set \(\Delta_{\mathrm{sp}}({\mathbf 0}):={\mathbf 0}\). For each \(g\in\{\mathrm{row},\mathrm{sp},\max\}\), these choices satisfy \(\Delta_g(G)\in\operatorname*{arg\,max}_{\psi_g(\Delta)\le1}\langle G,\Delta\rangle\), although the maximizer need not be unique.

\paragraph{Full-batch and random-reshuffling optimizers.}
For random reshuffling (RR), assume \(b\mid N\), set \(m:=N/b\), and independently permute and partition the data into \(m\) mini-batches of size \(b\) each epoch. Let \(B_t\) be the batch at step \(t\), define \(L_{B_t}(W):=|B_t|^{-1}\sum_{(i,s)\in B_t}-\log\mathbb S_i(Wx_{i,s})\), and set \(G_t:=\nabla L(W_t)\) under full-batch training and \(G_t:=\nabla L_{B_t}(W_t)\) under RR.
For \(g\in\{\mathrm{row},\mathrm{sp},\max\}\), \(\beta\in[0,1)\), and \(M_0={\mathbf 0}\), we write
\begin{equation}
M_{t+1}=\beta M_t+(1-\beta)G_t,\qquad
W_{t+1}=W_t-\eta_t\Delta_g(M_{t+1}).
\label{eq:unified-updates-clean}
\end{equation}
Under full-batch training, \(\beta=0\) gives row-normalized gradient descent (RowGD), spectral gradient descent (SpecGD), and sign gradient descent (SignGD) for \(g=\mathrm{row},\mathrm{sp},\max\), respectively; with \(\beta\in(0,1)\), the corresponding methods are row-normalized gradient descent with momentum (RowGDM), exact-SVD Muon (Muon),\footnote{Practical Muon \citep{jordan2024muon} approximates the orthogonalization step using a finite number of Newton--Schulz matrix-polynomial iterations. We instead use the exact SVD polar factor \(UV^\top\); this is equivalent to adding momentum to SpecGD and gives the variant we refer to as exact-SVD Muon.} and Signum. Under RR, we denote the momentum counterparts by RR-RowSGDM, RR-Muon, and RR-Signum. Thus, throughout the theory, Muon and RR-Muon refer to the exact-SVD variants.
Throughout, \(W_t^{\mathrm{opt}}\) denotes the parameter iterate of optimizer \(\mathrm{opt}\) at step \(t\).

\begin{assumption}[Bounded number of mini-batches]
\label{ass:bounded-minibatches}
Under random-reshuffling training with batch size \(b=b_d\), the number of mini-batches per epoch, \(m_d:=N(d)/b_d\), satisfies \(m_0:=\sup_d m_d<\infty\).
\end{assumption}

\paragraph{Adam optimizer.}
We consider Adam without bias correction or a numerical-stability constant, with \(0<\beta_1\le\beta_2<1\), using \(G_t\) from the corresponding full-batch or random-reshuffling setting:
\begin{equation}
\begin{aligned}
M_{t+1}=\beta_1M_t+(1-\beta_1)G_t, \mspace{15mu}
V_{t+1}=\beta_2V_t+(1-\beta_2)G_t^{\odot2}, \mspace{15mu}
W_{t+1}=W_t-\eta_tM_{t+1}\oslash V_{t+1}^{\odot1/2}.
\end{aligned}
\label{eq:adam-update-clean}
\end{equation}
We initialize \(M_0=V_0={\mathbf 0}\), and refer to the full-batch and random-reshuffling versions as Adam and RR-Adam, respectively.
The following condition holds almost surely at $W_0 = 0$.

\begin{assumption}[Full-batch Adam initialization]
\label{ass:adam-init}
For full-batch Adam, every entry of \(\nabla L(W_0)\) is nonzero.
\end{assumption}

\paragraph{Learning-rate schedule.} We use the following learning-rate conditions.
\begin{assumption}[Base learning-rate condition]
\label{ass:lr-base}
The sequence \(\{\eta_t\}_{t\ge0}\) is positive and nonincreasing, with \(\eta_t\to0\) and \(\sum_{t=0}^\infty \eta_t=\infty\).
\end{assumption}

\begin{assumption}[Momentum learning rate]
\label{ass:lr-momentum}
For every \(\beta\in(0,1)\) and \(c_1>0\), there exist \(t_0\in\mathbb N\) and \(c_2>0\) such that \(\sum_{s=0}^{t}\beta^s\bigl(\exp(c_1\sum_{\tau=1}^{s}\eta_{t-\tau})-1\bigr)\le c_2\eta_t\) for all \(t\ge t_0\).
\end{assumption}

Non-momentum methods require only Assumption~\ref{ass:lr-base}; momentum methods additionally require Assumption~\ref{ass:lr-momentum}. Both hold for \(\eta_t=c(t+1)^{-a}\), \(c>0\), \(a\in(0,1]\).

\section{Main Results}
\label{sec:main-results}

Section~\ref{subsec:deterministic-clean} establishes the deterministic optimizer population-accuracy ordering and identifies its limiting geometries. Section~\ref{subsec:stochastic-clean} studies robustness to random-reshuffling stochastic training, while Section~\ref{sec:anisotropic-extension} extends the analysis beyond the isotropic Gaussian-cloud data model.

\subsection{Comparison of Optimizers in the Deterministic Setting}
\label{subsec:deterministic-clean}

We first compare the full-batch optimizers defined in Section~\ref{sec:preliminaries}. Our first main result shows that RowGD and RowGDM achieve strictly higher population accuracy than SpecGD, Muon, SignGD, Signum, and Adam.

\begin{theorem}[Deterministic optimizer accuracy ordering]
\label{thm:det-optimizer-accuracy-order-clean}
Define the full-batch optimizer groups
\begin{equation}
\mathcal O_{\mathrm{row}}:=\{\mathrm{RowGD},\mathrm{RowGDM}\},\quad
\mathcal O_{\mathrm{comp}}:=\{\mathrm{SpecGD},\mathrm{Muon},\mathrm{SignGD},\mathrm{Signum},\mathrm{Adam}\}.
\label{eq:det-optimizer-groups}
\end{equation}
Suppose Assumptions~\ref{ass:gaussian-cloud}, \ref{ass:asymptotic-regime}, \ref{ass:adam-init}, \ref{ass:lr-base}, and~\ref{ass:lr-momentum} hold. Then, for every sufficiently large \(d\), with probability at least \(1-e^{-\sqrt d}\) over the training data, the following holds for all momentum parameters \(\beta \in (0,1)\) of RowGDM, Muon, and Signum:
\begin{equation}
\min_{\mathrm{opt}\in\mathcal O_{\mathrm{row}}}\liminf_{t\to\infty}\operatorname{Acc}_{\pi,d}\!\left(W_t^{\mathrm{opt}}\right)
>
\max_{\mathrm{opt}\in\mathcal O_{\mathrm{comp}}}\limsup_{t\to\infty}\operatorname{Acc}_{\pi,d}\!\left(W_t^{\mathrm{opt}}\right).
\label{eq:det-optimizer-accuracy-order-clean}
\end{equation}
\end{theorem}

\paragraph{Proof idea.}
The proof proceeds through the implicit biases of the optimizers. RowGD and RowGDM approach row-norm margin solutions, SpecGD and Muon approach spectral-norm margin solutions, and SignGD, Signum, and Adam approach max-norm margin solutions. Here, a \(\psi\)-norm margin solution minimizes \(\psi(W)\) among classifiers with training margin at least one; the formal solution sets are defined in Appendix~\ref{app:deterministic-margin}. Under the Gaussian-cloud model, these three solution sets have the distinct asymptotic classifier structures summarized in Proposition~\ref{prop:limiting-geometry}, and Gaussian pairwise-error bounds convert this geometric separation into the population-accuracy ordering.
Because Theorem~\ref{thm:det-optimizer-accuracy-order-clean} is proven through a uniform comparison over all norm-induced margin solutions approached by the optimizers, it does not require directional convergence of optimizers.

\begin{proposition}[Limiting classifier geometry]
\label{prop:limiting-geometry}
Let \(W_t^{\rm opt}\) have rows \((w_{i,t}^{\rm opt})^\top\), and let \(\Theta_t^{\rm opt}:=k^{-1}\sum_{i=1}^k\angle(w_{i,t}^{\rm opt},\mu_i)\) denote the average angle between the class mean and classifier row. 
Let \(M_\mu\in\mathbb R^{k\times d}\) have rows \(\mu_i^\top\), let \(P:=I_k-k^{-1}\mathbf1\mathbf1^\top\) and \(r_d:=\min\{k-1,d\}\), and let \(C_\mu:=PM_\mu\) have compact SVD \(U_\mu\Sigma_\mu V_\mu^\top\) and polar factor \(Q_\mu:=U_\mu V_\mu^\top\). Suppose Assumptions~\ref{ass:gaussian-cloud}, \ref{ass:asymptotic-regime}, \ref{ass:adam-init}, \ref{ass:lr-base}, and~\ref{ass:lr-momentum} hold. Then
\[
\begin{aligned}
\text{\rm RowGD/RowGDM:}\quad&
\limsup_{t\to\infty}\max_i\bigl\|\widehat w_{i,t}^{\rm opt}-\widehat\mu_i\bigr\|_2=o_{\mathbb P}(1),\\
\text{\rm SignGD/Signum/Adam:}\quad&
\limsup_{t\to\infty}\max_i\left\|\widehat w_{i,t}^{\rm opt}-\frac{\operatorname{sign}(\mu_i)}{\sqrt d}\right\|_2=o_{\mathbb P}(1),\\
\text{\rm SpecGD/Muon:}\quad&
\limsup_{t\to\infty}\frac1{r_d}\left\|\frac{PW_t^{\rm opt}}{\|PW_t^{\rm opt}\|_{\rm op}}-Q_\mu\right\|_{\rm F}^2=o_{\mathbb P}(1).
\end{aligned}
\]
Moreover, \(\limsup_{t\to\infty}|\Theta_t^{\rm opt}-\theta|=o_{\mathbb P}(1)\), where \(\theta=0\) for RowGD/RowGDM and \(\theta=\arccos\sqrt{2/\pi}\) for SignGD/Signum/Adam. 
If the aspect ratio \(\alpha>1\), then \(\theta=\arccos\kappa_{\rm sp}(\alpha)\) for SpecGD/Muon, where \(\kappa_{\rm sp}:(0,\infty)\to(0,1)\) is a positive constant depending only on \(\alpha\).\footnote{Its explicit formula is given in \eqref{eq:app-det-all-aspect-scales}. Numerically, \(\kappa_{\rm sp}(1)\approx0.84883\) and \(\kappa_{\rm sp}(2)\approx0.93280\).}
\end{proposition}

\paragraph{Underlying mechanism.}
Proposition~\ref{prop:limiting-geometry} makes the source of the accuracy gap explicit. RowGD and RowGDM preserve each class-mean direction, driving \(w_i\) toward a positive multiple of \(\mu_i\). In contrast, SignGD, Signum, and Adam approach the coordinate-wise sign direction of each class mean, effectively equalizing coordinate magnitudes.
SpecGD and Muon approach the polar factor of the centered class-mean matrix, effectively equalizing its singular values. Both transformations introduce nonvanishing directional distortion relative to the original class-mean geometry. 
Under isotropic Gaussian test noise, preserving the class-mean geometry leads to better separation of unseen examples, while the distortions induced by the other methods degrade this separation. This yields the population-accuracy advantage of the row-normalized methods in Theorem~\ref{thm:det-optimizer-accuracy-order-clean}.

\begin{corollary}[Population-error exponents]
\label{cor:det-optimizer-error-exponents}
Let \(\mathcal E_{t,d}^{\mathrm{opt}} := 1-\operatorname{Acc}_{\pi,d}(W_t^{\mathrm{opt}})\). 
Suppose Assumptions~\ref{ass:gaussian-cloud}, \ref{ass:asymptotic-regime}, \ref{ass:adam-init}, \ref{ass:lr-base}, and~\ref{ass:lr-momentum} hold. 
Then
\begin{align*}
\limsup_{t\to\infty}\left|\frac{r_{\mathrm{te}}^2}{d}\log \! \frac{\mathcal E_{t,d}^{\mathrm{Adam}}}{\mathcal E_{t,d}^{\mathrm{RowGDM}}}-\frac{\pi-2}{4\pi}\right|=o_{\mathbb P}(1), \ \ \
\liminf_{t\to\infty}\frac{r_{\mathrm{te}}^2}{d}\log \! \frac{\mathcal E_{t,d}^{\mathrm{Muon}}}{\mathcal E_{t,d}^{\mathrm{RowGDM}}}\ge\frac{1-\kappa_{\rm sp}(\alpha)^2}{4} \! - \!o_{\mathbb P}(1),
\end{align*}
where \(\kappa_{\rm sp}\) is the constant as in Proposition~\ref{prop:limiting-geometry}. 
The same conclusion holds with RowGDM replaced by RowGD, Adam by SignGD or Signum, and Muon by SpecGD.
\end{corollary}

Corollary~\ref{cor:det-optimizer-error-exponents} follows from Proposition~\ref{cor:app-det-optimizer-error-exponents} in Appendix~\ref{app:deterministic-margin}.
While Theorem~\ref{thm:det-optimizer-accuracy-order-clean} establishes a strict ordering of the limiting population accuracies, Corollary~\ref{cor:det-optimizer-error-exponents} quantifies the separation.
Although all three population errors vanish as \(d\to\infty\), RowGDM has exponentially smaller population error than Adam and Muon. This separation is captured by the logarithms of the error ratios on the scale \(d/r_{\mathrm{te}}^2\).

\subsection{Robustness to Stochastic Training}
\label{subsec:stochastic-clean}
\label{subsec:rr-adamproxy-clean}

Mini-batch updates can alter the implicit bias of an optimizer, so the full-batch ordering does not automatically extend to stochastic training. Under RR, however, sufficiently large momentum keeps RR-RowSGDM, RR-Muon, and RR-Signum close to their respective row-, spectral-, and max-norm margin solutions \citep{li2026implicit}. We show that this approximate preservation is enough for the strict population-accuracy ordering to persist.

\begin{proposition}[Stochastic optimizer accuracy ordering]
\label{thm:stoch-optimizer-accuracy-order-clean}
Suppose Assumptions~\ref{ass:gaussian-cloud}, \ref{ass:asymptotic-regime}, \ref{ass:bounded-minibatches}, \ref{ass:lr-base}, and~\ref{ass:lr-momentum} hold. Then there exists \(\beta_0=\beta_0(\alpha,r_0,m_0)\in[0,1)\) such that, for every sufficiently large \(d\), with probability at least \(1-e^{-\sqrt d}\) over the training data, the following holds for all momentum parameters \(\beta \in [\beta_0,1)\) of RR-RowSGDM, RR-Muon, and RR-Signum, and for every realization of the reshuffling:
\begin{equation}
\liminf_{t\to\infty}\operatorname{Acc}_{\pi,d}\!\left(W_t^{\mathrm{RR\text{-}RowSGDM}}\right)
>
\max_{\mathrm{opt}\in\{\mathrm{RR\text{-}Muon},\,\mathrm{RR\text{-}Signum}\}}
\limsup_{t\to\infty}\operatorname{Acc}_{\pi,d}\!\left(W_t^{\mathrm{opt}}\right).
\label{eq:stoch-optimizer-accuracy-order-clean}
\end{equation}
\end{proposition}

The condition \(m_d\le m_0\) in Assumption~\ref{ass:bounded-minibatches} means that each mini-batch contains at least a \(1/m_0\) fraction of the training set. For \(m_0 \ge 2\), the sufficient momentum threshold used in the proof has \(1-\beta_0\) proportional to \(1/[m_0(m_0^2-1)]\).

\paragraph{Approximating Adam via AdamProxy.}
Adam requires separate treatment under RR because averaging squared mini-batch gradients generally differs from squaring the averaged gradient, so its full-batch max-norm margin characterization need not persist \citep{baek2026implicit}. To address this, we develop a deterministic proxy for RR-Adam in the multiclass setting with general mini-batch size, extending the binary incremental AdamProxy construction of \citet{baek2026implicit}.
Using the flattened indexing \(\{(x_n,y_n)\}_{n=1}^N\), let \(g_n(W):=-(e_{y_n}-\mathbb S(Wx_n))x_n^\top\) denote the gradient of the cross-entropy loss for sample \((x_n,y_n)\). For each batch \(B\subset[N]\) of size \(b\), set \(G_B(W):=b^{-1}\sum_{n\in B}g_n(W)=\nabla L_B(W)\). Thus, under RR, \(G_t=G_{B_t}(W_t)\) in \eqref{eq:adam-update-clean}.

Drawing \(B\) uniformly from the \(b\)-subsets of \([N]\), define the AdamProxy direction
\begin{equation}
\Phi_{N,b}(W):=\mathbb E_B[G_B(W)]\oslash\mathbb E_B\!\left[G_B(W)^{\odot2}\right]^{\odot1/2}.
\label{eq:adamproxy-direction-clean}
\end{equation}
Equivalently, letting \(S(W):=\sum_{n=1}^N g_n(W)\), \(Q(W):=\sum_{n=1}^N g_n(W)^{\odot2}\),
\begin{equation}
\Phi_{N,b}(W) =
\sqrt{\frac{b(N-1)}{N}}\, S(W)\oslash \left(  (b-1) S(W)^{\odot2}+ (N-b) Q(W)\right)^{\odot1/2}.
\label{eq:adamproxy-direction-explicit-clean}
\end{equation}
Proposition~\ref{thm:rr-adamproxy-epoch-high-memory} shows that, in the $\beta_1, \beta_2 \to 1$ regime, this deterministic direction approximates the normalized one-epoch displacement of RR-Adam.

\begin{proposition}[One-epoch \(\beta_1,\beta_2\to1\) RR-Adam approximation]
\label{thm:rr-adamproxy-epoch-high-memory}
For each epoch \(r=0,1,2,\ldots\), let \(\tau_r:=mr\) denote the index of its first update and \(\gamma_r:=\sum_{s=0}^{m-1}\eta_{\tau_r+s}\) its total learning rate. Fix a training set with \(x_n[a]\neq0\) for all \(n\in[N]\) and \(a\in[d]\). 
Suppose that \(b<N\) and that Assumptions~\ref{ass:lr-base} and~\ref{ass:lr-momentum} hold. Then, for every \(\varepsilon>0\), RR-Adam in \eqref{eq:adam-update-clean} satisfies
\begin{equation}
\lim_{\substack{(\beta_1,\beta_2)\to(1,1)\\
0<\beta_1\le\beta_2<1}}
\limsup_{r\to\infty}
\mathbb P\left(
\left\|
\frac{W_{\tau_{r+1}}^{\mathrm{RR\text{-}Adam}}-W_{\tau_r}^{\mathrm{RR\text{-}Adam}}}{\gamma_r}
+\Phi_{N,b}\!\left(W_{\tau_r}^{\mathrm{RR\text{-}Adam}}\right)
\right\|_{\max}>\varepsilon
\right)=0,
\label{eq:rr-adamproxy-one-epoch-limit-clean}
\end{equation}
where the probability is over the random reshuffling.
\end{proposition}

Proposition~\ref{thm:rr-adamproxy-epoch-high-memory} is a one-epoch $\beta_1, \beta_2 \to 1$ approximation at fixed data and dimension. It does not establish long-horizon agreement between the RR-Adam and AdamProxy trajectories.

Motivated by Proposition~\ref{thm:rr-adamproxy-epoch-high-memory}, let $\gamma_t:=\sum_{s=0}^{m-1}\eta_{mt+s}$ and consider the AdamProxy dynamics
\begin{equation}
W_{t+1}^{\mathrm{AP}}
=
W_t^{\mathrm{AP}}
-
\gamma_t\Phi_{N,b}\!\left(W_t^{\mathrm{AP}}\right).
\label{eq:adamproxy-dynamics-clean}
\end{equation}

\paragraph{Comparison of AdamProxy and RR-RowSGDM.}
Although AdamProxy need not approach the exact sign geometry of full-batch Adam, it retains a weaker coordinate-wise saturation effect.
Once the classifier is scaled to have unit minimum training margin, every coordinate eventually stays within a constant multiple of the typical coordinate size (Proposition~\ref{prop:app-adamproxy-boundary-separation}). 
This forces a nonvanishing distortion from the class-mean geometry, whereas high-momentum RR-RowSGDM remains close to the class-mean geometry, yielding the population-accuracy ordering below.

\begin{proposition}[RR-RowSGDM versus AdamProxy accuracy ordering]
\label{thm:adamproxy-optimizer-accuracy-order-clean}
Suppose Assumptions~\ref{ass:gaussian-cloud}, \ref{ass:asymptotic-regime}, \ref{ass:bounded-minibatches}, \ref{ass:lr-base}, and~\ref{ass:lr-momentum} hold. Then there exists \(\beta_1=\beta_1(\alpha,r_0,m_0)\in[0,1)\) such that, for every sufficiently large \(d\), with probability at least \(1-e^{-\sqrt d}\) over the training data, the following holds for all momentum parameters \(\beta\in[\beta_1,1)\) of RR-RowSGDM and for every realization of the reshuffling:
\begin{equation}
\liminf_{t\to\infty}\operatorname{Acc}_{\pi,d}\!\left(W_t^{\mathrm{RR\text{-}RowSGDM}}\right)
>
\limsup_{t\to\infty}\operatorname{Acc}_{\pi,d}\!\left(W_t^{\mathrm{AP}}\right).
\end{equation}
\end{proposition}

\section{Extensions to Anisotropic Gaussian Models}
\label{sec:anisotropic-extension}

In Section~\ref{sec:main-results}, we studied an isotropic Gaussian-cloud data model in which both the class means and the within-class noise are isotropic. 
In this section, we examine when the advantage of row-normalized methods persists or reverses beyond isotropy.
Motivated by the nontrivial covariance structure that can arise in language-model representations, we extend the analysis to anisotropic Gaussian distributions.

To simplify the analysis and cover a broader range of anisotropy, we consider clean training, with no within-class noise in the training data, under the following common data model.

\begin{assumption}[Anisotropic clean Gaussian model]
\label{ass:aniso-clean}
Let \(\Gamma,\Sigma_1,\ldots,\Sigma_k\) be positive-definite covariance matrices with \(\operatorname{tr}\Gamma=\operatorname{tr}\Sigma_i=d\) for every \(i\in[k]\).\footnote{A fixed trace preserves total variance while allowing covariance spectra and orientations to vary.} Training uses the clean class signals \(\{(\mu_i,i)\}_{i=1}^k\). A test example with label \(Y\) is \(X^{\mathrm{te}}:=\mu_Y+r_{\mathrm{te}}z^{\mathrm{te}}\), where, conditionally on the covariance matrices, \((Y,z^{\mathrm{te}})\) is independent of \((\mu_i)_{i=1}^k\) and
\begin{equation}
\mu_i\overset{\mathrm{iid}}{\sim}\mathcal N({\mathbf 0},\Gamma),\qquad
\mathbb P(Y=i)=\pi_{i,d}>0,\qquad
z^{\mathrm{te}}\mid Y\sim\mathcal N({\mathbf 0},\Sigma_Y).
\label{eq:aniso-training-model}
\end{equation}
Along the sequence indexed by \(d\), \(k/d\to\alpha\in(0,\infty)\) and \(r_{\mathrm{te}}>0\) satisfies \(r_{\mathrm{te},0}:=\sup_d r_{\mathrm{te}}(d)<\infty\).
\end{assumption}

\subsection{Persistence of Row-Norm Advantage under Independent Orientations}

We first specialize Assumption~\ref{ass:aniso-clean} to trace-normalized power-law covariances with independent Haar eigenbases.
Suppressing dimension subscripts, define \(T_\vartheta:=\sum_{a=1}^d a^{-\vartheta}\) and \(D_\vartheta:=\operatorname{diag}(1,2^{-\vartheta},\ldots,d^{-\vartheta})\). For \(s_\mu,s_\Sigma\ge0\), set
\begin{equation}
\Gamma:=\frac{d}{T_{s_\mu}}U_\mu D_{s_\mu}U_\mu^\top,
\qquad
\Sigma_i:=\frac{d}{T_{s_\Sigma}}O_iD_{s_\Sigma}O_i^\top.
\label{eq:aniso-covariances}
\end{equation}
Here \(U_\mu,O_1,\ldots,O_k\) are mutually independent Haar matrices on \(\mathsf O(d)\); the class-mean draws are independent of \((O_i)_{i=1}^k\), and the optimizer initializations are fixed and do not depend on \((O_i)_{i=1}^k\).

\begin{theorem}[Accuracy ordering under anisotropy with independent orientations]
\label{thm:aniso-dd-ordering}
Let \(\mathcal O_{\mathrm{row}}\) and \(\mathcal O_{\mathrm{comp}}\) be as in \eqref{eq:det-optimizer-groups}, with fixed \(\beta\in(0,1)\) for RowGDM, Muon, and Signum. Suppose Assumptions~\ref{ass:aniso-clean}, \ref{ass:adam-init}, \ref{ass:lr-base}, and~\ref{ass:lr-momentum} hold, with the covariances in \eqref{eq:aniso-covariances} for \(0\le s_\mu<1\) and \(s_\Sigma\ge0\). Assume that \(\log(1/\pi_{\min,d})=o(d/\log d)\) and that \(W_t^{\mathrm{opt}}/\|W_t^{\mathrm{opt}}\|_{\mathrm F}\) converges almost surely as \(t\to\infty\) for every \(\mathrm{opt}\in\mathcal O_{\mathrm{row}}\cup\mathcal O_{\mathrm{comp}}\). Then there exist constants \(c,C>0\), independent of \(d\), such that, for every sufficiently large \(d\), with probability at least \(1-Cd^{-c}\) over \(U_\mu\), \(O_1,\ldots,O_k\), and the class means,
\begin{equation}
\min_{\mathrm{opt}\in\mathcal O_{\mathrm{row}}}\liminf_{t\to\infty}\operatorname{Acc}_{\pi,d}\!\left(W_t^{\mathrm{opt}}\right)
>
\max_{\mathrm{opt}\in\mathcal O_{\mathrm{comp}}}\limsup_{t\to\infty}\operatorname{Acc}_{\pi,d}\!\left(W_t^{\mathrm{opt}}\right).
\label{eq:aniso-det-optimizer-ordering-clean-data}
\end{equation}
\end{theorem}

Here \(c,C\) can be chosen independently of \(s_\mu\) and \(s_\Sigma\), although the sufficiently-large-\(d\) threshold may depend on them. In particular, under independent test-noise orientations, the advantage of row-normalized methods persists for arbitrary test-noise anisotropy \(s_\Sigma\ge0\).

\begin{remark}[Role of directional convergence]
Unlike the accuracy comparisons in Section~\ref{sec:main-results}, which hold uniformly over all classifier families that the normalized iterates approach and therefore do not require directional convergence, Theorem~\ref{thm:aniso-dd-ordering} assumes directional convergence for each optimizer.
Related directional-convergence assumptions are common in the implicit-bias literature \citep{gunasekar2018implicit, nacson2019convergence, chizat2020implicit, baek2026implicit}.
Under clean training, conditional on the training data, the limit of \(W_t^{\mathrm{opt}}/\|W_t^{\mathrm{opt}}\|_{\mathrm F}\) is fixed and independent of the test-noise eigenbases, which enter only at test time. 
Hence, the concentration bounds over the Haar-random test-noise eigenbases apply to the fixed limiting direction.
\end{remark}

\paragraph{Structural intuition: effective signal dimension.}
The condition \(s_\mu<1\) keeps the class-distinguishing signal sufficiently spread across dimensions, so different class means have weak pairwise overlap. This allows the row-norm margin solutions to remain close to the corresponding class-mean directions, while satisfying the multiclass margin constraints. 
If the signal is instead concentrated in only a few directions, some class means can have small pairwise angles, and margin constraints can push the row-norm margin solutions away from the class-signal-preserving geometry.

\begin{remark}[Connection to language-model representations]
We find that the fitted class-mean spectral exponents can be smaller than one (\(s_\mu<1\)) in language-model representations; see Appendix~\ref{app:language-model-represen}. This provides a possible connection between the anisotropic regime covered by Theorem~\ref{thm:aniso-dd-ordering} and representation geometry in language models.
\end{remark}

\subsection{Coordinate-aligned Reversal and Recovery under Haar Rotation}
\label{subsec:coordinate-reversal}

We next show that the advantage of row-normalized methods can reverse when the class-mean and test-noise covariances are diagonal and sufficiently close, but is restored after applying the same independently sampled Haar rotation to the class means and test-noise covariance.

The underlying issue is the interaction between signal and test-noise geometry. With a common test-noise covariance \(\Sigma\), the normal of the Bayes boundary between classes \(i\) and \(j\) is proportional to \(\Sigma^{-1}(\mu_i-\mu_j)\). Thus, if directions that strongly separate the class means also have large test variance, the Bayes boundary can differ substantially from the raw class-mean difference that the row-norm margin solution tends to preserve.
The construction below places this signal--noise alignment along the coordinate axes used by Adam's entrywise normalization, where Adam's tendency to equalize coordinate magnitudes can reduce the influence of such high-variance coordinates.

Let the class-mean covariance be diagonal and let the common test-noise covariance interpolate between isotropic noise and the same covariance profile:
\begin{equation}
\Gamma_d:=\operatorname{diag}(\gamma_{1,d},\ldots,\gamma_{d,d}),
\qquad
\Sigma_{\lambda,d}:=(1-\lambda)I_d+\lambda\Gamma_d,
\qquad \lambda\in[0,1].
\label{eq:coordinate-aligned-covariances}
\end{equation}
Thus, \(\Sigma_{\lambda, d}\) interpolates between isotropic test noise at \(\lambda=0\) and the class-mean covariance at \(\lambda=1\).

\begin{proposition}[Coordinate dependence of the RowGDM--Adam ordering]
\label{prop:main-coordinate-reversal}
Let \(\lambda\in[0,1]\). 
Suppose that Assumptions~\ref{ass:aniso-clean}, \ref{ass:adam-init}, \ref{ass:lr-base}, and~\ref{ass:lr-momentum} hold with \(\Gamma=\Gamma_d\) and \(\Sigma_1=\cdots=\Sigma_k=\Sigma_{\lambda,d}\) as in \eqref{eq:coordinate-aligned-covariances}.
Suppose also that \(\log(1/\pi_{\min,d})=o(d)\), \(0<\inf_{a,d}\gamma_{a,d}\le\sup_{a,d}\gamma_{a,d}<\infty\), and the limits \(\bar\gamma_q:=\lim_{d\to\infty}d^{-1}\sum_{a=1}^d\gamma_{a,d}^q\) exist \(\forall q\in\{1/2,2\}\) and satisfy \(2\bar\gamma_{1/2}^2\bar\gamma_2>\pi\).\footnote{This requires a sufficiently dispersed spectrum. Corollary~\ref{cor:app-coordinate-dispersed-profiles} gives a sufficient condition.}
Then \(\lambda_c:=(\pi/(2\bar\gamma_{1/2}^2)-1)/(\bar\gamma_2-1)\) lies in \((0,1)\), and the following hold.\footnote{The same conclusions hold with RowGDM replaced by RowGD.}
\begin{enumerate}[label=(\roman*),leftmargin=2em,nosep]
    \item \emph{Diagonal covariance.} For every \(\lambda\in[0,1]\setminus\{\lambda_c\}\), there exist \(c_\lambda,C_\lambda>0\) such that, for every sufficiently large \(d\), with probability at least \(1-C_\lambda e^{-c_\lambda d}\) over the class means, for all \(\beta \in (0,1)\),
    \[
        \begin{aligned}
        \liminf_{t\to\infty}\operatorname{Acc}_{\pi,d}\!\left(W_t^{\mathrm{RowGDM}}\right)
        &>
        \limsup_{t\to\infty}\operatorname{Acc}_{\pi,d}\!\left(W_t^{\mathrm{Adam}}\right)
        &&\text{if }\lambda<\lambda_c,\\
        \liminf_{t\to\infty}\operatorname{Acc}_{\pi,d}\!\left(W_t^{\mathrm{Adam}}\right)
        &>
        \limsup_{t\to\infty}\operatorname{Acc}_{\pi,d}\!\left(W_t^{\mathrm{RowGDM}}\right)
        &&\text{if }\lambda>\lambda_c.
        \end{aligned}
    \]
    \item \emph{Recovery under Haar rotation.} Let \(Q\) be Haar-uniformly random on \(\mathsf{SO}(d)\) and independent of the class means, and replace every training and test input \(x\) by \(Qx\). There exist \(c_{\mathrm{rot}},C_{\mathrm{rot}}>0\) such that, for every fixed \(\lambda\in[0,1]\) and every sufficiently large \(d\), with probability at least \(1-C_{\mathrm{rot}}e^{-c_{\mathrm{rot}}d}\) over the class means and \(Q\), the RowGDM-over-Adam ordering holds.
\end{enumerate}
\end{proposition}

\paragraph{Implication.}
Theorem~\ref{thm:aniso-dd-ordering} shows that anisotropy alone does not eliminate the advantage of row-normalized methods under independent orientations.
Proposition~\ref{prop:main-coordinate-reversal} identifies a structured failure mode. The RowGDM-over-Adam ordering can reverse when the high-variance class-mean and test-noise directions align with each other and with Adam's normalization coordinates.
However, this reversal is coordinate-specific.
Applying the same Haar rotation to the class means and test-noise covariance preserves the covariance spectra and their relative alignment while randomizing their orientation relative to the coordinate axes.
This restores the RowGDM-over-Adam ordering.
Hence, class-mean--test-noise alignment alone is not sufficient; the reversal additionally relies on its alignment with Adam's coordinate axes.

\begin{figure}[t]
\centering
\begin{minipage}[t]{0.49\linewidth}
\centering
\includegraphics[width=\linewidth]{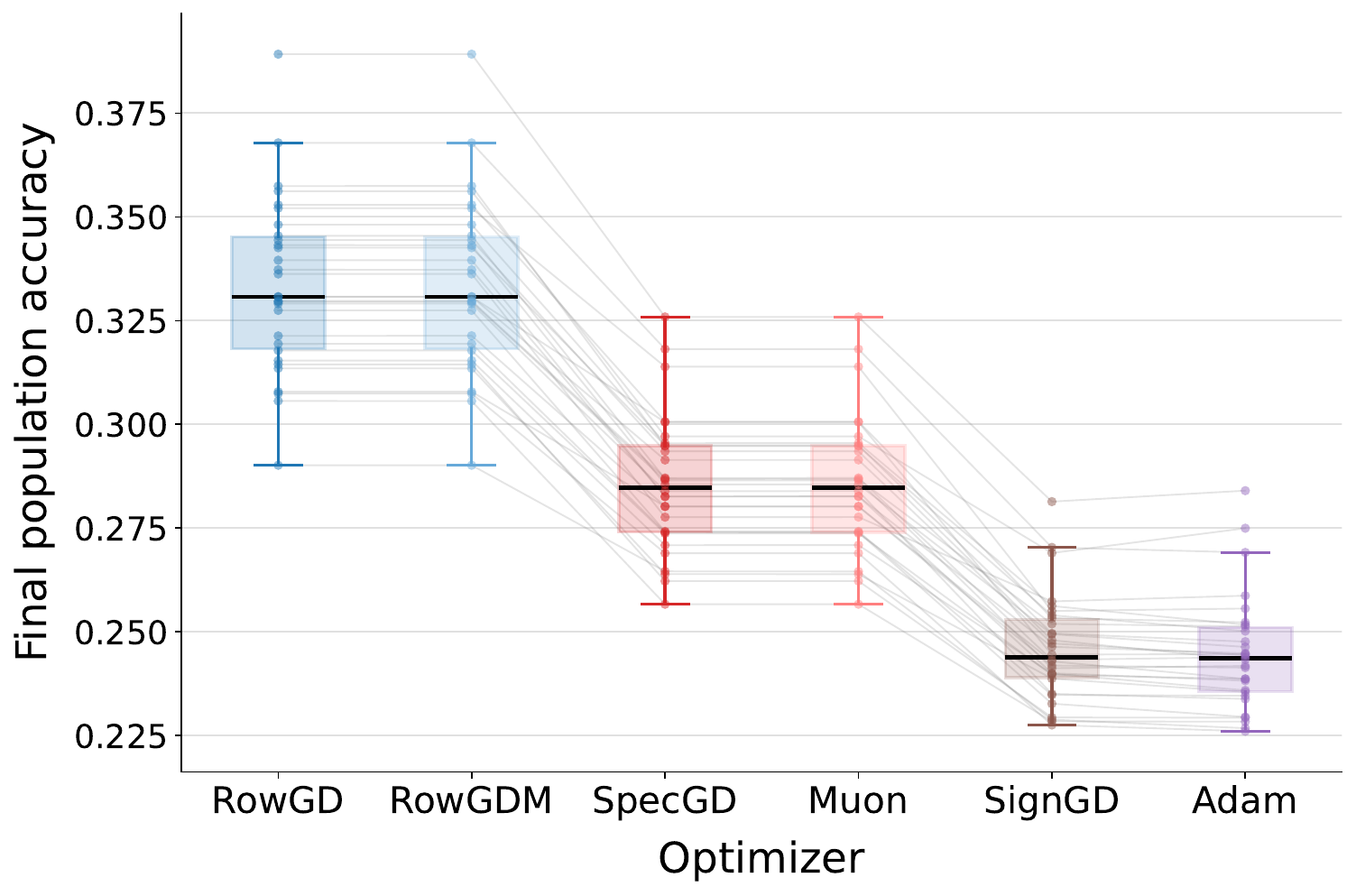}
\end{minipage}
\hfill
\begin{minipage}[t]{0.49\linewidth}
\centering
\includegraphics[width=\linewidth]{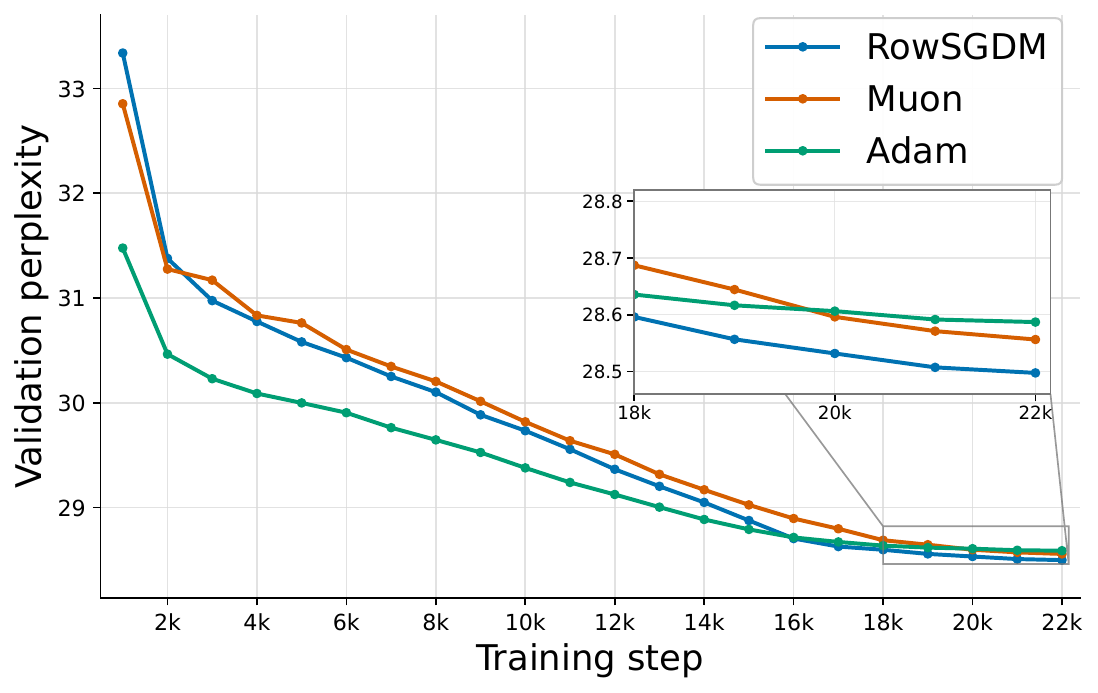}
\end{minipage}
\vspace{-8pt}
\caption{\textbf{Synthetic and language-model experiments.} \textbf{Left:} Final evaluation accuracy of six deterministic optimizers on 30 draws of Gaussian-cloud data. \textbf{Right:} Validation perplexity during last-layer-only LLaMA training with RowSGDM, Muon, and Adam. Lower perplexity is better.}
\label{fig:experiments}
\vspace{-10pt}
\end{figure}

\section{Experiments}

\paragraph{Synthetic experiment.}

On isotropic Gaussian-cloud data with \(k=60\), \(d=20\), RowGD and RowGDM outperform SpecGD, Muon, SignGD, and Adam on all 30 paired seeds (Figure~\ref{fig:experiments}, left), consistent with Theorem~\ref{thm:det-optimizer-accuracy-order-clean}. Additional comparisons and details are given in Appendix~\ref{app:experimental-details}.

\paragraph{Language model experiment.}

We freeze a trained LLaMA-\(60\)M model, zero-initialize its output head, and train only this layer with practical implementations of RowSGDM, Muon, or Adam. Figure~\ref{fig:experiments} (right) shows that RowSGDM attains the lowest final validation perplexity, consistent with the qualitative ordering our theory predicts. Multi-seed results and details are given in Appendix~\ref{app:language-model-experiment-details}.

\section{Conclusion}

\paragraph{Implications.}
We connect optimizer-induced implicit-bias geometry to population accuracy in high-dimensional multiclass classification, beyond prior work on optimizer-selected implicit biases \citep{fan2026implicit, li2026implicit, baek2026implicit}.
For linear classification abstracting the language model's last layer, row-normalized methods achieve strictly higher population accuracy than spectral and coordinate-wise alternatives in both full-batch and random-reshuffling settings by better preserving class-mean geometry.
These results complement empirical evidence of competitive performance by row-normalized methods in language model training \citep{glentis2025memory, deng2026rmnp}.

\paragraph{Limitations and outlook.}
Our analysis is limited to a linear classification model that can be viewed as a last-layer readout, together with a stylized data model with Gaussian class means and within-class noise. Although the data model shares some high-dimensional geometric properties with linguistic collapse \citep{wu2024linguistic}, modern language-model representations are substantially more complex. This motivates extensions to richer representation models and multilayer training.
Our anisotropic independent-orientation analysis assumes directional convergence, while our Adam results assume a zero numerical-stability constant and treat RR-Adam only via AdamProxy.

\subsection*{AI use statement}
In this work, we used generative AI tools to assist with improving English writing and readability, identifying potentially relevant related work that we might have overlooked, and drafting proof sections of the appendix by converting lengthy mathematical proofs and derivations into clean and readable LaTeX.
We additionally used generative AI tools to assist with debugging experimental code and improving the visual presentation of figures and tables.
All AI-assisted outputs were reviewed by the authors, who take responsibility for the final content of this work.

\subsection*{Reproducibility statement}
We state the assumptions and settings for our theoretical results in Sections~\ref{sec:preliminaries} and~\ref{sec:anisotropic-extension}, and provide complete proofs and supporting technical results in Appendices~\ref{app:max-sum-row}--\ref{sec:when2-results}.
For the empirical results, Appendix~\ref{app:additional-experiments} provides the experimental setups and information needed for reproduction, with synthetic experiments in Appendix~\ref{app:experimental-details}, the language-model experiment in Appendix~\ref{app:language-model-experiment-details}, and the representation analysis in Appendix~\ref{app:language-model-represen}.

\bibliographystyle{plainnat}
\bibliography{references}

\clearpage
\appendix

\etocsetlocaltop.toc{part}
\etocsettocdepth{subsection}
\etocsettocstyle{\section*{Appendix}}{}
\localtableofcontents

\clearpage

\section*{Overview of the Appendix}

The appendix is organized around the proof architecture of the paper. Appendix~\ref{app:add-rel} provides additional related work, and Appendix~\ref{app:max-sum-row} develops the optimizer-side margin and implicit-bias machinery for full-batch and random-reshuffling training. Appendices~\ref{app:deterministic-margin}--\ref{app:adamproxy-analysis} combine this machinery with high-dimensional geometry to prove the deterministic, stochastic, and AdamProxy population-accuracy results. Appendix~\ref{app:anisotropic-extension} extends the analysis to anisotropic Gaussian models with independently oriented test-noise covariances, Appendix~\ref{sec:when2-results} identifies a coordinate-specific RowGDM--Adam reversal and its recovery under a common Haar rotation, and Appendix~\ref{app:additional-experiments} provides the experimental details and additional empirical evidence.

\begin{enumerate}[label=(\arabic*), leftmargin=*, itemsep=3pt, topsep=4pt]

\item In Appendix~\ref{app:add-rel}, we provide additional related work on implicit bias in separable classification, adaptive and sign-based optimization, optimizer geometry and population performance, matrix-aware optimization geometry, row normalization and row--spectral hybrids, high-dimensional classification under mixture models, and last-layer geometry and neural collapse.

\item In Appendix~\ref{app:max-sum-row}, we establish the optimizer-side margin-convergence results used throughout the paper. We extend full-batch normalized steepest-descent and momentum arguments to max--sum dominated matrix norms and establish the max-norm maximum-margin characterization of full-batch Adam under our learning-rate assumptions. For random-reshuffling momentum methods, we develop a geometry-adapted analysis through the primal--dual atom radius \(R_\psi(X)\), which applies directly to the Row, Spectral, and Max geometries. Specializing these results to row normalization yields the exact full-batch RowGD and RowGDM direction-set convergence in Proposition~\ref{cor:app-row-full-batch} and the RR-RowSGDM near-margin guarantee in Proposition~\ref{cor:row-sgdm-rr-margin-clean}.

\item In Appendix~\ref{app:deterministic-margin}, we prove the deterministic population-accuracy ordering in Theorem~\ref{thm:det-optimizer-accuracy-order-clean} uniformly over the complete Row, Spectral, and Max maximum-margin direction sets. The proof establishes the all-aspect-ratio geometry of these three families: Row representatives align uniformly with the population mean differences, Max representatives exhibit a nonvanishing clipped-coordinate angular gap, and Spectral representatives are controlled through the reduced polar geometry of the centered empirical mean matrix, including the hard-edge regime at \(\alpha=1\). We then develop a geometry-to-population-error transfer and derive quantitative population-error exponents for the three geometries. The same exponent analysis also yields the full-batch Muon--Adam comparison in Corollary~\ref{cor:app-det-muon-adam-order}.

\item In Appendix~\ref{app:stochastic-margin}, we prove the random-reshuffling stochastic ordering in Proposition~\ref{thm:stoch-optimizer-accuracy-order-clean}. The geometry-adapted momentum result from Appendix~\ref{app:max-sum-row} places RR-RowSGDM, RR-Muon, and RR-Signum in near-maximum-margin direction sets for the Row, Spectral, and Max geometries. We reduce these near-margin directions to near-minimum-norm representatives, prove quantitative stability of the three near-minimizer families, and obtain a dimension-independent momentum threshold under which the deterministic geometric separation persists. The population-error transfer from Appendix~\ref{app:deterministic-margin} then gives the claimed stochastic accuracy ordering.

\item In Appendix~\ref{app:adamproxy-analysis}, we analyze RR-Adam and AdamProxy. We first prove the one-epoch \(\beta_1,\beta_2\to1\) RR-Adam approximation in Proposition~\ref{thm:rr-adamproxy-epoch-high-memory}; this result connects RR-Adam to the deterministic proxy over one epoch and does not assert long-horizon agreement of their trajectories. For AdamProxy itself, we establish positive margin growth and a trajectory-level coordinate cap, use the resulting coordinate saturation to obtain a nonvanishing boundary-direction gap, and combine this gap with the high-momentum RR-RowSGDM geometry and the population-error transfer from Appendix~\ref{app:deterministic-margin} to prove Proposition~\ref{thm:adamproxy-optimizer-accuracy-order-clean} without directional convergence. Separately, Proposition~\ref{thm:adamproxy-margin-clean} characterizes the limiting AdamProxy classifier through a self-consistent semidefinite quadratic margin problem when the normalized proxy direction converges.

\item In Appendix~\ref{app:anisotropic-extension}, we prove the anisotropic accuracy ordering in Theorem~\ref{thm:aniso-dd-ordering} under its directional-convergence assumption. Because training is clean, each optimizer-selected limiting direction depends only on the class means and is independent of the Haar eigenbases of the test-noise covariances. We characterize the selected Row, Spectral, and Max directions, expose the independent test-noise orientations, and compare their standardized Gaussian test signals. This establishes the Row advantage throughout \(0\le s_\mu<1\) and \(s_\Sigma\ge0\), including arbitrarily strong test-noise anisotropy.

\item In Appendix~\ref{sec:when2-results}, we prove the coordinate-specific RowGDM--Adam reversal in Proposition~\ref{prop:main-coordinate-reversal}. Working uniformly over the complete Row and Max maximum-margin solution sets, we derive an explicit transition point \(\lambda_c\) for diagonal class-mean covariance with aligned anisotropic test noise, below which Row dominates and above which the ordering reverses. We then show that applying a common independent Haar rotation to the same class-mean and test-noise geometry restores the Row-over-Max, and hence RowGDM-over-Adam, ordering while preserving the rotation-invariant geometry. We also give a sufficient spectral-dispersion condition under which the transition regime is nonvacuous.

\item In Appendix~\ref{app:additional-experiments}, we provide the full experimental setup and additional empirical evidence. This includes expanded deterministic and random-reshuffling synthetic optimizer comparisons, auxiliary optimizer geometries, robustness studies across data and noise settings, synthetic experiments for the anisotropic independent-orientation and coordinate-aligned predictions, controlled last-layer language-model experiments with RowSGDM, Muon, and Adam, and measurements of class-mean and within-class covariance spectral exponents in language-model representations.

\end{enumerate}

\section{Additional Related Work}
\label{app:add-rel}

\paragraph{Implicit bias in separable classification.}
The modern implicit-bias literature begins with the observation that, on linearly separable data, gradient descent on unregularized logistic or cross-entropy loss diverges in norm while converging in direction to a hard-margin classifier \citep{soudry2018implicit}.
More generally, the geometry of the update rule can determine the norm with respect to which the limiting margin is maximized, as formalized for steepest-descent and mirror-descent methods by \citet{gunasekar2018characterizing}.
For Euclidean methods, this picture has been extended from gradient descent to fixed-step stochastic gradient descent and momentum methods \citep{nacson2019stochastic,wang2022does}, and recent work gives multiclass convergence results for broad families of exponentially tailed losses \citep{ravi2024implicit}.
In homogeneous neural networks, norm-dependent steepest-descent dynamics generally converge to first-order stationary points of the corresponding margin problem rather than necessarily to its global optimum \citep{tsilivis2026flavors}, and a recent extension treats normalized and momentum variants of SignGD, Muon, and Adam \citep{gronich2026implicit}.
The results closest to ours characterize the max-margin directions selected by entrywise and Schatten geometries in multiclass linear models, their approximation under random reshuffling, and the data-dependent direction selected by incremental Adam \citep{fan2026implicit,li2026implicit,baek2026implicit}.
Our analysis complements these directional results by adding the row norm and by comparing the population accuracy of entire optimizer-selected margin faces rather than only identifying a limiting direction.

\paragraph{Adaptive and sign-based optimization.}
Adam combines momentum with coordinate-wise second-moment normalization \citep{kingma2015adam}, whereas signSGD and Signum retain only coordinate-wise gradient signs, with or without momentum \citep{bernstein2018signsgd}.
This relationship is not merely cosmetic: Adam can be decomposed into a sign direction and a variance-dependent magnitude, which helps explain both its robustness and its sensitivity to the coordinate system \citep{balles2018dissecting}; signSGD can also exhibit distinct compute-optimal scaling behavior from SGD \citep{kim2026scaling}.
The asymptotic classifier also depends on algorithmic details that are often suppressed in informal comparisons.
For example, AdaGrad can converge to a margin solution determined by its path-dependent diagonal conditioner \citep{qian2019implicit}, adaptive methods with a positive stability constant can become asymptotically Euclidean in homogeneous models \citep{wang2021implicit}, and Adam without a stability constant selects an \(\ell_\infty\)-max-margin direction in binary linear classification \citep{zhang2024implicit}.
Decoupled weight decay changes the problem again: the implicit regularization of AdamW is related to an \(\ell_\infty\)-constrained objective rather than to the unregularized margin limit \citep{loshchilov2017decoupled,xie2024implicit}.
Accordingly, our full-batch Adam statement concerns the zero-stability-constant, no-weight-decay dynamics used in our theoretical model, while our AdamProxy analysis isolates the additional data dependence created by incremental or mini-batch updates under random reshuffling \citep{baek2026implicit}.

\paragraph{Optimizer geometry and population performance.}
Implicit-bias results do not by themselves imply that one optimizer generalizes better than another.
Constructed separable problems show that Euclidean gradient methods and adaptive coordinate-wise methods can interpolate the same training data yet have sharply different test errors \citep{wilson2017marginal}, and norm-based analyses relate such differences to the interaction between an optimizer's geometry and the geometry of the data \citep{vaswani2020each}.
Experiments on transformers further suggest that the sign-like component of Adam, rather than stochastic noise alone, can account for part of its difference from SGD \citep{kunstner2023noise}.
Our contribution is a distribution-specific, high-probability ordering of row, spectral, and entrywise geometries in a multiclass model, uniform over the corresponding complete margin faces.

\paragraph{Matrix-aware optimization geometry.}
Several optimizer families exploit matrix or tensor structure without implementing the normalized steepest maps studied here.
Shampoo uses Kronecker-factored, mode-wise preconditioners for tensor parameters \citep{gupta2018shampoo}, while SOAP applies Adam in an evolving Shampoo eigenbasis \citep{vyas2025soap}.
The modular-duality framework instead assigns layerwise operator norms from the semantics of a module and derives corresponding duality maps \citep{bernstein2024modular}, and norm-constrained linear minimization oracles provide a related route to architecture-aware updates \citep{pethick2025training}.
Muon can be viewed as the matrix duality map obtained from the polar factor, commonly approximated by Newton--Schulz iterations \citep{jordan2024muon,bernstein2024modular}.
By contrast, our spectral comparator is the exact polar update \(UV^\top\) for a gradient with compact singular value decomposition \(U\Sigma V^\top\); our theorems do not automatically transfer to preconditioned methods such as Shampoo or SOAP, or to finite-iteration approximations of the polar factor.

\paragraph{Row normalization and row--spectral hybrids.}
The row-structured methods summarized in Section~\ref{rel} are algorithmically distinct even when they contain a superficially similar normalization step.
In particular, SinkGD alternates row and column normalization, NorMuon performs neuron-wise adaptation after Muon orthogonalization, RMNP analyzes row-normalized momentum through nonconvex convergence guarantees, and MOGA studies width scaling and hyperparameter transfer for the corresponding operator geometry \citep{scetbon2026gradient,li2025normuon,deng2026rmnp,xu2026width}.
Recent hybrid proposals further modify the polar update by projecting momentum before row alignment, equilibrating rows before orthogonalization, or imposing leverage-aware row uniformity while retaining spectral structure \citep{yuan2026nora,chang2026muoneq,dewulf2026aurora}.
SWAN is also a relevant stateless optimizer that combines column normalization with whitening \citep{ma2024swan}.
These methods strengthen the practical motivation for studying row geometry, but the extra projection, whitening, equilibration, or orthogonalization steps make them different dynamical systems.
Our results should therefore be read as an analysis of exact row-normalized steepest descent and momentum, and as a controlled comparison that isolates the row normalization component rather than as an implicit-bias theorem for every row--spectral hybrid.

\paragraph{High-dimensional classification under mixture models.}
Gaussian and sub-Gaussian mixture models have been used to obtain sharp asymptotic formulas or risk bounds for regularized convex classifiers and binary max-margin classifiers \citep{mignacco2020role,cao2021risk}.
In the multiclass setting, prior work relates the hard-margin multiclass SVM to minimum-norm interpolation and studies benign overfitting under Gaussian designs \citep{wang2021benign}, including regimes in which the number of classes grows with dimension \citep{subramanian2022generalization}.
These works primarily fix an estimator geometry, usually Euclidean, and characterize when its interpolation or maximum-margin bias generalizes.
We instead hold the data model fixed and compare the distinct margin geometries selected by several optimizers in the proportional regime \(k/d\to\alpha\in(0,\infty)\).
The uniform statements over complete margin faces are important in this regime because a norm-specific margin problem need not have a unique optimizer, so identifying one convenient solution would not establish an optimizer-level population comparison.

\paragraph{Last-layer geometry and neural collapse.}
The class-conditional mean geometry in our model is also related to neural collapse, where terminal-phase representations approach within-class concentration, centered class means form a simplex-like frame, and classifier weights align with those means \citep{papyan2020prevalence}.
Generalized neural-collapse formulations extend this viewpoint to settings with many more classes than feature dimensions by replacing the simplex equiangular tight frame with a one-versus-rest margin criterion \citep{jiang2023generalized}.
Related geometric regularities have been reported for the unembedding representations of language models under the name linguistic collapse \citep{wu2024linguistic}.
Our model does not attempt to explain the feature-learning dynamics that create such configurations.
Instead, it conditions on a fixed last-layer representation with random Gaussian class centers and isotropic or anisotropic residual variation, and asks which optimizer-induced norm produces the most accurate off-sample decision regions once that representation has been formed.

\section{Implicit Bias of Row-Normalized Methods: Full-Batch and Random Reshuffling}
\label{app:max-sum-row}

This appendix collects the optimizer-side implicit-bias results used throughout the paper. For full-batch normalized steepest descent, we extend the analysis of \citet{fan2026implicit} from entrywise and Schatten norms to every max--sum dominated matrix norm. 
For random-reshuffling momentum methods, we extend the analysis of \citet{li2026implicit} and use a sharper geometry-adapted analysis based on the primal--dual atom radius \(R_\psi(X)\), which avoids the common entrywise \(\ell_1\) feature bound and applies directly to the Row, Spectral, and Max geometries. We then specialize these general results to row normalization, thereby proving Proposition~\ref{cor:app-row-full-batch} and the RR-RowSGDM margin result used in the stochastic accuracy theorem.

\subsection{Common setup and auxiliary lemmas}

Let \(X=\{(x_n,y_n)\}_{n=1}^N\) be a fixed linearly separable multiclass dataset with \(x_n\in\mathbb R^d\) and \(y_n\in[k]\). A linear classifier is \(W\in\mathbb R^{k\times d}\), and the cross-entropy loss is
\begin{equation}
L(W):=\frac1N\sum_{n=1}^N\ell_n(W),\qquad \ell_n(W):=-\log{\mathbb{S}}(Wx_n)_{y_n}.
\label{eq:app-ms-ce}
\end{equation}
Define
\begin{equation}
\operatorname{mar}_X(W):=\min_{n\in[N]}\min_{c\ne y_n}(w_{y_n}-w_c)^\top x_n,\qquad \gamma_\psi(X):=\max_{\psi(W)\le1}\operatorname{mar}_X(W)
\label{eq:app-ms-margin-gamma}
\end{equation}
for any matrix norm \(\psi\), and define the standard cross-entropy proxy
\begin{equation}
\mathcal G(W):=\frac1N\sum_{n=1}^N\bigl(1-{\mathbb{S}}(Wx_n)_{y_n}\bigr).
\label{eq:app-common-proxy}
\end{equation}
For \(0\le\varepsilon<\gamma_\psi(X)\), define the near-max-margin direction set
\begin{equation}
\mathcal U_{\psi,\varepsilon}^X:=\left\{W\ne{\mathbf 0}:\frac{\operatorname{mar}_X(W)}{\psi(W)}\ge\gamma_\psi(X)-\varepsilon\right\},\qquad \mathcal U_\psi^X:=\mathcal U_{\psi,0}^X.
\label{eq:app-near-max-margin-direction-set}
\end{equation}
For a nonzero matrix \(A\) and a nonempty set \(\mathcal S\subset\mathbb R^{k\times d}\setminus\{{\mathbf 0}\}\), let
\begin{equation}
d_{\cos}(A,\mathcal S):=1-\sup_{V\in\mathcal S}\frac{\langle A,V\rangle}{\|A\|_{\mathrm F}\|V\|_{\mathrm F}}.
\label{eq:app-cosine-discrepancy}
\end{equation}

\begin{lemma}[Effective margin implies direction-set convergence]
\label{lem:effective-margin-direction-convergence}
Suppose \(W_t\ne{\mathbf 0}\) eventually and, for some \(0\le\varepsilon<\gamma_\psi(X)\),
\begin{equation}
\liminf_{t\to\infty}\frac{\operatorname{mar}_X(W_t)}{\psi(W_t)}\ge\gamma_\psi(X)-\varepsilon.
\label{eq:app-effective-margin-convergence}
\end{equation}
Then
\begin{equation}
d_{\cos}(W_t,\mathcal U_{\psi,\varepsilon}^X)\longrightarrow0.
\label{eq:app-effective-direction-convergence}
\end{equation}
\end{lemma}

\begin{proof}
By scale invariance of \(d_{\cos}\), consider \(W_t/\psi(W_t)\), which lies on the compact \(\psi\)-unit sphere. If the conclusion failed, a subsequence whose cosine discrepancy from \(\mathcal U_{\psi,\varepsilon}^X\) is bounded below by a positive constant would have a further subsequence converging to some \(U_\infty\). Continuity of the margin and~\eqref{eq:app-effective-margin-convergence} give \(\operatorname{mar}_X(U_\infty)\ge\gamma_\psi(X)-\varepsilon\), so \(U_\infty\in\mathcal U_{\psi,\varepsilon}^X\), contradicting continuity of the cosine discrepancy.
\end{proof}

\begin{lemma}[Exact margin implies direction-set convergence]
\label{lem:exact-margin-direction-convergence}
If
\begin{equation}
\frac{\operatorname{mar}_X(W_t)}{\psi(W_t)}\longrightarrow\gamma_\psi(X),
\label{eq:app-exact-margin-convergence}
\end{equation}
then
\begin{equation}
d_{\cos}(W_t,\mathcal U_\psi^X)\longrightarrow0.
\label{eq:app-exact-direction-convergence}
\end{equation}
\end{lemma}

\begin{proof}
Apply Lemma~\ref{lem:effective-margin-direction-convergence} with \(\varepsilon=0\).
\end{proof}

For a learning-rate sequence, write \(S_t:=\sum_{s=0}^{t-1}\eta_s\).

\begin{lemma}[Cumulative remainders for general learning rates]
\label{lem:app-general-lr-remainders}
Suppose Assumption~\ref{ass:lr-base} holds. Then \(S_t\to\infty\), and for every \(q>1\), \(c>0\), and fixed \(t_0\in\mathbb N\),
\begin{equation}
\frac{\sum_{s=0}^{t-1}\eta_s^q}{S_t}\longrightarrow0,\qquad \sum_{s=t_0}^{t-1}\eta_s\exp\!\left(-c\sum_{\tau=t_0}^{s-1}\eta_\tau\right)=O(1).
\label{eq:app-general-lr-remainders}
\end{equation}
\end{lemma}

\begin{proof}
Fix \(\varepsilon>0\) and choose \(T\) so that \(\eta_s^{q-1}\le\varepsilon\) for all \(s\ge T\). Then
\[
\frac{\sum_{s=0}^{t-1}\eta_s^q}{S_t}\le\frac{\sum_{s=0}^{T-1}\eta_s^q}{S_t}+\varepsilon,
\]
which proves the first limit. For the second assertion, enlarge \(T\ge t_0\) so that \(c\eta_s\le1\) for \(s\ge T\), and set \(A_s:=\sum_{\tau=t_0}^{s-1}\eta_\tau\). Since \(1-e^{-x}\ge x/2\) for \(x\in[0,1]\),
\[
e^{-cA_s}-e^{-cA_{s+1}}\ge\frac c2\eta_s e^{-cA_s}.
\]
The tail therefore telescopes to a bounded quantity, while the finite prefix is harmless.
\end{proof}

\begin{lemma}[Normalization of a cumulative margin bound]
\label{lem:app-general-lr-normalization}
Let \(\mathfrak m:\mathbb R^{k\times d}\to\mathbb R\) be continuous and positively homogeneous, let \(\Gamma_*:=\max_{\psi(W)\le1}\mathfrak m(W)>0\), and suppose \(0<\Gamma\le\Gamma_*\). Assume \(S_t\to\infty\). If
\[
\mathfrak m(W_t)\ge\Gamma S_t-r_t,\qquad \psi(W_t)\le C_0+S_t,\qquad r_t=o(S_t),
\]
then
\begin{equation}
\liminf_{t\to\infty}\frac{\mathfrak m(W_t)}{\psi(W_t)}\ge\Gamma.
\label{eq:app-general-lr-normalization}
\end{equation}
If \(\Gamma=\Gamma_*\), then \(\mathfrak m(W_t)/\psi(W_t)\to\Gamma_*\), and every accumulation point of \(W_t/\psi(W_t)\) maximizes \(\mathfrak m\) over the \(\psi\)-unit ball.
\end{lemma}

\begin{proof}
Since \(r_t=o(S_t)\) and \(S_t\to\infty\), the lower bound \(\Gamma S_t-r_t\) is positive for all sufficiently large \(t\). In particular, \(W_t\ne0\) eventually. For such \(t\),
\[
\frac{\mathfrak m(W_t)}{\psi(W_t)}\ge\frac{\Gamma S_t-r_t}{C_0+S_t}\longrightarrow\Gamma.
\]
This proves the liminf statement without imposing a sign condition on \(r_t\). When \(\Gamma=\Gamma_*\), positive homogeneity gives \(\mathfrak m(W_t)/\psi(W_t)\le\Gamma_*\), proving convergence. Compactness of the \(\psi\)-unit sphere and continuity of \(\mathfrak m\) give the accumulation-point statement.
\end{proof}

\subsection{Full-batch max--sum dominated norms}

For a matrix \(A\), write
\[
\|A\|_{\max}:=\max_{r,j}|A_{r,j}|,\qquad \|A\|_{\mathrm{sum}}:=\sum_{r,j}|A_{r,j}|.
\]
We call a matrix norm \(\psi\) \emph{max--sum dominated} if
\begin{equation}
\|A\|_{\max}\le\psi(A)\le\|A\|_{\mathrm{sum}}\qquad\text{for every }A.
\label{eq:app-ms-sandwich}
\end{equation}
Let \(R_1(X):=\max_{n\in[N]}\|x_n\|_1<\infty\).

\begin{lemma}[Consequences of max--sum domination]
\label{lem:app-ms-basic}
Let \(\psi\) satisfy~\eqref{eq:app-ms-sandwich}, and let \(\psi^*\) be its dual norm. Then, for every matrix \(G\), every \(h\) with \(\|h\|_1\le R_1(X)\), and every \(A\in\mathbb R^{k\times d}\),
\begin{equation}
\psi^*(G)\le\|G\|_{\mathrm{sum}},\qquad \|Ah\|_\infty\le R_1(X)\psi(A).
\label{eq:app-ms-basic-bounds}
\end{equation}
Moreover, if \(s\in\mathbb R^k\) is a probability vector and \(H(s):=\operatorname{diag}(s)-ss^\top\), then for every \(c\in[k]\),
\begin{equation}
h^\top A^\top H(s)Ah\le4R_1(X)^2(1-s_c)\psi(A)^2.
\label{eq:app-ms-hessian}
\end{equation}
\end{lemma}

\begin{proof}
Since the \(\psi\)-unit ball is contained in the entrywise max-norm unit ball,
\[
\psi^*(G)=\max_{\psi(A)\le1}\langle G,A\rangle\le\max_{\|A\|_{\max}\le1}\langle G,A\rangle=\|G\|_{\mathrm{sum}}.
\]
Also, \(\|Ah\|_\infty\le\|A\|_{\max}\|h\|_1\le R_1(X)\psi(A)\). For \(v=Ah\), the softmax covariance bound gives \(v^\top H(s)v\le4(1-s_c)\|v\|_\infty^2\), which proves~\eqref{eq:app-ms-hessian}.
\end{proof}

\begin{theorem}[Full-batch max--sum dominated margin convergence]
\label{thm:app-ms-full-batch}
Let \(\psi\) be max--sum dominated. Under Assumption~\ref{ass:lr-base}, normalized steepest descent
\begin{equation}
\Delta_t\in\operatorname*{arg\,max}_{\psi(\Delta)\le1}\langle\nabla L(W_t),\Delta\rangle,\qquad W_{t+1}=W_t-\eta_t\Delta_t
\label{eq:app-ms-nsd}
\end{equation}
satisfies
\begin{equation}
\frac{\operatorname{mar}_X(W_t)}{\psi(W_t)}\longrightarrow\gamma_\psi(X).
\label{eq:app-ms-nsd-margin}
\end{equation}
Under Assumptions~\ref{ass:lr-base} and~\ref{ass:lr-momentum}, the same conclusion holds for every fixed \(\beta\in(0,1)\) under normalized momentum descent with \(M_0={\mathbf 0}\):
\begin{equation}
M_{t+1}=\beta M_t+(1-\beta)\nabla L(W_t),\qquad
\Delta_t\in\operatorname*{arg\,max}_{\psi(\Delta)\le1}\langle M_{t+1},\Delta\rangle,
\qquad
W_{t+1}=W_t-\eta_t\Delta_t,
\label{eq:app-ms-nmd}
\end{equation}
In either case, every accumulation point of \(W_t/\psi(W_t)\) is a \(\psi\)-maximum-margin solution.
\end{theorem}

\begin{proof}
The geometry-dependent steps in the normalized steepest-descent and momentum proofs of \citet[Appendices~D and E]{fan2026implicit} use only norm duality, the lower bound \(\psi^*(\nabla L(W))\ge\gamma_\psi(X)\mathcal G(W)\), the upper bound \(\psi^*(\nabla L(W))\le2R_1(X)\mathcal G(W)\), and the logit and Hessian controls in Lemma~\ref{lem:app-ms-basic}. Substituting these bounds into their pre-rate scalar argument gives
\[
\operatorname{mar}_X(W_t)\ge\gamma_\psi(X)S_t-r_t,\qquad \psi(W_t)\le\psi(W_0)+S_t,
\]
where \(r_t\) is a finite-prefix term plus sums of the forms \(\sum_s\eta_s^2\) and \(\sum_s\eta_s\exp(-c\sum_{\tau<s}\eta_\tau)\); in the momentum case, Assumption~\ref{ass:lr-momentum} controls the additional geometric-memory term. Lemma~\ref{lem:app-general-lr-remainders} gives \(r_t=o(S_t)\), and Lemma~\ref{lem:app-general-lr-normalization} with \(\mathfrak m=\operatorname{mar}_X\) and \(\Gamma=\Gamma_*=\gamma_\psi(X)\) proves the result.
\end{proof}

The full-batch Adam characterization used in Theorem~\ref{thm:det-optimizer-accuracy-order-clean} also extends to the learning-rate assumptions of the present paper.

\begin{theorem}[Full-batch Adam under the present learning-rate assumptions]
\label{thm:app-ms-full-batch-adam}
Consider full-batch Adam without bias correction or a numerical-stability constant, initialized by \(M_0=V_0={\mathbf 0}\):
\begin{align}
M_{t+1}&=\beta_1M_t+(1-\beta_1)\nabla L(W_t),\nonumber\\
V_{t+1}&=\beta_2V_t+(1-\beta_2)\nabla L(W_t)^{\odot2},\nonumber\\
W_{t+1}&=W_t-\eta_tM_{t+1}\oslash V_{t+1}^{\odot1/2}.
\label{eq:app-full-batch-adam}
\end{align}
Assume \(0<\beta_1\le\beta_2<1\), Assumptions~\ref{ass:lr-base} and~\ref{ass:lr-momentum}, and \(\min_{c\in[k],\,j\in[d]}|\nabla L(W_0)[c,j]|^2\ge\omega\) for some \(\omega>0\). Then, with
\[
\gamma_{\max}(X):=\max_{\|W\|_{\max}\le1}\operatorname{mar}_X(W),
\]
one has
\begin{equation}
\frac{\operatorname{mar}_X(W_t)}{\|W_t\|_{\max}}\longrightarrow\gamma_{\max}(X),
\label{eq:app-adam-margin}
\end{equation}
and every accumulation point of \(W_t/\|W_t\|_{\max}\) is a max-norm maximum-margin solution.
\end{theorem}

\begin{proof}
The pre-rate part of \citet[Appendix~G]{fan2026implicit} gives
\[
\operatorname{mar}_X(W_t)\ge\gamma_{\max}(X)S_t-r_t,\qquad \|W_t\|_{\max}\le C_0+S_t,
\]
where \(r_t\) consists of a finite prefix, \(\sum_s\eta_s^{3/2}\), and an exponentially weighted remainder. Lemma~\ref{lem:app-general-lr-remainders} gives \(r_t=o(S_t)\), and Lemma~\ref{lem:app-general-lr-normalization} proves~\eqref{eq:app-adam-margin} and the accumulation-point conclusion.
\end{proof}

\subsection{Geometry-adapted random-reshuffling momentum}
\label{subsec:app-ga-rr-margin-input}

The full-batch max--sum argument above is sufficient for deterministic optimization, but using the same common feature radius \(\max_n\|x_n\|_1\) in the stochastic analysis is unnecessarily coarse when comparing different geometries. We therefore work directly with the primal--dual pair \((\psi,\psi^*)\) and define the atom radius
\begin{equation}
R_\psi(X):=\max_{n\in[N]}\max_{c\in[k]}\psi^*(e_cx_n^\top).
\label{eq:app-ga-radius}
\end{equation}
The stochastic results below hold for an arbitrary matrix norm \(\psi\); max--sum domination is not required. Assume \(b\mid N\), set \(m:=N/b\), and for every mini-batch \(B\subset[N]\) of size \(b\) define \(L_B(W):=b^{-1}\sum_{n\in B}\ell_n(W)\).

\begin{lemma}[Geometry-adapted cross-entropy bounds]
\label{lem:app-ga-atomic-package}
For every matrix norm \(\psi\), every sample \(x_n\), every matrix \(A\), and all classes \(a,c\in[k]\),
\begin{equation}
\|Ax_n\|_\infty\le R_\psi(X)\psi(A),\qquad \psi^*((e_a-e_c)x_n^\top)\le2R_\psi(X).
\label{eq:app-ga-logit-atom}
\end{equation}
For every \(W\),
\begin{equation}
\gamma_\psi(X)\mathcal G(W)\le\psi^*(\nabla L(W))\le2R_\psi(X)\mathcal G(W).
\label{eq:app-ga-gradient-proxy}
\end{equation}
For every sample \(n\), every \(W\), and every perturbation \(A\),
\begin{equation}
D^2\ell_n(W)[A,A]\le4R_\psi(X)^2\bigl(1-{\mathbb{S}}(Wx_n)_{y_n}\bigr)\psi(A)^2.
\label{eq:app-ga-hessian}
\end{equation}
Moreover, for all \(U,W\),
\begin{equation}
|L(U)-L(W)|\le2R_\psi(X)\psi(U-W).
\label{eq:app-ga-loss-lipschitz}
\end{equation}
\end{lemma}

\begin{proof}
For every \(c\in[k]\), norm duality gives
\[
|e_c^\top Ax_n|=|\langle A,e_cx_n^\top\rangle|\le\psi(A)\psi^*(e_cx_n^\top)\le R_\psi(X)\psi(A),
\]
which proves the first bound in~\eqref{eq:app-ga-logit-atom}; the second follows from the triangle inequality. Writing \(p_n(W):={\mathbb{S}}(Wx_n)\), the sample gradient has the decomposition
\[
\nabla\ell_n(W)=\sum_{c\ne y_n}p_{n,c}(W)(e_c-e_{y_n})x_n^\top,
\]
so
\[
\psi^*(\nabla\ell_n(W))\le2R_\psi(X)\bigl(1-p_{n,y_n}(W)\bigr).
\]
Averaging proves the upper bound in~\eqref{eq:app-ga-gradient-proxy}. For the lower bound, choose \(A_\star\) with \(\psi(A_\star)\le1\) and margin \(\gamma_\psi(X)\); then
\[
-\nabla L(W)=\frac1N\sum_{n=1}^N\sum_{c\ne y_n}p_{n,c}(W)(e_{y_n}-e_c)x_n^\top
\]
gives \(\psi^*(\nabla L(W))\ge\langle A_\star,-\nabla L(W)\rangle\ge\gamma_\psi(X)\mathcal G(W)\). For the Hessian bound, set \(v:=Ax_n\) and let \(C\) be categorical with law \(p_n(W)\). Then
\[
D^2\ell_n(W)[A,A]=\operatorname{Var}(v_C)\le\mathbb E[(v_C-v_{y_n})^2]\le4(1-p_{n,y_n}(W))\|v\|_\infty^2,
\]
which proves~\eqref{eq:app-ga-hessian}. Finally, convexity and the upper gradient bound, together with \(\mathcal G\le1\), yield \(L(U)-L(W)\le2R_\psi(X)\psi(U-W)\); exchanging \(U,W\) proves~\eqref{eq:app-ga-loss-lipschitz}.
\end{proof}

\begin{lemma}[Geometry-adapted softmax and gradient stability]
\label{lem:app-ga-stability-package}
For \(U,W\in\mathbb R^{k\times d}\), define \(h:=R_\psi(X)\psi(U-W)\). Then, for every sample \(n\) and class \(c\),
\begin{equation}
e^{-2h}\le\frac{p_{n,c}(U)}{p_{n,c}(W)}\le e^{2h},\qquad e^{-2h}\le\frac{1-p_{n,y_n}(U)}{1-p_{n,y_n}(W)}\le e^{2h}.
\label{eq:app-ga-softmax-ratios}
\end{equation}
Consequently,
\begin{equation}
\mathcal G(U)\le e^{2R_\psi(X)\psi(U-W)}\mathcal G(W),
\label{eq:app-ga-proxy-stability}
\end{equation}
\begin{equation}
\psi^*(\nabla L(U)-\nabla L(W))\le2R_\psi(X)\Bigl(e^{2R_\psi(X)\psi(U-W)}-1\Bigr)\mathcal G(W),
\label{eq:app-ga-full-gradient-stability}
\end{equation}
and, for every mini-batch \(B\),
\begin{equation}
\psi^*(\nabla L_B(U)-\nabla L_B(W))\le2mR_\psi(X)\Bigl(e^{2R_\psi(X)\psi(U-W)}-1\Bigr)\mathcal G(W).
\label{eq:app-ga-batch-gradient-stability}
\end{equation}
\end{lemma}

\begin{proof}
Set \(\delta_n:=(U-W)x_n\). Lemma~\ref{lem:app-ga-atomic-package} gives \(\|\delta_n\|_\infty\le h\), and the softmax formula gives
\[
\frac{p_{n,c}(U)}{p_{n,c}(W)}=e^{\delta_{n,c}}\frac{\sum_{a=1}^ke^{e_a^\top Wx_n}}{\sum_{a=1}^ke^{e_a^\top Wx_n+\delta_{n,a}}}\in[e^{-2h},e^{2h}].
\]
Summing over \(c\ne y_n\) proves the second ratio and~\eqref{eq:app-ga-proxy-stability}. If \(p:=p_n(W)\) and \(q:=p_n(U)\), the ratio bounds imply
\[
\sum_{c\ne y_n}|q_c-p_c|\le(e^{2h}-1)(1-p_{y_n}),\qquad |q_{y_n}-p_{y_n}|\le(e^{2h}-1)(1-p_{y_n}).
\]
Since \(\nabla\ell_n(U)-\nabla\ell_n(W)=(q-p)x_n^\top\),
\[
\psi^*((q-p)x_n^\top)\le\sum_{c=1}^k|q_c-p_c|\psi^*(e_cx_n^\top)\le2R_\psi(X)(e^{2h}-1)(1-p_{y_n}).
\]
Averaging over all samples proves~\eqref{eq:app-ga-full-gradient-stability}; averaging over \(B\), enlarging the resulting nonnegative sample sum to \([N]\), and using \(N/b=m\) proves~\eqref{eq:app-ga-batch-gradient-stability}.
\end{proof}

\begin{lemma}[Random-reshuffling noise and cancellation]
\label{lem:app-ga-batch-noise}
\label{lem:app-ga-frozen-rr}
For a mini-batch \(B\) of size \(b\), define \(\Xi_B(W):=\nabla L_B(W)-\nabla L(W)\). Then
\begin{equation}
\psi^*(\Xi_B(W))\le2(m-1)R_\psi(X)\mathcal G(W).
\label{eq:app-ga-batch-noise}
\end{equation}
Moreover, let \(\{B_t\}_{t\ge0}\) be any random-reshuffling sequence, so that \(B_{qm},\ldots,B_{qm+m-1}\) partition \([N]\) for every epoch \(q\), and define
\begin{equation}
E_t(W):=(1-\beta)\sum_{\tau=0}^t\beta^\tau\bigl(\nabla L_{B_{t-\tau}}(W)-\nabla L(W)\bigr).
\label{eq:app-ga-frozen-ema}
\end{equation}
Then, deterministically for every reshuffling realization,
\begin{equation}
\psi^*(E_t(W))\le(1-\beta)m(m^2-1)R_\psi(X)\mathcal G(W).
\label{eq:app-ga-frozen-rr-bound}
\end{equation}
\end{lemma}

\begin{proof}
When \(m\ge2\), let \(B^c=[N]\setminus B\). The finite-population identity
\[
\Xi_B(W)=\frac{m-1}{m}\bigl(\nabla L_B(W)-\nabla L_{B^c}(W)\bigr)
\]
and the single-sample bound in the proof of Lemma~\ref{lem:app-ga-atomic-package} give
\[
\psi^*(\Xi_B(W))\le\frac{m-1}{N}\sum_{n=1}^N\psi^*(\nabla\ell_n(W))\le2(m-1)R_\psi(X)\mathcal G(W),
\]
while \(m=1\) is immediate.

For the second assertion, write \(\Xi_t(W):=\nabla L_{B_t}(W)-\nabla L(W)\). Every complete epoch satisfies \(\sum_{r=0}^{m-1}\Xi_{qm+r}(W)={\mathbf 0}\). Put \(K_\Xi(W):=2(m-1)R_\psi(X)\mathcal G(W)\). For a complete epoch contained in \(\{0,\ldots,t\}\), subtracting the constant coefficient \(\beta^{t-qm}\) and using \(1-\beta^r\le r(1-\beta)\) gives
\[
\psi^*\left(\sum_{r=0}^{m-1}\beta^{t-qm-r}\Xi_{qm+r}(W)\right)\le K_\Xi(W)\beta^{t-(q+1)m}(1-\beta)\frac{m(m-1)}2.
\]
Summing over complete epochs and using \((1-\beta)/(1-\beta^m)\le1\), then adding the at most \(m\) terms in the current incomplete epoch, gives
\[
\psi^*(E_t(W))\le(1-\beta)\frac{m(m+1)}2K_\Xi(W)=(1-\beta)m(m^2-1)R_\psi(X)\mathcal G(W).
\]
\end{proof}

Consider now the moving-weight random-reshuffling momentum process
\begin{equation}
M_{t+1}=\beta M_t+(1-\beta)\nabla L_{B_t}(W_t),\qquad \Delta_t\in\operatorname*{arg\,max}_{\psi(\Delta)\le1}\langle M_{t+1},\Delta\rangle,\qquad W_{t+1}=W_t-\eta_t\Delta_t,
\label{eq:app-ga-momentum-dynamics}
\end{equation}
with \(M_0={\mathbf 0}\). 
Since \(\psi(\Delta_t)\le1\),
\begin{equation}
\psi(W_t-W_{t-\tau})\le\sum_{q=1}^{\tau}\eta_{t-q}.
\label{eq:app-ga-path-length}
\end{equation}

\begin{proposition}[Geometry-adapted random-reshuffling momentum control]
\label{prop:app-ga-momentum-approximation}
\label{prop:app-ga-descent-recurrence}
Assume Assumptions~\ref{ass:lr-base} and~\ref{ass:lr-momentum}. For every fixed finite dataset, fixed matrix norm \(\psi\), and fixed \(\beta\in(0,1)\), there exist \(t_{0,\psi,\beta}<\infty\) and \(c_{\psi,\beta}<\infty\) such that, for all \(t\ge t_{0,\psi,\beta}\),
\begin{align}
\psi^*(M_{t+1}-\nabla L(W_t))\le\Bigl[&2m(1-\beta)R_\psi(X)c_{\psi,\beta}\eta_t+(1-\beta)m(m^2-1)R_\psi(X)\nonumber\\
&+2R_\psi(X)\beta^{t+1}\Bigr]\mathcal G(W_t).
\label{eq:app-ga-momentum-approximation}
\end{align}
Define
\begin{equation}
\rho_{\psi,\beta,b}(X):=\gamma_\psi(X)-2(1-\beta)m(m^2-1)R_\psi(X),
\label{eq:app-ga-effective-margin}
\end{equation}
\begin{equation}
C_{\psi,\beta,b}:=4m(1-\beta)R_\psi(X)c_{\psi,\beta}+2R_\psi(X)^2e^{2R_\psi(X)\eta_0},\qquad D_\psi:=4R_\psi(X).
\label{eq:app-ga-descent-constants}
\end{equation}
Then, for all sufficiently large \(t\),
\begin{equation}
L(W_{t+1})\le L(W_t)-\eta_t\Bigl(\rho_{\psi,\beta,b}(X)-C_{\psi,\beta,b}\eta_t-D_\psi\beta^{t/2}\Bigr)\mathcal G(W_t).
\label{eq:app-ga-descent-recurrence}
\end{equation}
\end{proposition}

\begin{proof}
Unrolling the momentum recursion and adding and subtracting frozen batch gradients at \(W_t\) gives
\begin{align}
M_{t+1}-\nabla L(W_t)=&(1-\beta)\sum_{\tau=0}^t\beta^\tau\Bigl(\nabla L_{B_{t-\tau}}(W_{t-\tau})-\nabla L_{B_{t-\tau}}(W_t)\Bigr)\nonumber\\
&+(1-\beta)\sum_{\tau=0}^t\beta^\tau\Bigl(\nabla L_{B_{t-\tau}}(W_t)-\nabla L(W_t)\Bigr)-\beta^{t+1}\nabla L(W_t).
\label{eq:app-ga-momentum-decomposition}
\end{align}
Lemma~\ref{lem:app-ga-stability-package} and~\eqref{eq:app-ga-path-length} bound the first term by
\[
2m(1-\beta)R_\psi(X)\mathcal G(W_t)\sum_{\tau=0}^t\beta^\tau\left(\exp\left(2R_\psi(X)\sum_{q=1}^{\tau}\eta_{t-q}\right)-1\right).
\]
Assumption~\ref{ass:lr-momentum}, applied with \(c_1=2R_\psi(X)\), bounds the final sum by \(c_{\psi,\beta}\eta_t\) for all sufficiently large \(t\). Lemma~\ref{lem:app-ga-frozen-rr} controls the second term, and the upper gradient bound in~\eqref{eq:app-ga-gradient-proxy} controls the last, proving~\eqref{eq:app-ga-momentum-approximation}.

For the descent recurrence, the steepest-descent property and norm duality give
\begin{equation}
\langle\nabla L(W_t),\Delta_t\rangle\ge\psi^*(\nabla L(W_t))-2\psi^*(M_{t+1}-\nabla L(W_t)).
\label{eq:app-ga-first-order-alignment}
\end{equation}
The mean-value Taylor formula,~\eqref{eq:app-ga-hessian}, and \(\psi(\Delta_t)\le1\) bound the quadratic remainder in \(L(W_t-\eta_t\Delta_t)\) by
\[
2\eta_t^2R_\psi(X)^2\sup_{0\le s\le1}\mathcal G(W_t-s\eta_t\Delta_t).
\]
Lemma~\ref{lem:app-ga-stability-package} bounds the supremum by \(e^{2R_\psi(X)\eta_0}\mathcal G(W_t)\). Substituting~\eqref{eq:app-ga-gradient-proxy} and~\eqref{eq:app-ga-momentum-approximation} into~\eqref{eq:app-ga-first-order-alignment}, adding the quadratic remainder, and using \(\beta^{t+1}\le\beta^{t/2}\) proves~\eqref{eq:app-ga-descent-recurrence}.
\end{proof}

\begin{corollary}[Geometry-adapted random-reshuffling momentum margin]
\label{cor:app-ga-margin-transfer}
Under Assumptions~\ref{ass:lr-base} and~\ref{ass:lr-momentum}, suppose \(\rho_{\psi,\beta,b}(X)>0\). Then
\begin{equation}
\liminf_{t\to\infty}\frac{\operatorname{mar}_X(W_t)}{\psi(W_t)}\ge\rho_{\psi,\beta,b}(X).
\label{eq:app-ga-margin-conclusion}
\end{equation}
Consequently, with
\begin{equation}
\varepsilon_{\psi,\beta,b}(X):=2(1-\beta)m(m^2-1)R_\psi(X),
\label{eq:app-ga-epsilon}
\end{equation}
one has
\begin{equation}
d_{\cos}\!\left(W_t,\mathcal U_{\psi,\varepsilon_{\psi,\beta,b}(X)}^X\right)\longrightarrow0.
\label{eq:app-ga-direction-conclusion}
\end{equation}
More quantitatively, there exist a finite \(t_2\) and \(C_{\mathrm{ga}}<\infty\) such that
\begin{equation}
\rho_{\psi,\beta,b}(X)-\frac{\operatorname{mar}_X(W_t)}{\psi(W_t)}\le C_{\mathrm{ga}}\frac{\mathcal R_t^{\psi,\beta,b}}{S_t},
\label{eq:app-ga-general-lr-bound}
\end{equation}
where
\begin{align}
\mathcal R_t^{\psi,\beta,b}:=&1+\sum_{s=0}^{t_2-1}\eta_s+C_{\psi,\beta,b}\sum_{s=t_2}^{t-1}\eta_s^2+\frac{D_\psi\eta_0}{1-\sqrt\beta}\nonumber\\
&+\sum_{s=t_2}^{t-1}\eta_s\exp\!\left(-\frac{\rho_{\psi,\beta,b}(X)}4\sum_{\tau=t_2}^{s-1}\eta_\tau\right).
\label{eq:app-ga-remainder}
\end{align}
\end{corollary}

\begin{proof}
Write \(\rho:=\rho_{\psi,\beta,b}(X)>0\). Since \(\eta_t\to0\) and \(\beta^{t/2}\to0\), Proposition~\ref{prop:app-ga-descent-recurrence} gives a finite \(t_1\) after which the coefficient of \(\eta_t\mathcal G(W_t)\) is at least \(\rho/2\). The loss Lipschitz bound~\eqref{eq:app-ga-loss-lipschitz}, the divergence \(S_t\to\infty\), and the standard low-loss entrance argument give a finite \(t_2\ge t_1\) after which all empirical margins are nonnegative. The usual cross-entropy loss-to-margin conversion then yields
\[
\operatorname{mar}_X(W_t)\ge\rho\sum_{s=t_2}^{t-1}\eta_s\frac{\mathcal G(W_s)}{L(W_s)}-C_{\psi,\beta,b}\sum_{s=t_2}^{t-1}\eta_s^2-\frac{D_\psi\eta_0}{1-\sqrt\beta}-C_0,
\]
where \(C_0<\infty\) absorbs the finite prefix. The post-entrance descent recurrence and the standard cross-entropy proxy inequalities give
\[
\frac{\mathcal G(W_s)}{L(W_s)}\ge1-\exp\!\left(-\frac{\rho}{4}\sum_{\tau=t_2}^{s-1}\eta_\tau\right).
\]
Hence \(\operatorname{mar}_X(W_t)\ge\rho S_t-C_1\mathcal R_t^{\psi,\beta,b}\) for some \(C_1<\infty\). Lemma~\ref{lem:app-general-lr-remainders} gives \(\mathcal R_t^{\psi,\beta,b}=o(S_t)\), while \(\psi(W_t)\le\psi(W_0)+S_t\) follows directly from the normalized update. Lemma~\ref{lem:app-general-lr-normalization} proves~\eqref{eq:app-ga-general-lr-bound} and~\eqref{eq:app-ga-margin-conclusion}, and Lemma~\ref{lem:effective-margin-direction-convergence} gives~\eqref{eq:app-ga-direction-conclusion}.
\end{proof}

For the three geometries used in the main text, the atom radius is explicit:
\begin{equation}
R_{\mathrm{row}}(X)=\max_n\|x_n\|_2,\qquad R_{\mathrm{sp}}(X)=\max_n\|x_n\|_2,\qquad R_{\max}(X)=\max_n\|x_n\|_1.
\label{eq:app-ga-three-radii}
\end{equation}
Indeed, the duals of the row, spectral, and entrywise max norms are the sum-of-row-Euclidean, nuclear, and entrywise sum norms, respectively, while \(e_cx_n^\top\) has one nonzero row and rank one.

\subsection{Specialization to row normalization}

Define
\begin{equation}
\|A\|_{\mathrm{row}}:=\max_{r\in[k]}\|A_{r,:}\|_2.
\label{eq:app-row-norm}
\end{equation}
For a matrix \(G\), define its row-wise normalization by
\begin{equation}
\bigl(N_{\mathrm{row}}(G)\bigr)_{r,:}:=
\begin{cases}
G_{r,:}/\|G_{r,:}\|_2,&G_{r,:}\ne{\mathbf 0},\\
{\mathbf 0},&G_{r,:}={\mathbf 0}.
\end{cases}
\label{eq:app-row-normalization}
\end{equation}

\begin{proposition}[Row normalization is row-norm steepest descent]
\label{prop:app-row-steepest}
For every \(A,G\in\mathbb R^{k\times d}\),
\begin{equation}
\|A\|_{\max}\le\|A\|_{\mathrm{row}}\le\|A\|_{\mathrm{sum}},
\label{eq:app-row-max-sum}
\end{equation}
and
\begin{equation}
N_{\mathrm{row}}(G)\in\operatorname*{arg\,max}_{\|\Delta\|_{\mathrm{row}}\le1}\langle G,\Delta\rangle,\qquad \max_{\|\Delta\|_{\mathrm{row}}\le1}\langle G,\Delta\rangle=\sum_{r=1}^k\|G_{r,:}\|_2.
\label{eq:app-row-steepest}
\end{equation}
Thus row-wise normalized gradient methods are normalized steepest-descent methods for the row norm.
\end{proposition}

\begin{proof}
Every entry is bounded by the Euclidean norm of its row, while \(\max_r\|A_{r,:}\|_2\le\sum_r\|A_{r,:}\|_1=\|A\|_{\mathrm{sum}}\), proving~\eqref{eq:app-row-max-sum}. The constraint \(\|\Delta\|_{\mathrm{row}}\le1\) is equivalent to \(\|\Delta_{r,:}\|_2\le1\) for every row \(r\), so
\[
\max_{\|\Delta\|_{\mathrm{row}}\le1}\langle G,\Delta\rangle=\sum_{r=1}^k\max_{\|u\|_2\le1}\langle G_{r,:},u\rangle=\sum_{r=1}^k\|G_{r,:}\|_2,
\]
and \(N_{\mathrm{row}}(G)\) attains the maximum.
\end{proof}

\begin{proposition}[Row-normalized full-batch margin convergence]
\label{cor:app-row-full-batch}
Consider any fixed linearly separable multiclass dataset with bounded features. Under Assumption~\ref{ass:lr-base}, row-wise normalized GD satisfies
\begin{equation}
\frac{\operatorname{mar}_X(W_t)}{\|W_t\|_{\mathrm{row}}}\longrightarrow\gamma_{\mathrm{row}}(X),\qquad \inf_{V\in\mathcal U_{\mathrm{row}}^X}\left\|\frac{W_t}{\|W_t\|_{\mathrm F}}-\frac{V}{\|V\|_{\mathrm F}}\right\|_{\mathrm F}\longrightarrow0.
\label{eq:app-row-gd-direction}
\end{equation}
If Assumption~\ref{ass:lr-momentum} also holds, the same two conclusions hold for row-wise normalized GD with every fixed momentum parameter \(\beta\in(0,1)\).
\end{proposition}

\begin{proof}
Proposition~\ref{prop:app-row-steepest} shows that the row norm is max--sum dominated and that row normalization is its exact steepest direction. Theorem~\ref{thm:app-ms-full-batch} therefore gives normalized-margin convergence, and Lemma~\ref{lem:exact-margin-direction-convergence} gives \(d_{\cos}(W_t,\mathcal U_{\mathrm{row}}^X)\to0\). Since
\[
\inf_{V\in\mathcal U_{\mathrm{row}}^X}\left\|\frac{W_t}{\|W_t\|_{\mathrm F}}-\frac{V}{\|V\|_{\mathrm F}}\right\|_{\mathrm F}^2=2d_{\cos}(W_t,\mathcal U_{\mathrm{row}}^X),
\]
the Frobenius direction-set convergence follows.
\end{proof}

The stochastic counterpart follows from the geometry-adapted random-reshuffling result above and is stated separately because the approximation gap depends on the batch size and momentum.

\begin{proposition}[Random-reshuffling row-wise normalized SGDM margin convergence]
\label{cor:row-sgdm-rr-margin-clean}
Consider any fixed linearly separable multiclass dataset with bounded features. Run random-reshuffling row-wise normalized SGDM with \(b\mid N\), \(m:=N/b\), and fixed \(\beta\in(0,1)\). Define
\begin{equation}
R_{\mathrm{row}}(X):=\max_{n\in[N]}\|x_n\|_2,\qquad \varepsilon_{\mathrm{row},\beta,b}(X):=2(1-\beta)m(m^2-1)R_{\mathrm{row}}(X).
\label{eq:app-row-sgdm-epsilon}
\end{equation}
Under Assumptions~\ref{ass:lr-base} and~\ref{ass:lr-momentum}, if \(\varepsilon_{\mathrm{row},\beta,b}(X)<\gamma_{\mathrm{row}}(X)\), then
\begin{equation}
\liminf_{t\to\infty}\frac{\operatorname{mar}_X(W_t)}{\|W_t\|_{\mathrm{row}}}\ge\gamma_{\mathrm{row}}(X)-\varepsilon_{\mathrm{row},\beta,b}(X),
\label{eq:row-sgdm-rr-margin-clean}
\end{equation}
and
\begin{equation}
d_{\cos}\!\left(W_t,\mathcal U_{\mathrm{row},\varepsilon_{\mathrm{row},\beta,b}(X)}^X\right)\longrightarrow0.
\label{eq:row-sgdm-rr-direction-clean}
\end{equation}
\end{proposition}

\begin{proof}
Apply Corollary~\ref{cor:app-ga-margin-transfer} with \(\psi=\|\cdot\|_{\mathrm{row}}\). Equation~\eqref{eq:app-ga-three-radii} gives \(R_\psi(X)=R_{\mathrm{row}}(X)\), so~\eqref{eq:app-ga-effective-margin} becomes \(\gamma_{\mathrm{row}}(X)-\varepsilon_{\mathrm{row},\beta,b}(X)\), yielding~\eqref{eq:row-sgdm-rr-margin-clean}. Equation~\eqref{eq:row-sgdm-rr-direction-clean} follows from~\eqref{eq:app-ga-direction-conclusion}.
\end{proof}

\section{Proofs for Deterministic Optimizers}
\label{app:deterministic-margin}

This appendix proves Theorem~\ref{thm:det-optimizer-accuracy-order-clean} from optimizer-selected margin geometry to population accuracy. We first establish the all-aspect-ratio Row-versus-Spectral and Row-versus-Max geometry: Row representatives align uniformly with the population mean differences, Max representatives have a nonvanishing clipped-coordinate angular gap, and Spectral representatives are controlled through the reduced polar factor of the centered empirical mean matrix, with hard-edge truncation at \(\alpha=1\). We then convert these geometric statements into ordinary multiclass population-error bounds and prove the optimizer-level accuracy ordering. We also derive deterministic population-error exponents that quantify the separation between the three geometries.

For any matrix norm \(\psi\), define the unit-margin feasible set and the corresponding minimum-norm representative set by
\begin{equation}
\mathcal F_X:=\left\{W\in\mathbb R^{k\times d}:\operatorname{mar}_X(W)\ge1\right\},\qquad \mathcal M_\psi^X:=\operatorname*{arg\,min}_{W\in\mathcal F_X}\psi(W).
\label{eq:app-hard-margin-representatives}
\end{equation}
Whenever \(X\) is linearly separable, \(\mathcal F_X\) is nonempty, the minimum is attained, and we define
\begin{equation}
r_\psi^*(X):=\min_{W\in\mathcal F_X}\psi(W)=\frac{1}{\gamma_\psi(X)},\qquad \mathcal U_\psi^X=\left\{aW:a>0,\ W\in\mathcal M_\psi^X\right\}.
\label{eq:minmax-equivalence-clean}
\end{equation}
For \(q\in\{\mathrm{row},\mathrm{sp},\max\}\), we abbreviate \(\mathcal M_q^X:=\mathcal M_{\psi_q}^X\) and \(r_q^*(X):=r_{\psi_q}^*(X)\).
For every feasible classifier and every ordered pair \(i\ne j\), define
\begin{equation}
u_{ij}^W:=\frac{w_i-w_j}{\|w_i-w_j\|_2},\qquad u_{ij}^{\mu}:=\frac{\mu_i-\mu_j}{\|\mu_i-\mu_j\|_2},\qquad D_{\mathrm{avg\text{-}bdry}}(W):=\frac1{k(k-1)}\sum_{i\ne j}\|u_{ij}^W-u_{ij}^{\mu}\|_2.
\label{eq:app-det-boundary-definitions}
\end{equation}
These quantities are invariant under positive rescaling of \(W\) and under common row shifts.

\begin{lemma}[Exact invariance under common feature rescaling]
\label{lem:app-common-rescaling}
Let \(c>0\), let \(cX:=\{(cx_n,y_n)\}_{n=1}^N\), and let \(\psi\) be any matrix norm. Then
\begin{equation}
\mathcal F_{cX}=c^{-1}\mathcal F_X,\qquad \mathcal M_\psi^{cX}=c^{-1}\mathcal M_\psi^X,\qquad \mathcal U_\psi^{cX}=\mathcal U_\psi^X.
\label{eq:app-common-rescaling}
\end{equation}
The normalized pairwise row directions and \(D_{\mathrm{avg\text{-}bdry}}\) are unchanged. If the test inputs are also multiplied by \(c\), then \(\operatorname{Acc}_{\pi,d}(c^{-1}W;cX^{\mathrm{te}})=\operatorname{Acc}_{\pi,d}(W;X^{\mathrm{te}})\).
\end{lemma}

\begin{proof}
The constraint \((w_{y_n}-w_j)^\top(cx_n)\ge1\) is equivalent, after writing \(W=c^{-1}V\), to \((v_{y_n}-v_j)^\top x_n\ge1\). Positive homogeneity of \(\psi\) proves the identities for the feasible and minimum-norm sets, and the ray identity follows from \eqref{eq:minmax-equivalence-clean}. Finally, \((c^{-1}w_j)^\top(cx)=w_j^\top x\), so all predictions agree.
\end{proof}

Apply Lemma~\ref{lem:app-common-rescaling} with \(c=d^{-1/2}\). Thus, throughout the geometric and accuracy appendices, we work with
\begin{equation}
\widetilde\mu_i:=\frac{\mu_i}{\sqrt d},\qquad \widetilde z_{i,s}:=\frac{z_{i,s}}{\sqrt d},\qquad \widetilde x_{i,s}:=\frac{x_{i,s}}{\sqrt d},
\label{eq:app-normalized-data}
\end{equation}
and suppress tildes. Under this convention, \(\mu_i,z_{i,s}\overset{\mathrm{iid}}{\sim}\mathcal N({\mathbf 0},I_d/d)\) and \(X^{\mathrm{te}}\mid(Y=i)\sim\mathcal N(\mu_i,r_{\mathrm{te}}^2I_d/d)\). This is only a proof normalization: the direction sets, accuracy values, and radius-to-margin ratios used in the optimizer theorems are identical to those in the unnormalized model of Assumption~\ref{ass:gaussian-cloud}.

Let \(G\sim\mathcal N(0,1)\), and define
\begin{equation}
A_{\max}(\tau):=\mathbb E\!\left[|G|\min\{\tau|G|,1\}\right],\qquad B_{\max}(\tau):=\mathbb E\!\left[\min\{\tau^2G^2,1\}\right].
\label{eq:app-det-max-population-functions}
\end{equation}
Let \(\tau_*\in(0,1)\) be the unique solution of \(A_{\max}(\tau_*)=1/\sqrt\pi\), and define
\begin{equation}
C_{\max}:=\bigl(2\pi B_{\max}(\tau_*)\bigr)^{-1/2},\qquad \kappa_{\max}:=\sqrt2\,C_{\max}<1,\qquad \delta_{\max}:=\sqrt{2-2\kappa_{\max}}>0.
\label{eq:delta-max-clean}
\end{equation}
For \(\alpha>0\), let \(\nu_\alpha\) be the mean-one Marchenko--Pastur law with parameter \(1/\alpha\), including its zero atom when \(\alpha<1\):
\begin{align}
\nu_\alpha&=(1-\alpha)_+\delta_0+\frac{\alpha}{2\pi x}\sqrt{(b_\alpha-x)(x-a_\alpha)}\,\mathbf 1_{[a_\alpha,b_\alpha]}(x)\,dx,\nonumber\\
a_\alpha&=(1-\alpha^{-1/2})^2,\qquad b_\alpha=(1+\alpha^{-1/2})^2.
\label{eq:app-det-mp-law-all-aspect}
\end{align}
Set
\begin{equation}
\rho_\alpha:=\alpha\wedge1,\qquad m_{1/2}(\alpha):=\int\sqrt{x}\,d\nu_\alpha(x),\qquad s_\alpha:=\frac{m_{1/2}(\alpha)}{\sqrt\alpha},\qquad \kappa_{\mathrm{sp}}(\alpha):=\frac{m_{1/2}(\alpha)}{\sqrt{\rho_\alpha}},
\label{eq:app-det-all-aspect-scales}
\end{equation}
and
\begin{equation}
\delta_{\mathrm{sp}}(\alpha):=\sqrt{2-2\kappa_{\mathrm{sp}}(\alpha)}>0.
\label{eq:delta-sp-clean}
\end{equation}
Strict positivity follows from strict Cauchy--Schwarz on the positive part of \(\nu_\alpha\): \(m_{1/2}(\alpha)^2<\rho_\alpha\int x\,d\nu_\alpha(x)=\rho_\alpha\). At \(\alpha=1\), \(m_{1/2}(1)=8/(3\pi)\), so \(\delta_{\mathrm{sp}}(1)=\sqrt{2-16/(3\pi)}\).

\begin{theorem}[All-aspect deterministic margin geometry]
\label{thm:deterministic-separation-clean}
Under Assumptions~\ref{ass:gaussian-cloud} and~\ref{ass:asymptotic-regime}, there are deterministic \(\omega_d\downarrow0\) and events \(\mathcal G_d^{\mathrm{det}}\) such that, for every sufficiently large \(d\),
\begin{equation}
\mathbb P(\mathcal G_d^{\mathrm{det}})\ge1-e^{-\sqrt d},
\label{eq:app-det-clean-probability}
\end{equation}
and, on \(\mathcal G_d^{\mathrm{det}}\), \(\mathcal F_X\ne\varnothing\) and
\begin{align}
\sup_{U\in\mathcal U_{\mathrm{row}}^X}\max_{i\ne j}\|u_{ij}^U-u_{ij}^{\mu}\|_2&\le\omega_d,
\label{eq:det-row-uniform-clean}\\
\sup_{U\in\mathcal U_{\mathrm{row}}^X}D_{\mathrm{avg\text{-}bdry}}(U)&\le\omega_d,
\label{eq:det-row-clean}\\
\inf_{U\in\mathcal U_{\mathrm{sp}}^X}D_{\mathrm{avg\text{-}bdry}}(U)&\ge\delta_{\mathrm{sp}}(\alpha)-\omega_d,
\label{eq:det-sp-clean}\\
\inf_{U\in\mathcal U_{\max}^X}D_{\mathrm{avg\text{-}bdry}}(U)&\ge\delta_{\max}-\omega_d.
\label{eq:det-max-clean}
\end{align}
Consequently, for every sufficiently large \(d\) on \(\mathcal G_d^{\mathrm{det}}\),
\begin{equation}
\sup_{U\in\mathcal U_{\mathrm{row}}^X}D_{\mathrm{avg\text{-}bdry}}(U)<\min\left\{\inf_{V\in\mathcal U_{\mathrm{sp}}^X}D_{\mathrm{avg\text{-}bdry}}(V),\inf_{V\in\mathcal U_{\max}^X}D_{\mathrm{avg\text{-}bdry}}(V)\right\}.
\label{eq:det-set-order-clean}
\end{equation}
\end{theorem}

\subsection{Reduction from Direction Sets to Minimum-norm Representatives}
\label{app:det-reduction}

For every feasible matrix, all pairwise row differences are nonzero. Indeed, if $w_i=w_j$ for some $i\ne j$, then the constraint $(w_i-w_j)^\top x_{i,s}\ge1$ fails. Thus $D_{\mathrm{avg\text{-}bdry}}$ is well defined on $\mathcal F_X$ and on all three minimum-norm solution sets.

\begin{lemma}[Scale-invariant reduction]
\label{lem:app-det-ray-reduction}
For every $\psi\in\{\psi_{\mathrm{row}},\psi_{\mathrm{sp}},\psi_{\max}\}$,
\begin{align}
\sup_{U\in\mathcal U_\psi^X}D_{\mathrm{avg\text{-}bdry}}(U)
&=
\sup_{W\in\mathcal M_\psi^X}D_{\mathrm{avg\text{-}bdry}}(W),
\label{eq:app-det-ray-sup}\\
\inf_{U\in\mathcal U_\psi^X}D_{\mathrm{avg\text{-}bdry}}(U)
&=
\inf_{W\in\mathcal M_\psi^X}D_{\mathrm{avg\text{-}bdry}}(W).
\label{eq:app-det-ray-inf}
\end{align}
\end{lemma}

\begin{proof}
By \eqref{eq:minmax-equivalence-clean}, $\mathcal U_\psi^X=\{aW:a>0,\ W\in\mathcal M_\psi^X\}$. The conclusion follows from the positive scale invariance of $D_{\mathrm{avg\text{-}bdry}}$.
\end{proof}

We will therefore work with the minimum-norm representatives \(\mathcal M_{\mathrm{row}}^X\), \(\mathcal M_{\mathrm{sp}}^X\), and \(\mathcal M_{\max}^X\), and invoke Lemma~\ref{lem:app-det-ray-reduction} only at the end. The max-norm constants are those in \eqref{eq:delta-max-clean}.

\subsection{A Quantitative Common Gaussian-cloud Event}
\label{app:det-common-event}

For each class, define
\begin{equation}
\bar z_i:=\frac1{n_i}\sum_{s=1}^{n_i}z_{i,s},\qquad m_i:=\frac1{n_i}\sum_{s=1}^{n_i}x_{i,s}=\mu_i+r\bar z_i,\qquad e_{i,s}:=x_{i,s}-m_i=r(z_{i,s}-\bar z_i).
\label{eq:app-det-mean-residual}
\end{equation}
Then \(x_{i,s}=m_i+e_{i,s}\), \(\sum_s e_{i,s}={\mathbf 0}\), and \(m_i\) is independent of the centered residual family \((e_{i,1},\ldots,e_{i,n_i})\).
Moreover,
\begin{equation}
\operatorname{Cov}(e_{i,s})=r^2\left(1-\frac1{n_i}\right)\frac{I_d}{d}.
\label{eq:app-det-residual-covariance}
\end{equation}
For \(i\ne j\), write \(u_{ij}^m:=(m_i-m_j)/\|m_i-m_j\|_2\).
For \(u\in\mathbb S^{d-1}\), set
\begin{equation}
H_u(\tau):=\frac1d\sum_{\ell=1}^d\sqrt d\,|u_\ell|\min\{\tau\sqrt d\,|u_\ell|,1\},\qquad J_u(\tau):=\frac1d\sum_{\ell=1}^d\min\{\tau^2d u_\ell^2,1\}.
\label{eq:app-det-HJ}
\end{equation}

\begin{proposition}[Stretched-exponential common input]
\label{prop:app-det-common-event}
Under Assumptions~\ref{ass:gaussian-cloud} and~\ref{ass:asymptotic-regime}, there are deterministic \(\varepsilon_d\downarrow0\), universal \(C<\infty\), a constant \(c=c(r_0)>0\), and events \(\mathcal E_d^{\mathrm{det}}\) such that
\begin{equation}
\mathbb P\!\left((\mathcal E_d^{\mathrm{det}})^c\right)\le C\exp\!\left[-c\sqrt{d\log(2N(d)(d+1))}\right]
\label{eq:app-det-common-event-probability}
\end{equation}
and the following statements hold simultaneously on \(\mathcal E_d^{\mathrm{det}}\):
\begin{enumerate}[label=(H\arabic*)]
\item \label{item:app-det-H1} \(\max_i\|m_i-\mu_i\|_2\le\varepsilon_d\).
\item \label{item:app-det-H2} \(\max_i|\|m_i\|_2-1|\le\varepsilon_d\) and \(\max_{i\ne j}|m_i^\top m_j|\le\varepsilon_d\).
\item \label{item:app-det-H3} \(\max_{i,j,s}|m_j^\top e_{i,s}|\le\varepsilon_d\).
\item \label{item:app-det-H4} \(\max_{i\ne j}\|u_{ij}^m-u_{ij}^{\mu}\|_2\le\varepsilon_d\).
\item \label{item:app-det-H5} The max-norm optimum is finite and satisfies
\begin{equation}
2\sqrt d\,r_{\max}^*(X)\le\sqrt{2\pi}+\varepsilon_d.
\label{eq:app-det-max-optimum-upper}
\end{equation}
\item \label{item:app-det-H6} Uniformly over ordered pairs,
\begin{equation}
\sup_{0\le\tau\le1}|H_{u_{ij}^m}(\tau)-A_{\max}(\tau)|+\sup_{0\le\tau\le1}|J_{u_{ij}^m}(\tau)-B_{\max}(\tau)|\le\varepsilon_d.
\label{eq:app-det-HJ-uniform}
\end{equation}
\end{enumerate}
\end{proposition}

\begin{proof}
Put \(L_d:=\log(2N(d)(d+1))\) and \(a_d:=(L_d/d)^{1/4}\).
The assumptions \(\log N=o(d)\) and \(n_{\min}\to\infty\) imply \(a_d\to0\).
All normalized Gaussian vectors below obey the shell bound
\begin{equation}
\mathbb P\!\left(\left|\frac{\|g\|_2}{\sqrt d}-1\right|>t\right)\le2e^{-cdt^2},\qquad 0<t<1,
\label{eq:app-det-gaussian-shell}
\end{equation}
and every centered Gaussian scalar with variance at most \(C/d\) obeys \(\mathbb P(|Z|>t)\le2e^{-cdt^2/C}\).

Since \(\bar z_i\overset{\mathrm d}=g_i/\sqrt{n_i d}\), the event \(\max_i\|\bar z_i\|_2\le2/\sqrt{n_{\min}}\) fails with probability at most \(2N e^{-cd}\).
Thus
\begin{equation}
\max_i\|m_i-\mu_i\|_2\le\frac{2r_0}{\sqrt{n_{\min}}}=o(1)
\label{eq:app-det-mean-concentration}
\end{equation}
on an event whose failure probability is smaller than the right-hand side of \eqref{eq:app-det-common-event-probability}.
Applying \eqref{eq:app-det-gaussian-shell} to \(\mu_i\sim\mathcal N({\mathbf 0},I_d/d)\), with threshold \(Ca_d\), and taking a union bound proves \(\max_i|\|\mu_i\|_2-1|\le Ca_d\).
Conditionally on \(\mu_i\), the variable \(\mu_i^\top\mu_j\) is centered Gaussian with variance \(\|\mu_i\|_2^2/d\).
Another union bound over at most \(N^2\) ordered pairs gives \(\max_{i\ne j}|\mu_i^\top\mu_j|\le Ca_d\).
Combining these estimates with \eqref{eq:app-det-mean-concentration} proves (H1) and (H2).

Conditionally on all empirical means, \(m_j^\top e_{i,s}\) is centered Gaussian with variance at most \(r_0^2\|m_j\|_2^2/d\).
On the event \(\max_j\|m_j\|_2\le2\), a union bound over at most \(N^2\) triples gives \(\max_{i,j,s}|m_j^\top e_{i,s}|\le Ca_d\) with the required failure probability.
This proves (H3).
Items (H1)--(H2) imply \(\inf_{i\ne j}\|m_i-m_j\|_2\ge1\) and \(\inf_{i\ne j}\|\mu_i-\mu_j\|_2\ge1\) for all sufficiently large \(d\).
The normalization map is Lipschitz on vectors with norm at least one, so (H1) gives (H4).

For (H5), let \(s_i:=\operatorname{sign}(m_i)\in\{-1,1\}^d\).
Because each \(m_i\) is isotropic Gaussian, \(s_i\) is a uniform sign vector, and it is independent of every residual in its own class.
Bernstein's inequality for \(\|m_i\|_1\), Gaussian tails for \(s_j^\top m_i\) when \(i\ne j\), conditional Gaussian tails for \(s_j^\top e_{i,s}\), and a union bound give
\begin{equation}
\min_i\|m_i\|_1\ge\left(\sqrt{\frac2\pi}-o(1)\right)\sqrt d,\qquad \max_{i\ne j}|s_j^\top m_i|+\max_{i,j,s}|s_j^\top e_{i,s}|=o(\sqrt d)
\label{eq:app-det-sign-inputs}
\end{equation}
outside an event bounded as in \eqref{eq:app-det-common-event-probability}.
The matrix with rows \(w_i=t_ds_i\), where \(t_d\) is the reciprocal of the minimum bracket in
\begin{equation}
(s_i-s_j)^\top x_{i,s}\ge\|m_i\|_1-|s_j^\top m_i|-2\max_{\ell}|s_\ell^\top e_{i,s}|,
\label{eq:app-det-sign-margin}
\end{equation}
is feasible.
Equation \eqref{eq:app-det-sign-inputs} gives \(t_d\le(\sqrt{\pi/2}+o(1))/\sqrt d\), which is exactly (H5).

For (H6), each difference \(m_i-m_j\) is a centered isotropic Gaussian vector, even when \(n_i\ne n_j\), so \(u_{ij}^m\) is exactly uniform on \(\mathbb S^{d-1}\).
For one uniform direction, the two coordinate empirical processes in \eqref{eq:app-det-HJ} concentrate about \(A_{\max}\) and \(B_{\max}\) at rate \(e^{-cd\eta^2}\) for any fixed \(\tau\).
On the event \(d^{-1}\sum_\ell d u_\ell^2=1\), both processes are uniformly Lipschitz in \(\tau\).
A mesh of \([0,1]\) with spacing \(a_d\), followed by a union bound over the mesh and at most \(N^2\) ordered pairs, yields \eqref{eq:app-det-HJ-uniform} with error \(Ca_d\) and failure probability bounded by \eqref{eq:app-det-common-event-probability}.
Taking \(\varepsilon_d\) to dominate \(Ca_d\), \(2r_0/\sqrt{n_{\min}}\), and the deterministic vanishing envelopes in \eqref{eq:app-det-sign-inputs} completes the proof.
\end{proof}

\begin{lemma}[Normalization stability]
\label{lem:app-det-normalization}
Let \(x,y\in\mathbb R^d\) with \(y\ne{\mathbf 0}\).
If \(\|x-y\|_2\le\|y\|_2/2\), then
\begin{equation}
\left\|\frac{x}{\|x\|_2}-\frac{y}{\|y\|_2}\right\|_2\le\frac{4\|x-y\|_2}{\|y\|_2}.
\label{eq:app-det-normalization-bound}
\end{equation}
\end{lemma}

\begin{proof}
The reverse triangle inequality gives \(\|x\|_2\ge\|y\|_2/2\).
Adding and subtracting \(y/\|x\|_2\) and applying the reverse triangle inequality once more gives the displayed bound.
\end{proof}

\subsection{Row-norm Representatives Align with the Population Mean Differences}
\label{app:det-row}

The row-norm part is deterministic once the empirical means and residuals satisfy the small-error conditions in Proposition~\ref{prop:app-det-common-event}.

\begin{proposition}[Deterministic row-norm rigidity]
\label{prop:app-det-row-rigidity}
Fix deterministic samples $x_{i,s}=m_i+e_{i,s}$ with $\sum_s e_{i,s}={\mathbf 0}$ for every $i$. Suppose that, for some $0<h<1/20$,
\begin{equation}
\max_i\bigl|\|m_i\|_2-1\bigr|\le h,
\qquad
\max_{i\ne j}|m_i^\top m_j|\le h,
\qquad
\max_{i,j,s}|m_j^\top e_{i,s}|\le h.
\label{eq:app-det-row-rigidity-assumption}
\end{equation}
Then $\mathcal F_X$ is nonempty, and every $W\in\mathcal M_{\mathrm{row}}^X$ satisfies
\begin{equation}
\max_i\|w_i-m_i\|_2
\le
C\left(h+\frac1k\right)^{1/4}
\label{eq:app-det-row-rigidity-conclusion}
\end{equation}
for a universal constant $C$.
\end{proposition}

\begin{proof}
For $i\ne j$ and $s\le n_i$,
\begin{align}
(m_i-m_j)^\top x_{i,s}
&=\|m_i\|_2^2-m_j^\top m_i+(m_i-m_j)^\top e_{i,s}
\notag\\
&\ge (1-h)^2-h-2h
\ge 1-5h.
\label{eq:app-det-row-witness-margin}
\end{align}
Hence the matrix with rows $\widetilde w_i=m_i/(1-5h)$ is feasible. If $R^*:=\max_i\|w_i\|_2$ is the optimal row radius, then
\begin{equation}
R^*\le\frac{1+h}{1-5h}\le1+Ch.
\label{eq:app-det-row-radius-upper}
\end{equation}

Let $\bar w:=k^{-1}\sum_{r=1}^k w_r$. For any feasible $W$, averaging the constraints over $s=1,\ldots,n_i$ and summing over $j\ne i$ gives
\begin{equation}
\sum_{j\ne i}(w_i-w_j)^\top m_i\ge k-1.
\label{eq:app-det-row-summed-constraints}
\end{equation}
Since $\sum_{j\ne i}w_j=k\bar w-w_i$, this is equivalent to
\begin{equation}
(w_i-\bar w)^\top m_i\ge1-\frac1k.
\label{eq:app-det-row-centering-constraint}
\end{equation}
Consequently,
\begin{equation}
\|w_i-\bar w\|_2
\ge
L:=\frac{1-1/k}{1+h}.
\label{eq:app-det-row-centered-lower}
\end{equation}
For a row-norm minimizer, the variance identity yields
\begin{align}
L^2
&\le
\frac1k\sum_{i=1}^k\|w_i-\bar w\|_2^2
=
\frac1k\sum_{i=1}^k\|w_i\|_2^2-\|\bar w\|_2^2
\le
(R^*)^2-\|\bar w\|_2^2.
\label{eq:app-det-row-variance-identity}
\end{align}
Combining this with \eqref{eq:app-det-row-radius-upper} gives
\begin{equation}
\|\bar w\|_2^2\le C\left(h+\frac1k\right).
\label{eq:app-det-row-common-shift-small}
\end{equation}
Using \eqref{eq:app-det-row-centering-constraint} once more,
\begin{equation}
w_i^\top m_i
\ge
1-\frac1k+\bar w^\top m_i
\ge
1-C\left(h+\frac1k\right)^{1/2}.
\label{eq:app-det-row-alignment-lower}
\end{equation}
Together with $\|w_i\|_2\le1+Ch$ and $\|m_i\|_2\le1+h$, this implies
\begin{align}
\|w_i-m_i\|_2^2
&=\|w_i\|_2^2+\|m_i\|_2^2-2w_i^\top m_i
\le
C\left(h+\frac1k\right)^{1/2}.
\label{eq:app-det-row-distance-square}
\end{align}
Taking square roots proves \eqref{eq:app-det-row-rigidity-conclusion}.
\end{proof}

\begin{proposition}[Uniform row-norm alignment]
\label{prop:app-det-row-boundary}
Under \eqref{eq:gaussian-cloud}--\eqref{eq:asymptotic-regime-clean},
\begin{equation}
\sup_{W\in\mathcal M_{\mathrm{row}}^X}
D_{\mathrm{avg\text{-}bdry}}(W)
=o_{\mathbb P}(1).
\label{eq:app-det-row-boundary-convergence}
\end{equation}
In fact, the stronger uniform pairwise statement
\begin{equation}
\sup_{W\in\mathcal M_{\mathrm{row}}^X}
\max_{i\ne j}
\|u_{ij}^W-u_{ij}^{\mu}\|_2
=o_{\mathbb P}(1)
\label{eq:app-det-row-uniform-pair}
\end{equation}
holds.
\end{proposition}

\begin{proof}
Work on $\mathcal E_d^{\mathrm{det}}$ from Proposition~\ref{prop:app-det-common-event}. Apply Proposition~\ref{prop:app-det-row-rigidity} with $h=\varepsilon_d$. Since $k/d\to\alpha\in(0,\infty)$, we have $k\to\infty$, and therefore, uniformly over all row-norm minimizers,
\begin{equation}
\max_i\|w_i-m_i\|_2=o(1).
\label{eq:app-det-row-to-mean}
\end{equation}
Combining this with \ref{item:app-det-H1} gives
\begin{equation}
\Delta_d^{\mathrm{row}}
:=
\sup_{W\in\mathcal M_{\mathrm{row}}^X}
\max_i\|w_i-\mu_i\|_2
=o(1)
\label{eq:app-det-row-to-center}
\end{equation}
on $\mathcal E_d^{\mathrm{det}}$. Moreover, $\inf_{i\ne j}\|\mu_i-\mu_j\|_2$ is bounded below by a positive constant on the same event. Hence
\begin{equation}
\|(w_i-w_j)-(\mu_i-\mu_j)\|_2
\le2\Delta_d^{\mathrm{row}}
\label{eq:app-det-row-pair-perturbation}
\end{equation}
uniformly over $i\ne j$ and $W\in\mathcal M_{\mathrm{row}}^X$. Applying Lemma~\ref{lem:app-det-normalization} proves \eqref{eq:app-det-row-uniform-pair}, and averaging over ordered pairs gives \eqref{eq:app-det-row-boundary-convergence}.
\end{proof}

\subsection{Max-norm Representatives Have a Positive Angle Gap}
\label{app:det-max}

The max-norm proof has two ingredients. First, optimality imposes a coordinate cap of order $d^{-1/2}$ on every pairwise row difference. Second, feasibility forces every such difference to have a large projection onto the empirical mean-pair direction. The clipped minimum-energy problem below shows that these two requirements are incompatible with near-perfect angular alignment.

\subsubsection{The Clipped Minimum-energy Problem}

\begin{proposition}[Exact clipped problem]
\label{prop:app-det-clipped-exact}
Fix $u\in\mathbb S^{d-1}$ and $a,b>0$. Define
\begin{equation}
\mathfrak m_d(u;a,b)^2
:=
\min\left\{
\|x\|_2^2:
x^\top u\ge b,
\ \|x\|_\infty\le\frac{a}{\sqrt d}
\right\}.
\label{eq:app-det-clipped-problem}
\end{equation}
Let $z_\ell:=\sqrt d\,|u_\ell|$. If
\begin{equation}
0<b<a\frac1d\sum_{\ell=1}^d z_\ell,
\label{eq:app-det-clipped-feasibility}
\end{equation}
then the minimizer is unique and has coordinates
\begin{equation}
x_\ell
=
\frac{a}{\sqrt d}\operatorname{sign}(u_\ell)\min\{\tau z_\ell,1\},
\qquad
\ell=1,\ldots,d,
\label{eq:app-det-clipped-solution}
\end{equation}
where $\operatorname{sign}(0):=0$ and $\tau>0$ is the unique solution of
\begin{equation}
H_u(\tau)=\frac ba.
\label{eq:app-det-clipped-tau}
\end{equation}
Moreover,
\begin{equation}
\mathfrak m_d(u;a,b)^2=a^2J_u(\tau).
\label{eq:app-det-clipped-value}
\end{equation}
\end{proposition}

\begin{proof}
We first verify that the feasible set is nonempty. Define the coordinate-saturated vector $\widehat x\in\mathbb R^d$ by
\begin{equation}
\widehat x_\ell:=\frac{a}{\sqrt d}\operatorname{sign}(u_\ell).
\label{eq:app-det-clipped-saturated-vector}
\end{equation}
Its projection onto $u$ is
\begin{equation}
\widehat x^\top u
=
\frac{a}{\sqrt d}\sum_{\ell=1}^d|u_\ell|
=
\frac ad\sum_{\ell=1}^d z_\ell
>
b,
\label{eq:app-det-clipped-max-projection}
\end{equation}
where the strict inequality follows from \eqref{eq:app-det-clipped-feasibility}. Choose \(c\) such that \(b/(\widehat x^\top u)<c<1\). Then \(c\widehat x\) satisfies the projection constraint strictly and lies strictly inside the coordinate box. The feasible set is thus nonempty.
It is also closed and bounded, so a minimizer exists. Since the objective $x\mapsto\|x\|_2^2$ is strictly convex and the feasible set is convex, this minimizer is unique.

Let $x^\star$ denote the unique minimizer. We next show that every nonzero coordinate of $x^\star$ has the same sign as the corresponding coordinate of $u$. Given any feasible $x$, define $x^\sharp$ by
\begin{equation}
x^\sharp_\ell:=
\begin{cases}
\operatorname{sign}(u_\ell)|x_\ell|,&u_\ell\ne0,\\
0,&u_\ell=0.
\end{cases}
\label{eq:app-det-clipped-sign-alignment}
\end{equation}
Then $\|x^\sharp\|_\infty\le\|x\|_\infty$, $\|x^\sharp\|_2\le\|x\|_2$, and
\begin{equation}
(x^\sharp)^\top u
=
\sum_{\ell=1}^d|x_\ell||u_\ell|
\ge
\sum_{\ell=1}^d x_\ell u_\ell
=
x^\top u.
\label{eq:app-det-clipped-sign-projection}
\end{equation}
Applying this construction to $x^\star$ shows that $x^\sharp$ is feasible and has objective value no larger than that of $x^\star$. By optimality and uniqueness, $x^\star=x^\sharp$. Hence $x^\star_\ell$ has the sign of $u_\ell$ when $u_\ell\ne0$, and $x^\star_\ell=0$ when $u_\ell=0$.

The projection constraint must be active at the minimizer. Indeed, if $(x^\star)^\top u>b$, then
\begin{equation}
c:=\frac{b}{(x^\star)^\top u}
\end{equation}
satisfies $0<c<1$. The vector $cx^\star$ obeys $\|cx^\star\|_\infty\le a/\sqrt d$ and $(cx^\star)^\top u=b$, while
\begin{equation}
\|cx^\star\|_2^2=c^2\|x^\star\|_2^2<\|x^\star\|_2^2,
\end{equation}
contradicting the optimality of $x^\star$. Therefore,
\begin{equation}
(x^\star)^\top u=b.
\label{eq:app-det-clipped-active-projection}
\end{equation}

Because of the sign alignment, write
\begin{equation}
y_\ell:=\frac{\sqrt d}{a}|x^\star_\ell|,
\qquad
0\le y_\ell\le1.
\label{eq:app-det-clipped-y-definition}
\end{equation}
Using $|u_\ell|=z_\ell/\sqrt d$, the active projection constraint becomes
\begin{equation}
b
=
\sum_{\ell=1}^d|x^\star_\ell||u_\ell|
=
\frac ad\sum_{\ell=1}^d z_\ell y_\ell,
\end{equation}
or equivalently,
\begin{equation}
\sum_{\ell=1}^d z_\ell y_\ell=\frac{db}{a}.
\label{eq:app-det-clipped-reduced-constraint}
\end{equation}
Moreover,
\begin{equation}
\|x^\star\|_2^2
=
\frac{a^2}{d}\sum_{\ell=1}^d y_\ell^2.
\label{eq:app-det-clipped-reduced-objective}
\end{equation}
Therefore, finding $x^\star$ is equivalent to solving the convex optimization problem
\begin{equation}
\min_{0\le y_\ell\le1}
\frac12\sum_{\ell=1}^d y_\ell^2
\qquad
\text{subject to}
\qquad
\sum_{\ell=1}^d z_\ell y_\ell\ge\frac{db}{a}.
\label{eq:app-det-clipped-reduced-problem}
\end{equation}
The strict inequality in \eqref{eq:app-det-clipped-feasibility} implies that the right-hand side of the constraint is strictly smaller than $\sum_\ell z_\ell$, so the reduced problem has a strictly feasible point. Hence the KKT conditions are necessary and sufficient.

Let $\tau\ge0$ be the multiplier associated with the constraint in \eqref{eq:app-det-clipped-reduced-problem}. For fixed $\tau$, the Lagrangian minimization over the box $[0,1]^d$ separates across coordinates:
\begin{equation}
y_\ell
\in
\operatorname*{arg\,min}_{0\le y\le1}
\left\{
\frac12y^2-\tau z_\ell y
\right\}.
\label{eq:app-det-clipped-coordinate-kkt}
\end{equation}
The unconstrained minimizer of the scalar objective is $y=\tau z_\ell$. Projecting this value onto $[0,1]$ gives
\begin{equation}
y_\ell=\min\{\tau z_\ell,1\}.
\label{eq:app-det-clipped-y-solution}
\end{equation}
The multiplier must satisfy $\tau>0$. If $\tau=0$, then \eqref{eq:app-det-clipped-coordinate-kkt} gives $y_\ell=0$ for every $\ell$, which violates the constraint because $b>0$. Since $\tau>0$, complementary slackness implies that the projection constraint is active:
\begin{equation}
\sum_{\ell=1}^d z_\ell\min\{\tau z_\ell,1\}
=
\frac{db}{a}.
\label{eq:app-det-clipped-multiplier-equation}
\end{equation}
After dividing by $d$, this is exactly
\begin{equation}
H_u(\tau)
=
\frac1d\sum_{\ell=1}^d z_\ell\min\{\tau z_\ell,1\}
=
\frac ba.
\end{equation}

It remains to verify that this equation has a unique solution. The function $H_u$ is continuous and nondecreasing on $[0,\infty)$, with
\begin{equation}
H_u(0)=0,
\qquad
\lim_{\tau\to\infty}H_u(\tau)
=
\frac1d\sum_{\ell=1}^d z_\ell.
\label{eq:app-det-clipped-H-endpoints}
\end{equation}
The condition \eqref{eq:app-det-clipped-feasibility} places $b/a$ strictly between these two values, so existence follows from continuity. To prove uniqueness, take $0\le\tau_1<\tau_2$ and suppose that
\begin{equation}
H_u(\tau_1)<\frac1d\sum_{\ell=1}^d z_\ell.
\end{equation}
Then at least one coordinate with $z_\ell>0$ is not saturated at $\tau_1$, meaning $\tau_1z_\ell<1$. For that coordinate,
\begin{equation}
\min\{\tau_2z_\ell,1\}>\min\{\tau_1z_\ell,1\},
\end{equation}
while every other coordinate contributes a nondecreasing term. Thus $H_u(\tau_2)>H_u(\tau_1)$. Hence $H_u$ is strictly increasing before reaching its limiting value, and the equation $H_u(\tau)=b/a$ has a unique solution.

Substituting \eqref{eq:app-det-clipped-y-solution} into the sign-aligned representation of $x^\star$ gives
\begin{equation}
x^\star_\ell
=
\frac{a}{\sqrt d}\operatorname{sign}(u_\ell)\min\{\tau z_\ell,1\},
\end{equation}
which proves \eqref{eq:app-det-clipped-solution}. Finally,
\begin{align}
\mathfrak m_d(u;a,b)^2
&=
\|x^\star\|_2^2 \notag\\
&=
\frac{a^2}{d}\sum_{\ell=1}^d\min\{\tau z_\ell,1\}^2 \notag\\
&=
a^2\frac1d\sum_{\ell=1}^d\min\{\tau^2z_\ell^2,1\} \notag\\
&=
a^2J_u(\tau),
\end{align}
which proves \eqref{eq:app-det-clipped-value}.
\end{proof}

\begin{lemma}[Uniform asymptotics of the clipped problem]
\label{lem:app-det-clipped-asymptotic}
There exist constants $\delta_0,C_0>0$ such that the following holds. Suppose $u\in\mathbb S^{d-1}$ satisfies
\begin{equation}
\sup_{0\le\tau\le1}|H_u(\tau)-A_{\max}(\tau)|+\sup_{0\le\tau\le1}|J_u(\tau)-B_{\max}(\tau)|\le\delta
\label{eq:app-det-clipped-HJ-close}
\end{equation}
with $0<\delta\le\delta_0$, and suppose
\begin{equation}
|a-\sqrt{2\pi}|+|b-\sqrt2|\le\delta.
\label{eq:app-det-clipped-ab-close}
\end{equation}
Then \eqref{eq:app-det-clipped-feasibility} holds, the solution in Proposition~\ref{prop:app-det-clipped-exact} satisfies $\tau\in(0,1)$, and
\begin{align}
\frac1{\mathfrak m_d(u;a,b)}
&\le C_{\max}+C_0\delta,
\label{eq:app-det-clipped-reciprocal}\\
\frac b{\mathfrak m_d(u;a,b)}
&\le\kappa_{\max}+C_0\delta.
\label{eq:app-det-clipped-cosine}
\end{align}
\end{lemma}

\begin{proof}
Throughout the proof, $C>0$ denotes a constant independent of $d,u,a,b$, and $\delta$, and its value may change from line to line. Recall that $\tau_*>0$ is defined by
\begin{equation}
A_{\max}(\tau_*)=\frac1{\sqrt\pi}.
\label{eq:app-det-tau-star-recall}
\end{equation}
We first compare the required normalized projection $b/a$ with this reference value. By taking $\delta_0$ small enough, \eqref{eq:app-det-clipped-ab-close} ensures that $a\ge\sqrt{2\pi}/2$. Hence
\begin{align}
\left|\frac ba-\frac1{\sqrt\pi}\right|
&=
\left|\frac ba-\frac{\sqrt2}{\sqrt{2\pi}}\right| \notag\\
&\le
\frac{|b-\sqrt2|}{a}
+
\sqrt2\left|\frac1a-\frac1{\sqrt{2\pi}}\right| \notag\\
&=
\frac{|b-\sqrt2|}{a}
+
\frac{\sqrt2\,|a-\sqrt{2\pi}|}{a\sqrt{2\pi}}
\le C\delta.
\label{eq:app-det-clipped-ratio-close}
\end{align}
Thus the right-hand side of the equation $H_u(\tau)=b/a$ is within $O(\delta)$ of $A_{\max}(\tau_*)$.

We next show that the clipped problem is feasible and that its multiplier satisfies $\tau<1$. Since $\tau_*<1$ and $A_{\max}$ is strictly increasing, the constant
\begin{equation}
g_*:=A_{\max}(1)-A_{\max}(\tau_*)=A_{\max}(1)-\frac1{\sqrt\pi}
\label{eq:app-det-clipped-positive-gap}
\end{equation}
is strictly positive. By \eqref{eq:app-det-clipped-HJ-close},
\begin{equation}
H_u(1)\ge A_{\max}(1)-\delta=\frac1{\sqrt\pi}+g_*-\delta.
\label{eq:app-det-clipped-H-one-lower}
\end{equation}
On the other hand, \eqref{eq:app-det-clipped-ratio-close} gives
\begin{equation}
\frac ba\le\frac1{\sqrt\pi}+C\delta.
\label{eq:app-det-clipped-ratio-upper}
\end{equation}
After decreasing $\delta_0$ if necessary, these two bounds imply
\begin{equation}
0<\frac ba<H_u(1).
\label{eq:app-det-clipped-ratio-below-H-one}
\end{equation}
Because
\begin{equation}
H_u(1)=\frac1d\sum_{\ell=1}^d z_\ell\min\{z_\ell,1\}\le\frac1d\sum_{\ell=1}^d z_\ell,
\end{equation}
we obtain
\begin{equation}
0<b<a\frac1d\sum_{\ell=1}^d z_\ell.
\end{equation}
This is exactly \eqref{eq:app-det-clipped-feasibility}. Moreover, $H_u(0)=0$ and $H_u(1)>b/a$, so Proposition~\ref{prop:app-det-clipped-exact} gives a unique solution $\tau\in(0,1)$ of
\begin{equation}
H_u(\tau)=\frac ba.
\label{eq:app-det-clipped-solution-equation}
\end{equation}

We now show that this solution is close to $\tau_*$. Using \eqref{eq:app-det-clipped-solution-equation}, \eqref{eq:app-det-clipped-HJ-close}, \eqref{eq:app-det-clipped-ratio-close}, and \eqref{eq:app-det-tau-star-recall}, we obtain
\begin{align}
|A_{\max}(\tau)-A_{\max}(\tau_*)|
&\le
|A_{\max}(\tau)-H_u(\tau)|
+
\left|H_u(\tau)-A_{\max}(\tau_*)\right| \notag\\
&=
|A_{\max}(\tau)-H_u(\tau)|
+
\left|\frac ba-\frac1{\sqrt\pi}\right|
\le C\delta.
\label{eq:app-det-clipped-A-close}
\end{align}
For $s\in(0,1]$, differentiation under the expectation gives
\begin{equation}
A_{\max}'(s)=\mathbb E\left[G^2\mathbf{1}_{\{s|G|\le1\}}\right].
\label{eq:app-det-A-derivative}
\end{equation}
Since $s\le1$, the event $\{|G|\le1\}$ is contained in $\{s|G|\le1\}$. Therefore
\begin{equation}
A_{\max}'(s)\ge c_A,
\qquad
c_A:=\mathbb E\left[G^2\mathbf{1}_{\{|G|\le1\}}\right]>0.
\label{eq:app-det-A-derivative-lower}
\end{equation}
The mean value theorem and \eqref{eq:app-det-clipped-A-close} now give
\begin{equation}
c_A|\tau-\tau_*|
\le
|A_{\max}(\tau)-A_{\max}(\tau_*)|
\le C\delta,
\end{equation}
and hence
\begin{equation}
|\tau-\tau_*|\le C\delta.
\label{eq:app-det-clipped-tau-close}
\end{equation}

We next lower-bound $J_u(\tau)$. The first assumption of the lemma gives
\begin{equation}
J_u(\tau)\ge B_{\max}(\tau)-\delta.
\label{eq:app-det-clipped-J-first-lower}
\end{equation}
The function $B_{\max}$ is Lipschitz on $[0,1]$. Indeed, for $s,t\in[0,1]$,
\begin{align}
|B_{\max}(s)-B_{\max}(t)|
&\le
\mathbb E\left[\left|\min\{s^2G^2,1\}-\min\{t^2G^2,1\}\right|\right] \notag\\
&\le
|s^2-t^2|\mathbb E[G^2]
\le
2|s-t|.
\label{eq:app-det-B-Lipschitz}
\end{align}
Combining this bound with \eqref{eq:app-det-clipped-tau-close} gives
\begin{equation}
B_{\max}(\tau)\ge B_{\max}(\tau_*)-C\delta.
\end{equation}
Therefore,
\begin{equation}
J_u(\tau)\ge B_{\max}(\tau_*)-C\delta.
\label{eq:app-det-clipped-J-lower}
\end{equation}

The assumption on $a$ also gives
\begin{align}
a^2
&\ge
\left(\sqrt{2\pi}-\delta\right)^2
\ge
2\pi-C\delta.
\label{eq:app-det-clipped-a-square-lower}
\end{align}
Using the exact identity from Proposition~\ref{prop:app-det-clipped-exact},
\begin{equation}
\mathfrak m_d(u;a,b)^2=a^2J_u(\tau),
\end{equation}
together with \eqref{eq:app-det-clipped-J-lower} and \eqref{eq:app-det-clipped-a-square-lower}, we obtain
\begin{align}
\mathfrak m_d(u;a,b)^2
&\ge
(2\pi-C\delta)\left(B_{\max}(\tau_*)-C\delta\right) \notag\\
&\ge
2\pi B_{\max}(\tau_*)-C\delta.
\label{eq:app-det-clipped-energy-lower}
\end{align}

Set
\begin{equation}
M_*:=2\pi B_{\max}(\tau_*)>0.
\end{equation}
By taking $\delta_0$ smaller if necessary, \eqref{eq:app-det-clipped-energy-lower} implies
\begin{equation}
\mathfrak m_d(u;a,b)^2\ge M_*-C\delta\ge\frac{M_*}{2}.
\end{equation}
Consequently,
\begin{align}
\frac1{\mathfrak m_d(u;a,b)}
&\le
\frac1{\sqrt{M_*-C\delta}} \notag\\
&\le
\frac1{\sqrt{M_*}}+C\delta \notag\\
&=
C_{\max}+C\delta,
\label{eq:app-det-clipped-reciprocal-proof}
\end{align}
where the last equality follows from
\begin{equation}
C_{\max}=\left(2\pi B_{\max}(\tau_*)\right)^{-1/2}.
\end{equation}
After enlarging the constant, this proves \eqref{eq:app-det-clipped-reciprocal}.

Finally, \eqref{eq:app-det-clipped-ab-close} gives $b\le\sqrt2+\delta$. Therefore,
\begin{align}
\frac b{\mathfrak m_d(u;a,b)}
&\le
(\sqrt2+\delta)(C_{\max}+C_0\delta) \notag\\
&\le
\sqrt2\,C_{\max}+C\delta \notag\\
&=
\kappa_{\max}+C\delta,
\end{align}
because $\kappa_{\max}=\sqrt2\,C_{\max}$. Enlarging $C_0$ once more proves \eqref{eq:app-det-clipped-cosine}.
\end{proof}

\subsubsection{Application to Max-norm Minimizers}

\begin{proposition}[Max-norm boundary separation]
\label{prop:app-det-max-boundary}
Under \eqref{eq:gaussian-cloud}--\eqref{eq:asymptotic-regime-clean},
\begin{equation}
\inf_{W\in\mathcal M_{\max}^X}
D_{\mathrm{avg\text{-}bdry}}(W)
\ge\delta_{\max}-o_{\mathbb P}(1).
\label{eq:app-det-max-boundary-separation}
\end{equation}
In fact, on the same high-probability event, every ordered pair of every max-norm minimizer has boundary-normal distance at least $\delta_{\max}-o(1)$.
\end{proposition}

\begin{proof}
Work on $\mathcal E_d^{\mathrm{det}}$ from Proposition~\ref{prop:app-det-common-event}. Fix $W\in\mathcal M_{\max}^X$ and set $v_{ij}:=w_i-w_j$. By \ref{item:app-det-H5},
\begin{equation}
\|v_{ij}\|_\infty
\le2r_{\max}^*(X)
\le\frac{a_d}{\sqrt d},
\qquad
a_d:=\sqrt{2\pi}+\varepsilon_d.
\label{eq:app-det-max-pair-cap}
\end{equation}
Averaging the constraints $(w_i-w_j)^\top x_{i,s}\ge1$ over class $i$ gives $v_{ij}^\top m_i\ge1$. Applying the reverse constraints for class $j$ gives $v_{ij}^\top m_j\le-1$. Therefore
\begin{equation}
v_{ij}^\top(m_i-m_j)\ge2.
\label{eq:app-det-max-pair-projection}
\end{equation}
By \ref{item:app-det-H2}, uniformly over $i\ne j$,
\begin{equation}
\|m_i-m_j\|_2^2
=
\|m_i\|_2^2+\|m_j\|_2^2-2m_i^\top m_j
\le2+o(1).
\label{eq:app-det-mean-pair-norm-upper}
\end{equation}
Hence
\begin{equation}
v_{ij}^\top u_{ij}^m
\ge
\frac2{\|m_i-m_j\|_2}
\ge b_d,
\qquad
b_d:=\sqrt2-o(1),
\label{eq:app-det-max-unit-projection}
\end{equation}
where the error is deterministic on $\mathcal E_d^{\mathrm{det}}$ and uniform over all pairs.

Fix $i\ne j$, and abbreviate $u=u_{ij}^m$, $v=v_{ij}$, $a=a_d$, and $b=b_d$. Since $v^\top u\ge b>0$, define
\begin{equation}
y:=\frac b{v^\top u}v.
\label{eq:app-det-max-rescaled-pair}
\end{equation}
Then $y^\top u=b$, and because $b/(v^\top u)\le1$, \eqref{eq:app-det-max-pair-cap} gives $\|y\|_\infty\le a/\sqrt d$. Thus $y$ is feasible for \eqref{eq:app-det-clipped-problem}, and
\begin{equation}
\mathfrak m_d(u;a,b)
\le\|y\|_2
=
\frac{b\|v\|_2}{v^\top u}.
\label{eq:app-det-max-clipped-comparison}
\end{equation}
Equivalently,
\begin{equation}
\frac{v^\top u}{\|v\|_2}
\le
\frac b{\mathfrak m_d(u;a,b)}.
\label{eq:app-det-max-cosine-upper}
\end{equation}
By \ref{item:app-det-H6} and Lemma~\ref{lem:app-det-clipped-asymptotic}, the right-hand side is at most $\kappa_{\max}+o(1)$, uniformly over all ordered pairs. Since $\kappa_{\max}<1$,
\begin{equation}
\left\|
\frac{v_{ij}}{\|v_{ij}\|_2}-u_{ij}^m
\right\|_2
=
\sqrt{2-2\frac{v_{ij}^\top u_{ij}^m}{\|v_{ij}\|_2}}
\ge
\delta_{\max}-o(1).
\label{eq:app-det-max-away-from-empirical}
\end{equation}
Finally, \ref{item:app-det-H4} and the reverse triangle inequality imply
\begin{equation}
\left\|
\frac{v_{ij}}{\|v_{ij}\|_2}-u_{ij}^{\mu}
\right\|_2
\ge
\delta_{\max}-o(1),
\label{eq:app-det-max-away-from-bayes}
\end{equation}
uniformly over all pairs and all max-norm minimizers. Averaging proves \eqref{eq:app-det-max-boundary-separation}.
\end{proof}

\subsection{Spectral Polar Geometry}
\label{app:det-spectral}

Let
\begin{equation}
P:=I_k-\frac1k\one\one^\top,\qquad M_m:=\begin{bmatrix}m_1^\top\\ \vdots\\ m_k^\top\end{bmatrix},\qquad C_m:=PM_m,\qquad r_d:=\min\{k-1,d\}.
\label{eq:app-det-centered-mean-matrix}
\end{equation}
On the full reduced-rank event, write
\begin{equation}
C_m=U_m\Sigma_mV_m^\top,\qquad U_m\in\mathbb R^{k\times r_d},\quad V_m\in\mathbb R^{d\times r_d},
\label{eq:app-det-reduced-svd}
\end{equation}
and define the reduced polar factor and its positive-semidefinite companion by
\begin{equation}
Q_m:=U_mV_m^\top=C_m(C_m^\top C_m)^{\dagger1/2},\qquad S_m:=Q_mC_m^\top=(C_mC_m^\top)^{1/2}.
\label{eq:app-det-reduced-polar}
\end{equation}
Write \(q_i^\top\) for the \(i\)-th row of \(Q_m\).

\begin{lemma}[Quantitative Gaussian hard-edge count at the square aspect ratio]
\label{lem:app-det-gaussian-critical-count}
Let $p=p(d)$ satisfy $p/d\to1$, let $H_d\in\mathbb R^{p\times d}$ have independent $\mathcal N(0,1)$ entries, and set $A_d:=H_d/\sqrt d$ and $n_d:=\min\{p,d\}$. There are universal constants $T,c,C>0$ such that, for all sufficiently large $d$ and every $2/n_d\le u\le T$, an event $\mathcal H_d(u)$ satisfies
\[
\mathbb P\bigl(\mathcal H_d(u)^c\bigr)\le C\bigl(1+\log(T/u)\bigr)e^{-cdu^2},
\]
and, on this event,
\begin{equation}
\frac1{n_d}\#\{\ell:\sigma_\ell(A_d)\le t\}\le8t,\qquad u\le t\le T.
\label{eq:app-det-gaussian-critical-count}
\end{equation}
Moreover, $A_d$ has rank $n_d$ almost surely.
\end{lemma}

\begin{proof}
Write $n:=n_d$ and take $T=1/16$. Since $n/d\to1$, for all sufficiently large $d$ we have $n/d\ge9/10$. Let $B$ be a fixed $n\times n$ coordinate submatrix of $H_d$, obtained by retaining a fixed set of $n$ rows when $p\ge d$ and a fixed set of $n$ columns when $p<d$. If $p\ge d$, then $H_d^\top H_d=B^\top B+R^\top R\succeq B^\top B$ for the matrix $R$ of deleted rows, while if $p<d$, then $H_dH_d^\top=BB^\top+RR^\top\succeq BB^\top$ for the matrix $R$ of deleted columns. Hence, by monotonicity of eigenvalues under the positive-semidefinite order,
\[
\sigma_j(H_d)\ge\sigma_j(B),\qquad 1\le j\le n.
\]

Fix $t\in[1/n,T]$ and put $\ell:=\lceil3nt\rceil$. For all sufficiently large $n$, the choice $T=1/16$ ensures $\ell<n$. Let $F$ be the $n\times(n-\ell)$ matrix obtained by retaining a fixed set of $n-\ell$ columns of $B$. Apply \citet[Corollary~5.35]{vershynin2012nonasymptotic} to $F$ with row dimension $N=n$, column dimension $n-\ell$, and deviation parameter $s:=t\sqrt n/4$. Then, except with probability at most
\[
2e^{-s^2/2}=2e^{-nt^2/32},
\]
we have
\[
\sigma_{\min}(F)\ge\sqrt n-\sqrt{n-\ell}-s=\sqrt n-\sqrt{n-\ell}-\frac{t\sqrt n}{4}.
\]
Using
\[
\sqrt n-\sqrt{n-\ell}=\frac{\ell}{\sqrt n+\sqrt{n-\ell}}\ge\frac{\ell}{2\sqrt n}\ge\frac32t\sqrt n,
\]
we obtain
\[
\sigma_{\min}(F)\ge\frac54t\sqrt n>t\sqrt d,
\]
where the last inequality holds for all sufficiently large $d$ because $n/d\to1$.

Writing $B=[F\ \ R_F]$, we have $BB^\top=FF^\top+R_FR_F^\top\succeq FF^\top$. Therefore the largest $n-\ell$ eigenvalues of $BB^\top$ are at least the $n-\ell$ positive eigenvalues of $FF^\top$, and hence
\[
\sigma_j(B)>t\sqrt d,\qquad 1\le j\le n-\ell.
\]
Together with $\sigma_j(H_d)\ge\sigma_j(B)$, this gives
\[
\frac1n\#\{j:\sigma_j(A_d)\le t\}\le\frac{\ell}{n}\le3t+\frac1n\le4t.
\]

Now fix $u\in[2/n,T]$ and take the dyadic grid $t_j=2^{-j}T$ down to the first grid point lying in $[u/2,u]$. Every such grid point is at least $u/2\ge1/n$, and the number of grid points is at most $C(1+\log(T/u))$. Applying the preceding fixed-threshold estimate at every grid point and taking a union bound, the probability that any of these estimates fails is at most
\[
C\bigl(1+\log(T/u)\bigr)e^{-cdu^2},
\]
because the smallest grid point is at least $u/2$ and $n/d\to1$. On the complementary event, for every $t\in[u,T]$ choose a dyadic grid point $t_j\in[t,2t]$; monotonicity of the counting function gives
\[
\frac1n\#\{\ell:\sigma_\ell(A_d)\le t\}\le\frac1n\#\{\ell:\sigma_\ell(A_d)\le t_j\}\le4t_j\le8t.
\]
Finally, $A_d$ has rank $n$ almost surely because the rank-deficient matrices form a Lebesgue-null algebraic set.
\end{proof}

\begin{proposition}[All-aspect spectral input]
\label{prop:app-det-spectral-input}
There are deterministic \(\eta_d\downarrow0\), constants \(c,C>0\), and events \(\mathcal E_d^{\mathrm{sp}}\) such that
\begin{equation}
\mathbb P\!\left((\mathcal E_d^{\mathrm{sp}})^c\right)\le C\exp\!\left[-c\sqrt{d\log(2N(d)(d+1))}\right]
\label{eq:app-det-spectral-event-probability}
\end{equation}
and the following statements hold on \(\mathcal E_d^{\mathrm{sp}}\):
\begin{enumerate}[label=(S\arabic*)]
\item \label{item:app-det-S1} \(C_m\) has rank \(r_d\), \(PQ_m=Q_m\), \(\|Q_m\|_{\mathrm{op}}=1\), and
\begin{equation}
\left|\frac{\|C_m\|_*}{k}-s_\alpha\right|\le\eta_d,\qquad \inf_{i\ne j}\|q_i-q_j\|_2\ge c_\alpha>0.
\label{eq:app-det-spectral-S1}
\end{equation}
\item \label{item:app-det-S2} Uniformly over \(i\ne j\),
\begin{equation}
\max_i\|m_i-\mu_i\|_2+\|\bar m\|_2\le\eta_d,\qquad \bar m:=\frac1k\sum_i m_i,
\label{eq:app-det-spectral-mean-inputs}
\end{equation}
\begin{equation}
\max_i|S_m[i,i]-s_\alpha|+\max_{i\ne j}|S_m[i,j]|\le\eta_d,\qquad \max_{i\ne j,s}|(q_i-q_j)^\top e_{i,s}|\le\eta_d.
\label{eq:app-det-spectral-entry-inputs}
\end{equation}
\item \label{item:app-det-S3} The reduced polar margin
\begin{equation}
\rho_X:=\min_{i\ne j,s}(q_i-q_j)^\top x_{i,s}
\label{eq:app-det-polar-margin-definition}
\end{equation}
satisfies
\begin{equation}
|\rho_X-s_\alpha|\le C\eta_d,\qquad \rho_X>0.
\label{eq:app-det-polar-feasible-margin}
\end{equation}
\item \label{item:app-det-S4} The reduced polar geometry satisfies
\begin{equation}
\left|D_{\mathrm{avg\text{-}bdry}}(Q_m)-\delta_{\mathrm{sp}}(\alpha)\right|\le\eta_d.
\label{eq:app-det-polar-uniform-gap}
\end{equation}
\item \label{item:app-det-S5} If \(\alpha\ne1\), all positive singular values of \(C_m\) lie in a deterministic compact interval bounded away from zero:
\begin{equation}
c_\alpha\le\sigma_{r_d}(C_m)\le\sigma_1(C_m)\le C_\alpha.
\label{eq:app-det-C0-extreme-singular-values}
\end{equation}
If \(\alpha=1\), there are \(t_0,C_0>0\) such that
\begin{equation}
\frac1{r_d}\#\{\ell:\sigma_\ell(C_m)\le t\}\le C_0(t+\eta_d),\qquad \eta_d\le t\le t_0.
\label{eq:app-det-critical-hard-count}
\end{equation}
\end{enumerate}
\end{proposition}

\begin{proof}
Because \(m_i=\mu_i+r\bar z_i\) is Gaussian with covariance \((1+r^2/n_i)I_d/d\), the rows are independent and admit the exact representation
\begin{equation}
M_m=D_dG/\sqrt d,\qquad D_d:=\operatorname{diag}\!\left(\sqrt{1+r^2/n_1},\ldots,\sqrt{1+r^2/n_k}\right),
\label{eq:app-det-gaussian-mean-matrix}
\end{equation}
where \(G\) has independent standard Gaussian entries.
Since \(\|D_d-I_k\|_{\mathrm{op}}\le Cr_0^2/n_{\min}=o(1)\) and \(\|G/\sqrt d\|_{\mathrm{op}}=O(1)\) with exponentially high probability, the comparator \(C_0:=PG/\sqrt d\) obeys
\begin{equation}
\|C_m-C_0\|_{\mathrm{op}}\le\zeta_d,\qquad \zeta_d\downarrow0.
\label{eq:app-det-Weyl-transfer}
\end{equation}
After a fixed orthogonal change of basis on \(\mathbb R^k\), the positive singular values of \(C_0\) have the same law as those of a \((k-1)\times d\) standard Gaussian matrix divided by \(\sqrt d\).
Marchenko--Pastur convergence therefore gives \(\|C_0\|_*/k\to s_\alpha\).
The normalized nuclear norm is \(O(d^{-1})\)-Lipschitz in the Gaussian entries, so Gaussian concentration and the deterministic expectation bias place this convergence under a deterministic envelope with failure probability bounded as in \eqref{eq:app-det-spectral-event-probability}.
The rank assertion is almost sure, and Weyl's inequality yields \eqref{eq:app-det-C0-extreme-singular-values} when \(\alpha\ne1\).

At $\alpha=1$, put $L_d:=\log(2N(d)(d+1))$ and $u_d:=(L_d/d)^{1/4}$. Then $u_d\to0$, $u_d\ge2/r_d$ eventually, and Lemma~\ref{lem:app-det-gaussian-critical-count} gives failure at most $C\exp(-c\sqrt{dL_d})$ after absorbing its logarithmic prefactor. Weyl's inequality and \eqref{eq:app-det-Weyl-transfer} transfer the count from $C_0$ to $C_m$ by replacing $t$ with $t+\zeta_d$. Enlarge the common envelope so that $\eta_d\ge\max\{u_d,\zeta_d\}$ and take $t_0<T/2$; then \eqref{eq:app-det-critical-hard-count} follows with precisely the probability scale in \eqref{eq:app-det-spectral-event-probability}.

Let \(S_0:=(C_0C_0^\top)^{1/2}\).
Conditionally on its singular values, the left singular frame of \(C_0\) is Haar in \(\one^\perp\) and independent of those singular values.
Haar concentration for the quadratic forms defining \(S_0[i,j]\), followed by a union bound over \(k^2\) entries, gives
\begin{equation}
\max_i\left|S_0[i,i]-\frac{\|C_0\|_*}{k}\right|+\max_{i\ne j}|S_0[i,j]|\le\chi_d,\qquad \chi_d\downarrow0.
\label{eq:app-det-S0-entry-concentration}
\end{equation}
The gap-free square-root inequality for positive semidefinite matrices gives
\begin{equation}
\|S_m-S_0\|_{\mathrm{op}}\le\|C_mC_m^\top-C_0C_0^\top\|_{\mathrm{op}}^{1/2}\le C\zeta_d^{1/2}.
\label{eq:app-det-gap-free-square-root}
\end{equation}
This proves the entrywise bounds in (S2) without dividing by the smallest singular value.
The empirical-mean bound is part of Proposition~\ref{prop:app-det-common-event}, and the Gaussian average \(k^{-1}\sum_i m_i\) gives \(\|\bar m\|_2=o(1)\) under the same probability envelope.

Conditionally on the empirical means, \(Q_m\) is deterministic and the residuals are centered Gaussian.
Since \(\|q_i-q_j\|_2\le\sqrt2\), Gaussian tails and a union bound over at most \(N^2\) indices give the residual bound in (S2).
Using that row \(i\) of \(C_m\) is \((m_i-\bar m)^\top\), we obtain
\begin{equation}
(q_i-q_j)^\top m_i=S_m[i,i]-S_m[j,i]+(q_i-q_j)^\top\bar m=s_\alpha+O(\eta_d).
\label{eq:app-det-polar-class-margin}
\end{equation}
Adding the residual bound proves (S3).
The same identity and \(\max_{i,s}\|x_{i,s}\|_2=O(1)\) give the uniform lower bound on \(\|q_i-q_j\|_2\).

It remains to identify the angular constant.
The centering identities \(\one^\top Q_m={\mathbf 0}\) and \(\|Q_m\|_{\mathrm F}^2=r_d\) give
\begin{equation}
\frac1{k(k-1)}\sum_{i\ne j}\|q_i-q_j\|_2^2=\frac{2r_d}{k-1}\longrightarrow\frac{2\rho_\alpha}{\alpha}.
\label{eq:app-det-polar-pair-average}
\end{equation}

If \(\alpha<1\), then \(r_d=k-1\) and \(Q_mQ_m^\top=P\), so every pair norm is exactly \(\sqrt2\). For \(\alpha>1\), let \(Q_0\) be the reduced polar factor of \(C_0\) and put \(\Pi_m:=Q_mQ_m^\top\) and \(\Pi_0:=Q_0Q_0^\top\). The left singular frame of \(C_0\), but not in general that of \(C_m=PD_dG/\sqrt d\), is Haar in \(\one^\perp\) when the class counts differ. The positive singular-value lower bounds and \eqref{eq:app-det-Weyl-transfer} give
\[
\|(I-\Pi_0)\Pi_m\|_{\mathrm{op}}\le\|C_m-C_0\|_{\mathrm{op}}\|C_m^\dagger\|_{\mathrm{op}}\le C_\alpha\zeta_d,
\]
and the analogous estimate with \(0\) and \(m\) exchanged. Consequently, \(\|\Pi_m-\Pi_0\|_{\mathrm{op}}\le C_\alpha\zeta_d\). Haar-projector concentration applies to \(\Pi_0\) and yields \((e_i-e_j)^\top\Pi_0(e_i-e_j)=2/\alpha+o(1)\) uniformly over pairs on the same probability scale. Since each such quadratic form changes by at most \(2\|\Pi_m-\Pi_0\|_{\mathrm{op}}\), we obtain \(\|q_i-q_j\|_2=\sqrt{2/\alpha}+o(1)\) uniformly. At \(\alpha=1\), no perturbation of the polar factor at the hard edge is needed for this step. The deterministic bounds \(\|q_i-q_j\|_2\le\sqrt2\) and \eqref{eq:app-det-polar-pair-average} give
\[
\frac1{k(k-1)}\sum_{i\ne j}\left|\|q_i-q_j\|_2-\sqrt2\right|
\le\frac1{\sqrt2}\left(2-\frac{2r_d}{k-1}\right)\longrightarrow0.
\]
Thus the required pair-norm convergence holds in averaged \(L^1\) at \(\alpha=1\).

Furthermore,
\begin{equation}
(q_i-q_j)^\top(\mu_i-\mu_j)=S_m[i,i]+S_m[j,j]-S_m[i,j]-S_m[j,i]+o(1)=2s_\alpha+o(1),
\label{eq:app-det-polar-cross-inner}
\end{equation}
while \(\|\mu_i-\mu_j\|_2=\sqrt2+o(1)\) uniformly.
Consequently, the pairwise cosine converges in averaged \(L^1\) to
\begin{equation}
\frac{2s_\alpha}{\sqrt{2\rho_\alpha/\alpha}\sqrt2}=\frac{m_{1/2}(\alpha)}{\sqrt{\rho_\alpha}}=\kappa_{\mathrm{sp}}(\alpha).
\label{eq:app-det-polar-cosine}
\end{equation}
Uniform continuity of \(x\mapsto\sqrt{2-2x}\) on \([-1,1]\) proves (S4).
Collecting the deterministic envelopes proves the proposition.
\end{proof}

\begin{lemma}[Weighted trace gap and hard-edge truncation]
\label{lem:app-det-trace-truncation}
Let \(C=U\Sigma V^\top\) have rank \(r\), let \(Q=UV^\top\), and let \(B\) satisfy \(\|B\|_{\mathrm{op}}\le1\) and \(\operatorname{rank}(B)\le r\).
Put \(g_\ell:=u_\ell^\top Bv_\ell\) and \(a_\ell:=1-g_\ell\in[0,2]\).
Then
\begin{align}
\|B-Q\|_{\mathrm F}^2&\le2\sum_{\ell=1}^r a_\ell,
\label{eq:app-det-Frob-vs-trace}\\
\frac1r\sum_{\ell=1}^r a_\ell&\le2\frac{\#\{\ell:\sigma_\ell\le\tau\}}r+\frac1{\tau r}\sum_{\ell=1}^r\sigma_\ell a_\ell,\qquad \tau>0.
\label{eq:app-det-truncation}
\end{align}
\end{lemma}

\begin{proof}
The rank and operator-norm bounds give \(\|B\|_{\mathrm F}^2\le r\), so expanding \(\|B-Q\|_{\mathrm F}^2\) proves the first inequality.
For the second inequality, use \(a_\ell\le2\) when \(\sigma_\ell\le\tau\) and \(a_\ell\le\sigma_\ell a_\ell/\tau\) otherwise.
\end{proof}

\begin{lemma}[Average normalization transfer]
\label{lem:app-det-average-normalization}
Let \(B,Q\in\mathbb R^{k\times d}\) have rows \(b_i,q_i\), suppose \(b_i-b_j\ne{\mathbf 0}\) for \(i\ne j\), and assume \(\inf_{i\ne j}\|q_i-q_j\|_2\ge c_0>0\).
If
\begin{equation}
E_2:=\frac1{k(k-1)}\sum_{i\ne j}\|(b_i-b_j)-(q_i-q_j)\|_2^2,
\label{eq:app-det-average-normalization-error}
\end{equation}
then
\begin{equation}
\frac1{k(k-1)}\sum_{i\ne j}\left\|\frac{b_i-b_j}{\|b_i-b_j\|_2}-\frac{q_i-q_j}{\|q_i-q_j\|_2}\right\|_2\le\frac4{c_0}E_2^{1/2}+\frac8{c_0^2}E_2.
\label{eq:app-det-average-normalization-bound}
\end{equation}
\end{lemma}

\begin{proof}
For a pairwise difference error \(e_{ij}\le c_0/2\), the normalization map gives an error at most \(4e_{ij}/c_0\).
For the remaining pairs, use the trivial bound two and \(\mathbf1_{\{e_{ij}>c_0/2\}}\le4e_{ij}^2/c_0^2\), then average and apply Cauchy--Schwarz.
\end{proof}

\begin{proposition}[Exact and near spectral minimizers]
\label{prop:app-det-spectral-rigidity}
On the events of Proposition~\ref{prop:app-det-spectral-input}, let \(W\in\mathcal F_X\) satisfy
\begin{equation}
\|W\|_{\mathrm{op}}\le(1+\xi)r_{\mathrm{sp}}^*(X),\qquad 0\le\xi\le1.
\label{eq:app-det-near-spectral-assumption}
\end{equation}
Then
\begin{equation}
D_{\mathrm{avg\text{-}bdry}}(W)\ge\delta_{\mathrm{sp}}(\alpha)-\begin{cases}C_\alpha(\eta_d+\xi)^{1/2},&\alpha\ne1,\\ C(\eta_d+\xi)^{1/4},&\alpha=1.\end{cases}
\label{eq:app-det-near-spectral-bound}
\end{equation}
In particular, every exact minimizer satisfies
\begin{equation}
\frac1{r_d}\left\|\frac{PW}{\|PW\|_{\mathrm{op}}}-Q_m\right\|_{\mathrm F}^2\le\begin{cases}C_\alpha\eta_d,&\alpha\ne1,\\ C\sqrt{\eta_d},&\alpha=1,\end{cases}
\label{eq:app-det-spectral-frobenius-rigidity}
\end{equation}
uniformly over \(W\in\mathcal M_{\mathrm{sp}}^X\).
\end{proposition}

\begin{proof}
The polar witness \(Q_m/\rho_X\) is feasible, so
\begin{equation}
r_{\mathrm{sp}}^*(X)\le\rho_X^{-1}.
\label{eq:app-det-spectral-optimum-upper}
\end{equation}
For every feasible \(W\), summing the class-averaged constraints over ordered pairs gives
\begin{equation}
\langle PW,C_m\rangle\ge k-1.
\label{eq:app-det-spectral-dual-lower}
\end{equation}
Set \(A:=PW\), \(a:=\|A\|_{\mathrm{op}}>0\), and \(B:=A/a\).
Centering is an operator-norm contraction and preserves all row differences, while \(\operatorname{rank}(B)\le r_d\).
Equations \eqref{eq:app-det-spectral-optimum-upper}, \eqref{eq:app-det-spectral-dual-lower}, and Proposition~\ref{prop:app-det-spectral-input} give
\begin{equation}
\frac1{r_d}\sum_{\ell=1}^{r_d}\sigma_\ell(C_m)\left(1-u_\ell^\top Bv_\ell\right)\le C(\eta_d+\xi).
\label{eq:app-det-spectral-polar-inner-product}
\end{equation}
If \(\alpha\ne1\), the positive edge in \eqref{eq:app-det-C0-extreme-singular-values} and Lemma~\ref{lem:app-det-trace-truncation} yield \(r_d^{-1}\|B-Q_m\|_{\mathrm F}^2\le C_\alpha(\eta_d+\xi)\).
If \(\alpha=1\), set \(q:=\eta_d+\xi\) and use \(\tau=\sqrt q\) in \eqref{eq:app-det-truncation} together with \eqref{eq:app-det-critical-hard-count}; the result is
\begin{equation}
\frac1{r_d}\|B-Q_m\|_{\mathrm F}^2\le Cq^{1/2}.
\label{eq:app-det-critical-spectral-frobenius}
\end{equation}
The case of non-small \(q\) follows by increasing \(C\) and using the trivial bound.
Taking \(\xi=0\) proves \eqref{eq:app-det-spectral-frobenius-rigidity}.

Because both \(B\) and \(Q_m\) are centered,
\begin{equation}
\frac1{k(k-1)}\sum_{i\ne j}\|(b_i-b_j)-(q_i-q_j)\|_2^2=\frac{2}{k-1}\|B-Q_m\|_{\mathrm F}^2.
\label{eq:app-det-pairwise-Frobenius}
\end{equation}
Feasibility ensures \(b_i-b_j\ne{\mathbf 0}\), and Proposition~\ref{prop:app-det-spectral-input} supplies a uniform lower bound on the polar pair norms.
Lemma~\ref{lem:app-det-average-normalization}, followed by the reverse triangle inequality and \eqref{eq:app-det-polar-uniform-gap}, proves \eqref{eq:app-det-near-spectral-bound}.
\end{proof}

\begin{proposition}[Spectral-norm boundary separation]
\label{prop:app-det-spectral-boundary}
Under Assumptions~\ref{ass:gaussian-cloud} and~\ref{ass:asymptotic-regime},
\begin{equation}
\inf_{W\in\mathcal M_{\mathrm{sp}}^X}D_{\mathrm{avg\text{-}bdry}}(W)\ge\delta_{\mathrm{sp}}(\alpha)-o_{\mathbb P}(1).
\label{eq:app-det-spectral-boundary-separation}
\end{equation}
The bound holds on the events in Proposition~\ref{prop:app-det-spectral-input} with the deterministic error in \eqref{eq:app-det-near-spectral-bound} at \(\xi=0\).
\end{proposition}

\begin{proof}
Apply Proposition~\ref{prop:app-det-spectral-rigidity} with \(\xi=0\) and take the infimum over the exact minimizer set.
\end{proof}

\paragraph{Proof of the deterministic geometry theorem}
\label{app:det-main-proof}

\begin{proof}[Proof of Theorem~\ref{thm:deterministic-separation-clean}]
Let \(\mathcal G_d^{\mathrm{det}}\) be the intersection of \(\mathcal E_d^{\mathrm{det}}\) from Proposition~\ref{prop:app-det-common-event} and \(\mathcal E_d^{\mathrm{sp}}\) from Proposition~\ref{prop:app-det-spectral-input}.
Both complements satisfy a bound of the form
\begin{equation}
C\exp\!\left[-c\sqrt{d\log(2N(d)(d+1))}\right].
\label{eq:app-det-intersection-tail}
\end{equation}
Since \(\log(2N(d)(d+1))\to\infty\), the quantity in \eqref{eq:app-det-intersection-tail} is at most \(e^{-\sqrt d}\) for every sufficiently large \(d\), after increasing the deterministic threshold.
The sign construction in Proposition~\ref{prop:app-det-common-event} gives \(\mathcal F_X\ne\varnothing\) on this event, and compactness gives nonempty minimum-norm sets.

The proof of Proposition~\ref{prop:app-det-row-boundary} is deterministic on \(\mathcal E_d^{\mathrm{det}}\) and gives the maximum-pair statement \eqref{eq:det-row-uniform-clean} under a deterministic vanishing envelope.
Lemma~\ref{lem:app-det-ray-reduction} then gives \eqref{eq:det-row-clean}.
Proposition~\ref{prop:app-det-max-boundary} and Lemma~\ref{lem:app-det-ray-reduction} give \eqref{eq:det-max-clean}.
Proposition~\ref{prop:app-det-spectral-boundary} gives \eqref{eq:det-sp-clean} for every fixed \(\alpha>0\), using the reduced polar factor when \(\alpha<1\) and the hard-edge truncation when \(\alpha=1\).
Choose one deterministic \(\omega_d\downarrow0\) dominating all three errors.
Because \(\delta_{\mathrm{sp}}(\alpha)>0\) and \(\delta_{\max}>0\), \eqref{eq:det-set-order-clean} follows for every sufficiently large \(d\).
\end{proof}

\subsection{From Direction-set Geometry to Population Accuracy}
\label{app:det-population-transfer}

We now convert the geometric separation above into ordinary multiclass population accuracy. The lemmas in this subsection are stated for general classifier families so that the same transfer can be reused in the stochastic and AdamProxy analyses.

For every classifier with nonzero pairwise row differences, define
\begin{equation}
\mathcal E_{\pi,d}(W):=1-\operatorname{Acc}_{\pi,d}(W),\qquad p_{i\to j}(W):=\mathbb P\!\left(w_j^\top X^{\mathrm{te}}\ge w_i^\top X^{\mathrm{te}}\mid Y=i\right).
\label{eq:app-acc-error-directed}
\end{equation}
Under the normalized convention in \eqref{eq:app-normalized-data}, set \(\sigma_d:=r_{\mathrm{te}}/\sqrt d\), and for \(i\ne j\) write
\begin{equation}
d_{ij}(W):=\|u_{ij}^W-u_{ij}^{\mu}\|_2,\qquad s_{ij}:=\|\mu_i-\mu_j\|_2,\qquad m_{ij}:=\frac{\mu_i+\mu_j}{2}.
\label{eq:app-acc-pair-geometry}
\end{equation}

\begin{lemma}[Pairwise Gaussian error bounds]
\label{lem:app-acc-pairwise-gaussian}
For every \(i<j\), let \(e_{ij}(W):=\{p_{i\to j}(W)+p_{j\to i}(W)\}/2\). Then
\begin{equation}
\max_{i<j}\{\pi_{i,d}p_{i\to j}(W)+\pi_{j,d}p_{j\to i}(W)\}\le\mathcal E_{\pi,d}(W)\le\sum_{i<j}\{\pi_{i,d}p_{i\to j}(W)+\pi_{j,d}p_{j\to i}(W)\},
\label{eq:app-acc-weighted-sandwich}
\end{equation}
and hence \(\mathcal E_{\pi,d}(W)\ge2\pi_{\min,d}\max_{i<j}e_{ij}(W)\). Moreover, with \(c_{ij}(W):=1-d_{ij}(W)^2/2\),
\begin{equation}
e_{ij}(W)=\frac12\left[\overline\Phi\!\left(\frac{s_{ij}c_{ij}(W)/2+\langle u_{ij}^W,m_{ij}\rangle}{\sigma_d}\right)+\overline\Phi\!\left(\frac{s_{ij}c_{ij}(W)/2-\langle u_{ij}^W,m_{ij}\rangle}{\sigma_d}\right)\right]\ge\overline\Phi\!\left(\frac{s_{ij}(c_{ij}(W))_+}{2\sigma_d}\right),
\label{eq:app-acc-pairwise-midpoint-free}
\end{equation}
where \(\overline\Phi(a)=\mathbb P(Z\ge a)\) for \(Z\sim\mathcal N(0,1)\).
\end{lemma}

\begin{proof}
Conditionally on \(Y=i\), the multiclass error event is the union of the events \(\{w_j^\top X^{\mathrm{te}}\ge w_i^\top X^{\mathrm{te}}\}\), up to Gaussian-null ties. The union bound gives the upper half of \eqref{eq:app-acc-weighted-sandwich}, while any fixed unordered pair gives the lower half. For \(u=u_{ij}^W\), \(v=u_{ij}^{\mu}\), \(m=m_{ij}\), and \(s=s_{ij}\), the identities \(\mu_i=m+(s/2)v\) and \(\mu_j=m-(s/2)v\) give
\[
p_{i\to j}(W)=\overline\Phi\!\left(\frac{u^\top\mu_i}{\sigma_d}\right),\qquad p_{j\to i}(W)=\overline\Phi\!\left(\frac{-u^\top\mu_j}{\sigma_d}\right),
\]
which proves the equality in \eqref{eq:app-acc-pairwise-midpoint-free}. For \(q\ge0\), the even function \(z\mapsto[\overline\Phi(q+z)+\overline\Phi(q-z)]/2\) is minimized at zero, while for \(q<0\) the same average is at least \(1/2\). Since \(u^\top v=1-d_{ij}(W)^2/2\), the final inequality follows.
\end{proof}

Define
\begin{equation}
\tau_d:=\max_i\left|\|\mu_i\|_2^2-1\right|,\qquad \rho_d:=\max_{i\ne j}|\mu_i^\top\mu_j|.
\label{eq:app-acc-radial-coherence}
\end{equation}

\begin{lemma}[Quantitative Gaussian radial event]
\label{lem:app-acc-gaussian-radial-event}
There are deterministic \(\eta_d\downarrow0\) and events \(\mathcal R_d\), depending only on the class means, such that, for every sufficiently large \(d\),
\begin{equation}
\mathbb P(\mathcal R_d^c)\le\frac14e^{-\sqrt d},\qquad \tau_d+\rho_d\le\eta_d\quad\text{on }\mathcal R_d.
\label{eq:app-acc-radial-event-probability}
\end{equation}
\end{lemma}

\begin{proof}
Under the normalized convention, \(d\|\mu_i\|_2^2\) is chi-square with \(d\) degrees of freedom, while conditionally on \(\mu_i\), the inner product \(\mu_i^\top\mu_j\) is Gaussian with variance \(\|\mu_i\|_2^2/d\). Taking a sufficiently large constant \(C_\alpha\) and \(\eta_d=C_\alpha d^{-1/4}\), standard chi-square and Gaussian tails followed by a union bound over \(O(d^2)\) indices give \eqref{eq:app-acc-radial-event-probability}.
\end{proof}
For nonempty classifier families \(\mathcal R_d^{\mathrm{clf}}\) and \(\mathcal C_d^{\mathrm{clf}}\), put
\begin{equation}
h_d:=\sup_{W\in\mathcal R_d^{\mathrm{clf}}}\max_{i\ne j}d_{ij}(W),\qquad \Delta_d:=\inf_{W\in\mathcal C_d^{\mathrm{clf}}}D_{\mathrm{avg\text{-}bdry}}(W),
\label{eq:app-acc-family-geometry}
\end{equation}
and, whenever \(\tau_d+\rho_d<1\), define
\begin{align}
A_d&:=\frac{1-\tau_d-\rho_d}{\sqrt{2(1+\tau_d+\rho_d)}}-h_d\sqrt{1+\tau_d},\label{eq:app-acc-row-amplitude}\\
B_d&:=\sqrt{\frac{1+\tau_d+\rho_d}{2}}\left(1-\frac{\Delta_d^2}{2}\right)_+.
\label{eq:app-acc-competitor-amplitude}
\end{align}

\begin{proposition}[Finite-dimensional geometry-to-error transfer]
\label{prop:app-acc-finite-bounds}
If \(A_d>0\), then
\begin{equation}
\sup_{W\in\mathcal R_d^{\mathrm{clf}}}\mathcal E_{\pi,d}(W)\le(k-1)\overline\Phi\!\left(\frac{A_d}{\sigma_d}\right),\qquad \inf_{W\in\mathcal C_d^{\mathrm{clf}}}\mathcal E_{\pi,d}(W)\ge2\pi_{\min,d}\overline\Phi\!\left(\frac{B_d}{\sigma_d}\right).
\label{eq:app-acc-finite-bounds}
\end{equation}
\end{proposition}

\begin{proof}
For \(v=u_{ij}^{\mu}\), radial control gives
\[
v^\top\mu_i\ge\frac{1-\tau_d-\rho_d}{\sqrt{2(1+\tau_d+\rho_d)}},\qquad -v^\top\mu_j\ge\frac{1-\tau_d-\rho_d}{\sqrt{2(1+\tau_d+\rho_d)}}.
\]
If \(W\in\mathcal R_d^{\mathrm{clf}}\), then \(\|u_{ij}^W-v\|_2\le h_d\), so both directed population margins are at least \(A_d\). The upper bound follows from the union bound in Lemma~\ref{lem:app-acc-pairwise-gaussian}. If \(W\in\mathcal C_d^{\mathrm{clf}}\), some pair has \(d_{ij}(W)\ge\Delta_d\), while \(s_{ij}/2\le\sqrt{(1+\tau_d+\rho_d)/2}\). The midpoint-free lower bound in Lemma~\ref{lem:app-acc-pairwise-gaussian} then gives the second inequality.
\end{proof}

Define
\begin{equation}
\mathsf a(h):=\frac1{\sqrt2}-h,\qquad \mathsf b(\delta):=\frac1{\sqrt2}\left(1-\frac{\delta^2}{2}\right)_+.
\label{eq:app-acc-limiting-amplitudes}
\end{equation}

\begin{theorem}[Geometry-to-accuracy transfer]
\label{thm:app-acc-transfer}
Suppose \(\tau_d+\rho_d\to0\), \(\limsup_d h_d\le\overline h\), and \(\liminf_d\Delta_d\ge\underline\delta>0\). If \(\mathsf a(\overline h)>\mathsf b(\underline\delta)\), then under Assumption~\ref{ass:asymptotic-regime},
\begin{equation}
\frac{\sup_{W\in\mathcal R_d^{\mathrm{clf}}}\mathcal E_{\pi,d}(W)}{\inf_{W\in\mathcal C_d^{\mathrm{clf}}}\mathcal E_{\pi,d}(W)}\longrightarrow0,
\label{eq:app-acc-transfer-ratio}
\end{equation}
and therefore, for every sufficiently large \(d\),
\begin{equation}
\inf_{W\in\mathcal R_d^{\mathrm{clf}}}\operatorname{Acc}_{\pi,d}(W)>\sup_{W\in\mathcal C_d^{\mathrm{clf}}}\operatorname{Acc}_{\pi,d}(W).
\label{eq:app-acc-transfer-order}
\end{equation}
\end{theorem}

\begin{proof}
Choose constants \(0<b<a\) such that \(\mathsf b(\underline\delta)<b<a<\mathsf a(\overline h)\). Proposition~\ref{prop:app-acc-finite-bounds} gives, for all sufficiently large \(d\),
\[
\frac{\sup_{\mathcal R_d^{\mathrm{clf}}}\mathcal E_{\pi,d}}{\inf_{\mathcal C_d^{\mathrm{clf}}}\mathcal E_{\pi,d}}\le\frac{k-1}{2\pi_{\min,d}}\frac{\overline\Phi(a/\sigma_d)}{\overline\Phi(b/\sigma_d)}.
\]
The standard Mills bounds imply that the right-hand side is at most \(Ck\pi_{\min,d}^{-1}e^{-cd}\) for constants \(C,c>0\), because \(r_{\mathrm{te}}\le r_{\mathrm{te},0}\). Assumption~\ref{ass:asymptotic-regime} makes this quantity vanish.
\end{proof}

\begin{lemma}[Population-error continuity along a direction set]
\label{lem:app-acc-trajectory-continuity}
Fix \(d\), a linearly separable training set \(X\) containing at least one sample from every class in \([k]\), a matrix norm \(\psi\), and \(0\le\varepsilon<\gamma_\psi(X)\). If \(d_{\cos}(W_t,\mathcal U_{\psi,\varepsilon}^X)\to0\), then
\begin{equation}
\limsup_{t\to\infty}\mathcal E_{\pi,d}(W_t)\le\sup_{U\in\mathcal U_{\psi,\varepsilon}^X}\mathcal E_{\pi,d}(U),\qquad \liminf_{t\to\infty}\mathcal E_{\pi,d}(W_t)\ge\inf_{U\in\mathcal U_{\psi,\varepsilon}^X}\mathcal E_{\pi,d}(U).
\label{eq:app-acc-error-continuity}
\end{equation}
\end{lemma}

\begin{proof}
The norm-normalized section of \(\mathcal U_{\psi,\varepsilon}^X\) is compact, and its strictly positive empirical margin keeps every pairwise row difference uniformly away from zero. Population error is invariant under positive scaling and continuous on this section because the Gaussian test law gives zero mass to score-tie hyperplanes. Uniform continuity on a compact neighborhood proves the claim.
\end{proof}

\begin{lemma}[Boundary metric under direction-set convergence]
\label{lem:app-det-metric-transfer}
Under the fixed-data assumptions of Lemma~\ref{lem:app-acc-trajectory-continuity}, if \(d_{\cos}(W_t,\mathcal U_{\psi,\varepsilon}^X)\to0\) with \(0\le\varepsilon<\gamma_\psi(X)\), then
\begin{equation}
\limsup_{t\to\infty}D_{\mathrm{avg\text{-}bdry}}(W_t)\le\sup_{U\in\mathcal U_{\psi,\varepsilon}^X}D_{\mathrm{avg\text{-}bdry}}(U),\qquad \liminf_{t\to\infty}D_{\mathrm{avg\text{-}bdry}}(W_t)\ge\inf_{U\in\mathcal U_{\psi,\varepsilon}^X}D_{\mathrm{avg\text{-}bdry}}(U).
\label{eq:app-det-metric-transfer}
\end{equation}
\end{lemma}

\begin{proof}
The same compactness argument as in Lemma~\ref{lem:app-acc-trajectory-continuity} applies because \(D_{\mathrm{avg\text{-}bdry}}\) is scale invariant and continuous whenever all pairwise row differences are nonzero.
\end{proof}

\subsection{Proof of the Deterministic Optimizer Accuracy Ordering}
\label{app:det-optimizer-order}

\begin{proof}[Proof of Theorem~\ref{thm:det-optimizer-accuracy-order-clean}]
Work on the event \(\mathcal G_d^{\mathrm{det}}\) from Theorem~\ref{thm:deterministic-separation-clean}. Proposition~\ref{prop:app-det-common-event}, items~\ref{item:app-det-H1}--\ref{item:app-det-H2} imply \(\tau_d+\rho_d=o(1)\) on this event. Take \(\mathcal R_d^{\mathrm{clf}}=\mathcal U_{\mathrm{row}}^X\). Equation~\eqref{eq:det-row-uniform-clean} gives \(h_d=o(1)\). For \(\psi\in\{\mathrm{sp},\max\}\), take \(\mathcal C_d^{\mathrm{clf}}=\mathcal U_\psi^X\). Equations~\eqref{eq:det-sp-clean} and~\eqref{eq:det-max-clean} give \(\liminf_d\Delta_d\ge\delta_\psi>0\), where \(\delta_{\mathrm{sp}}=\delta_{\mathrm{sp}}(\alpha)\) and \(\delta_{\max}\) is defined in \eqref{eq:delta-max-clean}. Since
\begin{equation}
\mathsf a(0)=\frac1{\sqrt2}>\frac1{\sqrt2}\left(1-\frac{\delta_\psi^2}{2}\right)_+=\mathsf b(\delta_\psi),
\label{eq:app-acc-deterministic-amplitude-gap}
\end{equation}
Theorem~\ref{thm:app-acc-transfer} gives, for every sufficiently large \(d\) on \(\mathcal G_d^{\mathrm{det}}\),
\begin{equation}
\sup_{U\in\mathcal U_{\mathrm{row}}^X}\mathcal E_{\pi,d}(U)<\min_{\psi\in\{\mathrm{sp},\max\}}\inf_{V\in\mathcal U_\psi^X}\mathcal E_{\pi,d}(V).
\label{eq:app-acc-deterministic-set-error-order}
\end{equation}
The full-batch implicit-bias results in Appendix~\ref{app:max-sum-row}, together with Lemma~\ref{lem:exact-margin-direction-convergence}, give direction-set convergence to \(\mathcal U_{\mathrm{row}}^X\) for RowGD and RowGDM, to \(\mathcal U_{\mathrm{sp}}^X\) for Spectral-GD and exact-SVD Muon, and to \(\mathcal U_{\max}^X\) for SignGD, Signum, and full-batch Adam under the stated Adam assumptions. Lemma~\ref{lem:app-acc-trajectory-continuity} transfers \eqref{eq:app-acc-deterministic-set-error-order} to all seven trajectories. Since \(\mathbb P(\mathcal G_d^{\mathrm{det}})\ge1-e^{-\sqrt d}\), replacing error by one minus accuracy proves \eqref{eq:det-optimizer-accuracy-order-clean} with the claimed probability.
\end{proof}

\subsection{Proof of the Limiting Classifier Geometry}
\label{app:limiting-geometry-proof}

We prove Proposition~\ref{prop:limiting-geometry} under the normalized convention~\eqref{eq:app-normalized-data}, under which \(\mu_i,z_{i,s}\overset{\mathrm{iid}}{\sim}\mathcal N({\mathbf 0},I_d/d)\). Since all statements below are invariant under the common rescaling used in~\eqref{eq:app-normalized-data}, they transfer directly to the original normalization.

Let \(M_\mu\in\mathbb R^{k\times d}\) have rows \(\mu_i^\top\), let \(C_\mu:=PM_\mu\), and write a compact SVD \(C_\mu=U_\mu\Sigma_\mu V_\mu^\top\). Define \(Q_\mu:=U_\mu V_\mu^\top\), and write \(q_i^\top\) for its \(i\)-th row. Recall from~\eqref{eq:app-det-all-aspect-scales} that
\[
\kappa_{\mathrm{sp}}(\alpha)=\frac{m_{1/2}(\alpha)}{\sqrt{\alpha\wedge1}}\in(0,1).
\]
In particular, when \(\alpha>1\), \(\kappa_{\mathrm{sp}}(\alpha)=m_{1/2}(\alpha)\).

\begin{lemma}[Empirical-to-population polar transfer]
\label{lem:app-population-polar-transfer}
Under Assumptions~\ref{ass:gaussian-cloud} and~\ref{ass:asymptotic-regime},
\begin{equation}
\frac1{r_d}\|Q_m-Q_\mu\|_{\mathrm F}^2\xrightarrow{\mathbb P}0.
\label{eq:app-population-polar-transfer}
\end{equation}
\end{lemma}

\begin{proof}
By~\eqref{eq:app-det-mean-residual}, the \(i\)-th row of \(M_m-M_\mu\) is \(r\bar z_i^\top\), and the rows are independent centered Gaussian vectors with covariance \(r^2I_d/(n_i d)\). Hence \(M_m-M_\mu=D_d^{\mathrm e}H/\sqrt d\) for a standard Gaussian matrix \(H\) and a diagonal matrix satisfying \(\|D_d^{\mathrm e}\|_{\mathrm{op}}\le r_0/\sqrt{n_{\min}}\). Since \(k/d\to\alpha\) and \(n_{\min}\to\infty\),
\begin{equation}
\delta_d:=\|C_m-C_\mu\|_{\mathrm{op}}\le\|M_m-M_\mu\|_{\mathrm{op}}=o_{\mathbb P}(1).
\label{eq:app-population-polar-op-transfer}
\end{equation}

On the full reduced-rank event, \(\|Q_\mu\|_{\mathrm{op}}=1\), \(\operatorname{rank}(Q_\mu)=r_d\), and \(\langle Q_\mu,C_\mu\rangle=\|C_\mu\|_*\). Therefore
\begin{align}
0
&\le \|C_m\|_*-\langle Q_\mu,C_m\rangle \notag\\
&\le \bigl|\|C_m\|_*-\|C_\mu\|_*\bigr|+\bigl|\langle Q_\mu,C_m-C_\mu\rangle\bigr|
\le 2r_d\delta_d.
\label{eq:app-population-polar-deficit}
\end{align}
Write \(C_m=U_m\Sigma_mV_m^\top\), and set \(a_\ell:=1-u_{m,\ell}^\top Q_\mu v_{m,\ell}\). Then~\eqref{eq:app-population-polar-deficit} gives
\begin{equation}
\frac1{r_d}\sum_{\ell=1}^{r_d}\sigma_\ell(C_m)a_\ell\le2\delta_d.
\label{eq:app-population-polar-weighted-deficit}
\end{equation}

If \(\alpha\neq1\), Proposition~\ref{prop:app-det-spectral-input}, item~\ref{item:app-det-S5}, bounds every positive singular value of \(C_m\) away from zero, so Lemma~\ref{lem:app-det-trace-truncation} and~\eqref{eq:app-population-polar-weighted-deficit} yield \(r_d^{-1}\|Q_m-Q_\mu\|_{\mathrm F}^2=o_{\mathbb P}(1)\).

Suppose \(\alpha=1\). Choose a deterministic sequence \(\bar\delta_d\downarrow0\) such that \(\delta_d\le\bar\delta_d\) with probability tending to one, and set \(\tau_d:=\max\{\eta_d,\sqrt{\bar\delta_d}\}\), where \(\eta_d\) is the envelope in Proposition~\ref{prop:app-det-spectral-input}. For all sufficiently large \(d\), \(\tau_d\le t_0\), and item~\ref{item:app-det-S5} together with Lemma~\ref{lem:app-det-trace-truncation} gives
\[
\frac1{r_d}\|Q_m-Q_\mu\|_{\mathrm F}^2
\le 4C_0(\tau_d+\eta_d)+\frac{4\delta_d}{\tau_d}
=o_{\mathbb P}(1).
\]
This proves~\eqref{eq:app-population-polar-transfer}.
\end{proof}

\begin{proof}[Proof of Proposition~\ref{prop:limiting-geometry}]
We first establish the claimed geometry uniformly over the corresponding exact maximum-margin faces.

For the Row geometry, Proposition~\ref{prop:app-det-row-rigidity} and Proposition~\ref{prop:app-det-common-event}, item~\ref{item:app-det-H1}, give
\begin{equation}
\sup_{W\in\mathcal M_{\mathrm{row}}^X}\max_i\|w_i-\mu_i\|_2=o_{\mathbb P}(1).
\label{eq:app-limiting-row-population-rigidity}
\end{equation}
Since \(\min_i\|\mu_i\|_2=1-o_{\mathbb P}(1)\), Lemma~\ref{lem:app-det-normalization} yields
\begin{equation}
\sup_{W\in\mathcal M_{\mathrm{row}}^X}\max_i\|\widehat w_i-\widehat\mu_i\|_2=o_{\mathbb P}(1).
\label{eq:app-limiting-row-direction}
\end{equation}

For the Max geometry, put \(g:=\sqrt{2/\pi}\), let \(T_i:=\operatorname{sign}(\mu_i)\), and write \(t_{\max}:=r_{\max}^*(X)\). For \(W\in\mathcal M_{\max}^X\), set \(A:=W/t_{\max}\). The sharp radius estimate~\eqref{eq:app-det-sharp-radii} and the global sign rigidity~\eqref{eq:app-det-sharp-global-sign} imply
\begin{equation}
t_{\max}=\frac{1+o_{\mathbb P}(1)}{g\sqrt d},\qquad
\sup_{W\in\mathcal M_{\max}^X}\frac{\|A-T\|_{\mathrm F}^2}{kd}=o_{\mathbb P}(1),
\label{eq:app-limiting-max-global}
\end{equation}
where \(T\) has rows \(T_i^\top\). Let \(\bar A:=k^{-1}\sum_iA_i\) and \(\bar T:=k^{-1}\sum_iT_i\). Jensen's inequality,~\eqref{eq:app-limiting-max-global}, and \(\mathbb E\|\bar T\|_2^2=d/k=O(1)\) give \(\sup_W\|\bar A\|_2/\sqrt d=o_{\mathbb P}(1)\). The class-averaged feasibility constraint~\eqref{eq:app-det-row-centering-constraint} then gives, uniformly over \(i\) and \(W\in\mathcal M_{\max}^X\),
\[
A_i^\top\mu_i\ge g\sqrt d-o_{\mathbb P}(\sqrt d).
\]
Together with \(\max_i|\|\mu_i\|_1/\sqrt d-g|=o_{\mathbb P}(1)\) from~\eqref{eq:app-det-sharp-gaussian-inputs}, this implies \(D_i:=\|\mu_i\|_1-A_i^\top\mu_i=o_{\mathbb P}(\sqrt d)\) uniformly. Substituting this bound into~\eqref{eq:app-det-sharp-row-split} and using the uniform Gaussian small-ball estimate~\eqref{eq:app-det-sharp-small-ball} yields
\begin{equation}
\sup_{W\in\mathcal M_{\max}^X}\max_i\frac1d\|A_i-T_i\|_2^2=o_{\mathbb P}(1).
\label{eq:app-limiting-max-row-rigidity}
\end{equation}
Since \(\|T_i\|_2=\sqrt d\), normalization stability gives
\begin{equation}
\sup_{W\in\mathcal M_{\max}^X}\max_i\left\|\widehat w_i-\frac{\operatorname{sign}(\mu_i)}{\sqrt d}\right\|_2=o_{\mathbb P}(1).
\label{eq:app-limiting-max-direction}
\end{equation}

For the Spectral geometry, let \(B:=PW/\|PW\|_{\mathrm{op}}\). Combining~\eqref{eq:app-det-spectral-frobenius-rigidity} with Lemma~\ref{lem:app-population-polar-transfer} gives
\begin{equation}
\sup_{W\in\mathcal M_{\mathrm{sp}}^X}\frac1{r_d}\|B-Q_\mu\|_{\mathrm F}^2=o_{\mathbb P}(1).
\label{eq:app-limiting-spectral-population-rigidity}
\end{equation}
This proves the population-mean version of the spectral geometry claimed in the proposition.

It remains to identify the mean angles when \(\alpha>1\). For Row,~\eqref{eq:app-limiting-row-direction} immediately gives a limiting angle of zero. For Max,~\eqref{eq:app-limiting-max-direction}, \(\max_i|\|\mu_i\|_2-1|=o_{\mathbb P}(1)\), and~\eqref{eq:app-det-sharp-gaussian-inputs} give
\begin{equation}
\sup_{W\in\mathcal M_{\max}^X}\max_i\left|
\frac{w_i^\top\mu_i}{\|w_i\|_2\|\mu_i\|_2}
-\sqrt{\frac2\pi}
\right|=o_{\mathbb P}(1).
\label{eq:app-limiting-max-angle-cosine}
\end{equation}

For Spectral, \(\alpha>1\) implies \(r_d=d\). Since \(C_\mu\) is a centered Gaussian matrix, the left singular frame \(U_\mu\) is Haar in \(\mathbf1^\perp\) and independent of its singular values. Haar-projector concentration and the Marchenko--Pastur limit therefore give
\begin{equation}
\max_i\left|\|q_i\|_2^2-\frac1\alpha\right|=o_{\mathbb P}(1),\qquad
\max_i\left|q_i^\top\mu_i-s_\alpha\right|=o_{\mathbb P}(1).
\label{eq:app-limiting-population-polar-row}
\end{equation}
Indeed, if \(S_\mu:=Q_\mu C_\mu^\top=U_\mu\Sigma_\mu U_\mu^\top\), conditional Haar concentration gives \(\max_i|S_\mu[i,i]-\|C_\mu\|_*/k|=o_{\mathbb P}(1)\), while \(\|C_\mu\|_*/k\to s_\alpha\) and \(q_i^\top(\mu_i-\bar\mu)=S_\mu[i,i]\) with \(\|\bar\mu\|_2=o_{\mathbb P}(1)\). Since \(s_\alpha=m_{1/2}(\alpha)/\sqrt\alpha\) and \(\kappa_{\mathrm{sp}}(\alpha)=m_{1/2}(\alpha)\) for \(\alpha>1\),~\eqref{eq:app-limiting-population-polar-row} yields
\begin{equation}
\max_i\left|
\frac{q_i^\top\mu_i}{\|q_i\|_2\|\mu_i\|_2}
-\kappa_{\mathrm{sp}}(\alpha)
\right|=o_{\mathbb P}(1).
\label{eq:app-limiting-population-polar-cosine}
\end{equation}

We now transfer~\eqref{eq:app-limiting-population-polar-cosine} from \(Q_\mu\) to arbitrary exact Spectral minimizers. Fix \(W\in\mathcal M_{\mathrm{sp}}^X\), set \(a:=\|W\|_{\mathrm{op}}\), and let \(\bar w:=k^{-1}\sum_iw_i\) and \(h:=\bar w/a\). Centering preserves every margin and cannot increase the operator norm, so exact optimality gives \(\|PW\|_{\mathrm{op}}=\|W\|_{\mathrm{op}}=a\). Hence, with \(B:=PW/a\),
\begin{equation}
\frac Wa=B+\mathbf1h^\top,\qquad \mathbf1^\top B={\mathbf 0}^\top.
\label{eq:app-limiting-spectral-shift}
\end{equation}
Because \(\|Q_\mu\|_{\mathrm F}^2=d\),~\eqref{eq:app-limiting-spectral-population-rigidity} gives \(\|B\|_{\mathrm F}^2=d-o_{\mathbb P}(d)\) uniformly over \(\mathcal M_{\mathrm{sp}}^X\). The two terms in~\eqref{eq:app-limiting-spectral-shift} are Frobenius-orthogonal, while \(\|W/a\|_{\mathrm{op}}=1\) implies \(\|W/a\|_{\mathrm F}^2\le d\). Therefore
\[
d\ge\left\|\frac Wa\right\|_{\mathrm F}^2=\|B\|_{\mathrm F}^2+k\|h\|_2^2,
\]
and hence
\begin{equation}
\sup_{W\in\mathcal M_{\mathrm{sp}}^X}\|h\|_2=o_{\mathbb P}(1).
\label{eq:app-limiting-spectral-shift-small}
\end{equation}
Consequently,
\begin{equation}
\sup_{W\in\mathcal M_{\mathrm{sp}}^X}\frac1k\sum_{i=1}^k\left\|\frac{w_i}{a}-q_i\right\|_2^2
=
\sup_{W\in\mathcal M_{\mathrm{sp}}^X}\left\{\frac1k\|B-Q_\mu\|_{\mathrm F}^2+\|h\|_2^2\right\}
=o_{\mathbb P}(1).
\label{eq:app-limiting-spectral-row-average}
\end{equation}
Moreover, every spectral row is nonzero with probability tending to one: if \(w_i={\mathbf 0}\), then~\eqref{eq:app-det-row-centering-constraint} gives \(-ah^\top m_i\ge1-1/k\), contradicting \(a=O_{\mathbb P}(1)\),~\eqref{eq:app-limiting-spectral-shift-small}, and \(\max_i\|m_i\|_2=O_{\mathbb P}(1)\). Since~\eqref{eq:app-limiting-population-polar-row} bounds \(\min_i\|q_i\|_2\) away from zero,~\eqref{eq:app-limiting-spectral-row-average} and normalization stability imply that, for every fixed \(\varepsilon>0\),
\begin{equation}
\sup_{W\in\mathcal M_{\mathrm{sp}}^X}\frac1k\#\left\{i:
\left|
\arccos\!\left(\frac{w_i^\top\mu_i}{\|w_i\|_2\|\mu_i\|_2}\right)
-\arccos\kappa_{\mathrm{sp}}(\alpha)
\right|>\varepsilon
\right\}
\xrightarrow{\mathbb P}0.
\label{eq:app-limiting-spectral-angle-fraction}
\end{equation}
Because every angle lies in \([0,\pi]\),~\eqref{eq:app-limiting-spectral-angle-fraction} implies
\begin{equation}
\sup_{W\in\mathcal M_{\mathrm{sp}}^X}\left|
\frac1k\sum_{i=1}^k
\arccos\!\left(\frac{w_i^\top\mu_i}{\|w_i\|_2\|\mu_i\|_2}\right)
-\arccos\kappa_{\mathrm{sp}}(\alpha)
\right|
\xrightarrow{\mathbb P}0.
\label{eq:app-limiting-spectral-mean-angle}
\end{equation}

Finally,~\eqref{eq:minmax-equivalence-clean} identifies each maximum-margin direction set with the positive rays through the corresponding exact minimum-norm face. The full-batch implicit-bias results used in Theorem~\ref{thm:det-optimizer-accuracy-order-clean}, together with Lemma~\ref{lem:exact-margin-direction-convergence}, give direction-set convergence to \(\mathcal U_{\mathrm{row}}^X\) for RowGD and RowGDM, to \(\mathcal U_{\mathrm{sp}}^X\) for SpecGD and Muon, and to \(\mathcal U_{\max}^X\) for SignGD, Signum, and Adam. At each fixed \(d\), the normalized row directions, the centered spectral representative \(PW/\|PW\|_{\mathrm{op}}\), and the mean class--row angle are continuous on the corresponding compact direction sets on the high-probability events established above. Taking \(t\to\infty\) first and then \(d\to\infty\) therefore transfers~\eqref{eq:app-limiting-row-direction},~\eqref{eq:app-limiting-max-direction},~\eqref{eq:app-limiting-spectral-population-rigidity}, and the three angle limits to the optimizer trajectories, proving Proposition~\ref{prop:limiting-geometry}.
\end{proof}

\subsection{Quantitative Population-error Exponents for Deterministic Geometries}
\label{app:det-quantitative-errors}

Because all three accuracies may converge to one, a strict accuracy ordering can still correspond to an exponentially large separation in population error. To quantify this separation, define the directed-pair exponent and normalized log error by
\begin{equation}
\mathsf E_d(W):=\min_{i\ne j}\bigl[(u_{ij}^W)^\top\mu_i\bigr]_+^2,\qquad \mathsf F_d(W):=-\frac{2r_{\mathrm{te}}^2}{d}\log\mathcal E_{\pi,d}(W).
\label{eq:app-det-error-exponents}
\end{equation}
Both are invariant under positive rescaling and common row shifts.

\begin{lemma}[Directed-pair exponent controls multiclass error]
\label{lem:app-det-exponent-reduction}
For every classifier with nonzero pairwise row differences, set \(I_d(W):=d\mathsf E_d(W)/(2r_{\mathrm{te}}^2)\) and \(C_0:=3/2+\frac12\log(2\pi)\). Then
\begin{equation}
\max\{0,I_d(W)-\log(k-1)\}\le-\log\mathcal E_{\pi,d}(W)\le I_d(W)+\log(1/\pi_{\min,d})+\log(1+\sqrt{2I_d(W)})+C_0.
\label{eq:app-det-exponent-reduction}
\end{equation}
Consequently, uniformly over any classifier family on which \(\mathsf E_d(W)=O(1)\), one has \(\mathsf F_d(W)=\mathsf E_d(W)+o(1)\) under Assumption~\ref{ass:asymptotic-regime}.
\end{lemma}

\begin{proof}
For a directed pair, the exact error is \(\overline\Phi(\sqrt d\,(u_{ij}^W)^\top\mu_i/r_{\mathrm{te}})\). Let \(t=\sqrt{2I_d(W)}\). The Gaussian tail bounds
\[
\frac{e^{-C_0}}{1+t}e^{-t^2/2}\le\overline\Phi(t)\le\frac12e^{-t^2/2}
\]
show that one directed pair contributes at least \(\pi_{\min,d}e^{-C_0-I_d(W)}/(1+\sqrt{2I_d(W)})\) to the multiclass error, while a classwise union bound gives \(\mathcal E_{\pi,d}(W)\le(k-1)e^{-I_d(W)}\) when all directed population margins are nonnegative; if one is negative then \(I_d(W)=0\), and the same logarithmic lower bound is trivial. This proves \eqref{eq:app-det-exponent-reduction}. Multiplying by \(2r_{\mathrm{te}}^2/d\) and using \(k=O(d)\), \(\log(1/\pi_{\min,d})=o(d)\), and \(r_{\mathrm{te}}\le r_{\mathrm{te},0}\) proves the uniform \(o(1)\) reduction.
\end{proof}

Let \(g:=\sqrt{2/\pi}\). The following sharpening is needed only for the quantitative statement; the qualitative accuracy theorem above already follows from the angular gaps in Theorem~\ref{thm:deterministic-separation-clean}.

\begin{proposition}[Sharp directed signals on the Max and Spectral faces]
\label{prop:app-det-sharp-directed-signals}
Under Assumptions~\ref{ass:gaussian-cloud} and~\ref{ass:asymptotic-regime}, in probability,
\begin{equation}
r_{\max}^*(X)=\frac{1+o(1)}{g\sqrt d},\qquad r_{\mathrm{sp}}^*(X)=\frac{1+o(1)}{s_\alpha}.
\label{eq:app-det-sharp-radii}
\end{equation}
Moreover, uniformly over the complete exact faces,
\begin{align}
\sup_{W\in\mathcal M_{\max}^X}\max_{i\ne j}\left|(u_{ij}^W)^\top\mu_i-\frac1{\sqrt\pi}\right|&=o_{\mathbb P}(1),
\label{eq:app-det-max-sharp-signal}\\
\inf_{W\in\mathcal M_{\mathrm{sp}}^X}\min_{i\ne j}(u_{ij}^W)^\top\mu_i&\ge\frac{s_\alpha}{\sqrt2}-o_{\mathbb P}(1).
\label{eq:app-det-spectral-sharp-signal}
\end{align}
\end{proposition}

\begin{proof}
Let \(M_\mu\) have rows \(\mu_i^\top\), let \(C_\mu:=PM_\mu\), and let \(T:=\operatorname{sign}(M_\mu)\). Standard Gaussian concentration, uniformly over the \(k=\Theta(d)\) classes, gives
\begin{equation}
\frac{\|C_\mu\|_{\mathrm{sum}}}{k\sqrt d}=g+o_{\mathbb P}(1),\qquad \max_i\left|\frac{\|\mu_i\|_1}{\sqrt d}-g\right|=o_{\mathbb P}(1),\qquad \max_{i\ne j}\frac{|T_j^\top\mu_i|}{\sqrt d}=o_{\mathbb P}(1),
\label{eq:app-det-sharp-gaussian-inputs}
\end{equation}
and, for every fixed \(a>0\),
\begin{equation}
\frac1{kd}\#\{(i,\ell):\sqrt d\,|C_\mu[i,\ell]|\le a\}\xrightarrow{\mathbb P}\mathbb P(|G|\le a),\qquad \max_i\frac1d\#\{\ell:\sqrt d\,|\mu_i[\ell]|\le a\}\le\mathbb P(|G|\le a)+o_{\mathbb P}(1).
\label{eq:app-det-sharp-small-ball}
\end{equation}
Also \(\|C_m-C_\mu\|_{\mathrm{sum}}=o_{\mathbb P}(k\sqrt d)\) by Proposition~\ref{prop:app-det-common-event}. Every feasible \(W\) satisfies the averaged exposure inequality \(\langle W,C_m\rangle\ge k-1\). Max--sum duality and \eqref{eq:app-det-sharp-gaussian-inputs} therefore give \(r_{\max}^*(X)\ge(1+o_{\mathbb P}(1))/(g\sqrt d)\), while Proposition~\ref{prop:app-det-common-event}, item~\ref{item:app-det-H5} gives the matching upper bound. Similarly, operator--nuclear duality and Proposition~\ref{prop:app-det-spectral-input}, item~\ref{item:app-det-S1} give \(r_{\mathrm{sp}}^*(X)\ge(1+o_{\mathbb P}(1))/s_\alpha\), while the polar witness \(Q_m/\rho_X\) and Proposition~\ref{prop:app-det-spectral-input}, item~\ref{item:app-det-S3} give the matching upper bound. This proves \eqref{eq:app-det-sharp-radii}.

Fix \(W\in\mathcal M_{\max}^X\) and set \(A:=W/r_{\max}^*(X)\), so \(\|A\|_{\max}=1\). The exposure inequality and the exact Max radius imply
\begin{equation}
\langle A,C_\mu\rangle\ge\|C_\mu\|_{\mathrm{sum}}-o_{\mathbb P}(k\sqrt d)
\label{eq:app-det-sharp-max-exposure}
\end{equation}
uniformly over the complete Max face. Put \(S:=\operatorname{sign}(C_\mu)\) and \(\Delta_{i\ell}:=1-A[i,\ell]S[i,\ell]\in[0,2]\). Since \((A[i,\ell]-S[i,\ell])^2=\Delta_{i\ell}^2\le2\Delta_{i\ell}\), for every fixed \(a>0\),
\begin{equation}
\frac{\|A-S\|_{\mathrm F}^2}{kd}\le\frac{2\sqrt d}{akd}\left(\|C_\mu\|_{\mathrm{sum}}-\langle A,C_\mu\rangle\right)+\frac4{kd}\#\{(i,\ell):\sqrt d\,|C_\mu[i,\ell]|\le a\}.
\label{eq:app-det-sharp-global-split}
\end{equation}
First sending \(d\to\infty\) and then \(a\downarrow0\) gives \(\sup_W\|A-S\|_{\mathrm F}^2/(kd)=o_{\mathbb P}(1)\). Since \(C_\mu[i,\ell]=\mu_i[\ell]-\bar\mu[\ell]\), where \(\bar\mu=k^{-1}\sum_i\mu_i\) and \(\|\bar\mu\|_2=o_{\mathbb P}(1)\), the same small-ball argument transfers \(S\) to \(T\), yielding
\begin{equation}
\sup_{W\in\mathcal M_{\max}^X}\frac{\|A-T\|_{\mathrm F}^2}{kd}=o_{\mathbb P}(1).
\label{eq:app-det-sharp-global-sign}
\end{equation}
Let \(\bar A=k^{-1}\sum_iA_i\) and \(\bar T=k^{-1}\sum_iT_i\). Jensen's inequality and \eqref{eq:app-det-sharp-global-sign}, together with \(\mathbb E\|\bar T\|_2^2=d/k=O(1)\), imply \(\sup_W\|\bar A\|_2/\sqrt d=o_{\mathbb P}(1)\). Averaging the class-\(i\) constraints over its samples and over all competitors gives
\[
(A_i-\bar A)^\top m_i\ge\frac{k-1}{k\,r_{\max}^*(X)}=g\sqrt d-o_{\mathbb P}(\sqrt d).
\]
Proposition~\ref{prop:app-det-common-event}, item~\ref{item:app-det-H1}, the bound \(\|A_i-\bar A\|_2\le2\sqrt d\), and the preceding control of \(\bar A\) therefore imply \(A_i^\top\mu_i\ge g\sqrt d-o_{\mathbb P}(\sqrt d)\) uniformly over \(i\) and the complete face. The nonnegative deficiency \(D_i:=\|\mu_i\|_1-A_i^\top\mu_i\) is thus \(o_{\mathbb P}(\sqrt d)\) uniformly. For every fixed \(a>0\),
\begin{equation}
\frac1d\|A_i-T_i\|_2^2\le\frac{2D_i}{a\sqrt d}+\frac4d\#\{\ell:\sqrt d\,|\mu_i[\ell]|\le a\}.
\label{eq:app-det-sharp-row-split}
\end{equation}
Using \eqref{eq:app-det-sharp-small-ball} and then sending \(a\downarrow0\) proves \(\sup_W\max_i\|A_i-T_i\|_2/\sqrt d=o_{\mathbb P}(1)\). Finally, uniformly over \(i\ne j\), \((T_i-T_j)^\top\mu_i=g\sqrt d+o_{\mathbb P}(\sqrt d)\) and \(\|T_i-T_j\|_2=\sqrt{2d}+o_{\mathbb P}(\sqrt d)\), which proves \eqref{eq:app-det-max-sharp-signal}.

For Spectral, class-average feasibility gives \((w_i-w_j)^\top m_i\ge1\), while \(\|w_i-w_j\|_2\le\sqrt2\|W\|_{\mathrm{op}}=\sqrt2\,r_{\mathrm{sp}}^*(X)\). Proposition~\ref{prop:app-det-common-event}, item~\ref{item:app-det-H1} and \eqref{eq:app-det-sharp-radii} therefore give \((u_{ij}^W)^\top\mu_i\ge s_\alpha/\sqrt2-o_{\mathbb P}(1)\) uniformly, proving \eqref{eq:app-det-spectral-sharp-signal}.
\end{proof}

\begin{proposition}[Deterministic population-error exponents]
\label{prop:app-det-error-exponents}
Under Assumptions~\ref{ass:gaussian-cloud} and~\ref{ass:asymptotic-regime}, uniformly over the indicated complete maximum-margin direction sets,
\begin{align}
\sup_{R\in\mathcal U_{\mathrm{row}}^X}\left|\mathsf F_d(R)-\frac12\right|&=o_{\mathbb P}(1),
\label{eq:app-det-row-exponent}\\
\sup_{M\in\mathcal U_{\max}^X}\left|\mathsf F_d(M)-\frac1\pi\right|&=o_{\mathbb P}(1),
\label{eq:app-det-max-exponent}\\
\frac{s_\alpha^2}{2}-o_{\mathbb P}(1)\le\mathsf F_d(S)&\le\frac{\kappa_{\mathrm{sp}}(\alpha)^2}{2}+o_{\mathbb P}(1),\qquad S\in\mathcal U_{\mathrm{sp}}^X.
\label{eq:app-det-spectral-exponent}
\end{align}
Consequently, for \(0<\alpha\le1\), the spectral interval collapses to \(\kappa_{\mathrm{sp}}(\alpha)^2/2\); at \(\alpha=1\), this value is \(32/(9\pi^2)\). Uniformly over \(R\in\mathcal U_{\mathrm{row}}^X\), \(M\in\mathcal U_{\max}^X\), and \(S\in\mathcal U_{\mathrm{sp}}^X\),
\begin{equation}
\log\frac{\mathcal E_{\pi,d}(M)}{\mathcal E_{\pi,d}(R)}=\left(\frac{\pi-2}{4\pi}+o_{\mathbb P}(1)\right)\frac d{r_{\mathrm{te}}^2},
\label{eq:app-det-row-max-error-gain}
\end{equation}
\begin{equation}
\frac{1-\kappa_{\mathrm{sp}}(\alpha)^2}{2}-o_{\mathbb P}(1)\le\frac{2r_{\mathrm{te}}^2}{d}\log\frac{\mathcal E_{\pi,d}(S)}{\mathcal E_{\pi,d}(R)}\le\frac{1-s_\alpha^2}{2}+o_{\mathbb P}(1).
\label{eq:app-det-row-spectral-error-gain}
\end{equation}
The Row--Spectral gain is exact when \(\alpha\le1\).
\end{proposition}

\begin{proof}
For Row, Proposition~\ref{prop:app-det-row-boundary} and the Gaussian mean shell imply \((u_{ij}^R)^\top\mu_i=1/\sqrt2+o_{\mathbb P}(1)\) uniformly over all directed pairs and the complete Row face. Proposition~\ref{prop:app-det-sharp-directed-signals} gives the corresponding uniform Max signal \(1/\sqrt\pi+o_{\mathbb P}(1)\). Hence \(\mathsf E_d(R)=1/2+o_{\mathbb P}(1)\) and \(\mathsf E_d(M)=1/\pi+o_{\mathbb P}(1)\) uniformly.

For Spectral, Proposition~\ref{prop:app-det-sharp-directed-signals} gives \(\mathsf E_d(S)\ge s_\alpha^2/2-o_{\mathbb P}(1)\) uniformly. Theorem~\ref{thm:deterministic-separation-clean} implies that every spectral direction has an ordered pair with \(d_{ij}(S)\ge\delta_{\mathrm{sp}}(\alpha)-o_{\mathbb P}(1)\), and hence \((u_{ij}^S)^\top u_{ij}^{\mu}\le\kappa_{\mathrm{sp}}(\alpha)+o_{\mathbb P}(1)\). For that unordered pair,
\[
(u_{ij}^S)^\top\mu_i+(u_{ji}^S)^\top\mu_j=(u_{ij}^S)^\top(\mu_i-\mu_j)\le\sqrt2\,\kappa_{\mathrm{sp}}(\alpha)+o_{\mathbb P}(1).
\]
Both directed signals are positive by the preceding lower bound, so at least one is at most \(\kappa_{\mathrm{sp}}(\alpha)/\sqrt2+o_{\mathbb P}(1)\). This proves the upper endpoint for \(\mathsf E_d(S)\). Lemma~\ref{lem:app-det-exponent-reduction} transfers all three statements from \(\mathsf E_d\) to \(\mathsf F_d\). Finally, \(s_\alpha=\kappa_{\mathrm{sp}}(\alpha)\) for \(\alpha\le1\), while subtraction of the exponents gives \eqref{eq:app-det-row-max-error-gain}--\eqref{eq:app-det-row-spectral-error-gain}.
\end{proof}

\begin{proposition}[Population-error exponents of deterministic optimizers]
\label{cor:app-det-optimizer-error-exponents}
Under the assumptions and optimizer conventions of Theorem~\ref{thm:det-optimizer-accuracy-order-clean}, the deterministic full-batch optimizer trajectories satisfy, in probability,
\begin{align}
\limsup_{t\to\infty}\left|\mathsf F_d\!\left(W_t^{\mathrm{RowGDM}}\right)-\frac12\right|&=o_{\mathbb P}(1),
\label{eq:app-rowgdm-error-exponent}\\
\limsup_{t\to\infty}\left|\mathsf F_d\!\left(W_t^{\mathrm{Adam}}\right)-\frac1\pi\right|&=o_{\mathbb P}(1).
\label{eq:app-adam-error-exponent}
\end{align}
Exact-SVD Muon satisfies
\begin{equation}
\frac{s_\alpha^2}{2}-o_{\mathbb P}(1)\le\liminf_{t\to\infty}\mathsf F_d\!\left(W_t^{\mathrm{Muon}}\right)\le\limsup_{t\to\infty}\mathsf F_d\!\left(W_t^{\mathrm{Muon}}\right)\le\frac{\kappa_{\mathrm{sp}}(\alpha)^2}{2}+o_{\mathbb P}(1).
\label{eq:app-muon-error-exponent}
\end{equation}
Define the Row--Spectral exponent-gap constant
\begin{equation}
\Delta_{\rm exp}(\alpha):=\frac{1-\kappa_{\mathrm{sp}}(\alpha)^2}{4}=\frac14\left(1-\frac{m_{1/2}(\alpha)^2}{\alpha\wedge1}\right)>0,
\qquad
m_{1/2}(\alpha):=\frac{\alpha}{2\pi}\int_{a_\alpha}^{b_\alpha}\frac{\sqrt{(b_\alpha-x)(x-a_\alpha)}}{\sqrt{x}}\,dx,
\label{eq:app-det-c-alpha}
\end{equation}
where
\begin{equation}
a_\alpha=(1-\alpha^{-1/2})^2,\qquad b_\alpha=(1+\alpha^{-1/2})^2.
\end{equation}
Equivalently,
\begin{equation}
\Delta_{\rm exp}(\alpha)=\frac14\left[1-\frac{\alpha^2}{4\pi^2(\alpha\wedge1)}\left(\int_{a_\alpha}^{b_\alpha}\frac{\sqrt{(b_\alpha-x)(x-a_\alpha)}}{\sqrt{x}}\,dx\right)^2\right].
\label{eq:app-det-c-alpha-expanded}
\end{equation}
For example,
\begin{equation}
\Delta_{\rm exp}(1)=\frac14\left(1-\frac{64}{9\pi^2}\right)\approx0.06987,\qquad \Delta_{\rm exp}(2)\approx0.03247,\qquad \Delta_{\rm exp}(3)\approx0.02133.
\label{eq:app-det-c-alpha-examples}
\end{equation}
Consequently,
\begin{equation}
\limsup_{t\to\infty}\left|\frac{r_{\mathrm{te}}^2}{d}\log\frac{\mathcal E_{\pi,d}(W_t^{\mathrm{Adam}})}{\mathcal E_{\pi,d}(W_t^{\mathrm{RowGDM}})}-\frac{\pi-2}{4\pi}\right|=o_{\mathbb P}(1),
\label{eq:app-rowgdm-adam-error-gain}
\end{equation}
and
\begin{align}
\Delta_{\rm exp}(\alpha)-o_{\mathbb P}(1)
&\le\liminf_{t\to\infty}\frac{r_{\mathrm{te}}^2}{d}\log\frac{\mathcal E_{\pi,d}(W_t^{\mathrm{Muon}})}{\mathcal E_{\pi,d}(W_t^{\mathrm{RowGDM}})}
\nonumber\\
&\le\limsup_{t\to\infty}\frac{r_{\mathrm{te}}^2}{d}\log\frac{\mathcal E_{\pi,d}(W_t^{\mathrm{Muon}})}{\mathcal E_{\pi,d}(W_t^{\mathrm{RowGDM}})}
\le\frac{1-s_\alpha^2}{4}+o_{\mathbb P}(1).
\label{eq:app-rowgdm-muon-error-gain}
\end{align}
The same conclusions hold with RowGDM replaced by RowGD, with Adam replaced by SignGD or Signum, and with Muon replaced by Spectral-GD. When \(0<\alpha\le1\), the two Row--Muon bounds coincide. Here, \(o_{\mathbb P}(1)\) refers to the limit \(d\to\infty\), after taking \(t\to\infty\) at each fixed \(d\).
\end{proposition}

\begin{proof}
The full-batch margin-convergence results used in Theorem~\ref{thm:det-optimizer-accuracy-order-clean} imply that the normalized trajectories approach the corresponding row-, entrywise max-, and spectral-norm maximum-margin direction sets. Since Proposition~\ref{prop:app-det-error-exponents} holds uniformly over each complete direction set, Lemma~\ref{lem:app-acc-trajectory-continuity} transfers its exponent bounds to the long-time optimizer trajectories. This gives \eqref{eq:app-rowgdm-error-exponent}--\eqref{eq:app-muon-error-exponent}. The definition \eqref{eq:app-det-c-alpha} follows from \(\kappa_{\mathrm{sp}}(\alpha)=m_{1/2}(\alpha)/\sqrt{\alpha\wedge1}\), and the integral form follows from the Marchenko--Pastur density. The log-error ratio statements follow from
\[
\frac{r_{\mathrm{te}}^2}{d}\log\frac{\mathcal E_{\pi,d}(W^{(2)})}{\mathcal E_{\pi,d}(W^{(1)})}=\frac12\left(\mathsf F_d(W^{(1)})-\mathsf F_d(W^{(2)})\right).
\]
No uniqueness of the optimizer-selected limiting direction is required.
\end{proof}

\begin{corollary}[Full-batch Muon versus Adam]
\label{cor:app-det-muon-adam-order}
Under the assumptions and optimizer conventions of Theorem~\ref{thm:det-optimizer-accuracy-order-clean}, suppose \(0<\alpha<\alpha_\star\), where \(\alpha_\star\approx1.225\) is the aspect-ratio threshold defined by \(s_{\alpha_\star}^2=2/\pi\). Then, with
\[
c_{\mathrm{MA}}(\alpha):=\frac14\left(s_\alpha^2-\frac2\pi\right)>0,
\]
the deterministic full-batch trajectories satisfy
\begin{equation}
\limsup_{t\to\infty}\frac{r_{\mathrm{te}}^2}{d}\log\frac{\mathcal E_{\pi,d}(W_t^{\mathrm{Muon}})}{\mathcal E_{\pi,d}(W_t^{\mathrm{Adam}})}\le -c_{\mathrm{MA}}(\alpha)+o_{\mathbb P}(1).
\label{eq:app-det-muon-adam-order}
\end{equation}
Thus exact-SVD Muon has asymptotically smaller population error, and hence higher population accuracy, than full-batch Adam with probability tending to one. The same conclusion holds with Muon replaced by Spectral-GD and Adam replaced by SignGD or Signum.
\end{corollary}

\begin{proof}
Proposition~\ref{cor:app-det-optimizer-error-exponents} gives \(\liminf_{t\to\infty}\mathsf F_d(W_t^{\mathrm{Muon}})\ge s_\alpha^2/2-o_{\mathbb P}(1)\) and \(\mathsf F_d(W_t^{\mathrm{Adam}})=1/\pi+o_{\mathbb P}(1)\) at long times. Using \(\frac{r_{\mathrm{te}}^2}{d}\log(\mathcal E_2/\mathcal E_1)=\frac12(\mathsf F_d(W_1)-\mathsf F_d(W_2))\) gives \eqref{eq:app-det-muon-adam-order}.
\end{proof}

\begin{remark}[Relation to last-layer practice]
\label{rem:app-det-muon-adam-practice}
In practical last-layer training, Adam is typically preferred to Muon, whereas Corollary~\ref{cor:app-det-muon-adam-order} predicts the opposite ordering for \(\alpha<\alpha_\star\). A likely source of this mismatch is that our result compares only asymptotic convergence directions and therefore does not capture convergence speed or finite-time trajectories. The comparison is also idealized, using deterministic full-batch training, exact-SVD Muon, and Adam without a numerical-stability constant. For \(\alpha\ge\alpha_\star\), the ordering between Adam and Muon remains open.
\end{remark}

\section{Proofs for Stochastic Optimizers}
\label{app:stochastic-margin}

This appendix proves Proposition~\ref{thm:stoch-optimizer-accuracy-order-clean}. The random-reshuffling momentum analysis in Corollary~\ref{cor:app-ga-margin-transfer} places RR-RowSGDM, RR-Muon, and RR-Signum in near-maximum-margin direction sets for the row, spectral, and entrywise max norms, respectively. We show that, for sufficiently large momentum, these near-margin sets retain enough of the deterministic geometric separation to yield a strict population-accuracy ordering. Throughout this appendix, training time tends to infinity at each fixed dimension before the limit \(d\to\infty\).

For each \(d\), write \(m_d:=N(d)/b_d\). Recall the geometry-adapted atom radius \(R_\psi(X)\) from~\eqref{eq:app-ga-radius}, and for \(\psi\in\{\psi_{\mathrm{row}},\psi_{\mathrm{sp}},\psi_{\max}\}\) define
\begin{equation}
\varepsilon_{\psi,\beta,b}(X):=2(1-\beta)m_d(m_d^2-1)R_\psi(X),\qquad \mathcal U_{\psi,\beta,b}^X:=\mathcal U_{\psi,\varepsilon_{\psi,\beta,b}(X)}^X.
\label{eq:app-stoch-deficit-definition}
\end{equation}

\subsection{Near-margin directions and near-minimum-norm representatives}
\label{app:stoch-near-representatives}

\begin{lemma}[Near-margin reduction]
\label{lem:app-stoch-near-reduction}
Fix a linearly separable dataset \(X\), a matrix norm \(\psi\), and \(0\le\varepsilon<\gamma_\psi(X)\). For every \(U\in\mathcal U_{\psi,\varepsilon}^X\), define
\begin{equation}
V_U:=\frac{U}{\operatorname{mar}_X(U)},\qquad \xi_\psi(\varepsilon;X):=\frac{\varepsilon}{\gamma_\psi(X)-\varepsilon}.
\label{eq:app-stoch-near-representative}
\end{equation}
Then \(V_U\in\mathcal F_X\),
\begin{equation}
\psi(V_U)\le\left(1+\xi_\psi(\varepsilon;X)\right)r_\psi^*(X),
\label{eq:app-stoch-near-norm-bound}
\end{equation}
and \(V_U\) has the same normalized pairwise row differences as \(U\).
\end{lemma}

\begin{proof}
Set \(\widehat U:=U/\psi(U)\). By definition of \(\mathcal U_{\psi,\varepsilon}^X\), \(\operatorname{mar}_X(\widehat U)\ge\gamma_\psi(X)-\varepsilon>0\). Hence \(V_U=\widehat U/\operatorname{mar}_X(\widehat U)\) is feasible and
\[
\psi(V_U)=\frac1{\operatorname{mar}_X(\widehat U)}\le\frac1{\gamma_\psi(X)-\varepsilon}=\frac{r_\psi^*(X)}{1-\varepsilon/\gamma_\psi(X)}=\left(1+\xi_\psi(\varepsilon;X)\right)r_\psi^*(X).
\]
The normalized pairwise row differences are unchanged under positive scaling.
\end{proof}

\subsection{Quantitative stability of near minimizers}
\label{app:stoch-near-stability}

\begin{lemma}[Combined Gaussian-cloud event]
\label{lem:app-stoch-combined-input}
There are deterministic \(\omega_d\downarrow0\), constants \(C,c>0\), and events \(\mathcal E_d^{\mathrm{near}}\) such that
\begin{equation}
\mathbb P\!\left((\mathcal E_d^{\mathrm{near}})^c\right)\le C\exp\!\left[-c\sqrt{d\log(2N(d)(d+1))}\right],
\label{eq:app-stoch-near-event-probability}
\end{equation}
all conclusions of Propositions~\ref{prop:app-det-common-event} and~\ref{prop:app-det-spectral-input} hold with their vanishing errors bounded by \(\omega_d\), and
\begin{equation}
\max_{i,s}\|x_{i,s}\|_2\le1+2r_0+\omega_d,\qquad \max_{i,s}\|x_{i,s}\|_1\le(1+2r_0+\omega_d)\sqrt d.
\label{eq:app-stoch-data-radius-bounds}
\end{equation}
We may additionally take \(\omega_d\ge1/k\).
\end{lemma}

\begin{proof}
Intersect the events in Propositions~\ref{prop:app-det-common-event} and~\ref{prop:app-det-spectral-input} with the event \(\max_{i,s}\|z_{i,s}\|_2\le2\). The latter fails with probability at most \(2N e^{-cd}\), which is smaller than the right-hand side of~\eqref{eq:app-stoch-near-event-probability}. Under the normalized convention~\eqref{eq:app-normalized-data}, the mean shell gives \(\max_i\|\mu_i\|_2\le1+\omega_d\), so \(\|x_{i,s}\|_2\le1+\omega_d+2r_0\), and the \(\ell_1\) bound follows from \(\|x\|_1\le\sqrt d\,\|x\|_2\). Enlarge one deterministic envelope to dominate all finitely many errors and \(1/k\).
\end{proof}

\begin{proposition}[Near-minimizer stability]
\label{prop:app-stoch-near-stability}
There are constants \(C_{\mathrm{row}},C_{\max}>0\) and \(C_{\mathrm{sp}}(\alpha)>0\) such that, on \(\mathcal E_d^{\mathrm{near}}\), for every \(0\le\xi\le1\),
\begin{align}
W\in\mathcal F_X,\ \|W\|_{\mathrm{row}}\le(1+\xi)r_{\mathrm{row}}^*(X)&\Longrightarrow \max_{i\ne j}\|u_{ij}^W-u_{ij}^{\mu}\|_2\le C_{\mathrm{row}}(\omega_d+\xi)^{1/4},
\label{eq:app-stoch-row-near-stability}\\
W\in\mathcal F_X,\ \|W\|_{\mathrm{op}}\le(1+\xi)r_{\mathrm{sp}}^*(X)\Longrightarrow& D_{\mathrm{avg\text{-}bdry}}(W)\ge\delta_{\mathrm{sp}}(\alpha)-C_{\mathrm{sp}}(\alpha)(\omega_d+\xi)^{1/4},
\label{eq:app-stoch-sp-near-stability}\\
W\in\mathcal F_X,\
\|W\|_{\max}\le(1+\xi)r_{\max}^*(X)&\Longrightarrow D_{\mathrm{avg\text{-}bdry}}(W)\ge\delta_{\max}-C_{\max}(\omega_d+\xi).
\label{eq:app-stoch-max-near-stability}
\end{align}
\end{proposition}

\begin{proof}
For the row norm, put \(q:=\omega_d+\xi\). The empirical-mean witness in Proposition~\ref{prop:app-det-row-rigidity} gives \(r_{\mathrm{row}}^*(X)\le1+C\omega_d\), so \(\max_i\|w_i\|_2\le1+Cq\). Averaging feasibility within class \(i\) and then over all competing classes gives \((w_i-\bar w)^\top m_i\ge1-1/k\). The row variance identity then gives \(\|\bar w\|_2^2\le Cq\), and substituting this bound back into the averaged margin inequality yields \(\max_i\|w_i-m_i\|_2\le Cq^{1/4}\). The mean perturbation in Proposition~\ref{prop:app-det-common-event} and Lemma~\ref{lem:app-det-normalization} prove~\eqref{eq:app-stoch-row-near-stability}.

For the spectral norm, Proposition~\ref{prop:app-det-spectral-rigidity} gives an error \(C_\alpha q^{1/2}\) when \(\alpha\ne1\) and \(Cq^{1/4}\) when \(\alpha=1\). Since \(q^{1/2}\le q^{1/4}\) for \(q\le1\), enlarging the constant proves~\eqref{eq:app-stoch-sp-near-stability} uniformly over all aspect ratios.

For the max norm, fix \(i\ne j\), set \(v:=w_i-w_j\), and put \(u:=u_{ij}^m\). Near-optimality and Proposition~\ref{prop:app-det-common-event}(H5) give
\begin{equation}
\|v\|_\infty\le\frac{\sqrt{2\pi}+Cq}{\sqrt d}.
\label{eq:app-stoch-max-near-cap}
\end{equation}
The two class-averaged feasibility constraints imply \(v^\top(m_i-m_j)\ge2\), and hence \(v^\top u\ge\sqrt2-Cq\). Applying the clipped minimum-energy characterization in Proposition~\ref{prop:app-det-clipped-exact}, the uniform empirical laws in Proposition~\ref{prop:app-det-common-event}(H6), and Lemma~\ref{lem:app-det-clipped-asymptotic} gives
\begin{equation}
\frac{v^\top u}{\|v\|_2}\le\kappa_{\max}+Cq.
\label{eq:app-stoch-max-near-cosine}
\end{equation}
The cosine-distance identity and Proposition~\ref{prop:app-det-common-event}(H4) prove~\eqref{eq:app-stoch-max-near-stability}.
\end{proof}

\subsection{A dimension-independent momentum threshold}
\label{app:stoch-threshold}

\begin{proposition}[Uniform radius-to-margin ratios]
\label{prop:app-stoch-ratio-bounds}
On \(\mathcal E_d^{\mathrm{near}}\), for every sufficiently large \(d\),
\begin{equation}
\frac{R_{\mathrm{row}}(X)}{\gamma_{\mathrm{row}}(X)}\le K_{\mathrm{row}},\qquad \frac{R_{\mathrm{sp}}(X)}{\gamma_{\mathrm{sp}}(X)}\le K_{\mathrm{sp}},\qquad \frac{R_{\max}(X)}{\gamma_{\max}(X)}\le K_{\max},
\label{eq:app-stoch-ratio-bounds}
\end{equation}
where one may take
\begin{equation}
K_{\mathrm{row}}:=4(1+2r_0),\qquad K_{\mathrm{sp}}:=\frac{8(1+2r_0)}{s_\alpha},\qquad K_{\max}:=(\sqrt{2\pi}+1)(1+2r_0).
\label{eq:app-stoch-ratio-constants}
\end{equation}
\end{proposition}

\begin{proof}
Each ratio equals \(R_\psi(X)r_\psi^*(X)\). For every sufficiently large \(d\), Proposition~\ref{prop:app-det-row-rigidity} gives \(r_{\mathrm{row}}^*(X)\le2\), the reduced-polar witness~\eqref{eq:app-det-spectral-optimum-upper} and \(\rho_X\to s_\alpha>0\) give \(r_{\mathrm{sp}}^*(X)\le4/s_\alpha\), and Proposition~\ref{prop:app-det-common-event}(H5) gives \(r_{\max}^*(X)\le(\sqrt{2\pi}+1)/(2\sqrt d)\). Equation~\eqref{eq:app-stoch-data-radius-bounds} and \(\omega_d\to0\) imply, for sufficiently large \(d\),
\[
R_{\mathrm{row}}(X)=R_{\mathrm{sp}}(X)\le2(1+2r_0),\qquad R_{\max}(X)\le2(1+2r_0)\sqrt d.
\]
Combining these estimates proves~\eqref{eq:app-stoch-ratio-bounds}.
\end{proof}

Let
\begin{equation}
\delta_*:=\min\{\delta_{\mathrm{sp}}(\alpha),\delta_{\max}\}>0.
\label{eq:app-stoch-delta-star}
\end{equation}
Choose \(h_*>0\) sufficiently small that
\begin{equation}
\frac1{\sqrt2}-h_*>\frac1{\sqrt2}\left(1-\frac{(\delta_*/2)^2}{2}\right)_+.
\label{eq:app-stoch-accuracy-amplitude-budget}
\end{equation}
Choose \(b_*\in(0,1]\) sufficiently small that
\begin{equation}
C_{\mathrm{row}}b_*^{1/4}\le h_*,\qquad C_{\mathrm{sp}}(\alpha)b_*^{1/4}\le\frac{\delta_*}{8},\qquad C_{\max}b_*\le\frac{\delta_*}{8}.
\label{eq:app-stoch-stability-budget}
\end{equation}
Set
\begin{equation}
\theta_*:=\frac{b_*/2}{1+b_*/2},\qquad A_0:=m_0(m_0^2-1),\qquad K_*:=\max\{K_{\mathrm{row}},K_{\mathrm{sp}},K_{\max}\},
\label{eq:app-stoch-threshold-constants}
\end{equation}
and define
\begin{equation}
\beta_0:=\begin{cases}
\dfrac12,&A_0=0,\\[4pt]
\displaystyle\max\left\{\frac12,1-\frac{\theta_*}{2A_0K_*}\right\},&A_0>0.
\end{cases}
\label{eq:app-stoch-beta-zero}
\end{equation}
All constants in~\eqref{eq:app-stoch-delta-star}--\eqref{eq:app-stoch-beta-zero} depend only on \((\alpha,r_0,m_0)\).

\begin{proposition}[Uniform stochastic near-margin geometry]
\label{prop:app-stoch-near-geometry}
\label{thm:stochastic-separation-clean}
Under the data and batching assumptions of Proposition~\ref{thm:stoch-optimizer-accuracy-order-clean}, fix \(\beta\in[\beta_0,1)\). Then, for every sufficiently large \(d\), on \(\mathcal E_d^{\mathrm{near}}\), all three effective margins \(\gamma_\psi(X)-\varepsilon_{\psi,\beta,b}(X)\) are positive and
\begin{align}
\sup_{U\in\mathcal U_{\mathrm{row},\beta,b}^X}\max_{i\ne j}\|u_{ij}^U-u_{ij}^{\mu}\|_2&\le h_*,
\label{eq:app-stoch-row-set-upper}\\
\inf_{V\in\mathcal U_{\mathrm{sp},\beta,b}^X}D_{\mathrm{avg\text{-}bdry}}(V)&\ge\frac{\delta_*}{2},
\label{eq:app-stoch-sp-set-lower}\\
\inf_{V\in\mathcal U_{\max,\beta,b}^X}D_{\mathrm{avg\text{-}bdry}}(V)&\ge\frac{\delta_*}{2}.
\label{eq:app-stoch-max-set-lower}
\end{align}
Moreover, after increasing the deterministic dimension threshold if necessary,
\begin{equation}
\mathbb P(\mathcal E_d^{\mathrm{near}})\ge1-e^{-\sqrt d}.
\label{eq:app-stoch-clean-probability}
\end{equation}
\end{proposition}

\begin{proof}
For each geometry, define the relative stochastic deficit
\begin{equation}
\Delta_{\psi,d}(\beta):=\frac{\varepsilon_{\psi,\beta,b}(X)}{\gamma_\psi(X)}=2(1-\beta)m_d(m_d^2-1)\frac{R_\psi(X)}{\gamma_\psi(X)}.
\label{eq:app-stoch-relative-deficit}
\end{equation}
Since \(m_d\le m_0\), Proposition~\ref{prop:app-stoch-ratio-bounds} gives
\[
\Delta_{\psi,d}(\beta)\le2(1-\beta)A_0K_*\le\theta_*<1.
\]
Thus every effective margin is positive. By the definition of \(\theta_*\),
\begin{equation}
\frac{\Delta_{\psi,d}(\beta)}{1-\Delta_{\psi,d}(\beta)}\le\frac{b_*}{2}.
\label{eq:app-stoch-near-slack-bound}
\end{equation}
For every \(U\in\mathcal U_{\psi,\beta,b}^X\), Lemma~\ref{lem:app-stoch-near-reduction} therefore gives a feasible representative with identical normalized pairwise row differences and relative norm slack at most \(b_*/2\). For every sufficiently large \(d\), \(\omega_d\le b_*/2\), so Proposition~\ref{prop:app-stoch-near-stability} gives
\[
\sup_{U\in\mathcal U_{\mathrm{row},\beta,b}^X}\max_{i\ne j}\|u_{ij}^U-u_{ij}^{\mu}\|_2\le C_{\mathrm{row}}b_*^{1/4}\le h_*,
\]
\[
\inf_{V\in\mathcal U_{\mathrm{sp},\beta,b}^X}D_{\mathrm{avg\text{-}bdry}}(V)\ge\delta_{\mathrm{sp}}(\alpha)-C_{\mathrm{sp}}(\alpha)b_*^{1/4}\ge\delta_*-\frac{\delta_*}{8}\ge\frac{\delta_*}{2},
\]
and
\[
\inf_{V\in\mathcal U_{\max,\beta,b}^X}D_{\mathrm{avg\text{-}bdry}}(V)\ge\delta_{\max}-C_{\max}b_*\ge\delta_*-\frac{\delta_*}{8}\ge\frac{\delta_*}{2}.
\]
This proves~\eqref{eq:app-stoch-row-set-upper}--\eqref{eq:app-stoch-max-set-lower}. Finally, \(\log(2N(d)(d+1))\to\infty\), so the failure probability in~\eqref{eq:app-stoch-near-event-probability} is at most \(e^{-\sqrt d}\) for every sufficiently large \(d\).
\end{proof}

\subsection{Proof of the stochastic optimizer accuracy ordering}
\label{app:stoch-main-proof}

\begin{proof}[Proof of Proposition~\ref{thm:stoch-optimizer-accuracy-order-clean}]
Fix \(\beta\in[\beta_0,1)\), and let \(\mathcal R_d\) be the Gaussian radial event from Lemma~\ref{lem:app-acc-gaussian-radial-event}. Set
\[
\mathcal H_d:=\mathcal E_d^{\mathrm{near}}\cap\mathcal R_d.
\]
By~\eqref{eq:app-stoch-near-event-probability} and~\eqref{eq:app-acc-radial-event-probability}, after increasing the deterministic dimension threshold,
\begin{equation}
\mathbb P(\mathcal H_d)\ge1-e^{-\sqrt d}.
\label{eq:app-stoch-accuracy-event}
\end{equation}

Work on \(\mathcal H_d\). Define the Row family
\[
\mathcal R_d^{\mathrm{clf}}:=\mathcal U_{\mathrm{row},\beta,b}^X
\]
and, for \(\psi\in\{\mathrm{sp},\max\}\), the competitor family
\[
\mathcal C_{\psi,d}^{\mathrm{clf}}:=\mathcal U_{\psi,\beta,b}^X.
\]
Proposition~\ref{prop:app-stoch-near-geometry} gives
\begin{equation}
\sup_{U\in\mathcal R_d^{\mathrm{clf}}}\max_{i\ne j}\|u_{ij}^U-u_{ij}^{\mu}\|_2\le h_*,
\qquad
\inf_{V\in\mathcal C_{\psi,d}^{\mathrm{clf}}}D_{\mathrm{avg\text{-}bdry}}(V)\ge\frac{\delta_*}{2}.
\label{eq:app-acc-stochastic-geometry-input}
\end{equation}
The choice of \(h_*\) in~\eqref{eq:app-stoch-accuracy-amplitude-budget} gives the strict limiting amplitude gap
\begin{equation}
\mathsf a(h_*)>\mathsf b(\delta_*/2),
\label{eq:app-acc-stochastic-amplitude-gap}
\end{equation}
where \(\mathsf a\) and \(\mathsf b\) are defined in~\eqref{eq:app-acc-limiting-amplitudes}.

Choose fixed constants \(0<b<a\) such that
\begin{equation}
\mathsf b(\delta_*/2)<b<a<\mathsf a(h_*).
\label{eq:app-stoch-fixed-amplitude-gap}
\end{equation}
On \(\mathcal R_d\), the radial deviation and coherence in~\eqref{eq:app-acc-radial-coherence} converge to zero under a deterministic envelope. Therefore, for every sufficiently large \(d\), the finite-dimensional amplitudes in~\eqref{eq:app-acc-row-amplitude}--\eqref{eq:app-acc-competitor-amplitude} satisfy
\[
A_d\ge a,\qquad B_d\le b
\]
for the Row family and each of the two competitor families. Proposition~\ref{prop:app-acc-finite-bounds} consequently gives
\begin{equation}
\sup_{U\in\mathcal R_d^{\mathrm{clf}}}\mathcal E_{\pi,d}(U)\le(k-1)\overline\Phi\!\left(\frac{a}{\sigma_d}\right),
\qquad
\inf_{V\in\mathcal C_{\psi,d}^{\mathrm{clf}}}\mathcal E_{\pi,d}(V)\ge2\pi_{\min,d}\overline\Phi\!\left(\frac{b}{\sigma_d}\right).
\label{eq:app-stoch-family-error-bounds}
\end{equation}
For \(x\ge y\ge0\), the Gaussian upper tail satisfies \(\overline\Phi(x)/\overline\Phi(y)\le\exp(-(x^2-y^2)/2)\). Hence \(k=O(d)\), \(r_{\mathrm{te}}\le r_{\mathrm{te},0}\), and \(\log(1/\pi_{\min,d})=o(d)\) imply
\begin{equation}
\frac{\sup_{U\in\mathcal R_d^{\mathrm{clf}}}\mathcal E_{\pi,d}(U)}{\inf_{V\in\mathcal C_{\psi,d}^{\mathrm{clf}}}\mathcal E_{\pi,d}(V)}
\le
\frac{Ck}{\pi_{\min,d}}e^{-cd}
\longrightarrow0.
\label{eq:app-stoch-family-error-ratio}
\end{equation}
Hence, for every sufficiently large \(d\) on \(\mathcal H_d\),
\begin{equation}
\sup_{U\in\mathcal U_{\mathrm{row},\beta,b}^X}\mathcal E_{\pi,d}(U)
<
\min_{\psi\in\{\mathrm{sp},\max\}}
\inf_{V\in\mathcal U_{\psi,\beta,b}^X}\mathcal E_{\pi,d}(V).
\label{eq:app-stoch-set-error-order}
\end{equation}

It remains to transfer the set ordering to the optimizer trajectories. Proposition~\ref{prop:app-stoch-near-geometry} ensures that all effective margins are positive. Corollary~\ref{cor:app-ga-margin-transfer} therefore gives
\begin{align}
d_{\cos}\!\left(W_t^{\mathrm{RR\text{-}RowSGDM}},\mathcal U_{\mathrm{row},\beta,b}^X\right)&\longrightarrow0,\nonumber\\
d_{\cos}\!\left(W_t^{\mathrm{RR\text{-}Muon}},\mathcal U_{\mathrm{sp},\beta,b}^X\right)&\longrightarrow0,\nonumber\\
d_{\cos}\!\left(W_t^{\mathrm{RR\text{-}Signum}},\mathcal U_{\max,\beta,b}^X\right)&\longrightarrow0.
\label{eq:app-stoch-trajectory-near-set}
\end{align}
Applying Lemma~\ref{lem:app-acc-trajectory-continuity} to~\eqref{eq:app-stoch-trajectory-near-set} and then using~\eqref{eq:app-stoch-set-error-order} gives, for each \(\mathrm{opt}\in\{\mathrm{RR\text{-}Muon},\mathrm{RR\text{-}Signum}\}\),
\begin{align}
\limsup_{t\to\infty}\mathcal E_{\pi,d}\!\left(W_t^{\mathrm{RR\text{-}RowSGDM}}\right)
&\le
\sup_{U\in\mathcal U_{\mathrm{row},\beta,b}^X}\mathcal E_{\pi,d}(U)\nonumber\\
&<
\inf_{V\in\mathcal U_{\psi_{\mathrm{opt}},\beta,b}^X}\mathcal E_{\pi,d}(V)\nonumber\\
&\le
\liminf_{t\to\infty}\mathcal E_{\pi,d}\!\left(W_t^{\mathrm{opt}}\right),
\label{eq:app-stoch-trajectory-error-order}
\end{align}
where \(\psi_{\mathrm{Muon}}=\psi_{\mathrm{sp}}\) and \(\psi_{\mathrm{Signum}}=\psi_{\max}\). Since \(\operatorname{Acc}_{\pi,d}=1-\mathcal E_{\pi,d}\),
\[
\liminf_{t\to\infty}\operatorname{Acc}_{\pi,d}\!\left(W_t^{\mathrm{RR\text{-}RowSGDM}}\right)
>
\max_{\mathrm{opt}\in\{\mathrm{RR\text{-}Muon},\mathrm{RR\text{-}Signum}\}}
\limsup_{t\to\infty}\operatorname{Acc}_{\pi,d}\!\left(W_t^{\mathrm{opt}}\right).
\]
Since \(H_d\) and the dimension threshold above are independent of \(\beta\in[\beta_0,1)\), and the random-reshuffling margin bounds hold deterministically for every reshuffling realization, the conclusion holds simultaneously for all such \(\beta\) and every realization of the reshuffling.
Together with~\eqref{eq:app-stoch-accuracy-event}, this proves~\eqref{eq:stoch-optimizer-accuracy-order-clean}.
\end{proof}

\section{Analysis of AdamProxy for Random-reshuffling Mini-batch Adam}
\label{app:adamproxy-analysis}

This appendix proves Proposition~\ref{thm:rr-adamproxy-epoch-high-memory}, Proposition~\ref{thm:adamproxy-optimizer-accuracy-order-clean}, and Proposition~\ref{thm:adamproxy-margin-clean}. We first show that $\beta_1, \beta_2 \to 1$ random-reshuffling Adam is approximated over one epoch by the deterministic AdamProxy direction. We then prove a trajectory-level coordinate bound for AdamProxy and use it to establish a coordinate-induced angular gap, which, together with the near-row geometry from Appendix~\ref{app:stochastic-margin} and the geometry-to-accuracy transfer in Appendix~\ref{app:deterministic-margin}, yields the RR-RowSGDM--AdamProxy population-accuracy ordering. Separately, we characterize the limiting AdamProxy classifier through a self-consistent semidefinite quadratic margin problem when the normalized proxy direction converges.

\subsection{One-epoch \texorpdfstring{$\beta_1, \beta_2 \to 1$}{beta1, beta2 to 1} RR Adam approximation}
\label{app:rr-adamproxy-one-epoch-proof}

Throughout this subsection, probability is taken only over the independent random reshufflings while the training data are fixed. We assume \(1\le b<N\), use the gradient convention of the main text, namely \(g_n(W)=-(e_{y_n}-p_n(W))x_n^\top\), and use the standard Adam initialization \(M_0=V_0={\mathbf 0}\). For a uniformly random \(b\)-subset \(B\subset[N]\), recall \(G_B(W)=b^{-1}\sum_{n\in B}g_n(W)\), and set
\begin{equation}
S(W):=\sum_{n=1}^Ng_n(W),\qquad Q(W):=\sum_{n=1}^Ng_n(W)^{\odot2}.
\label{eq:app-adamproxy-SQ-gradient}
\end{equation}

\begin{lemma}[Finite-population batch moments and relative scale]
\label{lem:app-rr-adamproxy-batch-package}
Assume \(1\le b<N\). For \(z=(z_1,\ldots,z_N)\in\mathbb R^N\), define \(T_Bz:=b^{-1}\sum_{n\in B}z_n\) and \(\sigma(z)^2:=\mathbb E_B[(T_Bz)^2]\). Then
\begin{equation}
\mathbb E_B[T_Bz]=\frac1N\sum_{n=1}^Nz_n,
\label{eq:app-rr-adamproxy-batch-first-moment}
\end{equation}
\begin{equation}
\sigma(z)^2=\frac{b-1}{bN(N-1)}\left(\sum_{n=1}^Nz_n\right)^2+\frac{N-b}{bN(N-1)}\sum_{n=1}^Nz_n^2.
\label{eq:app-rr-adamproxy-batch-second-moment}
\end{equation}
Consequently, with \(\mu_1(W):=\mathbb E_B[G_B(W)]\) and \(\mu_2(W):=\mathbb E_B[G_B(W)^{\odot2}]\),
\begin{equation}
\mu_1(W)=\frac1N S(W),\qquad \mu_2(W)=\frac{(b-1)S(W)^{\odot2}+(N-b)Q(W)}{bN(N-1)},
\label{eq:app-rr-adamproxy-batch-moments-matrix}
\end{equation}
and therefore
\begin{equation}
\mu_1(W)\oslash\mu_2(W)^{\odot1/2}=\Phi_{N,b}(W).
\label{eq:app-rr-adamproxy-batch-law-direction}
\end{equation}
Moreover, because \(1\le b<N\), \(\sigma\) is a norm and
\begin{equation}
\underline c_{N,b}\|z\|_2\le\sigma(z)\le\frac1{\sqrt b}\|z\|_2,\qquad \underline c_{N,b}:=\sqrt{\frac{N-b}{bN(N-1)}},
\label{eq:app-rr-adamproxy-sigma-equivalence}
\end{equation}
while every \(b\)-subset satisfies
\begin{equation}
|T_Bz|\le K_{N,b}\sigma(z),\qquad \frac1b\sum_{n\in B}|z_n|\le K_{N,b}\sigma(z),\qquad K_{N,b}:=\sqrt{\frac{N(N-1)}{N-b}}.
\label{eq:app-rr-adamproxy-relative-batch-bound}
\end{equation}
For a matrix coordinate \(\rho=(r,a)\), define \(z_n^\rho(W):=g_n(W)[\rho]\) and \(\sigma_\rho(W):=\sigma(z^\rho(W))\). Under the nonzero-coordinate assumption in Proposition~\ref{thm:rr-adamproxy-epoch-high-memory}, \(\sigma_\rho(W)>0\) for every finite \(W\) and every \(\rho\).
\end{lemma}

\begin{proof}
Let \(I_n:=\mathbf 1_{\{n\in B\}}\). Since \(B\) is a uniform \(b\)-subset, \(\mathbb E I_n=b/N\) and \(\mathbb E[I_nI_{n'}]=b(b-1)/(N(N-1))\) for \(n\ne n'\). Expanding \(\mathbb E[(T_Bz)^2]\) gives \eqref{eq:app-rr-adamproxy-batch-second-moment}, while \eqref{eq:app-rr-adamproxy-batch-first-moment} is immediate.
Applying the two identities coordinate-wise to \(g_n(W)\) gives \eqref{eq:app-rr-adamproxy-batch-moments-matrix}, and substituting these expressions into \(\mu_1(W)\oslash\mu_2(W)^{\odot1/2}\) gives \eqref{eq:app-rr-adamproxy-batch-law-direction} by \eqref{eq:adamproxy-direction-explicit-clean}. The second term in \eqref{eq:app-rr-adamproxy-batch-second-moment} gives the lower bound in \eqref{eq:app-rr-adamproxy-sigma-equivalence}, while \(|T_Bz|\le b^{-1}\sum_{n\in B}|z_n|\le b^{-1/2}\|z\|_2\) gives the upper bound and \eqref{eq:app-rr-adamproxy-relative-batch-bound}.
Finally, at finite \(W\), every softmax probability is strictly positive. If \(r=y_n\), then \(g_n(W)[r,a]=-(1-p_{n,y_n}(W))x_n[a]\ne0\), while if \(r\ne y_n\), then \(g_n(W)[r,a]=p_{n,r}(W)x_n[a]\ne0\). Hence \(z^\rho(W)\ne0\), so \(\sigma_\rho(W)>0\).
\end{proof}

\begin{lemma}[Adam-step and signal stability]
\label{lem:app-rr-adamproxy-step-stability}
Let \(0<\beta_1\le\beta_2<1\) and define
\begin{equation}
A(\beta_1,\beta_2)^2:=\frac{(1-\beta_1)^2}{(1-\beta_2)(1-\beta_1^2/\beta_2)}.
\label{eq:app-rr-adamproxy-ratio-constant}
\end{equation}
Then \(\beta_1^2<\beta_2\), and for \(U_t:=M_{t+1}\oslash V_{t+1}^{\odot1/2}\),
\begin{equation}
\|U_t\|_{\max}\le A(\beta_1,\beta_2),\qquad \|W_{t+1}-W_t\|_{\mathrm F}\le\sqrt{kd}\,A(\beta_1,\beta_2)\eta_t.
\label{eq:app-rr-adamproxy-step-bound}
\end{equation}
There are finite data-dependent constants \(L,C\) such that, for all \(U,W\), every batch \(B\), and every coordinate \(\rho\),
\begin{equation}
\frac{|G_B(U)[\rho]-G_B(W)[\rho]|}{\sigma_\rho(W)}\le C\left(e^{L\|U-W\|_{\mathrm F}}-1\right),
\label{eq:app-rr-adamproxy-batch-stability}
\end{equation}
\begin{equation}
\frac{|G_B(U)[\rho]^2-G_B(W)[\rho]^2|}{\sigma_\rho(W)^2}\le C\left(e^{2L\|U-W\|_{\mathrm F}}-1\right).
\label{eq:app-rr-adamproxy-square-stability}
\end{equation}
\end{lemma}

\begin{proof}
Fix one coordinate and write \(Y_u:=G_{B_u}(W_u)[\rho]\). Zero initialization gives \(M_{t+1}=(1-\beta_1)\sum_{j=0}^t\beta_1^jY_{t-j}\) and \(V_{t+1}=(1-\beta_2)\sum_{j=0}^t\beta_2^jY_{t-j}^2\). If \(V_{t+1}=0\), then \(M_{t+1}=0\); otherwise Cauchy--Schwarz gives \(M_{t+1}^2\le A(\beta_1,\beta_2)^2V_{t+1}\), proving \eqref{eq:app-rr-adamproxy-step-bound}. For the stability bounds, let \(\Delta:=U-W\). The logit perturbations satisfy \(|e_c^\top\Delta x_n|\le\|\Delta\|_{\mathrm F}\|x_n\|_2\), so every softmax probability ratio, and also the ratio of \(1-p_{n,y_n}\), lies between \(e^{-L\|\Delta\|_{\mathrm F}}\) and \(e^{L\|\Delta\|_{\mathrm F}}\) after enlarging \(L\). Each coordinate of \(g_n\) is a fixed feature coordinate multiplied by one of these quantities, giving \(|g_n(U)[\rho]-g_n(W)[\rho]|\le(e^{L\|\Delta\|_{\mathrm F}}-1)|g_n(W)[\rho]|\). Summing over \(n\in B\) and applying \eqref{eq:app-rr-adamproxy-relative-batch-bound} proves \eqref{eq:app-rr-adamproxy-batch-stability}; the identity \(|a^2-b^2|\le|a-b|(|a|+|b|)\) then gives \eqref{eq:app-rr-adamproxy-square-stability}.
\end{proof}

For the next lemma only, extend the reshuffling process to all integer epochs by taking independent uniform ordered partitions for every epoch index in \(\mathbb Z\). For \(z\ne{\mathbf 0}\), define \(Y_t(z):=T_{B_t}z/\sigma(z)\).

\begin{lemma}[Uniform concentration of frozen exponential averages]
\label{lem:app-rr-adamproxy-ema-concentration}
For every \(\delta>0\), there exists \(\omega_\delta:(0,1)\to[0,1]\) with \(\omega_\delta(\beta)\to0\) as \(\beta\uparrow1\) such that, for every epoch start \(\tau=rm\),
\begin{equation}
\mathbb P\left(\max_{0\le s<m}\sup_{z\ne{\mathbf 0}}\left|(1-\beta)\sum_{j=0}^\infty\beta^j\left(Y_{\tau+s-j}(z)-\mathbb EY_0(z)\right)\right|>\delta\right)\le\omega_\delta(\beta).
\label{eq:app-rr-adamproxy-ema-first}
\end{equation}
The same statement holds with \(Y_t(z)^2\) in place of \(Y_t(z)\) and \(\mathbb E[Y_0(z)^2]=1\). For fixed \(N,b\), one may take \(\omega_\delta(\beta)\le C_\delta(1-\beta)\).
\end{lemma}

\begin{proof}
It is enough to work on the compact \(\sigma\)-unit sphere. Group the lags in the exponential average by reshuffling epoch. Contributions from distinct epochs are independent, every epoch block has mean zero, and Lemma~\ref{lem:app-rr-adamproxy-batch-package} uniformly bounds \(|Y_t(z)|\). Hence the variance of the centered exponential average is at most \(C(1-\beta)^2\sum_{q\ge0}\beta^{2qm}\le C'(1-\beta)\). Chebyshev's inequality gives the stated rate for each fixed \(z\). The maps \(z\mapsto T_Bz\) are uniformly Lipschitz in the \(\sigma\)-norm, so a finite \(\delta\)-net and a union bound make the estimate uniform over \(z\) and over \(s=0,\ldots,m-1\). The argument for \(Y_t(z)^2\) is identical because \(Y_t(z)\) is uniformly bounded.
\end{proof}

For fixed \(W\), define the two-sided frozen memories
\begin{equation}
\overline M_t^{\beta_1}(W):=(1-\beta_1)\sum_{j=0}^\infty\beta_1^jG_{B_{t-j}}(W),\qquad \overline V_t^{\beta_2}(W):=(1-\beta_2)\sum_{j=0}^\infty\beta_2^jG_{B_{t-j}}(W)^{\odot2}.
\label{eq:app-rr-adamproxy-frozen-memories}
\end{equation}
Their expectations are \(\mu_1(W)\) and \(\mu_2(W)\), respectively.

\begin{lemma}[Freezing one epoch]
\label{lem:app-rr-adamproxy-freezing}
Fix \(0<\beta_1\le\beta_2<1\). Under the assumptions of Proposition~\ref{thm:rr-adamproxy-epoch-high-memory}, for every coordinate \(\rho\), uniformly over \(0\le s<m\),
\begin{equation}
\frac{|M_{\tau_r+s+1}[\rho]-\overline M_{\tau_r+s}^{\beta_1}(W_{\tau_r})[\rho]|}{\sigma_\rho(W_{\tau_r})}\xrightarrow[r\to\infty]{\mathbb P}0,
\label{eq:app-rr-adamproxy-freezing-first}
\end{equation}
\begin{equation}
\frac{|V_{\tau_r+s+1}[\rho]-\overline V_{\tau_r+s}^{\beta_2}(W_{\tau_r})[\rho]|}{\sigma_\rho(W_{\tau_r})^2}\xrightarrow[r\to\infty]{\mathbb P}0.
\label{eq:app-rr-adamproxy-freezing-second}
\end{equation}
The convergence is uniform over the finitely many coordinates.
\end{lemma}

\begin{proof}
Fix \(t=\tau_r+s\). The first-moment difference consists of the missing negative-time tail and movement errors obtained by replacing \(W_{t-j}\) by \(W_{\tau_r}\). Lemma~\ref{lem:app-rr-adamproxy-batch-package} bounds the normalized missing tail by \(C\beta_1^{t+1}\). For \(0\le j\le s\), Lemma~\ref{lem:app-rr-adamproxy-step-stability} and monotonicity of \(\eta_t\) give \(\|W_{t-j}-W_{\tau_r}\|_{\mathrm F}\le C_{\beta_1,\beta_2}m\eta_{\tau_r}\to0\). For \(j>s\), writing \(q=j-s\) gives \(\|W_{\tau_r-q}-W_{\tau_r}\|_{\mathrm F}\le C_{\beta_1,\beta_2}\sum_{u=\tau_r-q}^{\tau_r-1}\eta_u\). Therefore \eqref{eq:app-rr-adamproxy-batch-stability} bounds the normalized past-lag contribution by
\[
C(1-\beta_1)\sum_{q=1}^{\tau_r}\beta_1^{q+s}\left(\exp\left(C\sum_{u=\tau_r-q}^{\tau_r-1}\eta_u\right)-1\right),
\]
which tends to zero by Assumption~\ref{ass:lr-momentum}. This proves \eqref{eq:app-rr-adamproxy-freezing-first}. The second-moment argument is identical using \eqref{eq:app-rr-adamproxy-square-stability} and \(\beta_2\), proving \eqref{eq:app-rr-adamproxy-freezing-second}.
\end{proof}

\begin{lemma}[Continuity of the Adam ratio]
\label{lem:app-rr-adamproxy-ratio-continuity}
There is a universal \(C_{\mathrm{rat}}<\infty\) such that, whenever \(\sigma>0\), \(|a_0|\le\sigma\), \(|a-a_0|\le\delta\sigma\), and \(|v-\sigma^2|\le\delta\sigma^2\) with \(0<\delta\le1/2\),
\begin{equation}
\left|\frac a{\sqrt v}-\frac{a_0}{\sigma}\right|\le C_{\mathrm{rat}}\delta.
\label{eq:app-rr-adamproxy-ratio-continuity}
\end{equation}
\end{lemma}

\begin{proof}
Since \(v\ge\sigma^2/2\),
\[
\left|\frac a{\sqrt v}-\frac{a_0}{\sigma}\right|\le\frac{|a-a_0|}{\sqrt v}+|a_0|\left|\frac1{\sqrt v}-\frac1\sigma\right|\le\sqrt2\,\delta+\left|(1+e)^{-1/2}-1\right|,
\]
where \(v=\sigma^2(1+e)\) and \(|e|\le\delta\). The map \(x\mapsto(1+x)^{-1/2}\) is Lipschitz on \([-1/2,1/2]\).
\end{proof}

\begin{proof}[Proof of Proposition~\ref{thm:rr-adamproxy-epoch-high-memory}]
Let \(f_{N,b}(W):=\mu_1(W)\oslash\mu_2(W)^{\odot1/2}=\Phi_{N,b}(W)\) by Lemma~\ref{lem:app-rr-adamproxy-batch-package}. Because the Adam update uses the gradient convention,
\begin{equation}
\frac{W_{\tau_r+m}-W_{\tau_r}}{\gamma_r}=-\sum_{s=0}^{m-1}\frac{\eta_{\tau_r+s}}{\gamma_r}U_{\tau_r+s},\qquad \sum_{s=0}^{m-1}\frac{\eta_{\tau_r+s}}{\gamma_r}=1.
\label{eq:app-rr-adamproxy-epoch-displacement}
\end{equation}
Fix \(0<\delta\le1/4\). Lemma~\ref{lem:app-rr-adamproxy-ema-concentration} gives, with probability at least \(1-\omega_\delta(\beta_1)-\omega_\delta(\beta_2)\), relative \(\delta\)-control of both frozen memories around \(\mu_1(W_{\tau_r})\) and \(\mu_2(W_{\tau_r})\), simultaneously over the epoch and all coordinates. Lemma~\ref{lem:app-rr-adamproxy-freezing} transfers these bounds to the actual memories with \(2\delta\) in place of \(\delta\), up to an event whose probability vanishes as \(r\to\infty\). Jensen's inequality gives \(|\mu_1(W)[\rho]|\le\sqrt{\mu_2(W)[\rho]}=\sigma_\rho(W)\), so Lemma~\ref{lem:app-rr-adamproxy-ratio-continuity} yields
\begin{equation}
\max_{0\le s<m}\|U_{\tau_r+s}-\Phi_{N,b}(W_{\tau_r})\|_{\max}\le C\delta
\label{eq:app-rr-adamproxy-stepwise-proxy-bound}
\end{equation}
with the same asymptotic probability. Combining \eqref{eq:app-rr-adamproxy-epoch-displacement} and \eqref{eq:app-rr-adamproxy-stepwise-proxy-bound} gives
\begin{equation}
\left\|\frac{W_{\tau_r+m}-W_{\tau_r}}{\gamma_r}+\Phi_{N,b}(W_{\tau_r})\right\|_{\max}\le C\delta.
\label{eq:app-rr-adamproxy-epoch-proxy-bound}
\end{equation}
Given \(\varepsilon>0\), choose \(\delta\) so that \(C\delta<\varepsilon\). Taking \(r\to\infty\) first yields a limsup probability bounded by \(\omega_\delta(\beta_1)+\omega_\delta(\beta_2)\), and then taking \((\beta_1,\beta_2)\to(1,1)\) through \(0<\beta_1\le\beta_2<1\) proves \eqref{eq:rr-adamproxy-one-epoch-limit-clean}.
\end{proof}

\subsection{RR-RowSGDM versus AdamProxy population accuracy}
\label{app:adamproxy-row-accuracy-proof}

We prove Proposition~\ref{thm:adamproxy-optimizer-accuracy-order-clean} by bounding the coordinates of the margin-normalized AdamProxy iterates and comparing the resulting classifier family with the near-row-margin family.
All training-time limits are taken at fixed dimension before the limit \(d\to\infty\).

\paragraph{A deterministic trajectory bound.}
Fix a finite linearly separable dataset \(X=\{(x_n,y_n)\}_{n=1}^N\), with \(x_n\in\mathbb R^d\) and \(y_n\in[k]\).
Write \(p_n(W)=\mathbb S(Wx_n)\), \(\ell_n(W)=-\log p_{n,y_n}(W)\), and \(L(W)=N^{-1}\sum_{n=1}^N\ell_n(W)\).
For \(\mathcal Q:=\{(n,c):n\in[N],\ c\ne y_n\}\), set \(A_{n,c}:=(e_{y_n}-e_c)x_n^\top\), so that
\[
\operatorname{mar}_X(W)=\min_{(n,c)\in\mathcal Q}\langle W,A_{n,c}\rangle,\qquad \gamma_{\max}(X)=\max_{\|V\|_{\max}\le1}\operatorname{mar}_X(V)>0.
\]
Let \(R_1(X):=\max_n\|x_n\|_1\) and \(r_{\max}^*(X)=\gamma_{\max}(X)^{-1}\).
For \(\kappa\ge0\), define
\begin{align}
h_n(W)&:=(e_{y_n}-p_n(W))x_n^\top=\sum_{c\ne y_n}p_{n,c}(W)A_{n,c},\nonumber\\
\widetilde S(W)&:=\sum_{n=1}^Nh_n(W),\qquad \widetilde Q(W):=\sum_{n=1}^Nh_n(W)^{\odot2},\nonumber\\
\Psi_\kappa(W)&:=\widetilde S(W)\oslash\left(\widetilde S(W)^{\odot2}+\kappa\widetilde Q(W)\right)^{\odot1/2}.
\label{eq:app-adamproxy-trajectory-map}
\end{align}
All powers and divisions are entrywise, with \(0/0:=0\).

\begin{proposition}[Positive margin growth and a trajectory-level coordinate bound]
\label{prop:app-adamproxy-coordinate-cap}
Let \(\kappa\ge0\), let \(W_0\in\mathbb R^{k\times d}\) be finite, and consider \(W_{t+1}=W_t+a_t\Psi_\kappa(W_t)\), where \(a_t>0\), \(a_t\to0\), and \(\sum_{t=0}^\infty a_t=\infty\).
Define
\[
T_t:=\sum_{s=0}^{t-1}a_s,\qquad K_\kappa(X):=2(1+\sqrt\kappa)R_1(X),\qquad \rho_\kappa(X):=\frac{\gamma_{\max}(X)^2}{K_\kappa(X)}>0.
\]
Then \(L(W_t)\to0\), \(\operatorname{mar}_X(W_t)\to\infty\), and
\begin{align}
\liminf_{t\to\infty}\frac{\operatorname{mar}_X(W_t)}{T_t}&\ge\rho_\kappa(X),
\label{eq:app-adamproxy-linear-margin-growth}\\
\liminf_{t\to\infty}\frac{\operatorname{mar}_X(W_t)}{\|W_t\|_{\max}}&\ge\rho_\kappa(X),
\label{eq:app-adamproxy-positive-normalized-margin}\\
\limsup_{t\to\infty}\left\|\frac{W_t}{\operatorname{mar}_X(W_t)}\right\|_{\max}&\le K_\kappa(X)r_{\max}^*(X)^2.
\label{eq:app-adamproxy-coordinate-cap}
\end{align}
The margin-normalized matrices in \eqref{eq:app-adamproxy-coordinate-cap} are well defined and have minimum training margin one for all sufficiently large \(t\).
\end{proposition}

\begin{proof}
For finite \(W\), define
\[
\mathcal G(W):=\frac1N\sum_{n=1}^N(1-p_{n,y_n}(W)),\qquad R(W):=\sum_{(n,c)\in\mathcal Q}p_{n,c}(W)=N\mathcal G(W)>0,
\]
and let \(\theta_{n,c}(W):=p_{n,c}(W)/R(W)\).
Then \(\theta(W)\) belongs to the probability simplex \(\Delta(\mathcal Q)\).
For any \(\theta\in\Delta(\mathcal Q)\), write
\begin{align*}
Z_n(\theta)&:=\sum_{c\ne y_n}\theta_{n,c}A_{n,c},\qquad \mathsf S(\theta):=\sum_{n=1}^NZ_n(\theta),\\
\mathsf D_\kappa(\theta)&:=\left(\mathsf S(\theta)^{\odot2}+\kappa\sum_{n=1}^NZ_n(\theta)^{\odot2}\right)^{\odot1/2}.
\end{align*}
Since \(h_n(W)=R(W)Z_n(\theta(W))\), the common factor \(R(W)\) cancels in the proxy ratio, giving
\begin{equation}
\Psi_\kappa(W)=\mathsf S(\theta(W))\oslash\mathsf D_\kappa(\theta(W)),\qquad \|\Psi_\kappa(W)\|_{\max}\le1.
\label{eq:app-adamproxy-trajectory-step-bound}
\end{equation}
The coordinate bound follows from \(\mathsf D_\kappa(\theta)[\rho]\ge|\mathsf S(\theta)[\rho]|\), including zero denominators under the stated convention.
The entrywise inequalities \(|\mathsf S(\theta)|\le\sum_n|Z_n(\theta)|\) and \((\sum_nZ_n(\theta)^{\odot2})^{\odot1/2}\le\sum_n|Z_n(\theta)|\) imply
\begin{equation}
\|\mathsf D_\kappa(\theta)\|_{\mathrm{sum}}\le(1+\sqrt\kappa)\sum_n\|Z_n(\theta)\|_{\mathrm{sum}}\le2(1+\sqrt\kappa)R_1(X)=K_\kappa(X).
\label{eq:app-adamproxy-trajectory-metric-mass}
\end{equation}
Here the last inequality uses \(\|A_{n,c}\|_{\mathrm{sum}}=2\|x_n\|_1\) and \(\sum_{(n,c)\in\mathcal Q}\theta_{n,c}=1\).
Let \(V_*\) attain \(\gamma_{\max}(X)\) on the compact max-norm unit ball.
Max--sum duality and the margin constraints for \(V_*\) give
\[
\|\mathsf S(\theta)\|_{\mathrm{sum}}\ge\langle V_*,\mathsf S(\theta)\rangle=\sum_{(n,c)\in\mathcal Q}\theta_{n,c}\langle V_*,A_{n,c}\rangle\ge\gamma_{\max}(X).
\]
If \(\mathsf D_\kappa(\theta)[\rho]=0\), then \(\mathsf S(\theta)[\rho]=0\).
Weighted Cauchy--Schwarz on the remaining coordinates therefore yields
\begin{equation}
\mathcal J_\kappa(\theta):=\sum_{\rho:\mathsf D_\kappa(\theta)[\rho]>0}\frac{\mathsf S(\theta)[\rho]^2}{\mathsf D_\kappa(\theta)[\rho]}\ge\frac{\|\mathsf S(\theta)\|_{\mathrm{sum}}^2}{\|\mathsf D_\kappa(\theta)\|_{\mathrm{sum}}}\ge\rho_\kappa(X).
\label{eq:app-adamproxy-uniform-dissipation}
\end{equation}
Since \(\nabla L(W)=-\mathcal G(W)\mathsf S(\theta(W))\), it follows that
\begin{equation}
-\langle\nabla L(W),\Psi_\kappa(W)\rangle=\mathcal G(W)\mathcal J_\kappa(\theta(W))\ge\rho_\kappa(X)\mathcal G(W).
\label{eq:app-adamproxy-uniform-descent}
\end{equation}

To control the discrete-time remainder, fix \(\|\Delta\|_{\max}\le1\).
Then \(\|\Delta x_n\|_\infty\le R_1(X)\), and the softmax Hessian satisfies
\[
D^2\ell_n(V)[\Delta,\Delta]=\operatorname{Var}_{C\sim p_n(V)}\bigl((\Delta x_n)_C\bigr)\le4R_1(X)^2\bigl(1-p_{n,y_n}(V)\bigr).
\]
Indeed, the variance is at most the second moment about \((\Delta x_n)_{y_n}\), and every logit difference has magnitude at most \(2R_1(X)\).
The softmax formula also gives \(\mathcal G(W+sa\Delta)\le e^{2R_1(X)a}\mathcal G(W)\) for \(a\ge0\) and \(s\in[0,1]\).
Because \(a_t\to0\), the quantity \(a_{\max}:=\sup_ta_t\) is finite.
Taylor's formula, \eqref{eq:app-adamproxy-trajectory-step-bound}, and \eqref{eq:app-adamproxy-uniform-descent} consequently imply
\begin{equation}
L(W_{t+1})\le L(W_t)-a_t\bigl(\rho_\kappa(X)-C_Xa_t\bigr)\mathcal G(W_t),\qquad C_X:=2R_1(X)^2e^{2R_1(X)a_{\max}}.
\label{eq:app-adamproxy-trajectory-loss-descent}
\end{equation}
For all sufficiently large \(t\), the coefficient in parentheses is at least \(\rho_\kappa(X)/2\), so \(L(W_t)\) is eventually nonincreasing and has a nonnegative limit.
If this limit were \(L_\infty>0\), then at each sufficiently large \(t\) some sample would satisfy \(\ell_n(W_t)\ge L(W_t)\ge L_\infty\), and hence \(\mathcal G(W_t)\ge(1-e^{-L_\infty})/N>0\).
Summing \eqref{eq:app-adamproxy-trajectory-loss-descent} would contradict \(L\ge0\) and \(\sum_ta_t=\infty\).
Thus \(L(W_t)\to0\).
Since \(0\le\ell_n(W_t)\le NL(W_t)\), every sample loss tends to zero, and
\[
\frac{\mathcal G(W_t)}{L(W_t)}=\frac{\sum_n(1-e^{-\ell_n(W_t)})}{\sum_n\ell_n(W_t)}\longrightarrow1.
\]
Moreover,
\[
\ell_n(W)=\log\left(1+\sum_{c\ne y_n}e^{-\langle W,A_{n,c}\rangle}\right)
\]
shows that every training margin tends to infinity.
For all sufficiently large \(t\), dividing \eqref{eq:app-adamproxy-trajectory-loss-descent} by \(L(W_t)\) and applying \(-\log(1-z)\ge z\) gives
\[
-\log L(W_{t+1})+\log L(W_t)\ge a_t\bigl(\rho_\kappa(X)-C_Xa_t\bigr)\frac{\mathcal G(W_t)}{L(W_t)}.
\]
The factor multiplying \(a_t\) converges to \(\rho_\kappa(X)\), so summation and weighted Ces\`aro averaging yield
\[
\liminf_{t\to\infty}\frac{-\log L(W_t)}{T_t}\ge\rho_\kappa(X).
\]
Write \(m_t:=\operatorname{mar}_X(W_t)\).
Once \(m_t\ge0\), a sample attaining the minimum margin gives
\[
L(W_t)\ge\frac1N\log(1+e^{-m_t})\ge\frac{e^{-m_t}}{2N},\qquad m_t\ge-\log L(W_t)-\log(2N).
\]
This proves \eqref{eq:app-adamproxy-linear-margin-growth}.
Finally, \eqref{eq:app-adamproxy-trajectory-step-bound} gives \(\|W_t\|_{\max}\le\|W_0\|_{\max}+T_t\).
Combining this bound with \eqref{eq:app-adamproxy-linear-margin-growth} proves \eqref{eq:app-adamproxy-positive-normalized-margin} and \eqref{eq:app-adamproxy-coordinate-cap}.
\end{proof}

For \(2\le b\le N\), the explicit AdamProxy formula gives
\[
-\Phi_{N,b}(W)=c_{N,b}^{\mathrm{AP}}\Psi_{\kappa_{N,b}}(W),\qquad c_{N,b}^{\mathrm{AP}}:=\sqrt{\frac{b(N-1)}{N(b-1)}},\qquad \kappa_{N,b}:=\frac{N-b}{b-1}.
\]
Thus \eqref{eq:adamproxy-dynamics-clean} has the recursion in Proposition~\ref{prop:app-adamproxy-coordinate-cap} with \(a_t=c_{N,b}^{\mathrm{AP}}\gamma_t\) and \(\gamma_t=\sum_{s=0}^{m-1}\eta_{mt+s}\).
At each fixed dimension, Assumption~\ref{ass:lr-base} implies \(a_t\to0\) and \(\sum_ta_t=\infty\).
The full-batch case \(b=N\) corresponds to \(\kappa_{N,b}=0\), \(c_{N,b}^{\mathrm{AP}}=1\), and \(\Psi_0=\operatorname{sign}(\widetilde S)\).

\paragraph{Gaussian-cloud normalization and coordinate caps.}
We use the normalized convention \eqref{eq:app-normalized-data}, under which \(\mu_i,z_{i,s}\sim\mathcal N({\mathbf 0},I_d/d)\) and \(X^{\mathrm{te}}\mid(Y=i)\sim\mathcal N(\mu_i,r_{\mathrm{te}}^2I_d/d)\).
This convention also preserves the relevant trajectory estimates.
Indeed, if \(X^{\mathrm{raw}}=\sqrt d\,X\) and \(V=\sqrt d\,W\), then every normalized-data sample gradient at \(V\) equals \(d^{-1/2}\) times the corresponding raw-data gradient at \(W\), so \(\Psi_{\kappa,X}(V)=\Psi_{\kappa,X^{\mathrm{raw}}}(W)\).
Consequently, \(V_t=\sqrt d\,W_t\) follows the normalized-data recursion with step sizes \(\sqrt d\,a_t\) and identical predictions.
For RR-RowSGDM, scaling its momentum by \(d^{-1/2}\) gives the same row-normalized update direction, so its transformed parameter trajectory uses step sizes \(\sqrt d\,\eta_t\).
Multiplication by this fixed positive constant preserves Assumptions~\ref{ass:lr-base} and~\ref{ass:lr-momentum} at each fixed \(d\).
We suppress the rescaling in the notation below.

Under the assumptions of Proposition~\ref{thm:adamproxy-optimizer-accuracy-order-clean}, \(b_d\ge N(d)/m_0\to\infty\), so \(b_d\ge2\) for all sufficiently large \(d\).
Moreover,
\begin{equation}
0\le\kappa_{N,b_d}=\frac{b_d(m_d-1)}{b_d-1}\le2(m_0-1)=:\overline\kappa.
\label{eq:app-adamproxy-kappa-bound}
\end{equation}
Set
\begin{equation}
R_*:=2(1+2r_0),\qquad C_{\mathrm{AP}}:=\pi R_*(1+\sqrt{\overline\kappa}),\qquad a_{\mathrm{AP}}:=\max\{3,2C_{\mathrm{AP}}+1\},
\label{eq:app-adamproxy-AP-cap-constants}
\end{equation}
and define the capped feasible family
\begin{equation}
\mathcal C_d:=\left\{V\in\mathcal F_X:\|V\|_{\max}\le\frac{C_{\mathrm{AP}}+1/2}{\sqrt d}\right\}.
\label{eq:app-adamproxy-capped-family}
\end{equation}

For the functions \(A_{\max}\) and \(B_{\max}\) in \eqref{eq:app-det-max-population-functions}, and for \(a\ge3\), let \(\tau_a\in(0,1)\) be the unique solution of
\begin{equation}
A_{\max}(\tau_a)=\frac{\sqrt2}{a}.
\label{eq:app-adamproxy-tau-a}
\end{equation}
Such a solution exists because \(A_{\max}(0)=0\), \(A_{\max}\) is strictly increasing, and \(A_{\max}(1)>1/\sqrt\pi>\sqrt2/a\).
Define
\begin{equation}
C_{\mathrm{clip}}(a):=\frac{\sqrt2}{a\sqrt{B_{\max}(\tau_a)}},\qquad \delta_{\mathrm{clip}}(a):=\sqrt{2-2C_{\mathrm{clip}}(a)},\qquad \delta_{\mathrm{AP}}:=\delta_{\mathrm{clip}}(a_{\mathrm{AP}}).
\label{eq:app-adamproxy-clipped-constants}
\end{equation}
Strict Cauchy--Schwarz gives \(A_{\max}(\tau)^2<B_{\max}(\tau)\) for every finite \(\tau>0\), so \(0<C_{\mathrm{clip}}(a)<1\) and \(\delta_{\mathrm{AP}}>0\).

\begin{lemma}[Uniform clipped-Gaussian angular barrier]
\label{lem:app-adamproxy-clipped-barrier}
Let \(u_{d,\xi}\in\mathbb S^{d-1}\) be indexed by a nonempty finite set \(\mathcal I_d\), and suppose the empirical functions in \eqref{eq:app-det-HJ} satisfy
\[
\sup_{\xi\in\mathcal I_d}\sup_{\tau\in[0,1]}\left(\left|H_{u_{d,\xi}}(\tau)-A_{\max}(\tau)\right|+\left|J_{u_{d,\xi}}(\tau)-B_{\max}(\tau)\right|\right)\longrightarrow0.
\]
Fix \(a\ge3\), let \(a_d\to a\) and \(\upsilon_d\to\sqrt2\) be positive deterministic sequences, and define
\[
\mathcal V_{d,\xi}:=\left\{v\in\mathbb R^d:\|v\|_\infty\le\frac{a_d}{\sqrt d},\quad v^\top u_{d,\xi}\ge\upsilon_d\right\}.
\]
For all sufficiently large \(d\), these sets are nonempty and
\begin{equation}
\inf_{\xi\in\mathcal I_d}\inf_{v\in\mathcal V_{d,\xi}}\left\|\frac{v}{\|v\|_2}-u_{d,\xi}\right\|_2\ge\delta_{\mathrm{clip}}(a)-o_d(1).
\label{eq:app-adamproxy-clipped-conclusion}
\end{equation}
The error depends only on the uniform empirical errors and the deviations of \(a_d\) and \(\upsilon_d\) from their limits.
\end{lemma}

\begin{proof}
Since \(\upsilon_d/a_d\to\sqrt2/a<A_{\max}(1)\), uniform convergence gives \(0<\upsilon_d/a_d<H_{u_{d,\xi}}(1)\) for every \(\xi\) and all sufficiently large \(d\).
The feasibility condition in Proposition~\ref{prop:app-det-clipped-exact} therefore holds, and that proposition gives
\[
\mathfrak m_{d,\xi}:=\min_{v\in\mathcal V_{d,\xi}}\|v\|_2=a_d\sqrt{J_{u_{d,\xi}}(\widehat\tau_{d,\xi})},\qquad H_{u_{d,\xi}}(\widehat\tau_{d,\xi})=\frac{\upsilon_d}{a_d},
\]
with \(\widehat\tau_{d,\xi}\in(0,1)\).
Uniform convergence and strict monotonicity of \(A_{\max}\) imply
\[
\sup_{\xi\in\mathcal I_d}|\widehat\tau_{d,\xi}-\tau_a|\longrightarrow0,\qquad \sup_{\xi\in\mathcal I_d}\left|\mathfrak m_{d,\xi}-a\sqrt{B_{\max}(\tau_a)}\right|\longrightarrow0.
\]
For every \(v\in\mathcal V_{d,\xi}\), the vector \(y=(\upsilon_d/(v^\top u_{d,\xi}))v\) also belongs to \(\mathcal V_{d,\xi}\).
Hence
\[
\frac{v^\top u_{d,\xi}}{\|v\|_2}\le\frac{\upsilon_d}{\mathfrak m_{d,\xi}}=C_{\mathrm{clip}}(a)+o_d(1),
\]
where the error is independent of \(v\) and uniform over \(\xi\).
The identity \(\|v/\|v\|_2-u\|_2^2=2-2v^\top u/\|v\|_2\) proves \eqref{eq:app-adamproxy-clipped-conclusion}.
\end{proof}

\begin{proposition}[Capped AdamProxy trajectories and boundary separation]
\label{prop:app-adamproxy-boundary-separation}
For every sufficiently large \(d\), on the event \(\mathcal E_d^{\mathrm{near}}\) from Lemma~\ref{lem:app-stoch-combined-input}, the AdamProxy trajectory satisfies
\begin{equation}
\frac{W_t^{\mathrm{AP}}}{\operatorname{mar}_X(W_t^{\mathrm{AP}})}\in\mathcal C_d\qquad\text{for all sufficiently large }t,
\label{eq:app-adamproxy-eventual-capped-membership}
\end{equation}
and
\begin{equation}
\inf_{V\in\mathcal C_d}\min_{i\ne j}\|u_{ij}^V-u_{ij}^{\mu}\|_2\ge\delta_{\mathrm{AP}}-o_d(1).
\label{eq:app-adamproxy-boundary-separation}
\end{equation}
The error is deterministic on \(\mathcal E_d^{\mathrm{near}}\) and uniform over \(V\in\mathcal C_d\).
The time threshold in \eqref{eq:app-adamproxy-eventual-capped-membership} may depend on the dimension, dataset, initialization, and learning-rate schedule.
\end{proposition}

\begin{proof}
The event guarantees separability, and Proposition~\ref{prop:app-det-common-event}(H5) and \eqref{eq:app-stoch-data-radius-bounds} give
\[
r_{\max}^*(X)\le\frac{\sqrt{2\pi}+o_d(1)}{2\sqrt d},\qquad \max_n\|x_n\|_2\le R_*,\qquad R_1(X)\le\sqrt d\,R_*.
\]
Applying Proposition~\ref{prop:app-adamproxy-coordinate-cap} and \eqref{eq:app-adamproxy-kappa-bound}, we obtain
\begin{align*}
\limsup_{t\to\infty}\sqrt d\left\|\frac{W_t^{\mathrm{AP}}}{\operatorname{mar}_X(W_t^{\mathrm{AP}})}\right\|_{\max}&\le2(1+\sqrt{\overline\kappa})dR_*r_{\max}^*(X)^2\\
&\le C_{\mathrm{AP}}+o_d(1).
\end{align*}
For all sufficiently large \(d\), the right-hand side is less than \(C_{\mathrm{AP}}+1/4\).
The definition of the time limsup and the unit margin of the normalized iterates then give \eqref{eq:app-adamproxy-eventual-capped-membership}.
In particular, \(\mathcal C_d\) is nonempty.

Fix \(V\in\mathcal C_d\) and an ordered pair \(i\ne j\), and write \(v:=v_i-v_j\) and \(u:=u_{ij}^m\).
The coordinate cap gives
\[
\|v\|_\infty\le\frac{2C_{\mathrm{AP}}+1}{\sqrt d}\le\frac{a_{\mathrm{AP}}}{\sqrt d}.
\]
Averaging feasibility over the samples from classes \(i\) and \(j\) gives \(v^\top m_i\ge1\) and \(v^\top m_j\le-1\), and hence
\[
v^\top u\ge\frac{2}{\|m_i-m_j\|_2}\ge\sqrt2-o_d(1).
\]
The last error is deterministic and uniform over all pairs and all \(V\in\mathcal C_d\), by Proposition~\ref{prop:app-det-common-event}(H2).
Choose a positive deterministic sequence \(\upsilon_d\to\sqrt2\) below these projection lower bounds.
Proposition~\ref{prop:app-det-common-event}(H6) supplies the empirical estimates in Lemma~\ref{lem:app-adamproxy-clipped-barrier}, which applies with \(a_d=a_{\mathrm{AP}}\) and \(\mathcal I_d=\{(i,j):i\ne j\}\).
Thus \(\|u_{ij}^V-u_{ij}^m\|_2\ge\delta_{\mathrm{AP}}-o_d(1)\) uniformly over all pairs and all \(V\in\mathcal C_d\).
Finally, Proposition~\ref{prop:app-det-common-event}(H4) and the reverse triangle inequality prove \eqref{eq:app-adamproxy-boundary-separation}.
\end{proof}

\paragraph{High-momentum row geometry.}
Let \(C_{\mathrm{row}}\) be the constant in Proposition~\ref{prop:app-stoch-near-stability}, and let \(K_{\mathrm{row}}\) be the row constant in Proposition~\ref{prop:app-stoch-ratio-bounds}.
Choose \(h_{\mathrm{AP}}:=\delta_{\mathrm{AP}}^2/(32\sqrt2)\), so that
\begin{equation}
\mathsf a(h_{\mathrm{AP}})>\mathsf b(\delta_{\mathrm{AP}}/2),
\label{eq:app-adamproxy-amplitude-gap}
\end{equation}
where \(\mathsf a\) and \(\mathsf b\) are defined in \eqref{eq:app-acc-limiting-amplitudes}.
Set
\begin{equation}
b_{\mathrm{AP}}:=\min\left\{1,\left(\frac{h_{\mathrm{AP}}}{C_{\mathrm{row}}}\right)^4\right\},\qquad \vartheta_{\mathrm{AP}}:=\frac{b_{\mathrm{AP}}/2}{1+b_{\mathrm{AP}}/2},\qquad A_0:=m_0(m_0^2-1),
\label{eq:app-adamproxy-row-budget}
\end{equation}
and
\begin{equation}
\beta_{\mathrm{AP}}:=\begin{cases}\dfrac12,&A_0=0,\\[3pt]\displaystyle\max\left\{\dfrac12,1-\frac{\vartheta_{\mathrm{AP}}}{2A_0K_{\mathrm{row}}}\right\},&A_0>0.\end{cases}
\label{eq:app-adamproxy-beta-threshold}
\end{equation}
These constants depend only on \((\alpha,r_0,m_0)\), and \(\beta_{\mathrm{AP}}\in(0,1)\).

\begin{lemma}[High-momentum Row near-set control]
\label{lem:app-adamproxy-row-control}
For every fixed \(\beta\in[\beta_{\mathrm{AP}},1)\), on \(\mathcal E_d^{\mathrm{near}}\) and for every sufficiently large \(d\), the RR-RowSGDM effective margin is positive and
\begin{equation}
\sup_{U\in\mathcal U_{\mathrm{row},\beta,b_d}^X}\max_{i\ne j}\|u_{ij}^U-u_{ij}^{\mu}\|_2\le h_{\mathrm{AP}}.
\label{eq:app-adamproxy-row-control}
\end{equation}
\end{lemma}

\begin{proof}
The relative margin deficit is
\[
\Delta_{\mathrm{row},d}:=\frac{\varepsilon_{\mathrm{row},\beta,b_d}(X)}{\gamma_{\mathrm{row}}(X)}=2(1-\beta)m_d(m_d^2-1)\frac{R_{\mathrm{row}}(X)}{\gamma_{\mathrm{row}}(X)}.
\]
Proposition~\ref{prop:app-stoch-ratio-bounds} and \eqref{eq:app-adamproxy-beta-threshold} give \(\Delta_{\mathrm{row},d}\le2(1-\beta)A_0K_{\mathrm{row}}\le\vartheta_{\mathrm{AP}}<1\), including \(\Delta_{\mathrm{row},d}=0\) when \(m_d=1\).
Thus the effective margin is positive, and
\[
\frac{\Delta_{\mathrm{row},d}}{1-\Delta_{\mathrm{row},d}}\le\frac{b_{\mathrm{AP}}}{2}.
\]
Lemma~\ref{lem:app-stoch-near-reduction} maps every \(U\in\mathcal U_{\mathrm{row},\beta,b_d}^X\) to a unit-margin feasible representative with the same pairwise row directions and relative row-norm slack at most \(b_{\mathrm{AP}}/2\).
For all sufficiently large \(d\), \(\omega_d\le b_{\mathrm{AP}}/2\), so Proposition~\ref{prop:app-stoch-near-stability} gives the uniform bound \(C_{\mathrm{row}}b_{\mathrm{AP}}^{1/4}\le h_{\mathrm{AP}}\).
\end{proof}

\begin{proof}[Proof of Proposition~\ref{thm:adamproxy-optimizer-accuracy-order-clean}]
Set \(\beta_1:=\beta_{\mathrm{AP}}\), fix \(\beta\in[\beta_1,1)\), and let \(\mathcal R_d\) be the Gaussian radial event from Lemma~\ref{lem:app-acc-gaussian-radial-event}.
Define \(\mathcal H_d:=\mathcal E_d^{\mathrm{near}}\cap\mathcal R_d\).
By \eqref{eq:app-stoch-near-event-probability} and \eqref{eq:app-acc-radial-event-probability}, for every sufficiently large \(d\),
\[
\mathbb P(\mathcal H_d)\ge1-e^{-\sqrt d}.
\]
Work on \(\mathcal H_d\), and take
\[
\mathcal R_d^{\mathrm{clf}}:=\mathcal U_{\mathrm{row},\beta,b_d}^X,\qquad \mathcal C_d^{\mathrm{clf}}:=\mathcal C_d.
\]
Lemma~\ref{lem:app-adamproxy-row-control} and Proposition~\ref{prop:app-adamproxy-boundary-separation} give, for all sufficiently large \(d\),
\[
\sup_{U\in\mathcal R_d^{\mathrm{clf}}}\max_{i\ne j}\|u_{ij}^U-u_{ij}^{\mu}\|_2\le h_{\mathrm{AP}},\qquad \inf_{V\in\mathcal C_d}D_{\mathrm{avg\text{-}bdry}}(V)\ge\frac{\delta_{\mathrm{AP}}}{2}.
\]
The radial deviations tend to zero on \(\mathcal R_d\) under a deterministic envelope, and the limiting amplitude condition is \eqref{eq:app-adamproxy-amplitude-gap}.
Choose fixed constants \(0<b_*<a_*\) with \(\mathsf b(\delta_{\mathrm{AP}}/2)<b_*<a_*<\mathsf a(h_{\mathrm{AP}})\).
Proposition~\ref{prop:app-acc-finite-bounds} gives, for \(\sigma_d=r_{\mathrm{te}}/\sqrt d\) and all sufficiently large \(d\),
\begin{align*}
\sup_{U\in\mathcal U_{\mathrm{row},\beta,b_d}^X}\mathcal E_{\pi,d}(U)&\le(k-1)\overline\Phi(a_*/\sigma_d),\\
\inf_{V\in\mathcal C_d}\mathcal E_{\pi,d}(V)&\ge2\pi_{\min,d}\overline\Phi(b_*/\sigma_d).
\end{align*}
The Gaussian upper-tail inequality \(\overline\Phi(x)/\overline\Phi(y)\le e^{-(x^2-y^2)/2}\) for \(x\ge y\ge0\) therefore yields
\[
\frac{\sup_{U\in\mathcal U_{\mathrm{row},\beta,b_d}^X}\mathcal E_{\pi,d}(U)}{\inf_{V\in\mathcal C_d}\mathcal E_{\pi,d}(V)}\le\frac{k-1}{2\pi_{\min,d}}\exp\left(-\frac{(a_*^2-b_*^2)d}{2r_{\mathrm{te},0}^2}\right)\longrightarrow0,
\]
since \(k=O(d)\), \(\log(1/\pi_{\min,d})=o(d)\), and \(r_{\mathrm{te},0}:=\sup_dr_{\mathrm{te}}(d)<\infty\).
Consequently, for every sufficiently large \(d\) on \(\mathcal H_d\),
\begin{equation}
\sup_{U\in\mathcal U_{\mathrm{row},\beta,b_d}^X}\mathcal E_{\pi,d}(U)<\inf_{V\in\mathcal C_d}\mathcal E_{\pi,d}(V).
\label{eq:app-adamproxy-set-error-order}
\end{equation}
The positive Row effective margin and Corollary~\ref{cor:app-ga-margin-transfer} give
\[
d_{\cos}\!\left(W_t^{\mathrm{RR\text{-}RowSGDM}},\mathcal U_{\mathrm{row},\beta,b_d}^X\right)\longrightarrow0.
\]
Lemma~\ref{lem:app-acc-trajectory-continuity} hence implies
\[
\limsup_{t\to\infty}\mathcal E_{\pi,d}\!\left(W_t^{\mathrm{RR\text{-}RowSGDM}}\right)\le\sup_{U\in\mathcal U_{\mathrm{row},\beta,b_d}^X}\mathcal E_{\pi,d}(U).
\]
For AdamProxy, \eqref{eq:app-adamproxy-eventual-capped-membership} and positive scale invariance of population error give
\[
\liminf_{t\to\infty}\mathcal E_{\pi,d}(W_t^{\mathrm{AP}})\ge\inf_{V\in\mathcal C_d}\mathcal E_{\pi,d}(V).
\]
Combining these two bounds with \eqref{eq:app-adamproxy-set-error-order} yields
\[
\limsup_{t\to\infty}\mathcal E_{\pi,d}\!\left(W_t^{\mathrm{RR\text{-}RowSGDM}}\right)<\liminf_{t\to\infty}\mathcal E_{\pi,d}(W_t^{\mathrm{AP}}).
\]
Since \(\operatorname{Acc}_{\pi,d}=1-\mathcal E_{\pi,d}\), this is the claimed accuracy ordering.
The Row momentum estimates hold for every reshuffling realization, and the AdamProxy estimates are deterministic conditional on the training data, so no additional exceptional event is required.
\end{proof}

\subsection{AdamProxy margin characterization}
\label{app:adamproxy-margin-proof}

AdamProxy is not steepest descent in a fixed norm, so we characterize its limit through a data-dependent quadratic margin problem. Let \(\mathcal Q:=\{(n,c):n\in[N],\,c\neq y_n\}\), let \(\Delta(\mathcal Q)\) be the probability simplex, and for \(\theta\in\Delta(\mathcal Q)\) define
\begin{align}
Z_n(\theta)&:=\sum_{c\neq y_n}\theta_{n,c}(e_{y_n}-e_c)x_n^\top,\nonumber\\
M_{N,b}(\theta)&:=\operatorname{diag}\!\left(\operatorname{vec}\!\left[\left(\frac{N}{b(N-1)}\left((b-1)\left(\sum_{n=1}^N Z_n(\theta)\right)^{\odot2}+(N-b)\sum_{n=1}^N Z_n(\theta)^{\odot2}\right)\right)^{\odot1/2}\right]\right).
\label{eq:adamproxy-metric-clean}
\end{align}
Extending the parametric margin formulation of \citet{baek2026implicit}, let
\begin{equation}
\mathcal P_{N,b}(\theta):=\operatorname*{arg\,min}_{W}\frac12\operatorname{vec}(W)^\top M_{N,b}(\theta)\operatorname{vec}(W)\quad\text{s.t.}\quad (w_{y_n}-w_c)^\top x_n\ge1,\ \ (n,c)\in\mathcal Q.
\label{eq:adamproxy-primal-clean}
\end{equation}

\begin{proposition}[AdamProxy margin characterization]
\label{thm:adamproxy-margin-clean}
Under Assumption~\ref{ass:lr-base}, let \(W_t\) follow~\eqref{eq:adamproxy-dynamics-clean} with \(1\le b\le N\). Assume linear separability, \(x_n[a]\neq0\) for all \(n,a\), and \(\Phi_{N,b}(W_t)/\|\Phi_{N,b}(W_t)\|_{\mathrm F}\to\overline\Phi\). Then \(W_t/\|W_t\|_{\mathrm F}\to\overline W:=-\overline\Phi\) and \(\gamma_{\mathrm{AP}}:=\operatorname{mar}_X(\overline W)>0\). For \(p^\star:=\overline W/\gamma_{\mathrm{AP}}\), there exist \(\theta^\star\in\Delta(\mathcal Q)\) and a nonzero dual optimizer \(\lambda^\star\ge0\) of the corresponding problem such that
\[
p^\star\in\mathcal P_{N,b}(\theta^\star),\qquad \theta^\star=\lambda^\star/\|\lambda^\star\|_1.
\]
\end{proposition}

Thus, the AdamProxy margin direction is characterized by \(\mathcal P_{N,b}(\theta^\star)\) at a self-consistent parameter \(\theta^\star\), which coincides with the normalized dual weights of the same variational problem.

\paragraph{Proof setup.}
The remainder of this subsection proves Proposition~\ref{thm:adamproxy-margin-clean}. We first rewrite the AdamProxy dynamics in terms of descent signals and normalized wrong-class coefficients, then characterize the trajectory-selected limiting coefficients, and finally verify the primal--dual optimality conditions of the quadratic margin problem.

Define the descent signals
\begin{equation}
h_n(W):=-g_n(W)=(e_{y_n}-p_n(W))x_n^\top,
\label{eq:app-adamproxy-descent-signal}
\end{equation}
and write \(A_{n,c}:=(e_{y_n}-e_c)x_n^\top\) for \((n,c)\in\mathcal Q\). Then \(h_n(W)=\sum_{c\ne y_n}p_{n,c}(W)A_{n,c}\). Set
\begin{equation}
\widetilde S(W):=\sum_{n=1}^Nh_n(W)=-S(W),\qquad \widetilde Q(W):=\sum_{n=1}^Nh_n(W)^{\odot2}=Q(W).
\label{eq:app-adamproxy-descent-SQ}
\end{equation}

For \(\theta\in\Delta(\mathcal Q)\), write
\begin{equation}
\begin{aligned}
\mathsf S(\theta)&:=\sum_{n=1}^NZ_n(\theta),\qquad \mathsf Q(\theta):=\sum_{n=1}^NZ_n(\theta)^{\odot2}, \\ \mathsf D_{N,b}(\theta)&:=\left(\frac{N}{b(N-1)}\left((b-1)\mathsf S(\theta)^{\odot2}+(N-b)\mathsf Q(\theta)\right)\right)^{\odot1/2}.
\label{eq:app-adamproxy-SQ-theta}
\end{aligned}
\end{equation}
Thus \(M_{N,b}(\theta)=\operatorname{diag}(\operatorname{vec}\mathsf D_{N,b}(\theta))\). Define
\begin{equation}
\Psi_{N,b}(W):=\widetilde S(W)\oslash\left(\frac{N}{b(N-1)}\left((b-1)\widetilde S(W)^{\odot2}+(N-b)\widetilde Q(W)\right)\right)^{\odot1/2},
\label{eq:app-adamproxy-descent-proxy-relation}
\end{equation}
with \(0/0:=0\) entrywise. The explicit proxy formula in the main text gives \(\Psi_{N,b}(W)=-\Phi_{N,b}(W)\), so the AdamProxy dynamics can be written as \(W_{t+1}^{\mathrm{AP}}=W_t^{\mathrm{AP}}+\gamma_t\Psi_{N,b}(W_t^{\mathrm{AP}})\).

\begin{lemma}[Coefficient representation and trajectory-selected limit]
\label{lem:app-adamproxy-selected-limit}
For finite \(W\), define
\begin{equation}
R(W):=\sum_{(n,c)\in\mathcal Q}p_{n,c}(W)>0,\qquad \theta_{n,c}(W):=\frac{p_{n,c}(W)}{R(W)}.
\label{eq:app-adamproxy-theta}
\end{equation}
Then
\begin{equation}
h_n(W)=R(W)Z_n(\theta(W)),\qquad \widetilde S(W)=R(W)\mathsf S(\theta(W)),\qquad \widetilde Q(W)=R(W)^2\mathsf Q(\theta(W)),
\label{eq:app-adamproxy-coefficient-factorization}
\end{equation}
\begin{equation}
\Psi_{N,b}(W)=\mathsf S(\theta(W))\oslash\mathsf D_{N,b}(\theta(W)),\qquad \|\Psi_{N,b}(W)\|_{\max}\le1.
\label{eq:app-adamproxy-coefficient-proxy}
\end{equation}
Under the assumptions of Proposition~\ref{thm:adamproxy-margin-clean},
\begin{equation}
\frac{W_t^{\mathrm{AP}}}{\|W_t^{\mathrm{AP}}\|_{\mathrm F}}\longrightarrow\overline W:=-\overline\Phi,
\label{eq:app-adamproxy-Wbar-sign}
\end{equation}
and there exist a subsequence \(t_j\to\infty\), \(\theta^\star\in\Delta(\mathcal Q)\), and \(\nu\in(0,\infty)\) such that
\begin{equation}
\theta(W_{t_j}^{\mathrm{AP}})\to\theta^\star,\qquad \|\Psi_{N,b}(W_{t_j}^{\mathrm{AP}})\|_{\mathrm F}\to\nu,\qquad \Psi_{N,b}(W_{t_j}^{\mathrm{AP}})\to\nu\overline W,
\label{eq:app-adamproxy-selected-subsequence}
\end{equation}
with \(\mathsf S(\theta^\star)\ne{\mathbf 0}\) and \(\nu\|\overline W\|_{\max}\le1\).
\end{lemma}

\begin{proof}
The factorization follows directly from \(h_n(W)=\sum_{c\ne y_n}p_{n,c}(W)A_{n,c}\). For every coordinate \(\rho\), Cauchy--Schwarz gives \(\mathsf S(\theta)[\rho]^2\le N\mathsf Q(\theta)[\rho]\), and therefore
\[
\mathsf D_{N,b}(\theta)[\rho]^2-\mathsf S(\theta)[\rho]^2=\frac{N-b}{b(N-1)}\left(N\mathsf Q(\theta)[\rho]-\mathsf S(\theta)[\rho]^2\right)\ge0.
\]
Hence \(\mathsf D_{N,b}(\theta)[\rho]\ge|\mathsf S(\theta)[\rho]|\), and if \(\mathsf D_{N,b}(\theta)[\rho]=0\), then \(\mathsf S(\theta)[\rho]=0\). With the convention \(0/0:=0\), this proves \eqref{eq:app-adamproxy-coefficient-proxy} and the coordinate bound. By linear separability, choose \(V\) and \(\sigma_{\mathrm{sep}}>0\) with \(\langle V,A_q\rangle\ge\sigma_{\mathrm{sep}}\) for every \(q\). Since \(\mathsf D_{N,b}\) is continuous on the compact simplex, \(D_{\max}:=\max_\theta\|\mathsf D_{N,b}(\theta)\|_{\max}<\infty\), and
\[
\sigma_{\mathrm{sep}}\le\langle V,\mathsf S(\theta(W))\rangle\le\|V\|_{\mathrm F}D_{\max}\|\Psi_{N,b}(W)\|_{\mathrm F}.
\]
Thus \(\|\Psi_{N,b}(W_t)\|_{\mathrm F}\) is bounded above and below by positive constants. Since \(\Psi_{N,b}(W)=-\Phi_{N,b}(W)\), the normalized descent-proxy directions satisfy \(\Psi_{N,b}(W_t)/\|\Psi_{N,b}(W_t)\|_{\mathrm F}\to-\overline\Phi\). Assumption~\ref{ass:lr-base} implies that the epochwise schedule \(\gamma_t\) is positive, tends to zero, and satisfies \(\sum_t\gamma_t=\infty\). With \(T_t:=\sum_{s=0}^{t-1}\gamma_s\|\Psi_{N,b}(W_s)\|_{\mathrm F}\), one has \(T_t\to\infty\) and
\[
\frac{W_t}{T_t}=\frac{W_0}{T_t}+\frac1{T_t}\sum_{s=0}^{t-1}\gamma_s\|\Psi_{N,b}(W_s)\|_{\mathrm F}\frac{\Psi_{N,b}(W_s)}{\|\Psi_{N,b}(W_s)\|_{\mathrm F}}\longrightarrow-\overline\Phi.
\]
Hence \(\|W_t\|_{\mathrm F}/T_t\to1\) and \(W_t/\|W_t\|_{\mathrm F}\to-\overline\Phi\), proving \eqref{eq:app-adamproxy-Wbar-sign}. Compactness of \(\Delta(\mathcal Q)\) and the two-sided bound on \(\|\Psi_{N,b}(W_t)\|_{\mathrm F}\) give a subsequence satisfying \eqref{eq:app-adamproxy-selected-subsequence}. The separator inequality passes to the limit and gives \(\mathsf S(\theta^\star)\ne{\mathbf 0}\), while \(\|\Psi_{N,b}(W_{t_j})\|_{\max}\le1\) yields \(\nu\|\overline W\|_{\max}\le1\).
\end{proof}

\begin{proposition}[Selected stationarity and positive margin]
\label{prop:app-adamproxy-stationarity-margin}
For the limit in Lemma~\ref{lem:app-adamproxy-selected-limit}, there exists \(\alpha=\nu^{-1}>0\) such that
\begin{equation}
M_{N,b}(\theta^\star)\operatorname{vec}(\overline W)=\alpha\operatorname{vec}(\mathsf S(\theta^\star))=\alpha\sum_{q\in\mathcal Q}\theta_q^\star\operatorname{vec}(A_q).
\label{eq:app-adamproxy-selected-stationarity}
\end{equation}
Moreover,
\begin{equation}
\gamma_{\mathrm{AP}}:=\min_{q\in\mathcal Q}\langle\overline W,A_q\rangle>0,
\qquad
\operatorname{supp}(\theta^\star)\subseteq\operatorname*{arg\,min}_{q\in\mathcal Q}\langle\overline W,A_q\rangle.
\label{eq:app-adamproxy-positive-margin-support}
\end{equation}
\end{proposition}

\begin{proof}
Let \(\mathsf S^\star:=\mathsf S(\theta^\star)\) and \(\mathsf D^\star:=\mathsf D_{N,b}(\theta^\star)\). If \(\mathsf D^\star[\rho]>0\), continuity and \eqref{eq:app-adamproxy-selected-subsequence} give \(\mathsf S^\star[\rho]/\mathsf D^\star[\rho]=\nu\overline W[\rho]\). If \(\mathsf D^\star[\rho]=0\), then the inequality \(\mathsf D_{N,b}(\theta)[\rho]\ge|\mathsf S(\theta)[\rho]|\) proved above gives \(\mathsf S^\star[\rho]=0\). Hence \eqref{eq:app-adamproxy-selected-stationarity} holds on every coordinate. It also gives
\begin{equation}
\langle\overline W,\mathsf S(\theta^\star)\rangle=\nu\sum_\rho\mathsf D^\star[\rho]\overline W[\rho]^2>0.
\label{eq:app-adamproxy-positive-average}
\end{equation}
To identify the support, write \(\overline m_{n,c}:=\langle\overline W,A_{n,c}\rangle\), \(\mu_n:=\min_{c\ne y_n}\overline m_{n,c}\), and \(h_{n,c}:=\overline m_{n,c}-\min\{0,\mu_n\}\). Since \(\|W_t\|_{\mathrm F}\to\infty\) and \(W_t/\|W_t\|_{\mathrm F}\to\overline W\), finite log-sum-exp asymptotics give
\[
-\frac1{\|W_t\|_{\mathrm F}}\log p_{n,c}(W_t)\longrightarrow h_{n,c}.
\]
Therefore every subsequential limit \(\theta^\star\) is supported on the atoms minimizing \(h_{n,c}\). If some \(\mu_n\le0\), then the minimum of \(h\) is zero and every minimizing atom has \(\overline m_{n,c}\le0\), contradicting \eqref{eq:app-adamproxy-positive-average}. Thus every \(\mu_n>0\), so \(h_{n,c}=\overline m_{n,c}\), proving both claims in \eqref{eq:app-adamproxy-positive-margin-support}.
\end{proof}

For fixed \(\theta\in\Delta(\mathcal Q)\), let \(a_q:=\operatorname{vec}(A_q)\) and define
\begin{equation}
\mathcal L_\theta(u,\lambda):=\frac12u^\top M_{N,b}(\theta)u+\sum_{q\in\mathcal Q}\lambda_q(1-a_q^\top u),\qquad g_\theta(\lambda):=\inf_{u\in\mathbb R^{kd}}\mathcal L_\theta(u,\lambda),\qquad \lambda\ge0.
\label{eq:app-adamproxy-lagrangian}
\end{equation}

\begin{lemma}[Semidefinite primal-dual certificate]
\label{lem:app-adamproxy-psd-kkt}
Fix \(\theta\in\Delta(\mathcal Q)\). Suppose \(U\) is feasible for \eqref{eq:adamproxy-primal-clean} and \(\lambda\ge0\) satisfies
\begin{equation}
M_{N,b}(\theta)\operatorname{vec}(U)=\sum_{q\in\mathcal Q}\lambda_q\operatorname{vec}(A_q),\qquad \lambda_q(\langle U,A_q\rangle-1)=0\quad(q\in\mathcal Q).
\label{eq:app-adamproxy-psd-kkt}
\end{equation}
Then \(U\in\mathcal P_{N,b}(\theta)\), and \(\lambda\) is a dual optimizer. No invertibility of \(M_{N,b}(\theta)\) is required.
\end{lemma}

\begin{proof}
Write \(u:=\operatorname{vec}(U)\), \(M:=M_{N,b}(\theta)\), and let \(V\) be any feasible matrix with \(v:=\operatorname{vec}(V)\). Since \(M\succeq0\),
\[
\frac12v^\top Mv-\frac12u^\top Mu=\frac12(v-u)^\top M(v-u)+\sum_q\lambda_q\langle V-U,A_q\rangle\ge0,
\]
where feasibility and complementarity are used in the last step. Thus \(U\) is primal optimal. Stationarity makes \(u\) a global minimizer of the Lagrangian for this \(\lambda\), and complementarity makes the dual value equal the primal value, so \(\lambda\) is dual optimal by weak duality.
\end{proof}

\begin{proof}[Proof of Proposition~\ref{thm:adamproxy-margin-clean}]
Set \(p^\star:=\overline W/\gamma_{\mathrm{AP}}\), where \(\overline W=-\overline\Phi\) and \(\gamma_{\mathrm{AP}}>0\) are given by Lemma~\ref{lem:app-adamproxy-selected-limit} and Proposition~\ref{prop:app-adamproxy-stationarity-margin}. Let \(\alpha\) be the scalar in \eqref{eq:app-adamproxy-selected-stationarity} and define
\begin{equation}
\lambda_q^\star:=\frac{\alpha}{\gamma_{\mathrm{AP}}}\theta_q^\star.
\label{eq:app-adamproxy-dual-candidate}
\end{equation}
Then \(p^\star\) is feasible because \(\langle p^\star,A_q\rangle\ge1\), and dividing \eqref{eq:app-adamproxy-selected-stationarity} by \(\gamma_{\mathrm{AP}}\) gives the stationarity equation in Lemma~\ref{lem:app-adamproxy-psd-kkt}. If \(\lambda_q^\star>0\), then \(\theta_q^\star>0\), so \eqref{eq:app-adamproxy-positive-margin-support} gives \(\langle\overline W,A_q\rangle=\gamma_{\mathrm{AP}}\) and hence complementarity. Lemma~\ref{lem:app-adamproxy-psd-kkt} therefore yields \(p^\star\in\mathcal P_{N,b}(\theta^\star)\) and dual optimality of \(\lambda^\star\). Finally,
\[
\|\lambda^\star\|_1=\frac{\alpha}{\gamma_{\mathrm{AP}}}\sum_q\theta_q^\star=\frac{\alpha}{\gamma_{\mathrm{AP}}}>0,
\]
so \(\theta^\star=\lambda^\star/\|\lambda^\star\|_1\). This proves the proposition.
\end{proof}

\section{Extension to Anisotropic Gaussian with Independent Orientations}
\label{app:anisotropic-extension}

This appendix proves Theorem~\ref{thm:aniso-dd-ordering}. The proof uses clean training in an essential way: the optimizer-selected limiting directions depend only on the class means and are therefore independent of the Haar eigenbases of the test-noise covariances. We first characterize the geometry of these selected Row, Spectral, and Max directions and then compare their Gaussian test errors after exposing the independent test-noise orientations.

\paragraph{Normalization and Gaussian error transfer.}
Apply Lemma~\ref{lem:app-common-rescaling} with the common factor \(d^{-1/2}\), and suppress the rescaling in the notation. Thus throughout this appendix
\begin{equation}
\mu_i=\frac{1}{\sqrt d}\Gamma^{1/2}g_i,\qquad g_i\overset{\mathrm{iid}}{\sim}\mathcal N({\mathbf 0},I_d),\qquad X^{\mathrm{te}}\mid(Y=i)\sim\mathcal N\!\left(\mu_i,\frac{r_{\mathrm{te}}^2}{d}\Sigma_i\right),
\label{eq:app-aniso-normalized-model}
\end{equation}
and the clean training set is \(X=\{(\mu_i,i)\}_{i=1}^k\). This rescaling changes neither classifier predictions nor margin-direction sets. Set \(r_{\mathrm{te},0}:=\sup_d r_{\mathrm{te}}(d)<\infty\). For any classifier with \(w_i\ne w_j\), define
\begin{equation}
u_{ij}^W:=\frac{w_i-w_j}{\|w_i-w_j\|_2},\qquad Z_{i\to j}(W):=\frac{\sqrt d\,(u_{ij}^W)^\top\mu_i}{\sqrt{(u_{ij}^W)^\top\Sigma_i u_{ij}^W}}.
\label{eq:app-aniso-standardized-signal}
\end{equation}

\begin{lemma}[Gaussian error transfer]
\label{lem:app-aniso-error-comparison}
Let \(\mathcal R_d\) and \(\mathcal C_d\) be nonempty classifier families. Suppose every \(W\in\mathcal R_d\) satisfies \(Z_{i\to j}(W)\ge A_d>0\) for all \(i\ne j\), while every \(V\in\mathcal C_d\) has a directed pair \(i\ne j\) satisfying \(Z_{i\to j}(V)\le B_d\), where \(0\le B_d<A_d\). Then
\begin{equation}
\frac{\sup_{W\in\mathcal R_d}\operatorname{Err}_{\pi,d}(W)}{\inf_{V\in\mathcal C_d}\operatorname{Err}_{\pi,d}(V)}
\le
\frac{k-1}{\pi_{\min,d}}
\exp\!\left(-\frac{A_d^2-B_d^2}{2r_{\mathrm{te},0}^2}\right).
\label{eq:app-aniso-error-ratio}
\end{equation}
Consequently, if \(A_d^2\to\infty\), \(\limsup_d B_d^2/A_d^2<1\), and \(\log((k-1)/\pi_{\min,d})=o(A_d^2)\), then every classifier in \(\mathcal R_d\) has strictly larger population accuracy than every classifier in \(\mathcal C_d\) for all sufficiently large \(d\).
\end{lemma}

\begin{proof}
For a directed pair \(i\ne j\),
\[
\mathbb P\!\left(w_j^\top X^{\mathrm{te}}\ge w_i^\top X^{\mathrm{te}}\mid Y=i\right)
=
\overline\Phi\!\left(\frac{Z_{i\to j}(W)}{r_{\mathrm{te}}}\right),
\]
where \(\overline\Phi(x)=\mathbb P(G\ge x)\) for \(G\sim\mathcal N(0,1)\). A union bound over the \(k-1\) competitors gives \(\operatorname{Err}_{\pi,d}(W)\le(k-1)\overline\Phi(A_d/r_{\mathrm{te}})\) for every \(W\in\mathcal R_d\). For \(V\in\mathcal C_d\), the directed pair supplied by the assumption is contained in the multiclass error event, so \(\operatorname{Err}_{\pi,d}(V)\ge\pi_{\min,d}\overline\Phi(B_d/r_{\mathrm{te}})\). Since \(\overline\Phi(x)/\overline\Phi(y)\le\exp(-(x^2-y^2)/2)\) for \(x\ge y\ge0\), \eqref{eq:app-aniso-error-ratio} follows. The final assertion is immediate.
\end{proof}

For \(\theta\ge0\), write
\begin{equation}
T_{\theta,d}:=\sum_{a=1}^d a^{-\theta}.
\label{eq:app-aniso-power-sum}
\end{equation}
For every fixed \(\theta\ge0\),
\begin{equation}
T_{\theta,d}\asymp
\begin{cases}
d^{1-\theta},&0\le\theta<1,\\
\log d,&\theta=1,\\
1,&\theta>1.
\end{cases}
\label{eq:app-aniso-power-sum-asymptotics}
\end{equation}
In particular, for fixed \(0\le s_\mu<1\), the trace-normalized mean covariance obeys
\begin{equation}
c_{s_\mu}I_d\preceq\Gamma,\qquad \|\Gamma\|_{\mathrm{op}}\le C_{s_\mu}d^{s_\mu},\qquad \operatorname{tr}(\Gamma^2)\le C_{s_\mu}(d+d^{2s_\mu})\log d.
\label{eq:app-aniso-mean-moments}
\end{equation}

\paragraph{Optimizer-selected directions.}
Let \(\mathcal F_{\mathrm{cl}}\) be the sigma-field generated by the class means together with the prescribed optimizer initializations and tie-breaking variables. Because training uses only \(\{(\mu_i,i)\}_{i=1}^k\), \(\mathcal F_{\mathrm{cl}}\) is independent of \(O_1,\ldots,O_k\).

\begin{lemma}[Residual-exogenous optimizer limits]
\label{lem:app-aniso-clean-optimizer-limits}
On the probability-one event on which the directional limits assumed in Theorem~\ref{thm:aniso-dd-ordering} exist, define
\[
U_{\mathrm{opt}}:=\lim_{t\to\infty}\frac{W_t^{\mathrm{opt}}}{\|W_t^{\mathrm{opt}}\|_{\mathrm F}},\qquad \mathrm{opt}\in\mathcal O_{\mathrm{row}}\cup\mathcal O_{\mathrm{comp}}.
\]
Every \(U_{\mathrm{opt}}\) is \(\mathcal F_{\mathrm{cl}}\)-measurable and hence independent of \(O_1,\ldots,O_k\). Moreover,
\[
U_{\mathrm{opt}}\in
\begin{cases}
\mathcal U_{\mathrm{row}}^X,&\mathrm{opt}\in\mathcal O_{\mathrm{row}},\\
\mathcal U_{\mathrm{sp}}^X,&\mathrm{opt}\in\{\mathrm{SpecGD},\mathrm{Muon}\},\\
\mathcal U_{\max}^X,&\mathrm{opt}\in\{\mathrm{SignGD},\mathrm{Signum},\mathrm{Adam}\},
\end{cases}
\]
and
\begin{equation}
\operatorname{Acc}_{\pi,d}\!\left(W_t^{\mathrm{opt}}\right)\longrightarrow\operatorname{Acc}_{\pi,d}\!\left(U_{\mathrm{opt}}\right).
\label{eq:app-aniso-clean-accuracy-limit}
\end{equation}
\end{lemma}

\begin{proof}
The empirical loss, every full-batch gradient, and every optimizer state are \(\mathcal F_{\mathrm{cl}}\)-measurable, so induction gives the same measurability for every iterate and hence for every normalized directional limit. The full-batch implicit-bias results used in Theorem~\ref{thm:det-optimizer-accuracy-order-clean}, together with Lemma~\ref{lem:exact-margin-direction-convergence}, place the three optimizer groups in the corresponding maximum-margin direction sets. Finally, population accuracy is invariant under positive rescaling, each margin direction has nonzero pairwise row differences, and the positive-definite Gaussian test laws assign zero probability to score-tie hyperplanes. Dominated convergence therefore gives \eqref{eq:app-aniso-clean-accuracy-limit}.
\end{proof}

\paragraph{Mean geometry of the selected directions.}
Let \(M^\mu\in\mathbb R^{k\times d}\) have \(i\)-th row \(\mu_i^\top\), let \(P:=I_k-k^{-1}\mathbf 1_k\mathbf 1_k^\top\), and set \(C^\mu:=PM^\mu\). Recall \(\kappa_{\mathrm{sp}}(\alpha)<1\) from \eqref{eq:app-det-all-aspect-scales}, and define
\begin{equation}
\vartheta_{\mathrm{sp}}(\alpha):=\kappa_{\mathrm{sp}}(\alpha)^2,\qquad \vartheta_{\max}:=\frac{2}{\pi},\qquad \vartheta_\alpha:=\max\{\vartheta_{\mathrm{sp}}(\alpha),\vartheta_{\max}\}<1.
\label{eq:app-aniso-clean-theta}
\end{equation}

\begin{lemma}[Selected mean geometry]
\label{lem:app-aniso-clean-quantified-mean}
Fix \(0\le s_\mu<1\) and any finite collection of \(\mathcal F_{\mathrm{cl}}\)-measurable Row, Spectral, and Max selectors. For every fixed \(0<\xi<1-s_\mu\), with probability at least \(1-e^{-d^\xi}\) for all sufficiently large \(d\), the following hold.
\begin{enumerate}[label=(\roman*),leftmargin=2em]
\item Every selected Row direction satisfies
\begin{equation}
\min_{i\ne j}(u_{ij}^{U_{\mathrm{row}}})^\top\mu_i\ge\frac{1}{\sqrt2}-o(1).
\label{eq:app-aniso-clean-row-signal}
\end{equation}
\item For every selected \(U_\psi\), \(\psi\in\{\mathrm{sp},\max\}\), and every fixed \(\varepsilon>0\), there exist unit vectors \(a_1^\psi,\ldots,a_k^\psi\), a directed edge set \(\mathcal G_d^\psi\) with \(|\mathcal G_d^\psi|\ge c_\varepsilon k(k-1)\), and
\[
p_{ij}^\psi:=\frac{a_i^\psi-a_j^\psi}{\|a_i^\psi-a_j^\psi\|_2}
\]
such that
\begin{align}
\max_{i\ne j}|(a_i^\psi)^\top a_j^\psi|&=o(1/\log d),\label{eq:app-aniso-clean-code-coherence}\\
\max_{(i,j)\in\mathcal G_d^\psi}\|u_{ij}^{U_\psi}-p_{ij}^\psi\|_2&\le\frac{1}{\log d},\label{eq:app-aniso-clean-code-transfer}\\
\max_{(i,j)\in\mathcal G_d^\psi}\bigl[(u_{ij}^{U_\psi})^\top\mu_i\bigr]_+^2&\le\frac{\vartheta_\psi+\varepsilon}{2}.
\label{eq:app-aniso-clean-code-signal}
\end{align}
\end{enumerate}
The event can be chosen simultaneously for the entire finite selector collection.
\end{lemma}

\begin{proof}
Write \(L_d:=\log d\). Weighted Gaussian quadratic and bilinear concentration using \eqref{eq:app-aniso-mean-moments}, followed by a union bound over \(O(k^2)\) quantities, gives
\begin{equation}
\max_i\bigl|\|\mu_i\|_2^2-1\bigr|+\max_{i\ne j}|\mu_i^\top\mu_j|\le L_d^{-64}
\label{eq:app-aniso-clean-shell}
\end{equation}
except with probability at most \(C\exp\{-cd^{1-s_\mu}/L_d^C\}\). This failure probability is at most \(e^{-d^\xi}\) for every fixed \(\xi<1-s_\mu\) and sufficiently large \(d\).

For Row, apply Proposition~\ref{prop:app-det-row-rigidity} to the clean samples with \(m_i=\mu_i\) and \(e_{i,s}={\mathbf 0}\). Equation~\eqref{eq:app-aniso-clean-shell} gives \(\max_i\|w_i-\mu_i\|_2\le CL_d^{-16}\), uniformly over the exact Row face. Since \(\|\mu_i-\mu_j\|_2=\sqrt2+o(1)\) uniformly over pairs, normalization stability yields \(u_{ij}^W=(\mu_i-\mu_j)/\|\mu_i-\mu_j\|_2+O(L_d^{-16})\), and \eqref{eq:app-aniso-clean-row-signal} follows.

For Spectral, write the reduced singular value decomposition
\[
C^\mu=U\Lambda V^\top,\qquad Q:=UV^\top,\qquad r:=\min\{k-1,d\},
\]
and let \(q_i^\top\) denote the \(i\)-th row of \(Q\). After identifying \(\mathbf 1_k^\perp\) with \(\mathbb R^{k-1}\), the nonzero part of \(C^\mu\) has the same law as \(G\Gamma^{1/2}/\sqrt d\), where \(G\in\mathbb R^{(k-1)\times d}\) has i.i.d.\ standard Gaussian entries. Its left singular frame is therefore Haar in \(\mathbf 1_k^\perp\) and independent of the singular values. Haar projector concentration gives, on the same probability scale,
\begin{equation}
\max_i\left|\|q_i\|_2^2-\frac{r}{k}\right|+\max_{i\ne j}|q_i^\top q_j|\le L_d^{-16}.
\label{eq:app-aniso-clean-polar-code}
\end{equation}
Moreover, since \(QC^{\mu\top}=U\Lambda U^\top\), conditional Haar concentration and the bound \(\|\bar\mu\|_2=o(1)\), where \(\bar\mu=k^{-1}\sum_i\mu_i\), give
\begin{equation}
\max_{i\ne j}\left|(q_i-q_j)^\top\mu_i-\frac{\|C^\mu\|_*}{k-1}\right|\le L_d^{-32}.
\label{eq:app-aniso-clean-polar-flatness}
\end{equation}

The nuclear-to-Frobenius ratio satisfies
\begin{equation}
\frac{\|C^\mu\|_*}{\sqrt r\,\|C^\mu\|_{\mathrm F}}\le\kappa_{\mathrm{sp}}(\alpha)+\eta
\label{eq:app-aniso-clean-polar-ratio}
\end{equation}
for every fixed \(\eta>0\), with the same stretched-exponential confidence. Indeed, for fixed \(d\), the map
\[
A\mapsto\mathbb E_G\operatorname{tr}\!\left[\left(\frac{GAG^\top}{d}\right)^{1/2}\right]
\]
is concave and invariant under orthogonal conjugation, hence Schur-concave in the eigenvalues of \(A\). Among positive semidefinite \(A\) with \(\operatorname{tr}A=d\), it is maximized by \(I_d\). Gaussian concentration and the Marchenko--Pastur limit therefore bound the numerator by its isotropic limit, while \eqref{eq:app-aniso-clean-shell} gives \(\|C^\mu\|_{\mathrm F}/\sqrt k=1+o(1)\).

It remains to transfer the polar geometry to every exact Spectral margin direction. Since \(s_\mu<1\), \eqref{eq:app-aniso-power-sum-asymptotics} gives \(\Gamma\succeq c_{s_\mu}I_d\). Hence the positive singular values of \(C^\mu\) dominate \(\sqrt{c_{s_\mu}}\) times those of the corresponding standard Gaussian matrix. If \(\alpha\ne1\), the smallest positive singular value is therefore bounded below by a constant with probability \(1-e^{-cd}\). If \(\alpha=1\), the standard Gaussian hard-edge estimate and a dyadic union bound give
\begin{equation}
\frac1r\#\{\ell:\lambda_\ell(C^\mu)\le t\}\le Ct,\qquad L_d^{-8}\le t\le1.
\label{eq:app-aniso-clean-hard-edge}
\end{equation}
with probability at least \(1-Ce^{-cd/L_d^{16}}\).

Let \(W\in\mathcal U_{\mathrm{sp}}^X\), put \(B:=PW/\|PW\|_{\mathrm{op}}\), and note that \(\|PW\|_{\mathrm{op}}=\|W\|_{\mathrm{op}}\) because centering preserves all margins. The polar factor \(Q\) has clean margin \(\|C^\mu\|_*/(k-1)+O(L_d^{-32})\) by \eqref{eq:app-aniso-clean-polar-flatness}. Summing the normalized margin constraints for \(B\) therefore gives
\[
\langle B,C^\mu\rangle\ge\|C^\mu\|_*-kL_d^{-32}.
\]
Writing \(d_\ell:=1-u_\ell^\top Bv_\ell\in[0,2]\), where \(u_\ell,v_\ell\) are the singular vectors of \(C^\mu\), yields
\[
\frac1r\sum_{\ell=1}^r\lambda_\ell(C^\mu)d_\ell\le CL_d^{-32}.
\]
The positive spectral edge for \(\alpha\ne1\), and \eqref{eq:app-aniso-clean-hard-edge} with a split at \(t=L_d^{-8}\) for \(\alpha=1\), imply
\[
\frac1r\|B-Q\|_{\mathrm F}^2\le CL_d^{-8}.
\]
The centered pair identity then gives
\begin{equation}
\sup_{W\in\mathcal U_{\mathrm{sp}}^X}\frac1{k(k-1)}\sum_{i\ne j}\left\|u_{ij}^W-\frac{q_i-q_j}{\|q_i-q_j\|_2}\right\|_2^2\le CL_d^{-8}.
\label{eq:app-aniso-clean-spectral-transfer}
\end{equation}

Define
\[
a_i^{\mathrm{sp}}:=\frac{q_i}{\|q_i\|_2},\qquad p_{ij}^{\mathrm{sp}}:=\frac{a_i^{\mathrm{sp}}-a_j^{\mathrm{sp}}}{\|a_i^{\mathrm{sp}}-a_j^{\mathrm{sp}}\|_2}.
\]
Equation~\eqref{eq:app-aniso-clean-polar-code} gives \(\max_{i\ne j}|(a_i^{\mathrm{sp}})^\top a_j^{\mathrm{sp}}|\le CL_d^{-16}\) and shows that \(p_{ij}^{\mathrm{sp}}\) differs from \((q_i-q_j)/\|q_i-q_j\|_2\) by \(O(L_d^{-16})\). Combining \eqref{eq:app-aniso-clean-polar-flatness}, \eqref{eq:app-aniso-clean-polar-ratio}, and \(\|C^\mu\|_{\mathrm F}/\sqrt k=1+o(1)\) gives
\begin{equation}
\max_{i\ne j}\bigl[(p_{ij}^{\mathrm{sp}})^\top\mu_i\bigr]_+\le\frac{\kappa_{\mathrm{sp}}(\alpha)+2\eta}{\sqrt2}.
\label{eq:app-aniso-clean-spectral-canonical-signal}
\end{equation}
Markov's inequality applied to \eqref{eq:app-aniso-clean-spectral-transfer} shows that all but \(o(k^2)\) ordered pairs satisfy \(\|u_{ij}^{U_{\mathrm{sp}}}-p_{ij}^{\mathrm{sp}}\|_2\le L_d^{-1}\). Choosing \(\eta>0\) sufficiently small in terms of the prescribed \(\varepsilon\) proves part~(ii) for Spectral selectors.

For Max, use the common Haar eigenbasis of the mean covariance. Write
\[
\mu_i=U_\mu\widetilde\mu_i,\qquad \widetilde\mu_i:=\frac{D_{s_\mu}^{1/2}g_i}{\sqrt{T_{s_\mu,d}}},\qquad \widetilde c_i:=\widetilde\mu_i-\frac1k\sum_{\ell=1}^k\widetilde\mu_\ell.
\]
Then the \(i\)-th row of \(C^\mu\) is \((U_\mu\widetilde c_i)^\top\). Weighted Gaussian concentration gives
\begin{equation}
\max_i\bigl|\|\widetilde c_i\|_2^2-1\bigr|+\max_{i\ne j}|\widetilde c_i^\top\widetilde c_j|+\max_i\bigl|\widetilde c_i^\top\widetilde\mu_i-1\bigr|+\max_{i\ne j}|\widetilde c_j^\top\widetilde\mu_i|\le L_d^{-64}
\label{eq:app-aniso-clean-prehaar-shell}
\end{equation}
outside an event of probability \(C\exp\{-cd^{1-s_\mu}/L_d^C\}\).

Put \(S^\mu:=\operatorname{sign}(C^\mu)\). Conditional on the vectors in \eqref{eq:app-aniso-clean-prehaar-shell}, standard smoothed-sign Haar concentration, using the identities
\[
\mathbb E[\operatorname{sign}(q^\top x)\operatorname{sign}(q^\top y)]=\frac{2}{\pi}\arcsin\!\left(\frac{x^\top y}{\|x\|_2\|y\|_2}\right)
\]
and
\[
\sqrt d\,\mathbb E[\operatorname{sign}(q^\top x)\,q^\top v]=\sqrt d\,\mathbb E|q_1|\frac{x^\top v}{\|x\|_2}=\left(\sqrt{\frac{2}{\pi}}+o(1)\right)\frac{x^\top v}{\|x\|_2},
\]
gives
\begin{align}
\left|\frac{\|C^\mu\|_{\mathrm{sum}}}{k\sqrt d}-\sqrt{\frac2\pi}\right|&\le CL_d^{-32},\label{eq:app-aniso-clean-sign-l1}\\
\max_{i\ne j}\left|\frac{(S_i^\mu-S_j^\mu)^\top\mu_i}{\sqrt d}-\sqrt{\frac2\pi}\right|&\le CL_d^{-32},\label{eq:app-aniso-clean-sign-signal}\\
\max_{i\ne j}\left|\frac{(S_i^\mu)^\top S_j^\mu}{d}\right|&\le CL_d^{-8},\label{eq:app-aniso-clean-sign-coherence}\\
\frac1{kd}\#\{(i,a):\sqrt d\,|C^\mu[i,a]|\le L_d^{-16}\}&\le CL_d^{-16}.
\label{eq:app-aniso-clean-sign-small}
\end{align}
For example, these bounds follow by replacing the sign function by a piecewise-linear approximation at scale \(L_d^{-128}\), applying concentration on \(\mathsf O(d)\), and then controlling the smoothing error by the spherical small-ball probability. The resulting failure probability is \(C\exp\{-cd/L_d^C\}\).

Let \(W\in\mathcal U_{\max}^X\) and \(B:=W/\|W\|_{\max}\). The sign matrix is a feasible unit-Max direction with margin at least \((\sqrt{2/\pi}-CL_d^{-32})\sqrt d\), so \(\gamma_{\max}(X)\) is at least this value. Summing the normalized margin constraints for \(B\) and using \eqref{eq:app-aniso-clean-sign-l1} gives
\[
\frac{\|C^\mu\|_{\mathrm{sum}}-\langle B,C^\mu\rangle}{k\sqrt d}\le CL_d^{-32}.
\]
Since \(|B[i,a]|\le1\), splitting the coordinates at \(\sqrt d\,|C^\mu[i,a]|=L_d^{-16}\) and using \eqref{eq:app-aniso-clean-sign-small} yields
\begin{equation}
\sup_{W\in\mathcal U_{\max}^X}\frac1{kd}\left\|\frac{W}{\|W\|_{\max}}-S^\mu\right\|_{\mathrm F}^2\le CL_d^{-16}.
\label{eq:app-aniso-clean-max-transfer-matrix}
\end{equation}
Equation~\eqref{eq:app-aniso-clean-sign-coherence} implies \(\|S_i^\mu-S_j^\mu\|_2^2=2d(1+O(L_d^{-8}))\). The centered pair identity and \eqref{eq:app-aniso-clean-max-transfer-matrix} therefore give
\begin{equation}
\sup_{W\in\mathcal U_{\max}^X}\frac1{k(k-1)}\sum_{i\ne j}\left\|u_{ij}^W-\frac{S_i^\mu-S_j^\mu}{\|S_i^\mu-S_j^\mu\|_2}\right\|_2^2\le CL_d^{-16}.
\label{eq:app-aniso-clean-max-transfer-pair}
\end{equation}

Define
\[
a_i^{\max}:=\frac{S_i^\mu}{\sqrt d},\qquad p_{ij}^{\max}:=\frac{S_i^\mu-S_j^\mu}{\|S_i^\mu-S_j^\mu\|_2}.
\]
Equations~\eqref{eq:app-aniso-clean-sign-signal}--\eqref{eq:app-aniso-clean-sign-coherence} give
\[
\max_{i\ne j}|(a_i^{\max})^\top a_j^{\max}|\le CL_d^{-8},\qquad \max_{i\ne j}\left|\bigl[(p_{ij}^{\max})^\top\mu_i\bigr]^2-\frac1\pi\right|\le CL_d^{-8}.
\]
Markov's inequality applied to \eqref{eq:app-aniso-clean-max-transfer-pair} leaves all but \(o(k^2)\) pairs with direction error at most \(L_d^{-1}\). Since \(1/\pi=\vartheta_{\max}/2\), part~(ii) follows for Max selectors. The preceding bounds are uniform over the complete exact Spectral and Max faces, so the same event works for every member of the fixed finite selector collection.
\end{proof}

\paragraph{Test-noise geometry.}
The preceding lemma depends only on the class means. We now expose the independent test-noise bases.

\begin{lemma}[Test-noise energies for selected directions]
\label{lem:app-aniso-clean-sharp-common}
Condition on an event from Lemma~\ref{lem:app-aniso-clean-quantified-mean}, and consider the corresponding finite collection of selected Row, Spectral, and Max directions.
\begin{enumerate}[label=(\roman*),leftmargin=2em]
\item If \(0\le s_\Sigma<1\), then for every fixed \(0<\chi<1-s_\Sigma\), with conditional probability at least \(1-e^{-d^\chi}\) for all sufficiently large \(d\),
\begin{equation}
\max_{\mathrm{selected}\ U}\max_{i\ne j}\left|(u_{ij}^U)^\top\Sigma_i u_{ij}^U-1\right|=o(1).
\label{eq:app-aniso-clean-low-residual-energy}
\end{equation}
\item Suppose \(s_\Sigma\ge1\) and put
\begin{equation}
c_d:=\frac{\log d}{T_{s_\Sigma,d}}.
\label{eq:app-aniso-clean-cd}
\end{equation}
For every fixed \(b>0\) and \(\varepsilon>0\), all selected Row directions satisfy
\begin{equation}
\max_{i\ne j}(u_{ij}^{U_{\mathrm{row}}})^\top\Sigma_i u_{ij}^{U_{\mathrm{row}}}\le1+(4+2b+\varepsilon)c_d
\label{eq:app-aniso-clean-row-sharp-energy}
\end{equation}
except with conditional probability at most \(C_{b,\varepsilon}d^{-b}\). For every fixed \(q\in(0,1)\), each selected competitor \(U_\psi\), \(\psi\in\{\mathrm{sp},\max\}\), has an edge \((i,j)\in\mathcal G_d^\psi\) such that
\begin{equation}
(u_{ij}^{U_\psi})^\top\Sigma_i u_{ij}^{U_\psi}\ge4qc_d+\mathbf 1_{\{s_\Sigma=1\}}(1-o(1))
\label{eq:app-aniso-clean-competitor-extreme-energy}
\end{equation}
except with conditional probability at most \(C_{q,a}d^{-a}\) for every fixed \(0<a<1-q\).
\end{enumerate}
\end{lemma}

\begin{proof}
For a unit vector \(u\) independent of \(O_i\), write \(O_i^\top u=g/\|g\|_2\). The quadratic form is
\[
u^\top\Sigma_i u=\frac{d}{\|g\|_2^2T_{s_\Sigma,d}}\sum_{r=1}^d r^{-s_\Sigma}g_r^2.
\]
When \(s_\Sigma<1\), weighted chi-square concentration gives deviations \(o(1)\) with failure at most \(e^{-d^\chi}\) for every \(\chi<1-s_\Sigma\). A union bound over the finite number of selectors and their \(O(k^2)=O(d^2)\) pair directions proves \eqref{eq:app-aniso-clean-low-residual-energy}.

Now suppose \(s_\Sigma\ge1\). Laurent--Massart's inequality gives, for every \(t>0\),
\[
\frac1{T_{s_\Sigma,d}}\sum_{r=1}^dr^{-s_\Sigma}g_r^2\le1+\frac{2\sqrt{T_{2s_\Sigma,d}t}+2t}{T_{s_\Sigma,d}}
\]
except with probability \(e^{-t}\), while \(\|g\|_2^2/d=1+o(1)\) exponentially fast. Taking \(t=(2+b)\log d\) and union-bounding over \(O(d^2)\) selected Row pair directions gives \eqref{eq:app-aniso-clean-row-sharp-energy}; the square-root term is \(o(c_d)\).

For the competitor lower bound, fix \(q\in(0,1)\) and \(0<a<1-q\), and choose \(q_+\in(q,1-a)\). Write \(L_d:=\log d\), let \(h_i:=O_ie_1\), and represent \(h_i=\xi_i/\|\xi_i\|_2\) using independent \(\xi_i\sim\mathcal N({\mathbf 0},I_d)\). Fix one selected competitor and its code \((a_i^\psi)\) and edge set \(\mathcal G_d^\psi\) from Lemma~\ref{lem:app-aniso-clean-quantified-mean}. Put \(t_d:=\sqrt{2q_+L_d}\) and define
\[
N_d:=\sum_{(i,j)\in\mathcal G_d^\psi}\mathbf 1\{\xi_i^\top a_i^\psi\ge t_d,\ \xi_i^\top a_j^\psi\le-t_d\}.
\]
The coherence bound \(\max_{i\ne j}|(a_i^\psi)^\top a_j^\psi|=o(1/L_d)\) implies, uniformly over the relevant distinct indices,
\[
\mathbb P(\xi_i^\top a_i^\psi\ge t_d,\xi_i^\top a_j^\psi\le-t_d)=d^{-2q_++o(1)},\qquad \mathbb P(\xi_i^\top a_i^\psi\ge t_d,\xi_i^\top a_j^\psi\le-t_d,\xi_i^\top a_\ell^\psi\le-t_d)=d^{-3q_++o(1)}.
\]
Indeed, the two- and three-dimensional covariance matrices differ from the identity by \(o(1/L_d)\), and Gaussian density comparison on the relevant orthants changes the logarithmic tail asymptotics by \(o(L_d)\). Edge indicators with different first indices are independent. Since \(|\mathcal G_d^\psi|=\Theta(d^2)\),
\[
\mathbb E N_d=d^{2-2q_++o(1)},\qquad \operatorname{Var}(N_d)\le\mathbb E N_d+d^{3-3q_++o(1)}.
\]
Chebyshev's inequality gives
\begin{equation}
\mathbb P(N_d=0)\le d^{-(1-q_+)+o(1)}=o(d^{-a}).
\label{eq:app-aniso-clean-extreme-failure}
\end{equation}
On \(\{N_d>0\}\), choose a counted edge using a fixed ordering of \(\mathcal G_d^\psi\). This choice depends only on the class-side information and \((\xi_i)_{i=1}^k\), not on the remaining residual eigenframes. Uniform Gaussian norm concentration and \(\|a_i^\psi-a_j^\psi\|_2=\sqrt2+o(1)\) give
\[
h_i^\top p_{ij}^\psi\ge\frac{2t_d}{(1+o(1))\sqrt{2d}}.
\]
For any fixed \(A>0\), spherical concentration and a union bound give, except with probability \(O(d^{-A})\),
\[
\max_{(i,j)\in\mathcal G_d^\psi}\sqrt d\,\left|h_i^\top(u_{ij}^{U_\psi}-p_{ij}^\psi)\right|\le C_A L_d^{-1/2},\qquad \max_{i\ne j}d\left|h_i^\top u_{ij}^{U_\psi}\right|^2\le C_A L_d.
\]
Here the first bound uses \(\|u_{ij}^{U_\psi}-p_{ij}^\psi\|_2\le L_d^{-1}\), and all directions in these bounds are fixed after conditioning on the class-side information. The selected edge therefore satisfies
\[
d\bigl(h_i^\top u_{ij}^{U_\psi}\bigr)^2\ge(4q_+-o(1))L_d\ge4qL_d
\]
for all sufficiently large \(d\). Since the top eigenvalue of \(\Sigma_i\) is \(d/T_{s_\Sigma,d}\), its contribution is at least \(4qc_d\).

Suppose \(s_\Sigma=1\). Conditional on all top eigenvectors, their auxiliary Gaussian radii, and the selected edge, the remaining eigenframe of \(O_i\) is Haar on \(h_i^\perp\). Write \(a_{ij}:=h_i^\top u_{ij}^{U_\psi}\). The remaining contribution has the conditional distribution
\[
\frac{d(1-a_{ij}^2)}{T_{1,d}}\frac{\sum_{r=2}^d r^{-1}g_r^2}{\sum_{r=2}^d g_r^2},\qquad g_2,\ldots,g_d\overset{\mathrm{iid}}{\sim}\mathcal N(0,1).
\]
The preceding upper bound gives \(a_{ij}^2=O(L_d/d)\). The weighted chi-square lower-tail inequality, with tail parameter \(L_d^{3/2}\), gives
\[
\sum_{r=2}^d r^{-1}g_r^2\ge T_{1,d}-1-O(L_d^{3/4})
\]
except with probability \(\exp(-L_d^{3/2})\), since \(\sum_{r=2}^d r^{-2}=O(1)\). Also, \(\sum_{r=2}^d g_r^2=(d-1)(1+O(L_d^{-2}))\) except with probability \(C\exp(-cd/L_d^4)\). Thus the remaining contribution is at least \(1-O(L_d^{-1/4})=1-o(1)\) with superpolynomial confidence. Choosing \(A>a\), integrating these conditional bounds, and using \eqref{eq:app-aniso-clean-extreme-failure} proves \eqref{eq:app-aniso-clean-competitor-extreme-energy} with failure at most \(C_{q,a}d^{-a}\). A finite union handles all selected competitors.
\end{proof}

\paragraph{Selected-direction accuracy ordering.}

\begin{proposition}[Selected-direction ordering]
\label{prop:app-aniso-clean-quantified}
Fix \(0\le s_\mu<1\), \(s_\Sigma\ge0\), and any finite collection of \(\mathcal F_{\mathrm{cl}}\)-measurable Row, Spectral, and Max selectors. Assume
\[
\log\frac{k-1}{\pi_{\min,d}}=o(d/\log d).
\]
Then the following hold.
\begin{enumerate}[label=(\roman*),leftmargin=2em]
\item If \(0\le s_\Sigma<1\), then for every fixed \(0<\chi<\min\{1-s_\mu,1-s_\Sigma\}\), every selected Row classifier has strictly larger population accuracy than every selected Spectral or Max classifier with probability at least \(1-e^{-d^\chi}\) for every sufficiently large \(d\).
\item If \(s_\Sigma\ge1\), then for every fixed
\begin{equation}
0<\zeta<\frac{2(1-\vartheta_\alpha)}{2+\vartheta_\alpha},
\label{eq:app-aniso-clean-zeta-range}
\end{equation}
the same strict ordering holds with probability at least \(1-C_\zeta d^{-\zeta}\) for every sufficiently large \(d\).
\end{enumerate}
\end{proposition}

\begin{proof}
Suppose first that \(s_\Sigma<1\). Fix \(0<\chi<\min\{1-s_\mu,1-s_\Sigma\}\), choose \(\chi_+\in(\chi,\min\{1-s_\mu,1-s_\Sigma\})\), and fix \(0<\varepsilon_0<1-\vartheta_\alpha\). Applying Lemma~\ref{lem:app-aniso-clean-quantified-mean} and Lemma~\ref{lem:app-aniso-clean-sharp-common}(i) with confidence exponent \(\chi_+\) implies that every selected Row direction has
\[
Z_{i\to j}(U_{\mathrm{row}})^2\ge(1/2-o(1))d
\]
uniformly over all directed pairs, whereas every selected competitor \(U_\psi\) has a directed pair with
\[
[Z_{i\to j}(U_\psi)]_+^2\le\frac{\vartheta_\psi+\varepsilon_0}{2}\,d+o(d).
\]
Since \(\vartheta_\psi+\varepsilon_0<1\), the competitor-to-Row squared-signal ratio has limsup strictly below one. The Row coefficient is of order \(d\), and \(\log((k-1)/\pi_{\min,d})=o(d/\log d)=o(d)\), so Lemma~\ref{lem:app-aniso-error-comparison} gives the strict accuracy ordering. The two mean and residual events have joint failure at most \(e^{-d^\chi}\) after increasing the dimension threshold.

Now suppose \(s_\Sigma\ge1\). Fix \(\zeta\) satisfying \eqref{eq:app-aniso-clean-zeta-range}, choose \(b>\zeta\) sufficiently close to \(\zeta\), and then choose \(q\) so that
\begin{equation}
\vartheta_\alpha\left(1+\frac b2\right)<q<1-\zeta.
\label{eq:app-aniso-clean-bq-choice}
\end{equation}
Such a choice is possible because of \eqref{eq:app-aniso-clean-zeta-range}. Choose fixed \(\varepsilon_R,\varepsilon_M>0\) sufficiently small that
\[
(\vartheta_\alpha+\varepsilon_M)(4+2b+\varepsilon_R)<4q,\qquad (\vartheta_\alpha+\varepsilon_M)(5+2b+\varepsilon_R)<1+4q.
\]
The first inequality holds with zero slacks by \eqref{eq:app-aniso-clean-bq-choice}; the second follows from the same condition and \(\vartheta_\alpha<1\). Apply Lemma~\ref{lem:app-aniso-clean-quantified-mean} with signal slack \(\varepsilon_M\) and Lemma~\ref{lem:app-aniso-clean-sharp-common} with Row-energy slack \(\varepsilon_R\). They supply squared comparison amplitudes satisfying
\begin{align}
A_{R,d}^2&\ge\frac{(1/2-o(1))d}{1+(4+2b+\varepsilon_R)c_d},\nonumber\\
B_{\psi,d}^2&\le\frac{((\vartheta_\psi+\varepsilon_M)/2)d}{4qc_d+\mathbf 1_{\{s_\Sigma=1\}}(1-o(1))}.
\label{eq:app-aniso-clean-high-coefficients}
\end{align}
If \(s_\Sigma>1\), then \(c_d\to\infty\), and
\[
\limsup_d\frac{B_{\psi,d}^2}{A_{R,d}^2}\le\frac{(\vartheta_\psi+\varepsilon_M)(4+2b+\varepsilon_R)}{4q}<1.
\]
If \(s_\Sigma=1\), then \(c_d\to1\), and
\[
\limsup_d\frac{B_{\psi,d}^2}{A_{R,d}^2}\le\frac{(\vartheta_\psi+\varepsilon_M)(5+2b+\varepsilon_R)}{1+4q}<1.
\]
Moreover, \(A_{R,d}^2\gtrsim d/\log d\), so the prior assumption is sufficient for Lemma~\ref{lem:app-aniso-error-comparison}. The Row residual-energy failure is \(O(d^{-b})\), while the competitor-extreme failure is \(O(d^{-a})\) for every \(a<1-q\). Choose \(a\in(\zeta,1-q)\). The mean-geometry failure is stretched-exponential, so intersecting the finitely many selected-direction events gives total failure at most \(C_\zeta d^{-\zeta}\).
\end{proof}

\begin{proof}[Proof of Theorem~\ref{thm:aniso-dd-ordering}]
Lemma~\ref{lem:app-aniso-clean-optimizer-limits} turns the seven optimizer limits into a finite collection of \(\mathcal F_{\mathrm{cl}}\)-measurable Row, Spectral, and Max selectors and gives convergence of their population accuracies. Choose any fixed
\begin{equation}
0<c<\frac{2(1-\vartheta_\alpha)}{2+\vartheta_\alpha}.
\label{eq:app-aniso-clean-theorem-c}
\end{equation}
If \(0\le s_\Sigma<1\), choose any fixed \(0<\chi<\min\{1-s_\mu,1-s_\Sigma\}\). Proposition~\ref{prop:app-aniso-clean-quantified}(i) gives failure probability at most \(e^{-d^\chi}\), which is at most \(d^{-c}\) for every sufficiently large \(d\). If \(s_\Sigma\ge1\), Proposition~\ref{prop:app-aniso-clean-quantified}(ii), applied with \(\zeta=c\), gives failure probability at most \(C_cd^{-c}\). After enlarging a constant \(C\), both cases therefore hold with probability at least \(1-Cd^{-c}\). The exponent \(c\) and the constant \(C\) can be chosen independently of \(s_\mu\) and \(s_\Sigma\); only the sufficiently-large-\(d\) threshold depends on these exponents. On this event, every selected Row limit has strictly larger population accuracy than every selected Spectral or Max limit. Equation~\eqref{eq:app-aniso-clean-accuracy-limit} then gives
\[
\min_{\mathrm{opt}\in\mathcal O_{\mathrm{row}}}\liminf_{t\to\infty}\operatorname{Acc}_{\pi,d}\!\left(W_t^{\mathrm{opt}}\right)
>
\max_{\mathrm{opt}\in\mathcal O_{\mathrm{comp}}}\limsup_{t\to\infty}\operatorname{Acc}_{\pi,d}\!\left(W_t^{\mathrm{opt}}\right),
\]
which is \eqref{eq:aniso-det-optimizer-ordering-clean-data}.
\end{proof}

\section{Coordinate-Specific RowGDM--Adam Reversal}
\label{sec:when2-results}
\label{sec:when2-proofs}

This section proves the coordinate-aligned transition and the recovery under a common Haar rotation in Proposition~\ref{prop:main-coordinate-reversal}.
The comparisons hold uniformly over the complete Row and Max maximum-margin solution sets.

\paragraph{Setup.}
We rescale all features by \(d^{-1/2}\), which preserves the maximum-margin direction sets and population accuracies by Lemma~\ref{lem:app-common-rescaling}.
Let \(k/d\to\alpha\in(0,\infty)\), and suppose
\begin{equation}
\Gamma_d=\operatorname{diag}(\gamma_{1,d},\ldots,\gamma_{d,d}),\qquad \frac1d\operatorname{tr}\Gamma_d=1,\qquad 0<c_\Gamma\le\gamma_{a,d}\le C_\Gamma<\infty.
\label{eq:when2-diet-profile}
\end{equation}
The independent class means and clean training set are
\begin{equation}
\mu_i=\frac1{\sqrt d}\Gamma_d^{1/2}g_i,\qquad g_i\overset{\mathrm{iid}}{\sim}\mathcal N({\mathbf 0},I_d),\qquad \mathcal X=\{(\mu_i,i):i\in[k]\}.
\label{eq:when2-diet-means}
\end{equation}
Write \(M\in\mathbb R^{k\times d}\) for the matrix with rows \(\mu_i^\top\), and define
\[
\mathcal F(M):=\{W:(w_i-w_j)^\top\mu_i\ge1\text{ for every }i\ne j\}.
\]
The two minimum-norm solution sets are
\begin{equation}
\mathcal M_{\mathrm{row}}(M):=\operatorname*{arg\,min}_{W\in\mathcal F(M)}\|W\|_{\mathrm{row}},\qquad \mathcal M_{\mathrm{max}}(M):=\operatorname*{arg\,min}_{W\in\mathcal F(M)}\|W\|_{\max}.
\label{eq:when2-diet-faces}
\end{equation}
For \(\Sigma\succ0\) and a classifier with positive directed mean margins, set
\begin{equation}
L_{i\to j}(W;\Sigma):=\frac{((w_i-w_j)^\top\mu_i)^2}{(w_i-w_j)^\top\Sigma(w_i-w_j)},\qquad E_d(W;\Sigma):=\min_{i\ne j}L_{i\to j}(W;\Sigma).
\label{eq:when2-diet-coefficient}
\end{equation}
The test distribution is
\begin{equation}
\mathbb P(Y=i)=\pi_{i,d},\qquad X^{\mathrm{te}}\mid(Y=i)\sim\mathcal N\!\left(\mu_i,\frac{r_{\mathrm{te}}^2}{d}\Sigma\right),
\label{eq:when2-diet-test}
\end{equation}
where \(0<r_{\mathrm{te}}\le r_{\mathrm{te},0}<\infty\), \(\pi_{\min,d}:=\min_i\pi_{i,d}>0\), and
\begin{equation}
\log(1/\pi_{\min,d})=o(d).
\label{eq:when2-diet-prior}
\end{equation}
Write \(\operatorname{Err}_{\pi,d}=1-\operatorname{Acc}_{\pi,d}\) and \(g_0:=\sqrt{2/\pi}\).
Constants denoted by \(c,C\) may change between displays; their permitted dependence is specified in each statement.

\begin{lemma}[Gaussian coefficient comparison]
\label{lem:when2-diet-error-transfer}
Every classifier with positive directed mean margins satisfies
\begin{equation}
\pi_{\min,d}\,\overline\Phi\!\left(\frac{\sqrt{dE_d(W;\Sigma)}}{r_{\mathrm{te}}}\right)\le\operatorname{Err}_{\pi,d}(W)\le(k-1)\,\overline\Phi\!\left(\frac{\sqrt{dE_d(W;\Sigma)}}{r_{\mathrm{te}}}\right),
\label{eq:when2-diet-tail}
\end{equation}
where \(\overline\Phi(x)=\mathbb P(Z\ge x)\) for \(Z\sim\mathcal N(0,1)\).
If nonempty classifier families \(\mathcal R,\mathcal C\) have positive directed mean margins and satisfy \(E_d(R;\Sigma)\ge A>B\ge E_d(V;\Sigma)>0\) for every \(R\in\mathcal R\) and \(V\in\mathcal C\), then
\begin{equation}
\frac{\sup_{R\in\mathcal R}\operatorname{Err}_{\pi,d}(R)}{\inf_{V\in\mathcal C}\operatorname{Err}_{\pi,d}(V)}\le\frac{k-1}{\pi_{\min,d}}\exp\!\left(-\frac{d(A-B)}{2r_{\mathrm{te},0}^2}\right).
\label{eq:when2-quant-error-ratio}
\end{equation}
In particular, a fixed positive coefficient gap gives strict uniform accuracy ordering for all sufficiently large \(d\), without any additional probabilistic loss.
\end{lemma}

\begin{proof}
The directed pair error equals \(\overline\Phi(\sqrt{dL_{i\to j}(W;\Sigma)}/r_{\mathrm{te}})\).
A classwise union bound gives the upper bound in \eqref{eq:when2-diet-tail}, and a pair attaining \(E_d\) gives the lower bound.
For \(x\ge y\ge0\), changing variables in the Gaussian tail integral gives
\[
\overline\Phi(x)=\int_y^\infty\frac{e^{-(u+x-y)^2/2}}{\sqrt{2\pi}}\,du\le e^{-(x^2-y^2)/2}\overline\Phi(y).
\]
Applying this inequality with \(x=\sqrt{dA}/r_{\mathrm{te}}\) and \(y=\sqrt{dB}/r_{\mathrm{te}}\) proves \eqref{eq:when2-quant-error-ratio}.
The final assertion follows from \(k=\Theta(d)\) and \eqref{eq:when2-diet-prior}.
\end{proof}

\begin{lemma}[Deterministic whole-face stability]
\label{lem:when2-quant-deterministic-rigidity}
Fix \(h_*>0\) and \(K<\infty\).
There exist \(\varepsilon_0,C>0\), depending only on \(h_*,K\), with the following property.
Let \(0<\varepsilon\le\varepsilon_0\), \(k^{-1}\le\varepsilon\), and \(h\in[h_*,1]\).
Suppose deterministic means have no zero coordinates, and, with \(\bar\mu=k^{-1}\sum_i\mu_i\) and \(T_i=\operatorname{sign}(\mu_i)\), satisfy
\begin{align}
\max_i|\|\mu_i\|_2^2-1|&\le\varepsilon,& \max_{i\ne j}|\mu_i^\top\mu_j|&\le\varepsilon,& \|\bar\mu\|_2&\le\varepsilon,\nonumber\\
\max_i\left|\frac{\|\mu_i\|_1}{\sqrt d}-h\right|&\le\varepsilon,& \max_{i\ne j}\frac{|T_j^\top\mu_i|}{\sqrt d}&\le\varepsilon,& \max_{i\ne j}\frac{|T_i^\top T_j|}{d}&\le\varepsilon.
\label{eq:when2-quant-rigidity-input}
\end{align}
Assume also that, for \(s\in\{\varepsilon^{1/2},\varepsilon^{1/8}\}\),
\begin{equation}
\max_i\frac1d\#\{a:\sqrt d\,|\mu_i[a]|\le s\}\le Ks+\varepsilon.
\label{eq:when2-quant-small-ball-input}
\end{equation}
Then both solution sets are nonempty, and, writing \(t_{\max}:=\min_{W\in\mathcal F(M)}\|W\|_{\max}\),
\begin{align}
\sup_{W\in\mathcal M_{\mathrm{row}}(M)}\max_i\|w_i-\mu_i\|_2&\le C\varepsilon^{1/4},\label{eq:when2-quant-row-stability}\\
\left|\frac1{t_{\max}\sqrt d}-h\right|+\sup_{W\in\mathcal M_{\mathrm{max}}(M)}\max_i\frac{\|w_i/t_{\max}-T_i\|_2}{\sqrt d}&\le C\varepsilon^{1/16}.
\label{eq:when2-quant-max-stability}
\end{align}
\end{lemma}

\begin{proof}
The matrix with rows \(\mu_i/(1-2\varepsilon)\) is feasible, so the optimal row radius \(R_*\) is at most \(\sqrt{1+\varepsilon}/(1-2\varepsilon)\le1+C\varepsilon\).
For any feasible \(W\), averaging the class-\(i\) constraints over competitors gives
\begin{equation}
(w_i-\bar w)^\top\mu_i\ge1-\frac1k,\qquad \bar w:=\frac1k\sum_iw_i.
\label{eq:when2-quant-averaged-constraint}
\end{equation}
For a Row minimizer, Cauchy--Schwarz and the row variance identity therefore imply
\[
\frac{(1-1/k)^2}{1+\varepsilon}\le\frac1k\sum_i\|w_i-\bar w\|_2^2\le R_*^2-\|\bar w\|_2^2.
\]
Thus \(\|\bar w\|_2\le C\varepsilon^{1/2}\), and \eqref{eq:when2-quant-averaged-constraint} gives \(w_i^\top\mu_i\ge1-C\varepsilon^{1/2}\) uniformly in \(i\).
Expanding \(\|w_i-\mu_i\|_2^2\) proves \eqref{eq:when2-quant-row-stability}.

The sign matrix divided by \(\sqrt d(h-2\varepsilon)\) is feasible, so \(t_{\max}\le[\sqrt d(h-2\varepsilon)]^{-1}\).
Fix a Max minimizer and set \(H=W/t_{\max}\), \(\bar H=k^{-1}\sum_iH_i\), and
\[
D_i:=\frac{\|\mu_i\|_1-H_i^\top\mu_i}{\sqrt d}\ge0.
\]
Since \(\|H\|_{\max}=1\), averaging \eqref{eq:when2-quant-averaged-constraint} over \(i\) yields
\begin{align*}
\frac1k\sum_iD_i&\le h+\varepsilon-\frac{1-1/k}{t_{\max}\sqrt d}-\frac{\bar H^\top\bar\mu}{\sqrt d}\\
&\le h+\varepsilon-(1-1/k)(h-2\varepsilon)+\|\bar\mu\|_2\le C\varepsilon.
\end{align*}
For every \(s>0\), the coordinate bound on \(H\) gives
\begin{equation}
\frac{\|H_i-T_i\|_2^2}{d}\le\frac{2D_i}{s}+\frac4d\#\{a:\sqrt d\,|\mu_i[a]|\le s\}.
\label{eq:when2-diet-threshold-split}
\end{equation}
Indeed, outside the indicated small-coordinate set, use \((H_i[a]-T_i[a])^2\le2(1-H_i[a]T_i[a])\), and use the bound four on that set.
Taking \(s=\varepsilon^{1/2}\) and averaging gives \(\|H-T\|_{\mathrm F}^2/(kd)\le C\varepsilon^{1/2}\).
Moreover, \(\|k^{-1}\sum_iT_i\|_2^2/d\le k^{-1}+\varepsilon\le2\varepsilon\), so Jensen's inequality gives \(\|\bar H\|_2/\sqrt d\le C\varepsilon^{1/4}\).
Returning to the individual constraints now yields
\[
D_i\le h+\varepsilon-(1-1/k)(h-2\varepsilon)+\frac{\|\bar H\|_2\|\mu_i\|_2}{\sqrt d}\le C\varepsilon^{1/4}.
\]
Taking \(s=\varepsilon^{1/8}\) in \eqref{eq:when2-diet-threshold-split} proves the claimed uniform row error for \(H\).
Finally, averaging the constraints and using \(H_i^\top\mu_i\le\|\mu_i\|_1\) gives
\[
h-2\varepsilon\le\frac1{t_{\max}\sqrt d}\le\frac{h+2\varepsilon}{1-1/k}=h+O(\varepsilon).
\]
This proves \eqref{eq:when2-quant-max-stability}.
The feasible witnesses and compact norm sublevel sets ensure existence of both minima.
\end{proof}

\begin{lemma}[Quantitative whole-face geometry in aligned coordinates]
\label{lem:when2-diet-aligned-geometry}
Define
\begin{equation}
\bar\gamma_{1/2,d}:=\frac1d\sum_a\sqrt{\gamma_{a,d}},\qquad h_d:=g_0\bar\gamma_{1/2,d},\qquad \bar\gamma_{2,d}:=\frac1d\sum_a\gamma_{a,d}^2.
\label{eq:when2-diet-hd}
\end{equation}
There exist \(\varepsilon_0,c,C>0\), depending only on \(c_\Gamma,C_\Gamma\), such that, whenever \(0<\varepsilon\le\varepsilon_0\) and \(k\ge C\varepsilon^{-2}\), an event \(\mathcal A_d(\varepsilon)\) satisfies
\begin{equation}
\mathbb P(\mathcal A_d(\varepsilon)^c)\le Ck^2e^{-cd\varepsilon^2}.
\label{eq:when2-quant-aligned-probability}
\end{equation}
On this event, both solution sets are nonempty and
\begin{align}
\sup_{W\in\mathcal M_{\mathrm{row}}(M)}\max_i\|w_i-\mu_i\|_2&\le C\varepsilon^{1/4},\label{eq:when2-diet-row-rigidity}\\
\sup_{W\in\mathcal M_{\mathrm{max}}(M)}\max_i\frac{\|w_i/t_{\max}-T_i\|_2}{\sqrt d}&\le C\varepsilon^{1/16},\qquad T_i:=\operatorname{sign}(\mu_i).
\label{eq:when2-diet-max-rigidity}
\end{align}
Also, uniformly over \(i\ne j\),
\begin{align}
\bigl|\|\mu_i\|_2^2-1\bigr|+|\mu_i^\top\mu_j|&\le C\varepsilon,\nonumber\\
\left|\frac{(T_i-T_j)^\top\mu_i}{\sqrt d}-h_d\right|+\left|\frac{\|T_i-T_j\|_2^2}{2d}-1\right|&\le C\varepsilon.
\label{eq:when2-diet-aligned-laws}
\end{align}
For \(\Sigma_\lambda:=(1-\lambda)I_d+\lambda\Gamma_d\) and \(A_d(\lambda):=1+\lambda(\bar\gamma_{2,d}-1)\), the same event satisfies
\begin{align}
\sup_{\lambda\in[0,1]}\max_{i\ne j}\left|(\mu_i-\mu_j)^\top\Sigma_\lambda(\mu_i-\mu_j)-2A_d(\lambda)\right|&\le C\varepsilon,\nonumber\\
\sup_{\lambda\in[0,1]}\max_{i\ne j}\left|\frac{(T_i-T_j)^\top\Sigma_\lambda(T_i-T_j)}{2d}-1\right|&\le C\varepsilon.
\label{eq:when2-quant-aligned-energies}
\end{align}
\end{lemma}

\begin{proof}
We verify the inputs of Lemma~\ref{lem:when2-quant-deterministic-rigidity} at precision \(\varepsilon\).
The independent-coordinate Gaussian quadratic-form and Bernstein bounds used below follow, for example, from the concentration tools in \citet{vershynin2012nonasymptotic}.
For the quadratic terms, the identity \(\mathbb E e^{u(Z^2-1)}=e^{-u}(1-2u)^{-1/2}\) for \(|u|<1/2\), together with the bounded spectral weights, gives deviation probability at most \(2e^{-cd\varepsilon^2}\) for each normalized quadratic form.
Polarization gives the corresponding bilinear bounds.
The variables \(\sqrt{\gamma_{a,d}}|g_i[a]|\) have uniformly bounded exponential moments, so Bernstein's inequality gives
\[
\mathbb P\!\left(\left|\frac{\|\mu_i\|_1}{\sqrt d}-h_d\right|>\varepsilon\right)\le2e^{-cd\varepsilon^2}.
\]
For \(j\ne i\), conditioning on \(T_j\) gives \(T_j^\top\mu_i/\sqrt d\sim\mathcal N(0,1/d)\).
The products \(T_i[a]T_j[a]\) are independent Rademacher variables across \(a\), so their averages satisfy the same exponential bound.
For either prescribed threshold \(s\), Gaussian small-ball probabilities and Hoeffding's inequality give
\[
\mathbb P\!\left(\frac1d\#\{a:\sqrt d\,|\mu_i[a]|\le s\}>Ks+\varepsilon\right)\le e^{-2d\varepsilon^2},
\]
where \(K\) depends only on \(c_\Gamma\).
Also, \(\bar\mu\) has covariance \(\Gamma_d/(kd)\); when \(k\ge C\varepsilon^{-2}\), the event \(\|\bar\mu\|_2>\varepsilon\) has probability at most \(Ce^{-cd}\).
All coordinates are nonzero almost surely, and \(g_0\sqrt{c_\Gamma}\le h_d\le g_0<1\).
A union bound over these \(O(k^2)\) events verifies the deterministic lemma's hypotheses and proves \eqref{eq:when2-diet-row-rigidity}--\eqref{eq:when2-diet-aligned-laws}.

For \eqref{eq:when2-quant-aligned-energies}, it suffices to treat \(\lambda=0,1\), since the quadratic forms and their asserted centers are affine in \(\lambda\).
The Gaussian difference \(\mu_i-\mu_j\) has covariance \(2\Gamma_d/d\), giving mean quadratic form \(2d^{-1}\operatorname{tr}(\Gamma_d\Sigma_\lambda)=2A_d(\lambda)\).
The same quadratic-form bound applies at both endpoints.
For the sign contrasts, \((T_i[a]-T_j[a])^2\) are independent across coordinates, lie in \(\{0,4\}\), and have mean two.
Their bounded weighted sums therefore concentrate about \(2\operatorname{tr}\Sigma_\lambda=2d\) with failure at most \(2e^{-cd\varepsilon^2}\) at each endpoint.
Adding these events proves \eqref{eq:when2-quant-aligned-probability} and \eqref{eq:when2-quant-aligned-energies}.
\end{proof}

\begin{proposition}[An exact Row--Max transition under aligned test anisotropy]
\label{prop:when2-transition}
Suppose
\begin{equation}
\bar\gamma_{1/2,d}\to\bar\gamma_{1/2},\qquad \bar\gamma_{2,d}\to\bar\gamma_2,
\label{eq:when2-diet-profile-limits}
\end{equation}
and set
\begin{equation}
h:=g_0\bar\gamma_{1/2},\qquad h^2\bar\gamma_2>1.
\label{eq:when2-diet-h}
\end{equation}
For
\begin{equation}
\Sigma_\lambda:=(1-\lambda)I_d+\lambda\Gamma_d,\qquad \lambda\in[0,1],
\label{eq:when2-diet-sigma-t}
\end{equation}
there are constants \(\varepsilon_0,C>0\), depending only on the spectral bounds, such that, whenever \(0<\varepsilon\le\varepsilon_0\) and \(k\ge C\varepsilon^{-2}\), the event \(\mathcal A_d(\varepsilon)\) in Lemma~\ref{lem:when2-diet-aligned-geometry} satisfies
\begin{align}
\sup_{\lambda\in[0,1]}\sup_{W\in\mathcal M_{\mathrm{row}}(M)}\left|E_d(W;\Sigma_\lambda)-\frac1{2A_d(\lambda)}\right|&\le C\varepsilon^{1/4},\nonumber\\
\sup_{\lambda\in[0,1]}\sup_{W\in\mathcal M_{\mathrm{max}}(M)}\left|E_d(W;\Sigma_\lambda)-\frac{h_d^2}{2}\right|&\le C\varepsilon^{1/16}.
\label{eq:when2-diet-transition-coefficients}
\end{align}
The limiting coefficients cross uniquely at
\begin{equation}
\lambda_c:=\frac{h^{-2}-1}{\bar\gamma_2-1}\in(0,1).
\label{eq:when2-diet-tc}
\end{equation}
For every fixed \(\lambda\ne\lambda_c\), there exist \(c_\lambda,C_\lambda>0\), independent of \(d\), such that, for every sufficiently large \(d\),
\begin{align}
\lambda<\lambda_c\quad&\Longrightarrow\quad\mathbb P\!\left\{\inf_{W\in\mathcal M_{\mathrm{row}}}\operatorname{Acc}_{\pi,d}(W)>\sup_{W\in\mathcal M_{\mathrm{max}}}\operatorname{Acc}_{\pi,d}(W)\right\}\ge1-C_\lambda e^{-c_\lambda d},\label{eq:when2-diet-row-win}\\
\lambda>\lambda_c\quad&\Longrightarrow\quad\mathbb P\!\left\{\inf_{W\in\mathcal M_{\mathrm{max}}}\operatorname{Acc}_{\pi,d}(W)>\sup_{W\in\mathcal M_{\mathrm{row}}}\operatorname{Acc}_{\pi,d}(W)\right\}\ge1-C_\lambda e^{-c_\lambda d}.
\label{eq:when2-diet-max-win}
\end{align}
At \(\lambda=\lambda_c\), only equality of the limiting coefficients is asserted.
\end{proposition}

\begin{proof}
The spectral bounds give \(c_\Gamma I_d\preceq\Sigma_\lambda\preceq C_\Gamma I_d\) for every \(\lambda\in[0,1]\), since trace normalization implies \(c_\Gamma\le1\le C_\Gamma\).
For Row, \eqref{eq:when2-diet-row-rigidity} allows us to replace \(w_i-w_j\) by \(\mu_i-\mu_j\) with Euclidean error \(C\varepsilon^{1/4}\).
The latter vector has directed signal \(1+O(\varepsilon)\) and test-noise quadratic form \(2A_d(\lambda)+O(\varepsilon)\), uniformly over pairs and \(\lambda\).
The denominators are bounded away from zero, proving the first estimate in \eqref{eq:when2-diet-transition-coefficients}.
For Max, scale invariance allows us to use the contrast \((H_i-H_j)/\sqrt d\), where \(H=W/t_{\max}\).
By \eqref{eq:when2-diet-max-rigidity}, this contrast differs from \((T_i-T_j)/\sqrt d\) by at most \(C\varepsilon^{1/16}\).
The canonical sign contrast has directed signal \(h_d+O(\varepsilon)\) and test-noise quadratic form \(2+O(\varepsilon)\), proving the second estimate.
These arguments also prove the same estimates for every individual \(L_{i\to j}\).

Since \(\bar\gamma_{1/2}\le1\), we have \(0<h<1\); the condition \(h^2\bar\gamma_2>1\) implies \(\bar\gamma_2>1\) and \(0<\lambda_c<1\).
Fix \(\lambda\ne\lambda_c\), and let
\[
\Delta_\lambda:=\left|\frac1{2[1+\lambda(\bar\gamma_2-1)]}-\frac{h^2}{2}\right|>0.
\]
Choose a fixed \(\varepsilon>0\) small enough that the sum of the two errors in \eqref{eq:when2-diet-transition-coefficients} is at most \(\Delta_\lambda/4\).
By \eqref{eq:when2-diet-profile-limits}, the finite-dimensional coefficient gap has the asserted sign and magnitude at least \(3\Delta_\lambda/4\) for all sufficiently large \(d\).
Thus the complete-face gap is at least \(\Delta_\lambda/2\) on \(\mathcal A_d(\varepsilon)\).
Lemma~\ref{lem:when2-diet-error-transfer} proves the corresponding strict accuracy ordering.
Finally, \(k=\Theta(d)\) allows the factor \(k^2\) in \eqref{eq:when2-quant-aligned-probability} to be absorbed into \(C_\lambda e^{-c_\lambda d}\).
No convergence rate for either spectral summary is required; their convergence affects only the sufficiently-large-d threshold.
\end{proof}

The Haar-rotation statements below hold equally for Haar measure on \(O(d)\) and on \(\mathrm{SO}(d)\). Indeed, for the fixed orthonormal frames used below, both measures induce the same uniform Stiefel distribution, and, for sufficiently large \(d\), conditional on a fixed two-frame, the unexposed direction is uniform on the orthogonal sphere in either case.

\begin{lemma}[Quantitative sign statistics under a Haar rotation]
\label{lem:when2-quant-haar-sign}
Let deterministic vectors \(x_1,\ldots,x_k\in\mathbb R^d\) satisfy \(1/2\le\|x_i\|_2\le2\), let \(R\) be Haar on \(O(d)\) or \(\mathrm{SO}(d)\), and put \(T_i=\operatorname{sign}(Rx_i)\).
There are universal \(\varepsilon_0,c,C>0\) such that, for \(0<\varepsilon\le\varepsilon_0\) and any two prescribed thresholds \(s_1,s_2\in(0,1]\), an event of probability at least \(1-Ck^2e^{-cd\varepsilon^6}\) satisfies
\begin{align}
\max_{i,j}\left|\frac{T_i^\top Rx_j}{\sqrt d}-g_0\frac{x_i^\top x_j}{\|x_i\|_2}\right|&\le C\varepsilon,\nonumber\\
\max_{i,j}\left|\frac{T_i^\top T_j}{d}-\frac2\pi\arcsin\!\left(\frac{x_i^\top x_j}{\|x_i\|_2\|x_j\|_2}\right)\right|&\le C\varepsilon,\nonumber\\
\max_i\frac1d\#\{a:\sqrt d\,|(Rx_i)[a]|\le s_\ell\}&\le Cs_\ell+C\varepsilon^2,\qquad \ell=1,2.
\label{eq:when2-quant-haar-statistics}
\end{align}
\end{lemma}

\begin{proof}
Fix a pair and choose an orthonormal basis \(F\in\mathbb R^{d\times m}\) of its span, where \(m\le2\).
Write \(x_i=Fu_i\), \(x_j=Fu_j\), and let \(Z\in\mathbb R^{d\times m}\) have independent standard Gaussian entries.
Then
\[
\sqrt d\,RF\ \overset{\mathrm d}=\ Y:=ZB,\qquad B:=(Z^\top Z/d)^{-1/2}.
\]
Indeed, \(Y/\sqrt d\) has orthonormal columns and its distribution is invariant under every left orthogonal transformation.
The Gaussian quadratic-form bound gives
\[
\mathbb P\!\left(\|Z^\top Z/d-I_m\|_{\mathrm{op}}>\varepsilon^3\right)\le Ce^{-cd\varepsilon^6}.
\]
On the complementary event, \(\|B-I_m\|_{\mathrm{op}}\le C\varepsilon^3\), and hence
\[
\frac1d\|(Y-Z)u\|_2^2\le C\varepsilon^6
\]
for either coefficient vector \(u\).
Set \(X=Zu\), \(\widetilde X=Yu\), and \(b=\varepsilon^2\).
Gaussian small-ball probabilities and Hoeffding's inequality imply \(d^{-1}\#\{a:|X_a|\le b\}\le C\varepsilon^2\), except with probability \(Ce^{-cd\varepsilon^4}\).
Consequently,
\[
\frac1d\#\{a:\operatorname{sign}(\widetilde X_a)\ne\operatorname{sign}(X_a)\}\le\frac1d\#\{a:|X_a|\le b\}+\frac{\|\widetilde X-X\|_2^2}{db^2}\le C\varepsilon^2.
\]
Thus the normalized Euclidean distance between the corresponding sign vectors is at most \(C\varepsilon\).
By Cauchy--Schwarz, replacing \(Y\) by \(Z\) changes the empirical averages involving one sign factor and one linear factor by at most \(C\varepsilon\), while the empirical correlations between the corresponding sign vectors change by at most \(C\varepsilon^2\).
For a standard Gaussian row \(z^\top\) of \(Z\), the corresponding expectations are
\[
\mathbb E[\operatorname{sign}(z^\top u)z^\top v]=g_0\frac{u^\top v}{\|u\|_2},\qquad \mathbb E[\operatorname{sign}(z^\top u)\operatorname{sign}(z^\top v)]=\frac2\pi\arcsin\!\left(\frac{u^\top v}{\|u\|_2\|v\|_2}\right).
\]
The first identity follows by conditioning \(z^\top v\) on \(z^\top u\), and the second is the standard arcsine formula for the signs of two jointly Gaussian variables.
Bernstein's and Hoeffding's inequalities concentrate the corresponding empirical averages within \(\varepsilon\) with failure at most \(Ce^{-cd\varepsilon^2}\).
Finally, for either prescribed \(s\),
\[
\frac1d\#\{a:|\widetilde X_a|\le s\}\le\frac1d\#\{a:|X_a|\le s+b\}+\frac{\|\widetilde X-X\|_2^2}{db^2}\le Cs+C\varepsilon^2
\]
except with probability \(Ce^{-cd\varepsilon^4}\).
A union bound over the pairs and the two thresholds proves \eqref{eq:when2-quant-haar-statistics}; independence between different pairs is not needed.
\end{proof}

\begin{lemma}[Generic common rotation restores the Row coefficient advantage]
\label{lem:when2-diet-haar-recovery}
For deterministic means, define
\begin{equation}
\delta_d:=\max_i|\|\mu_i\|_2^2-1|+\max_{i\ne j}|\mu_i^\top\mu_j|.
\label{eq:when2-diet-regular}
\end{equation}
Let \(\Sigma\succ0\) satisfy \(\operatorname{tr}\Sigma=d\) and \(c_\Sigma I_d\preceq\Sigma\preceq C_\Sigma I_d\), and set
\[
\mu_i^\circ=R\mu_i,\qquad \Sigma^\circ=R\Sigma R^\top,\qquad v_{ij}:=\frac{\mu_i-\mu_j}{\|\mu_i-\mu_j\|_2},\qquad a_{ij}:=v_{ij}^\top\Sigma v_{ij},
\]
where \(R\) is an independent Haar matrix.
There are \(\varepsilon_0,c,C>0\), depending only on \(c_\Sigma,C_\Sigma\), such that, whenever
\[
0<\varepsilon\le\varepsilon_0,\qquad \delta_d+d^{-1}+k^{-1}\le\varepsilon^2,
\]
an event of probability at least \(1-Ck^2e^{-cd\varepsilon^6}\), over \(R\) alone, satisfies
\begin{align}
\sup_{W\in\mathcal M_{\mathrm{row}}^\circ}\max_{i\ne j}\left|L_{i\to j}(W;\Sigma^\circ)-\frac1{2a_{ij}}\right|&\le C\varepsilon^{1/4},\nonumber\\
\sup_{W\in\mathcal M_{\mathrm{max}}^\circ}\max_{i\ne j}\left|L_{i\to j}(W;\Sigma^\circ)-\frac1{2a_{ij}+\pi-2}\right|&\le C\varepsilon^{1/16}.
\label{eq:when2-diet-haar-coefficients}
\end{align}
Here the superscript \(\circ\) denotes solution sets and accuracies for the rotated problem.
Consequently, if \(\delta_d\to0\), \(k/d\to\alpha\in(0,\infty)\), and \eqref{eq:when2-diet-prior} holds, there exist \(c_*,C_*>0\), independent of \(d\), such that, for every sufficiently large \(d\),
\begin{equation}
\mathbb P_R\!\left\{\inf_{W\in\mathcal M_{\mathrm{row}}^\circ}\operatorname{Acc}_{\pi,d}^\circ(W)>\sup_{W\in\mathcal M_{\mathrm{max}}^\circ}\operatorname{Acc}_{\pi,d}^\circ(W)\right\}\ge1-C_*e^{-c_*d}.
\label{eq:when2-diet-haar-row-win}
\end{equation}
\end{lemma}

\begin{proof}
The hypotheses imply \(\|\bar\mu\|_2^2\le(1+\delta_d)/k+\delta_d\le C\varepsilon^2\).
Apply Lemma~\ref{lem:when2-quant-haar-sign} to the means.
The norm and coherence assumptions, together with \eqref{eq:when2-quant-haar-statistics}, give, for \(T_i=\operatorname{sign}(R\mu_i)\),
\begin{equation}
\max_i\left|\frac{\|R\mu_i\|_1}{\sqrt d}-g_0\right|+\max_{i\ne j}\frac{|T_j^\top R\mu_i|}{\sqrt d}+\max_{i\ne j}\frac{|T_i^\top T_j|}{d}\le C\varepsilon.
\label{eq:when2-diet-haar-sign-laws}
\end{equation}
Choose a fixed sufficiently large constant \(C_0\) and apply that lemma with thresholds \((C_0\varepsilon)^{1/2}\) and \((C_0\varepsilon)^{1/8}\).
After decreasing \(\varepsilon_0\), Lemma~\ref{lem:when2-quant-deterministic-rigidity} applies at precision \(C_0\varepsilon\), with \(h=g_0\).
It follows that every Row minimizer is within \(C\varepsilon^{1/4}\) per row of the rotated means, while every rescaled Max minimizer is within \(C\varepsilon^{1/16}\sqrt d\) per row of \(T_i\).

It remains to calculate the test-noise energy of the canonical sign contrast
\[
b_{ij}:=\frac{T_i-T_j}{\|T_i-T_j\|_2}.
\]
Fix a pair, let \(S=\operatorname{span}\{\mu_i,\mu_j\}\), and choose an orthonormal basis \(F\in\mathbb R^{d\times2}\) of \(S\).
The Gram matrix of the pair is uniformly well-conditioned under the assumed bound on \(\delta_d\).
Conditioning only on \(RF\), the vector \(b_{ij}\) is fixed whenever it is defined.
Set \(p=(RF)^\top b_{ij}\), \(q=Fp\), and \(r^2=1-\|p\|_2^2\).
The unexposed action of \(R\) on the orthogonal complement is Haar, so conditionally
\begin{equation}
R^\top b_{ij}=q+r\zeta,
\label{eq:when2-diet-haar-decomposition}
\end{equation}
where \(\zeta\) is uniform on the unit sphere of \(S^\perp\).
On the sign-statistics event, \eqref{eq:when2-quant-haar-statistics} gives
\[
b_{ij}^\top R\mu_i=\frac{g_0}{\sqrt2}+O(\varepsilon),\qquad b_{ij}^\top R\mu_j=-\frac{g_0}{\sqrt2}+O(\varepsilon).
\]
The same two inner products for \(g_0v_{ij}\) are \(g_0/\sqrt2+O(\delta_d)\) and \(-g_0/\sqrt2+O(\delta_d)\).
Inverting the well-conditioned two-by-two Gram matrix therefore yields
\begin{equation}
\|q-g_0v_{ij}\|_2\le C\varepsilon,\qquad |r^2-(1-g_0^2)|\le C\varepsilon.
\label{eq:when2-diet-haar-projection}
\end{equation}
For the conditional complement calculation, let \(P_S\) denote the orthogonal projector onto \(S\).
Representing \(\zeta\) by a normalized standard Gaussian in \(S^\perp\) and applying the bounded-weight quadratic-form and Gaussian linear-form bounds gives
\[
\left|\zeta^\top\Sigma\zeta-\frac{\operatorname{tr}((I_d-P_S)\Sigma)}{d-2}\right|+|q^\top\Sigma\zeta|\le C\varepsilon
\]
with conditional failure at most \(Ce^{-cd\varepsilon^2}\), uniformly in the exposed frame.
The trace ratio is \(1+O(d^{-1})\), since \(\operatorname{tr}\Sigma=d\) and \(\operatorname{rank}P_S=2\).
The conditional failure bound can be integrated and then union-bounded over pairs; we do not condition on the sign vectors of the other classes.
Combining this bound with \eqref{eq:when2-diet-haar-decomposition} and \eqref{eq:when2-diet-haar-projection} gives
\begin{equation}
\max_{i\ne j}\left|b_{ij}^\top\Sigma^\circ b_{ij}-(g_0^2a_{ij}+1-g_0^2)\right|\le C\varepsilon.
\label{eq:when2-diet-haar-energy}
\end{equation}
The canonical Row direction \(Rv_{ij}\) has directed signal \(1/\sqrt2+O(\delta_d)\) and test-noise energy exactly \(a_{ij}\).
The canonical Max direction has directed signal \(g_0/\sqrt2+O(\varepsilon)\), and
\[
\frac{g_0^2/2}{g_0^2a_{ij}+1-g_0^2}=\frac1{2a_{ij}+\pi-2}.
\]
Both canonical contrast norms are bounded away from zero after the Max contrasts are divided by \(\sqrt d\).
Normalization stability, the whole-face bounds already proved, and \(\Sigma^\circ\succeq c_\Sigma I_d\) therefore transfer these coefficients uniformly to the complete solution sets, proving \eqref{eq:when2-diet-haar-coefficients}.
The combined failure is at most \(Ck^2e^{-cd\varepsilon^6}\).

To pass to the minimum directed coefficient, put \(a_*:=\max_{i\ne j}a_{ij}\).
Then \eqref{eq:when2-diet-haar-coefficients} implies
\[
\inf_{W\in\mathcal M_{\mathrm{row}}^\circ}E_d(W;\Sigma^\circ)-\sup_{W\in\mathcal M_{\mathrm{max}}^\circ}E_d(W;\Sigma^\circ)\ge\frac1{2a_*}-\frac1{2a_*+\pi-2}-C\varepsilon^{1/16}.
\]
Since \(a_*\le C_\Sigma\), the difference of the two displayed fractions is at least
\begin{equation}
\Delta_{\mathrm{rot}}:=\frac{\pi-2}{2C_\Sigma(2C_\Sigma+\pi-2)}>0.
\label{eq:when2-quant-rotation-gap}
\end{equation}
Choose a fixed \(\varepsilon>0\) so that \(C\varepsilon^{1/16}\le\Delta_{\mathrm{rot}}/2\).
The regularity assumptions eventually ensure \(\delta_d+d^{-1}+k^{-1}\le\varepsilon^2\).
Lemma~\ref{lem:when2-diet-error-transfer} gives the strict accuracy ordering, and absorbing \(k^2\) into the exponential proves \eqref{eq:when2-diet-haar-row-win}.
\end{proof}

\begin{proposition}[Opposite Row--Max orderings before and after a common Haar rotation]
\label{prop:when2-rotation-flip}
Under the hypotheses of Proposition~\ref{prop:when2-transition}, fix \(\lambda\in(\lambda_c,1]\).
Compare the original problem with means \(\mu_i\) and covariance \(\Sigma_\lambda\) to the problem obtained by the independent common rotation \(\mu_i\mapsto R\mu_i\), \(\Sigma_\lambda\mapsto R\Sigma_\lambda R^\top\).
There exist \(c,C>0\), independent of \(d\), such that, for every sufficiently large \(d\), with joint probability at least \(1-Ce^{-cd}\),
\begin{align}
\inf_{W\in\mathcal M_{\mathrm{max}}}\operatorname{Acc}_{\pi,d}(W)&>\sup_{W\in\mathcal M_{\mathrm{row}}}\operatorname{Acc}_{\pi,d}(W),\nonumber\\
\inf_{W\in\mathcal M_{\mathrm{row}}^\circ}\operatorname{Acc}_{\pi,d}^\circ(W)&>\sup_{W\in\mathcal M_{\mathrm{max}}^\circ}\operatorname{Acc}_{\pi,d}^\circ(W).
\label{eq:when2-diet-rotation-flip}
\end{align}
The two problems have identical class Gram matrices, class-mean norms and pair distances, mean singular values, test-noise eigenvalues, and mean-contrast test-noise variances.
\end{proposition}

\begin{proof}
Proposition~\ref{prop:when2-transition} gives the first ordering with exponentially small failure.
For the second ordering, choose the fixed precision \(\varepsilon\) in the last paragraph of Lemma~\ref{lem:when2-diet-haar-recovery}.
Gaussian quadratic and bilinear concentration gives
\[
\mathbb P\{\delta_d>\varepsilon^2/2\}\le Ck^2e^{-cd\varepsilon^4}.
\]
For all sufficiently large \(d\), we also have \(d^{-1}+k^{-1}\le\varepsilon^2/2\).
Conditional on any means in this event, Lemma~\ref{lem:when2-diet-haar-recovery} applies with the same constants, since \(c_\Gamma I_d\preceq\Sigma_\lambda\preceq C_\Gamma I_d\).
Its conditional failure is at most \(Ck^2e^{-cd\varepsilon^6}\).
Integrating this uniform conditional bound and intersecting with the first ordering proves \eqref{eq:when2-diet-rotation-flip}.
The invariances follow directly from orthogonality; in particular,
\[
(Rv_{ij})^\top(R\Sigma_\lambda R^\top)(Rv_{ij})=v_{ij}^\top\Sigma_\lambda v_{ij}.
\]
\end{proof}

\begin{proof}[Proof of Proposition~\ref{prop:main-coordinate-reversal}]
Lemma~\ref{lem:app-common-rescaling} identifies the direction sets and accuracies in the main-text normalization with those analyzed here.
Since \(h=g_0\bar\gamma_{1/2}\), the transition point in Proposition~\ref{prop:when2-transition} is
\[
\frac{h^{-2}-1}{\bar\gamma_2-1}=\frac{\pi/(2\bar\gamma_{1/2}^2)-1}{\bar\gamma_2-1}=\lambda_c.
\]
That proposition therefore gives the diagonal complete-face ordering with probability at least \(1-C_\lambda e^{-c_\lambda d}\) for every fixed \(\lambda\ne\lambda_c\) and sufficiently large \(d\).

For every fixed \(\lambda\in[0,1]\), the conditional argument in the proof of Proposition~\ref{prop:when2-rotation-flip} gives the rotated Row-over-Max ordering with probability at least \(1-C_{\mathrm{rot}}e^{-c_{\mathrm{rot}}d}\).
This argument does not require \(\lambda>\lambda_c\).
Its spectral bounds and the positive gap in \eqref{eq:when2-quant-rotation-gap} are uniform in \(\lambda\), so \(c_{\mathrm{rot}},C_{\mathrm{rot}}\) can be chosen independently of \(\lambda\).

At each fixed dimension on these events, Proposition~\ref{cor:app-row-full-batch} and Theorem~\ref{thm:app-ms-full-batch-adam}, together with Lemma~\ref{lem:exact-margin-direction-convergence}, place the normalized RowGDM and Adam trajectories asymptotically in their respective Row and Max maximum-margin direction sets.
Every classifier in these direction sets has nonzero pairwise row differences.
The compactness and continuity argument of Lemma~\ref{lem:app-acc-trajectory-continuity} applies to the present positive-definite Gaussian test covariances as well, since all score-tie hyperplanes have probability zero.
It transfers each uniform complete-face ordering to the stated strict comparison between the trajectory liminf and limsup accuracies, without an additional exceptional event or a directional-convergence assumption.
The same argument applies to RowGD.
\end{proof}

\begin{corollary}[Spectral dispersion implies a transition]
\label{cor:app-coordinate-dispersed-profiles}
Suppose the assumptions of Proposition~\ref{prop:main-coordinate-reversal} concerning the existence of the limiting spectral summaries hold.
Assume additionally that, for some fixed \(0<\underline\gamma<1<\overline\gamma<\infty\),
\[
\underline\gamma\le\gamma_{a,d}\le\overline\gamma\qquad\text{for all }a,d,
\]
and that
\[
v_\Gamma:=\lim_{d\to\infty}\frac1d\sum_{a=1}^d(\gamma_{a,d}-1)^2>\frac{\pi}{2q_*^2}-1,\qquad q_*:=\frac{(\overline\gamma-1)\sqrt{\underline\gamma}+(1-\underline\gamma)\sqrt{\overline\gamma}}{\overline\gamma-\underline\gamma}.
\]
Then \(2\bar\gamma_{1/2}^2\bar\gamma_2/\pi>1\), and hence the transition point in Proposition~\ref{prop:main-coordinate-reversal} lies in \((0,1)\).
In particular, the transition condition does not require a special spectral shape: within a fixed bounded spectral range, the displayed lower bound on the limiting spectral variance is sufficient.
\end{corollary}

\begin{proof}
Concavity of \(x\mapsto\sqrt{x}\) places its graph above the chord joining \((\underline\gamma,\sqrt{\underline\gamma})\) and \((\overline\gamma,\sqrt{\overline\gamma})\).
Averaging this chord inequality and using \(d^{-1}\sum_a\gamma_{a,d}=1\) gives \(\bar\gamma_{1/2}\ge q_*\).
Trace normalization also gives \(\bar\gamma_2=1+v_\Gamma\), and therefore
\[
\frac{2\bar\gamma_{1/2}^2\bar\gamma_2}{\pi}\ge\frac{2q_*^2(1+v_\Gamma)}{\pi}>1.
\]
\end{proof}

\begin{remark}[Non-vacuity of the dispersion condition]
\label{rem:app-coordinate-dispersion-nonvacuous}
Write \(l:=\underline\gamma\) and \(u:=\overline\gamma\).
The variance permitted by the spectral range is itself bounded.
Indeed, since \(l\le\gamma_{a,d}\le u\), we have \((\gamma_{a,d}-l)(u-\gamma_{a,d})\ge0\), or equivalently \(\gamma_{a,d}^2\le(l+u)\gamma_{a,d}-lu\).
Averaging and using the trace normalization \(d^{-1}\sum_a\gamma_{a,d}=1\) gives \(\bar\gamma_2\le l+u-lu\), and hence \(v_\Gamma=\bar\gamma_2-1\le(1-l)(u-1)\).
This bound is sharp: it is attained asymptotically by a two-point spectral profile supported on \(l\) and \(u\) with respective masses \((u-1)/(u-l)\) and \((1-l)/(u-l)\).

Consequently, the dispersion condition in Corollary~\ref{cor:app-coordinate-dispersed-profiles} is non-vacuous if and only if \((1-l)(u-1)>\pi/(2q_*^2)-1\).
Using \(q_*=(1+\sqrt{lu})/(\sqrt l+\sqrt u)\), this is equivalently
\[
\frac{2(1+\sqrt{lu})^2(l+u-lu)}{\pi(\sqrt l+\sqrt u)^2}>1.
\]
For each fixed \(l\in(0,1)\), the left-hand side is strictly increasing in \(u>1\), so there is a unique critical value \(u_c(l)>1\) above which the sufficient dispersion condition can be satisfied.
\end{remark}

\section{Additional Experiments and Experimental Details}
\label{app:additional-experiments}

\subsection{Synthetic Experiments: Setup and Additional Results}
\label{app:experimental-details}

\paragraph{Setup.}
The main-text experiment in Figure~\ref{fig:experiments} (left) uses \(k=60\) classes and dimension \(d=20\). For each class \(i\), we draw \(\mu_i\sim\mathcal N({\mathbf 0},I_d/d)\), draw the number of training examples independently and uniformly from \(\{1,\ldots,5\}\), and generate \(x_{i,r}^{\mathrm{tr}}=\mu_i+0.2236\,\varepsilon_{i,r}^{\mathrm{tr}}\) and \(x_{i,r}^{\mathrm{te}}=\mu_i+2.2361\,\varepsilon_{i,r}^{\mathrm{te}}\), where all residual vectors are independent \(\mathcal N({\mathbf 0},I_d/d)\). Test-class proportions are drawn independently from \(\operatorname{Dirichlet}(\mathbf 1_k)\), and each of the 30 paired seeds uses \(120{,}000\) evaluation examples. All methods within a seed use the same training and evaluation data.

\paragraph{Optimization and learning-rate selection.}
We train a bias-free multiclass linear classifier from \(W_0={\mathbf 0}\) with cross-entropy for \(20{,}000\) updates in float64 and use \(\eta_t=\eta_0(t+1)^{-1/2}\). For every optimizer and seed, we sweep the common effective-Frobenius coefficients \(c\in\{0.03,0.1,0.3,1,3,10,30\}\), setting \(\eta_0=c/\|\Delta_0\|_{\mathrm F}\), and select \(c\) by maximum final evaluation accuracy, with ties broken by lower training loss and then smaller \(c\). Because the same evaluation sample is used for learning-rate selection and reporting, these results are comparisons under the stated discrete sweep rather than untouched-test estimates; some RowGD and RowGDM runs select the lower endpoint of the grid.

\paragraph{Additional deterministic and stochastic comparisons.}
Figure~\ref{fig:synthetic-core-appendix} complements the six-optimizer deterministic comparison in Figure~\ref{fig:experiments} (left) in two ways. The left panel adds Signum to give the full deterministic comparison among RowGD, RowGDM, SpecGD, Muon, SignGD, Signum, and Adam, while the right panel compares the corresponding random-reshuffling methods RR-RowSGDM, RR-Muon, RR-Signum, and RR-Adam. RowGDM, Muon, and Signum use momentum \(\beta=0.9\); SpecGD and Muon use the exact compact-SVD polar factor; and Adam uses \((\beta_1,\beta_2)=(0.9,0.999)\) without bias correction or a numerical-stability constant. For random reshuffling, we use batch size \(b=1\), draw an independent permutation of the finite training set at each epoch, use momentum \(\beta=0.99\) for RR-RowSGDM, RR-Muon, and RR-Signum, and use \((\beta_1,\beta_2)=(0.99,0.999)\) for RR-Adam. The deterministic panel preserves the Row-over-Spectral and Row-over-coordinate-wise ordering of the main-text experiment, and RR-RowSGDM likewise outperforms all three random-reshuffling comparators on every paired seed. RR-Adam is included as an empirical comparator; our theoretical stochastic Adam comparison is formulated through AdamProxy.

\begin{figure}[t]
\centering
\begin{minipage}[t]{0.49\linewidth}
\centering
\includegraphics[width=\linewidth]{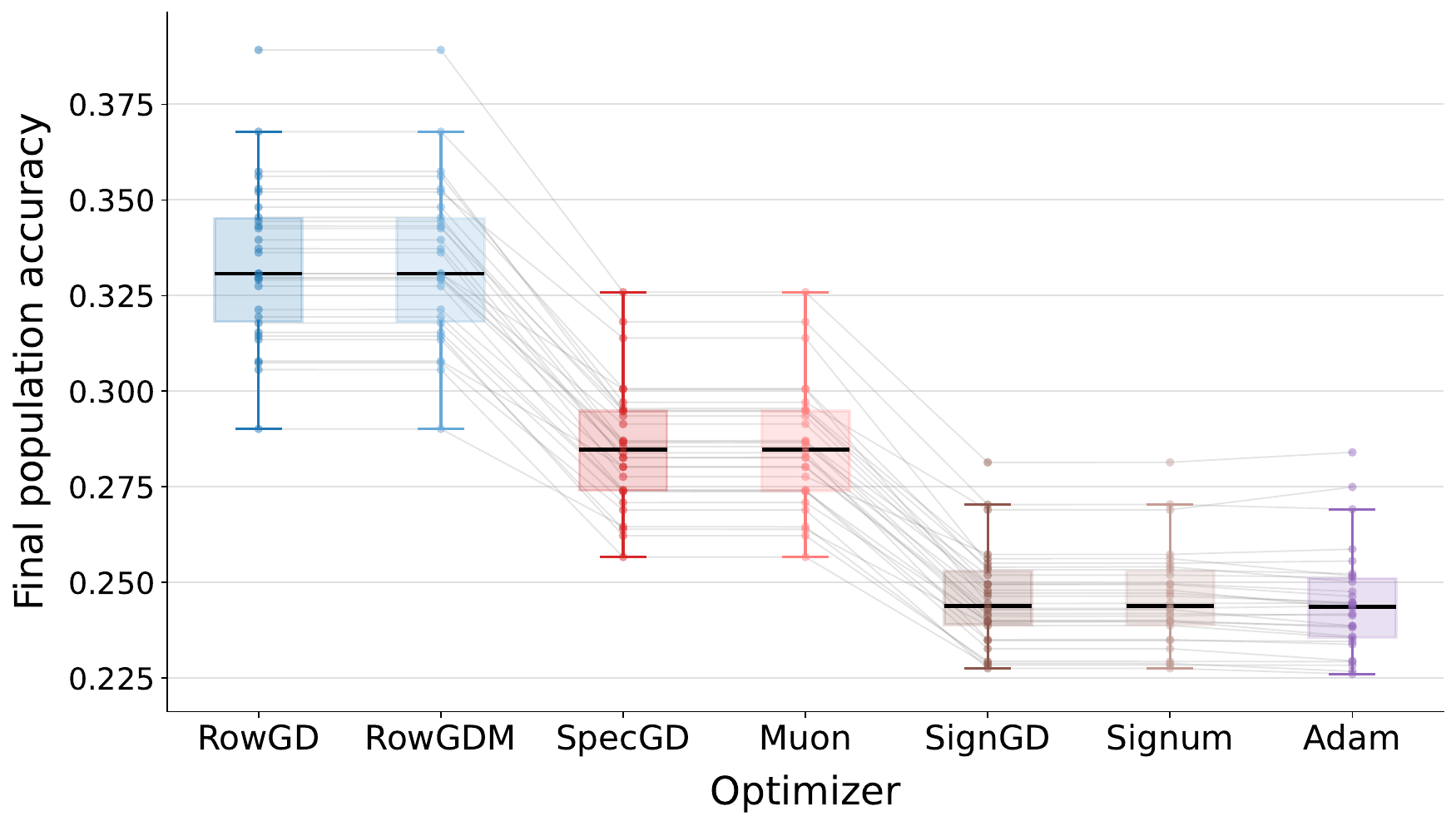}
\end{minipage}
\hfill
\begin{minipage}[t]{0.49\linewidth}
\centering
\includegraphics[width=\linewidth]{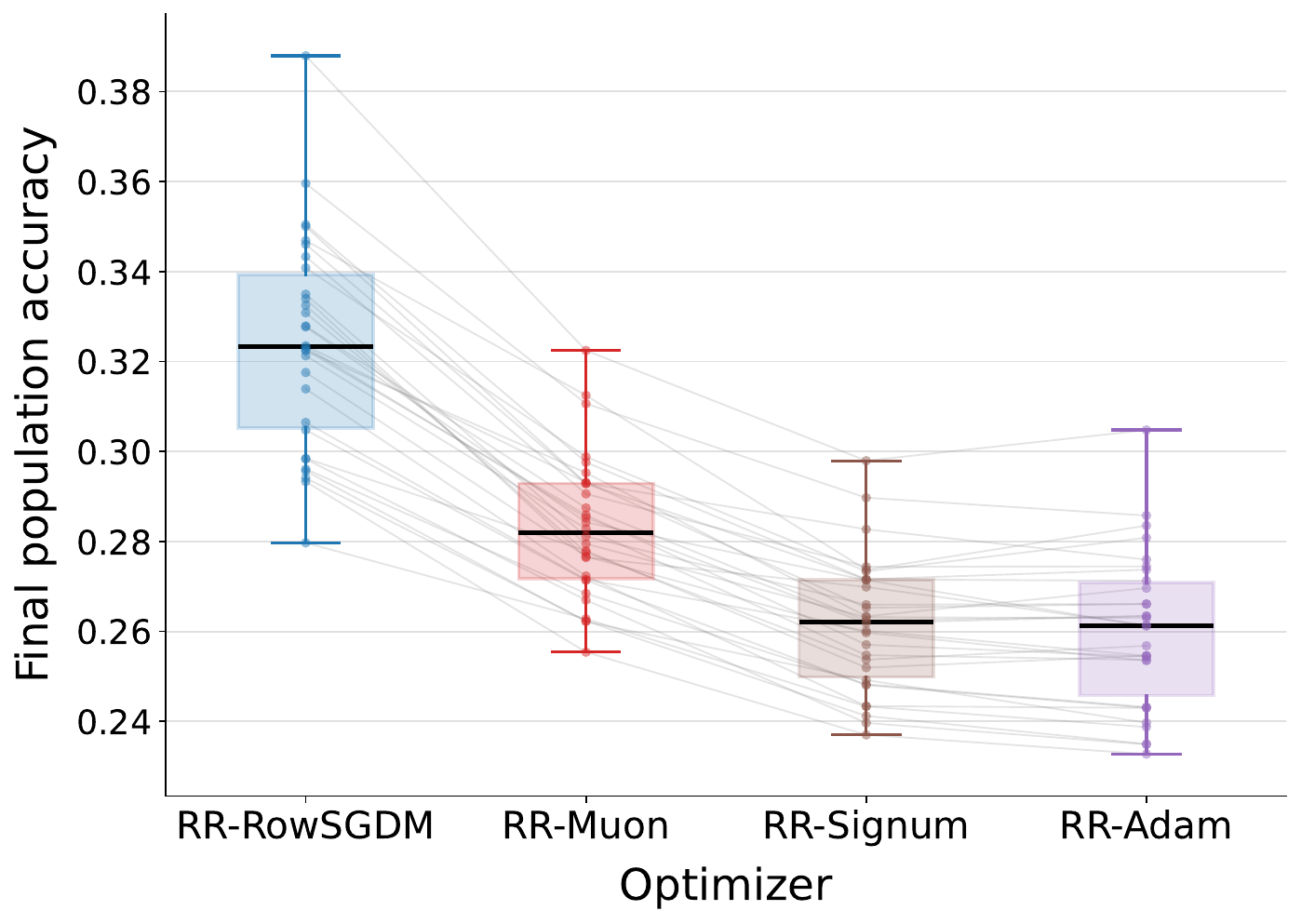}
\end{minipage}
\caption{\textbf{Additional optimizer comparisons on the baseline synthetic model.} \textbf{Left:} The deterministic comparison in Figure~\ref{fig:experiments} (left), augmented with Signum. \textbf{Right:} Random-reshuffling comparison of RR-RowSGDM, RR-Muon, RR-Signum, and RR-Adam with batch size \(b=1\). Points are paired seeds and gray lines connect methods evaluated on the same dataset.}
\label{fig:synthetic-core-appendix}
\end{figure}

\paragraph{Auxiliary geometries.}
Figure~\ref{fig:synthetic-auxiliary} supplements the main-text optimizer comparison with three reference methods that are not part of our theoretical ordering. Frobenius-normalized GD uses \(\Delta_t=G_t/\|G_t\|_{\mathrm F}\), ColumnGD independently normalizes each column of \(G_t\) in Euclidean norm, and ordinary GD uses \(\Delta_t=G_t\). The figure uses \textbf{FrobGD} as an abbreviation for Frobenius-normalized GD. These methods provide additional finite-dimensional context rather than additional theoretical comparators.

\begin{figure}[t]
\centering
\includegraphics[width=0.72\linewidth]{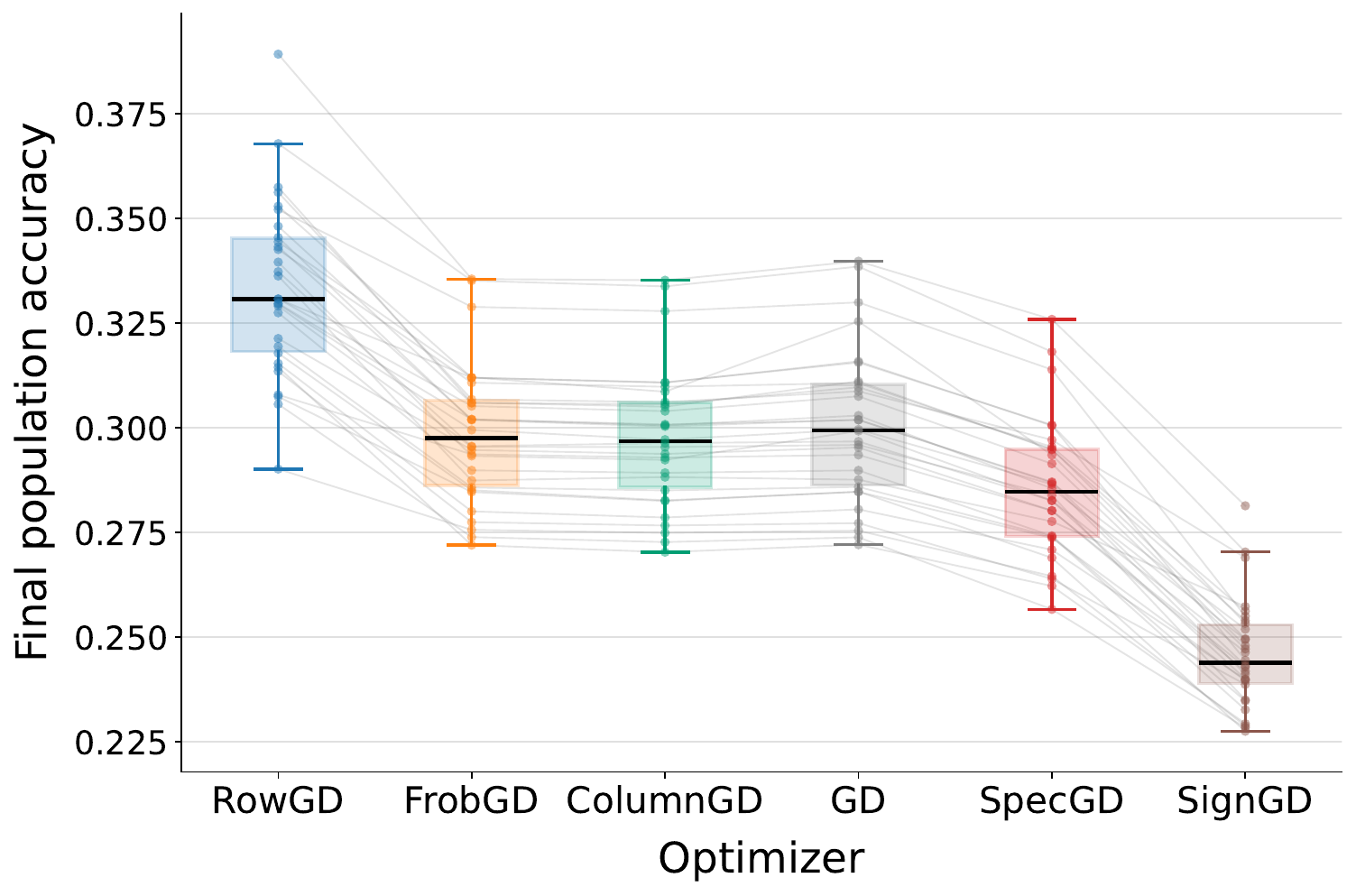}
\caption{\textbf{Auxiliary optimization geometries.} RowGD, FrobGD, ColumnGD, ordinary GD, SpecGD, and SignGD on the same 30 paired baseline datasets, where FrobGD denotes Frobenius-normalized GD.}
\label{fig:synthetic-auxiliary}
\end{figure}

\paragraph{Robustness of the deterministic ordering.}
To test whether the ordering in Figure~\ref{fig:experiments} (left) depends on the particular baseline configuration, we repeat the same six-optimizer comparison under eleven one-factor-at-a-time variations. We vary the training-noise coefficient from \(0.2236\) to \(0.0447\) or \(0.4472\), the test-noise coefficient from \(2.2361\) to \(1.3416\) or \(3.1305\), the test-class proportions from \(\operatorname{Dirichlet}(\mathbf 1_k)\) to balanced classes or \(\operatorname{Dirichlet}(0.3\mathbf 1_k)\), the number of training examples per class to exactly one or five, the problem size to \((k,d)=(30,20)\) or \((120,40)\), and the class-mean distribution from Gaussian to uniform on the unit sphere. In all Gaussian-noise settings, the residual vectors are drawn from \(\mathcal N({\mathbf 0},I_d/d)\); in the \((120,40)\) setting, the training- and test-noise coefficients are \(0.3162\) and \(3.1623\), respectively. Each setting uses 30 paired seeds and the same training and learning-rate protocol as the main-text experiment. Across all settings, RowGD and RowGDM have the two highest mean accuracies; each outperforms SignGD and Adam on every paired seed, and each outperforms SpecGD and Muon on all 30 seeds except in the one- and five-training-example settings, where the Row methods win 29 of 30 comparisons. Figures~\ref{fig:synthetic-robust-noise}--\ref{fig:synthetic-robust-size-means} show that the main-text ordering is stable across these changes.

\begin{figure}[t]
\centering
\begin{minipage}[t]{0.49\linewidth}
\centering
\includegraphics[width=\linewidth]{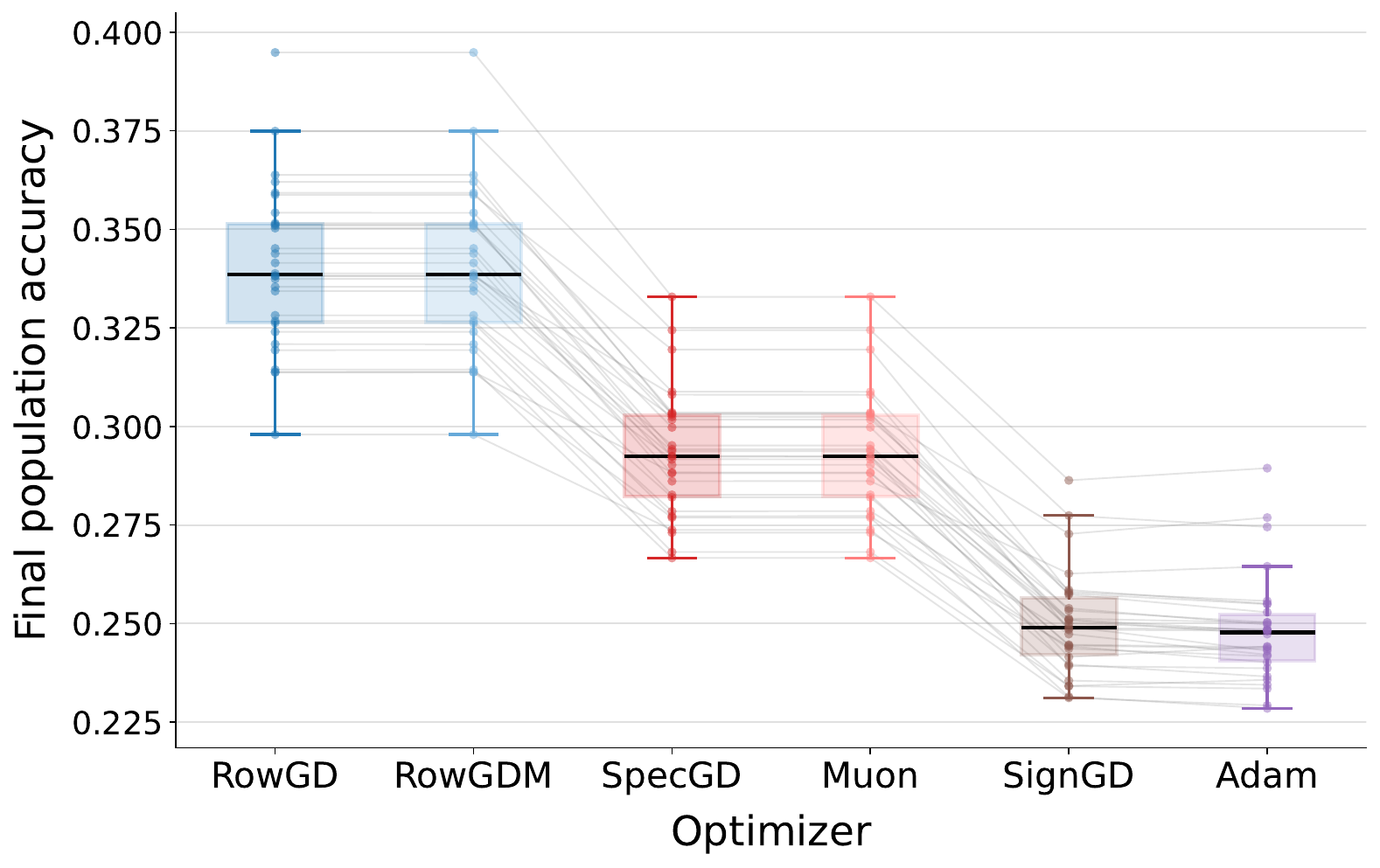}
\end{minipage}
\hfill
\begin{minipage}[t]{0.49\linewidth}
\centering
\includegraphics[width=\linewidth]{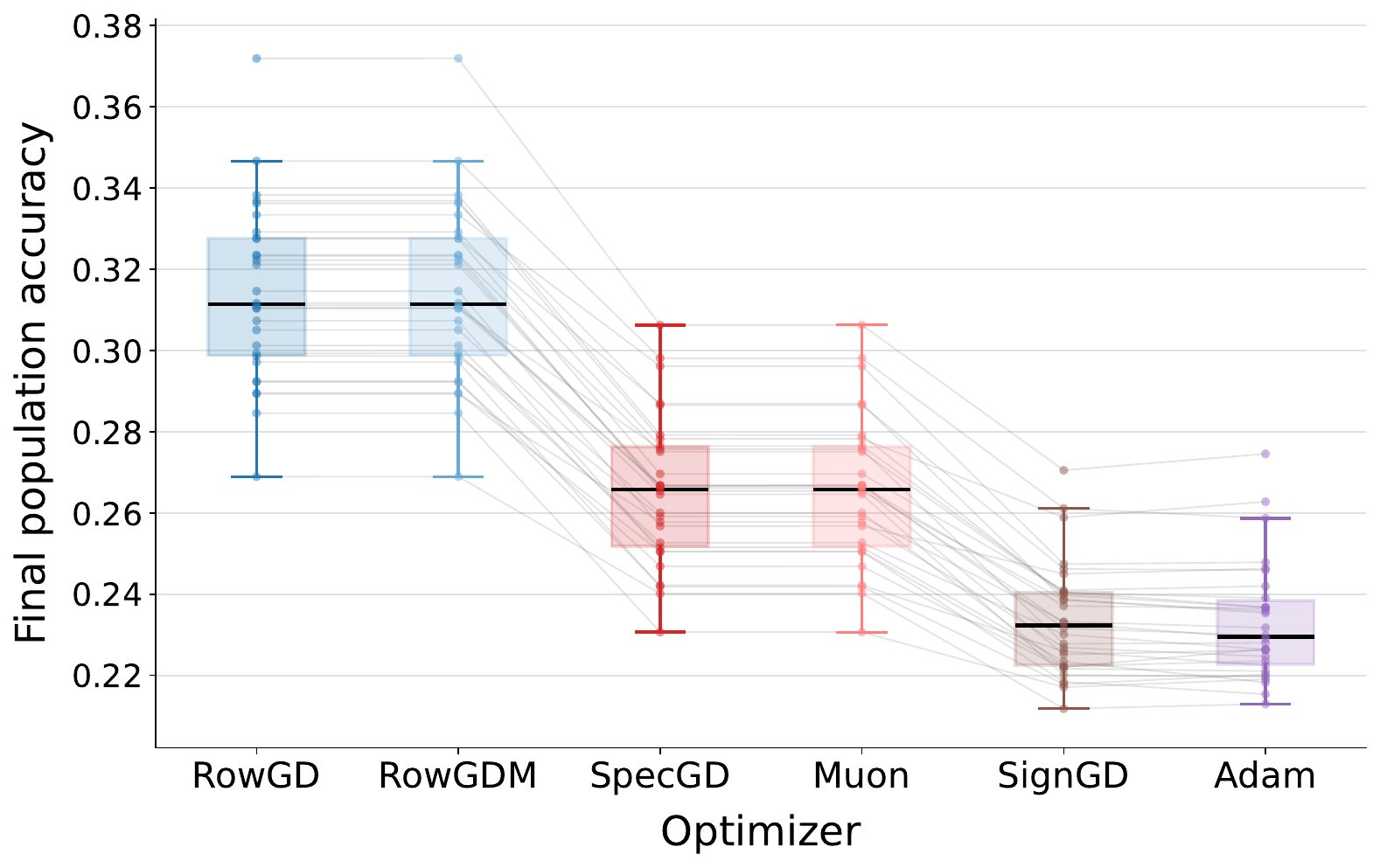}
\end{minipage}

\vspace{2mm}

\begin{minipage}[t]{0.49\linewidth}
\centering
\includegraphics[width=\linewidth]{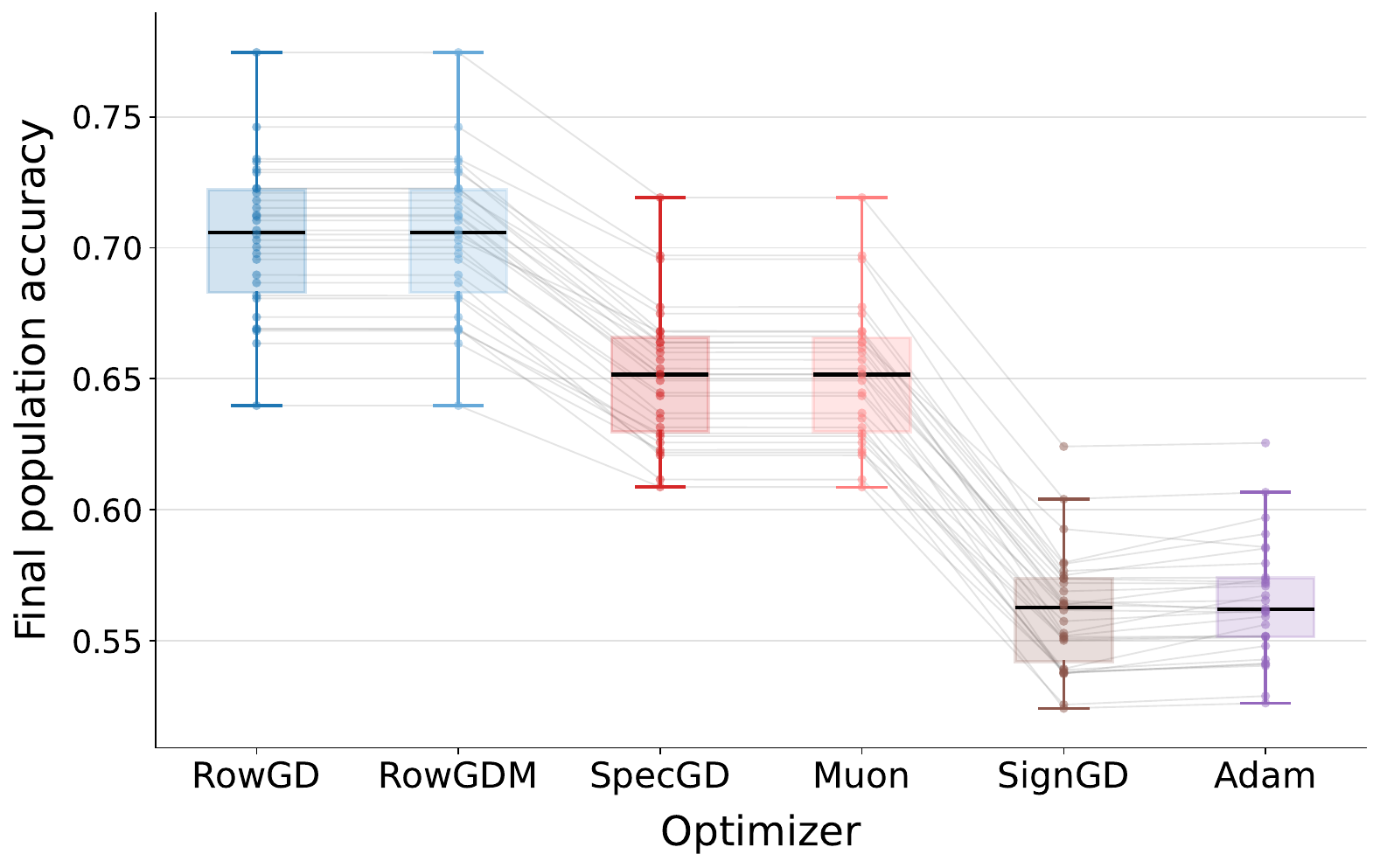}
\end{minipage}
\hfill
\begin{minipage}[t]{0.49\linewidth}
\centering
\includegraphics[width=\linewidth]{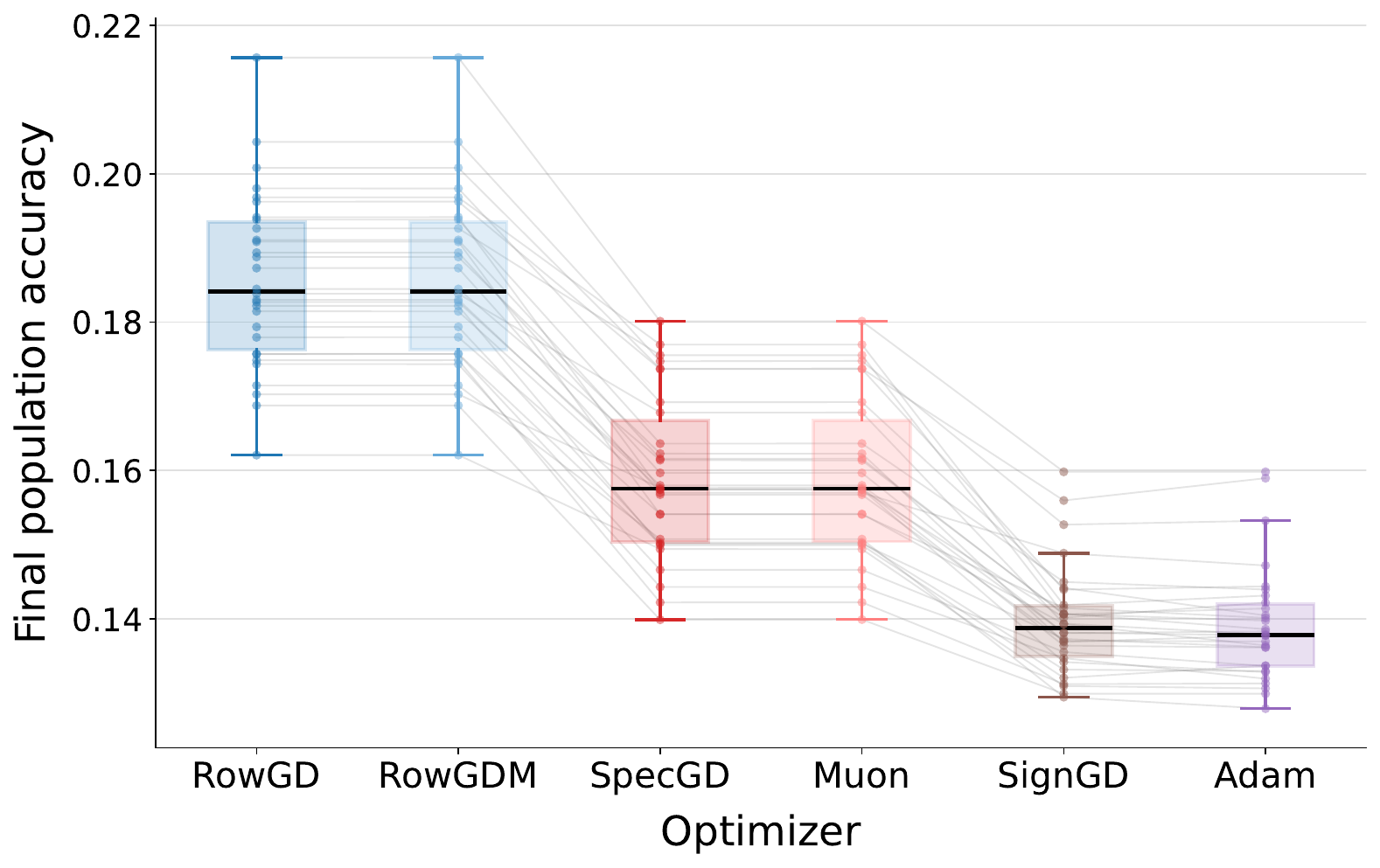}
\end{minipage}
\caption{\textbf{Robustness to training and test noise.} Top: training-noise coefficient \(0.0447\) (left) and \(0.4472\) (right), compared with \(0.2236\) in Figure~\ref{fig:experiments} (left). Bottom: test-noise coefficient \(1.3416\) (left) and \(3.1305\) (right), compared with \(2.2361\) in the main-text setting. In all cases, the residual vectors are drawn from \(\mathcal N({\mathbf 0},I_d/d)\).}
\label{fig:synthetic-robust-noise}
\end{figure}

\begin{figure}[t]
\centering
\begin{minipage}[t]{0.49\linewidth}
\centering
\includegraphics[width=\linewidth]{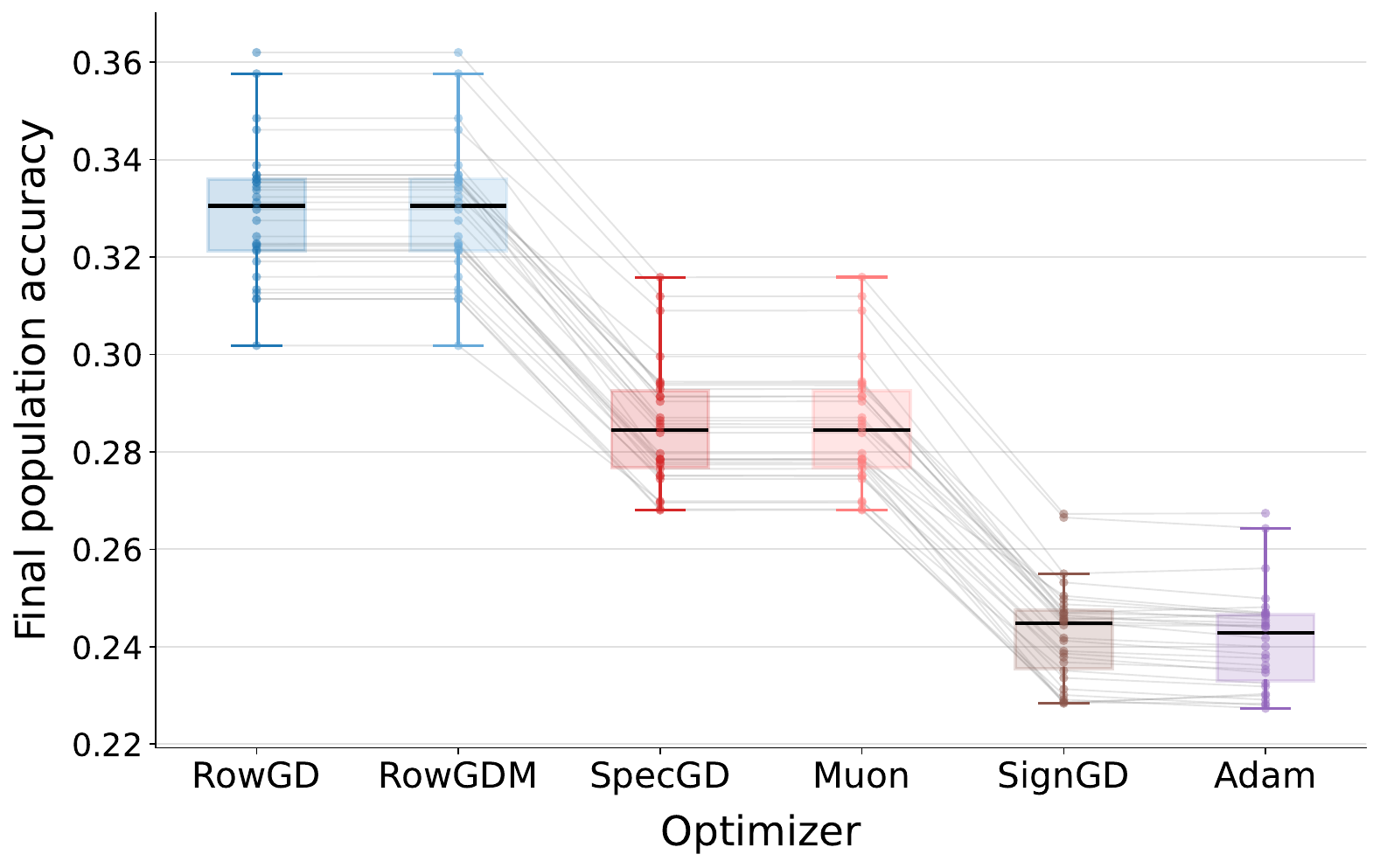}
\end{minipage}
\hfill
\begin{minipage}[t]{0.49\linewidth}
\centering
\includegraphics[width=\linewidth]{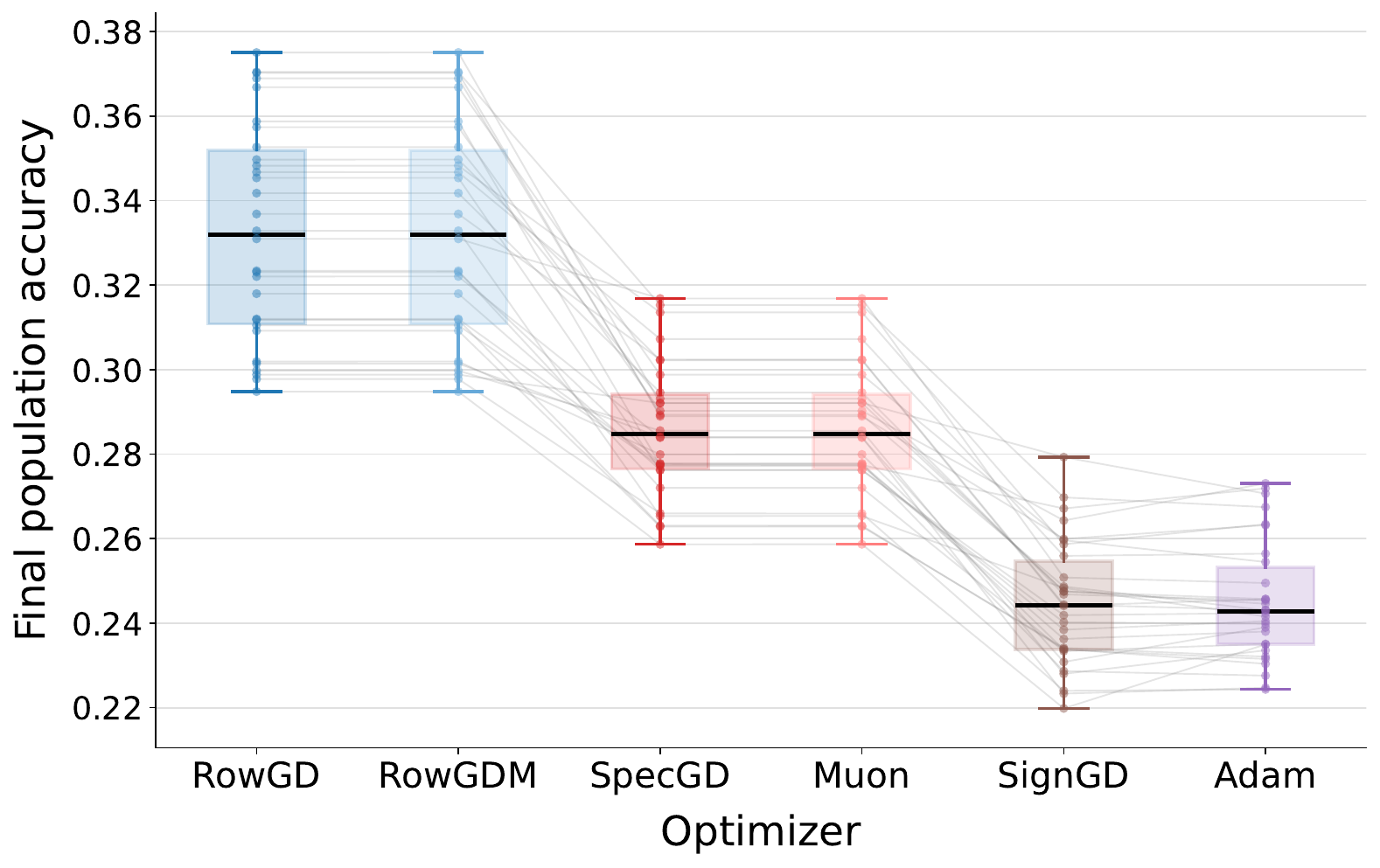}
\end{minipage}

\vspace{2mm}

\begin{minipage}[t]{0.49\linewidth}
\centering
\includegraphics[width=\linewidth]{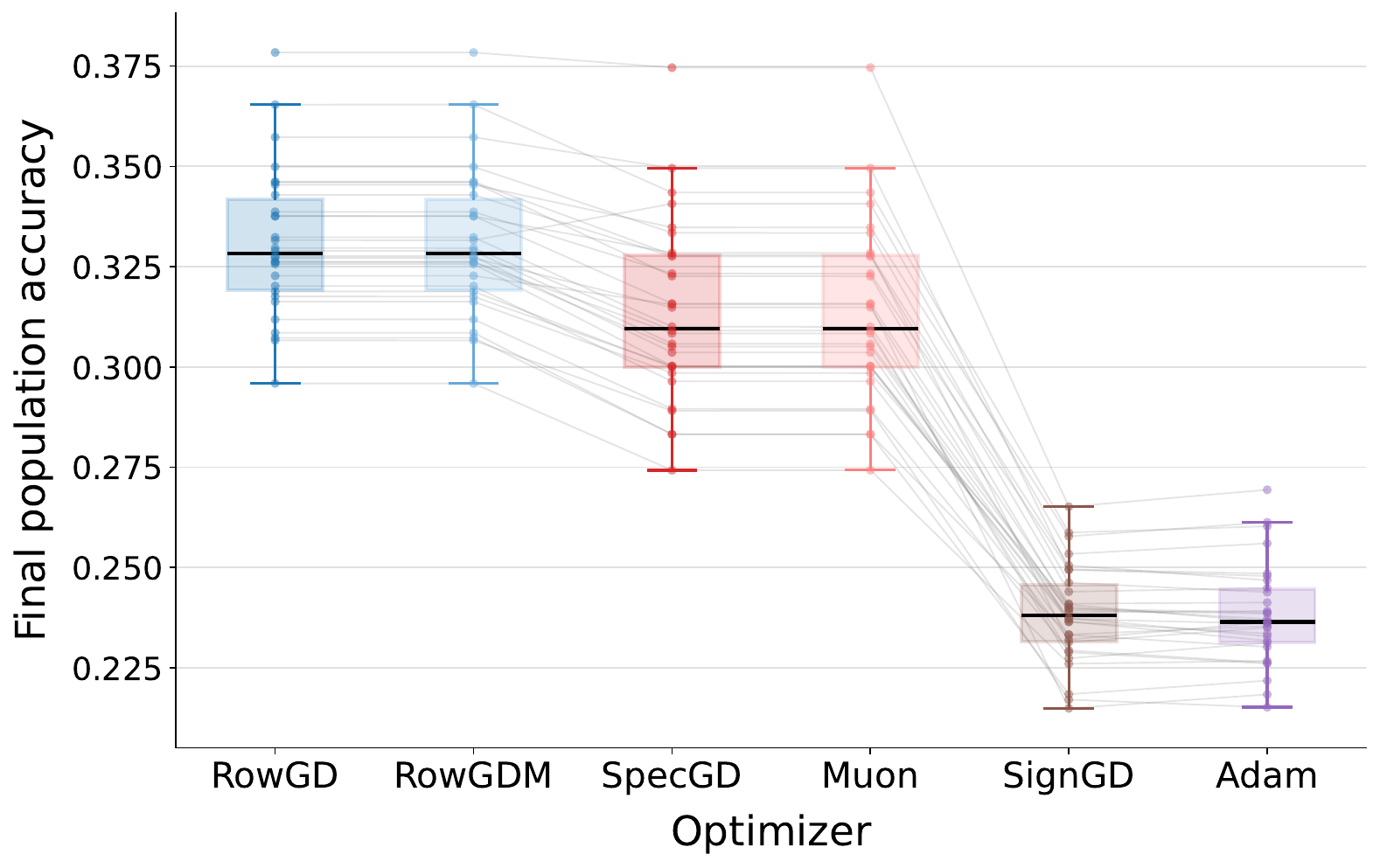}
\end{minipage}
\hfill
\begin{minipage}[t]{0.49\linewidth}
\centering
\includegraphics[width=\linewidth]{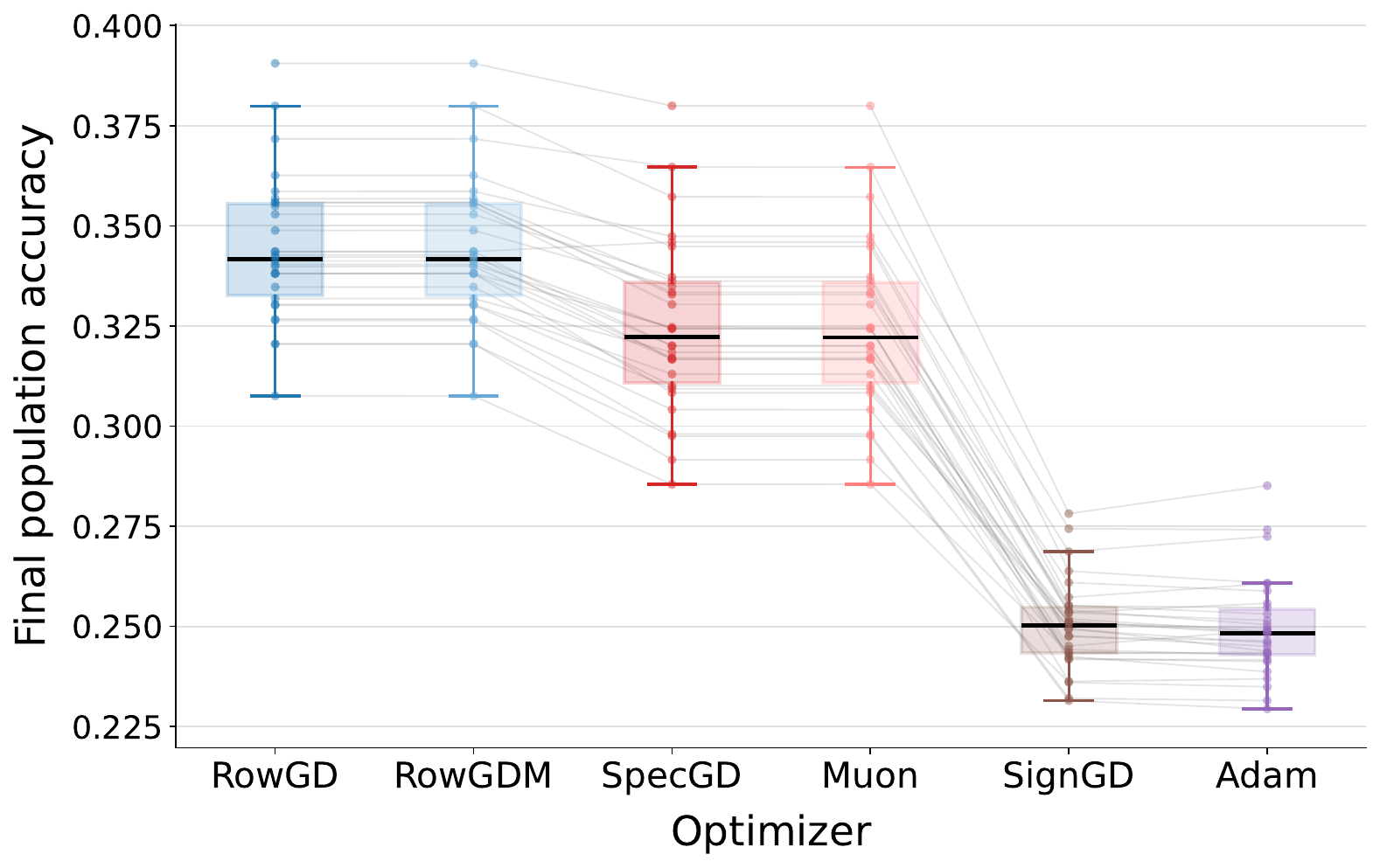}
\end{minipage}
\caption{\textbf{Robustness to class proportions and training sample count.} Top: balanced test classes (left) and \(\operatorname{Dirichlet}(0.3\mathbf 1_k)\) test proportions (right), compared with \(\operatorname{Dirichlet}(\mathbf 1_k)\) in Figure~\ref{fig:experiments} (left). Bottom: exactly one (left) and five (right) training examples per class, compared with independent counts drawn uniformly from \(\{1,\ldots,5\}\) in the main-text setting.}
\label{fig:synthetic-robust-data}
\end{figure}

\begin{figure}[t]
\centering
\begin{minipage}[t]{0.49\linewidth}
\centering
\includegraphics[width=\linewidth]{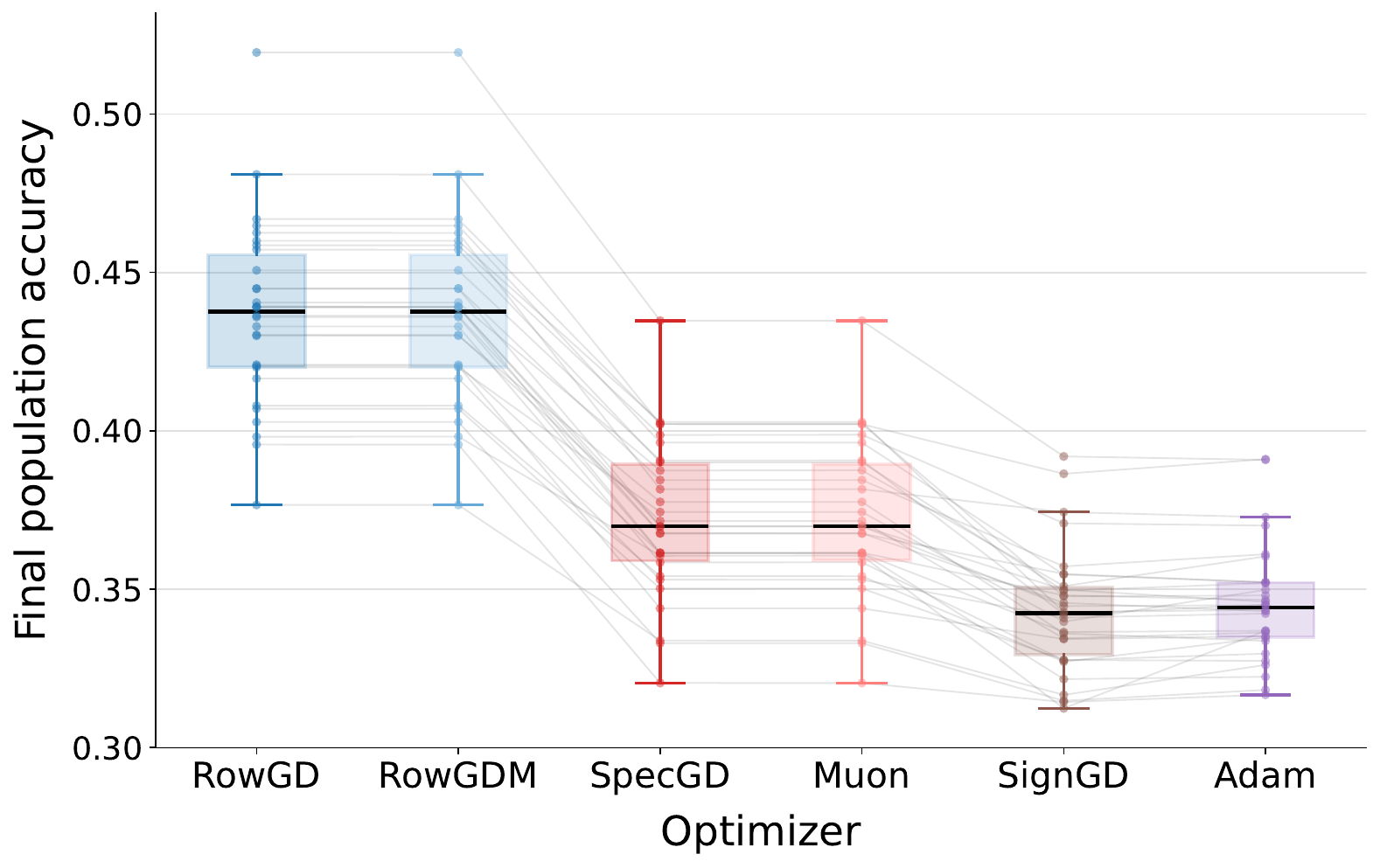}
\end{minipage}
\hfill
\begin{minipage}[t]{0.49\linewidth}
\centering
\includegraphics[width=\linewidth]{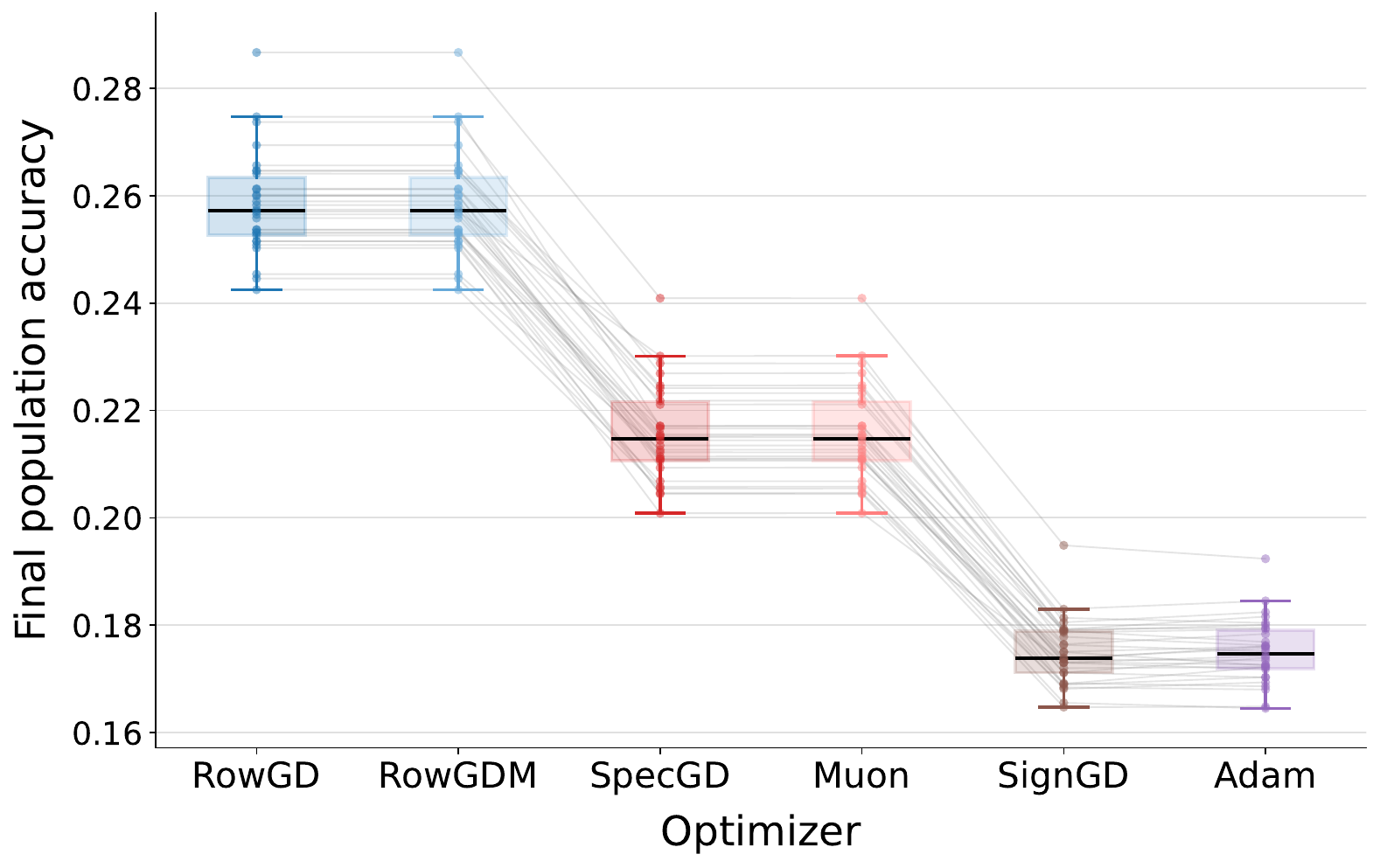}
\end{minipage}

\vspace{2mm}

\begin{minipage}[t]{0.49\linewidth}
\centering
\includegraphics[width=\linewidth]{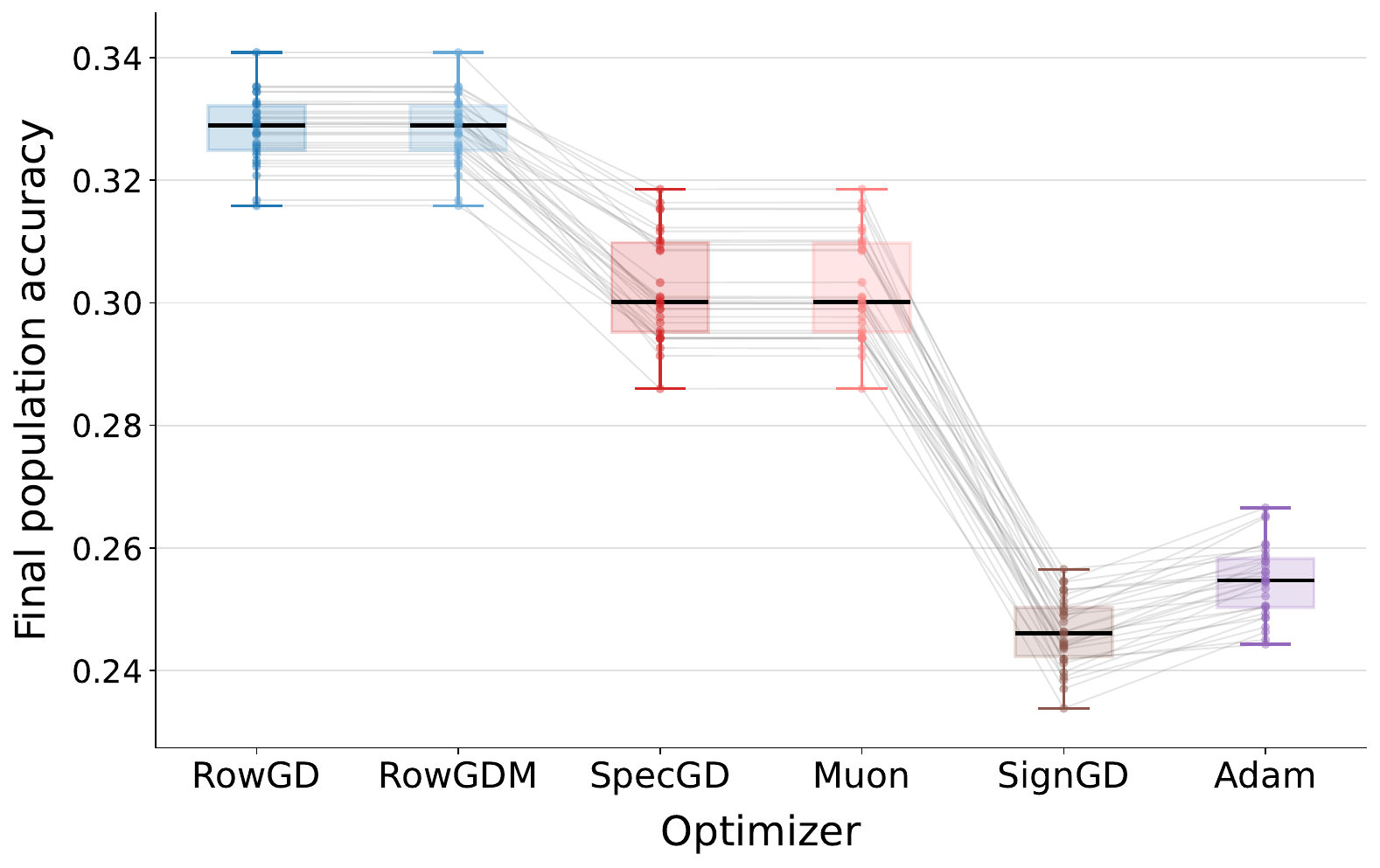}
\end{minipage}
\caption{\textbf{Robustness to problem size and class-mean distribution.} Top: \((k,d)=(30,20)\) (left) and \((120,40)\) (right), compared with \((60,20)\) in Figure~\ref{fig:experiments} (left). Bottom: class means drawn independently and uniformly from the unit sphere instead of from \(\mathcal N({\mathbf 0},I_d/d)\).}
\label{fig:synthetic-robust-size-means}
\end{figure}

\paragraph{Anisotropic synthetic experiments.}
We also test the anisotropic predictions of Section~\ref{sec:anisotropic-extension}. Unless otherwise noted, the training and evaluation protocol is the same as in the baseline synthetic experiment above. Figure~\ref{fig:synthetic-anisotropic-independent} shows a representative independent-orientation setting with power-law exponents \(s_\mu=0.5\) and \(s_\Sigma=2\); across six such settings, RowGD and RowGDM outperform SpecGD, Muon, SignGD, and Adam on all \(30/30\) paired seeds in every setting, with conservative mean advantages between \(2.82\) and \(5.73\) percentage points. Figure~\ref{fig:synthetic-anisotropic-coordinate} illustrates the coordinate-aligned reversal and recovery phenomenon from Proposition~\ref{prop:main-coordinate-reversal}: in the diagonal setting, RowGDM outperforms Adam on all \(30/30\) seeds at \(\lambda=0.25\), Adam outperforms RowGDM on all \(30/30\) seeds at \(\lambda=1\), and applying the same Haar rotation restores the RowGDM-over-Adam ordering on all \(30/30\) seeds. The coordinate-aligned reversal-and-recovery experiments use \(10^6\) training steps to reduce sensitivity to finite-time transients.

\begin{figure}[t]
\centering
\includegraphics[width=0.75\linewidth]{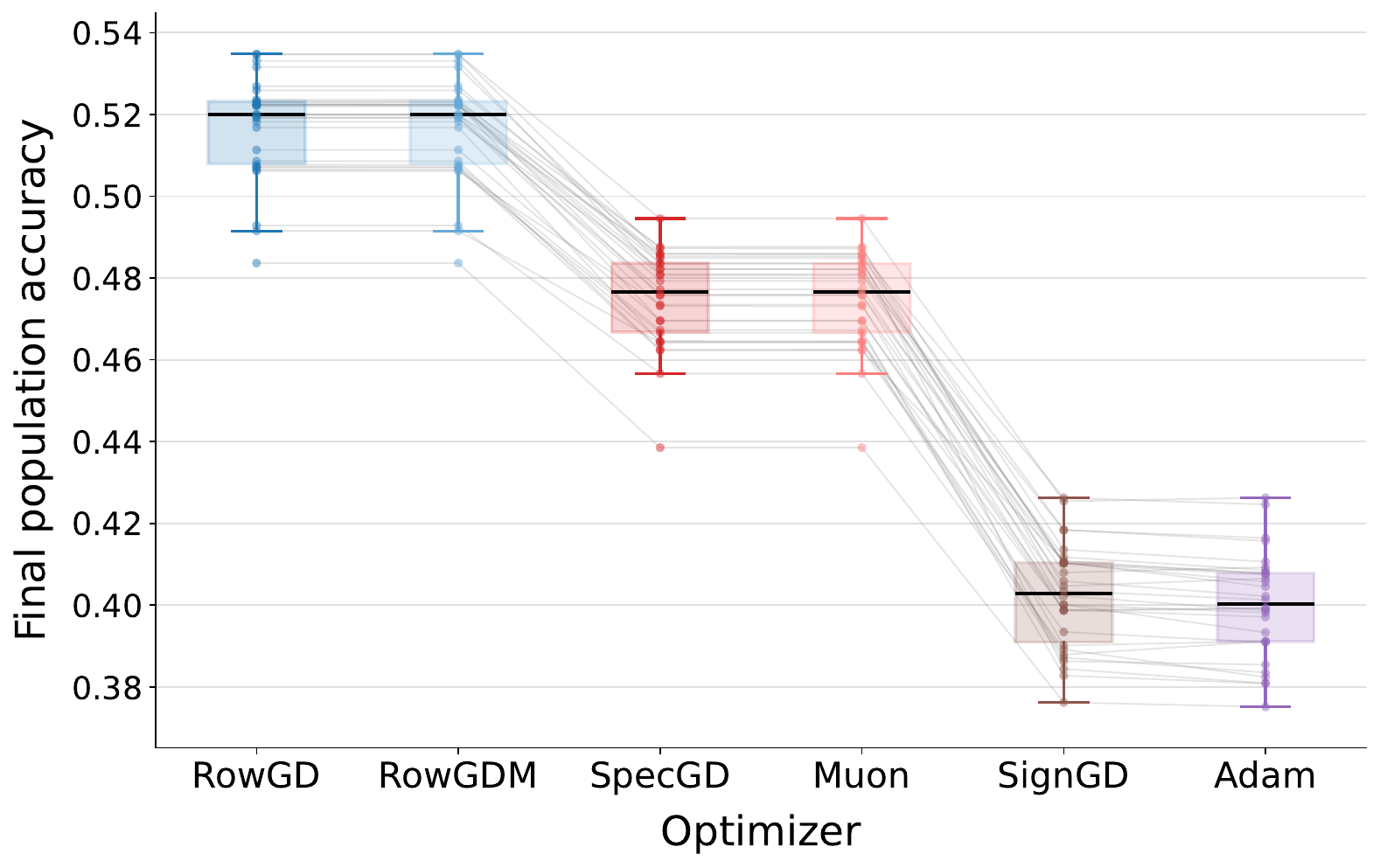}
\caption{\textbf{Independent-orientation anisotropic experiment.} Representative setting with \(s_\mu=0.5\) and \(s_\Sigma=2\), where RowGD and RowGDM outperform SpecGD, Muon, SignGD, and Adam on all \(30\) paired seeds.}
\label{fig:synthetic-anisotropic-independent}
\end{figure}

\begin{figure}[t]
\centering
\includegraphics[width=1.0\linewidth]{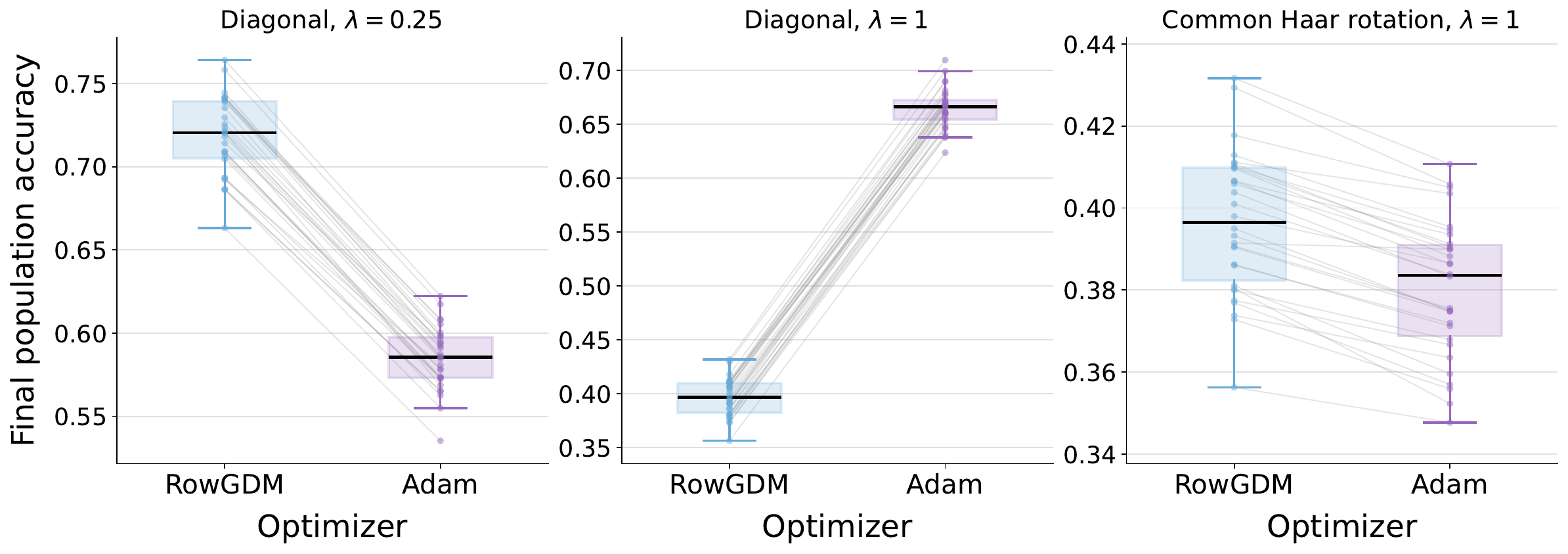}
\caption{\textbf{Coordinate-aligned reversal and recovery.} RowGDM outperforms Adam at \(\lambda=0.25\), Adam outperforms RowGDM at \(\lambda=1\), and the RowGDM-over-Adam ordering is restored after applying the same Haar rotation.}
\label{fig:synthetic-anisotropic-coordinate}
\end{figure}

\FloatBarrier

\subsection{Language Model Experiment}
\label{app:language-model-experiment-details}

\paragraph{Model and starting checkpoint.}
We use a LLaMA-\(60\)M causal language model with \(58.07\)M parameters exactly, consisting of \(8\) transformer layers with hidden size \(512\), intermediate size \(1376\), \(8\) attention heads, and vocabulary size \(32{,}000\). The starting checkpoint is obtained by training the full model from scratch on C4 for \(11{,}000\) updates with global batch size \(512\), sequence length \(256\), a peak learning rate of \(0.005\), cosine decay, and \(1{,}100\) warmup steps. This pretraining uses a hybrid optimizer: Muon is applied to the two-dimensional attention and MLP projection matrices, while AdamW is used for the token embedding, language-model head, and one-dimensional parameters. The resulting checkpoint has validation perplexity \(28.23501\). For the last-layer experiment, we freeze the entire model body, discard the pretrained output-head weights, zero-initialize the \(32{,}000\times512\) output head, and train only this \(16{,}384{,}000\)-parameter matrix.

\paragraph{Last-layer training horizon.}
The main experiment trains the zero-initialized output head for \(22{,}000\) updates in bfloat16 using distributed data parallel training on two NVIDIA RTX A6000 GPUs, with global batch size \(512\), sequence length \(256\), a cosine learning-rate schedule with \(1{,}100\) warmup steps, and zero weight decay and gradient clipping. An \(11{,}000\)-update last-layer run processes \(1.10\)B non-padding training tokens, or approximately \(18.9\) tokens per total model parameter, close to the conventional \(20\)-tokens-per-parameter Chinchilla reference; the \(22{,}000\)-update horizon approximately doubles this token budget to \(2.20\)B tokens. We use the longer horizon to examine a later-training regime in which transient optimization effects are reduced and the trajectories may be closer to the limiting-direction behavior studied theoretically, without assuming that directional convergence has occurred by \(22{,}000\) updates. The Chinchilla comparison is only a token-budget reference, since the model body is frozen and only the output head is optimized in this experiment.

\paragraph{Horizon check.}
We also run a separate \(11{,}000\)-update comparison with its own cosine schedule ending at \(11{,}000\) steps and learning rates selected independently at that horizon. The selected learning rates are \(0.00625\) for RowSGDM, \(0.0045\) for Muon, and \(0.005\) for Adam, yielding final monitoring-set perplexities \(28.5871\), \(28.6373\), and \(28.5914\), respectively. Thus, RowSGDM is already numerically best at approximately the Chinchilla token-budget reference, although its advantage over Adam is only \(0.0152\%\) in this single-seed comparison. The \(22{,}000\)-step experiment therefore probes whether the qualitative ordering persists and becomes clearer at a later training horizon rather than relying on the \(11{,}000\)-step difference alone.

\paragraph{Data and evaluation.}
We use the Hugging Face \texttt{allenai/c4} dataset in streaming mode and tokenize with \texttt{t5-base}, truncating and padding sequences to length \(256\). For each training seed, RowSGDM, Muon, and Adam receive the same shuffled stream and batch sequence. This streaming shuffle is not the epoch-wise permutation of a fixed finite dataset used by the random-reshuffling optimizers in Section~\ref{sec:preliminaries}. The validation trajectory in Figure~\ref{fig:experiments} (right) is obtained from a monitoring set evaluated every \(1{,}000\) updates and containing \(7{,}625{,}745\) valid next-token targets. Final evaluation uses a separate set of \(237{,}723\) C4 validation documents containing \(46{,}186{,}882\) valid targets and having zero document-hash overlap with the monitoring set. The held-out set is not used for learning-rate selection, early stopping, or checkpoint selection. NLL is averaged over the exact number of valid non-padding shifted targets, and perplexity is its exponential.

\paragraph{Optimizers and learning rates.}
RowSGDM uses momentum \(0.9\), row-wise RMS normalization, multiplier \(0.2\), and numerical stabilizer \(\epsilon=10^{-8}\). Muon uses momentum \(0.95\), five Newton--Schulz iterations, multiplier \(0.2\), and zero weight decay. Adam uses \((\beta_1,\beta_2)=(0.9,0.95)\), \(\epsilon=10^{-8}\), and zero weight decay. These implementations are practical counterparts of the RowSGDM, exact-SVD Muon, and zero-stability-constant Adam considered in our theory. For the \(22{,}000\)-step experiment, learning rates are selected using the monitoring set: we sweep \(\{0.004,0.005,0.006,0.007\}\) for RowSGDM and Adam and \(\{0.0025,0.003,0.0035,0.004,0.0045,0.005,0.0055,0.006\}\) for Muon, selecting \(0.005\), \(0.005\), and \(0.004\), respectively.

\paragraph{Boundary-alignment diagnostic.}
We additionally test whether the performance ordering is accompanied by the boundary geometry predicted by our analysis. For each valid held-out next-token position, let \(h_t\in\mathbb R^{512}\) denote the frozen hidden representation immediately before the output head and let \(y_t\) denote its next-token class. For class \(i\), we estimate \(\widehat\mu_i=n_i^{-1}\sum_{t:y_t=i}h_t\) and retain the \(5{,}093\) classes with \(n_i\ge1000\). We uniformly sample a fixed set \(\mathcal P\) of \(1{,}000{,}000\) unordered class pairs and use the same pairs for every optimizer and seed. For a trained output head with rows \(w_i^\top\), define
\[
\widehat u^W_{ij}:=\frac{w_i-w_j}{\|w_i-w_j\|_2},\qquad \widehat u^\mu_{ij}:=\frac{\widehat\mu_i-\widehat\mu_j}{\|\widehat\mu_i-\widehat\mu_j\|_2},
\]
and measure
\[
\widehat D_{\mathrm{avg\text{-}bdry}}(W):=\frac{1}{|\mathcal P|}\sum_{\{i,j\}\in\mathcal P}\|\widehat u^W_{ij}-\widehat u^\mu_{ij}\|_2,\qquad \widehat C_{\mathrm{bdry}}(W):=\frac{1}{|\mathcal P|}\sum_{\{i,j\}\in\mathcal P}\langle\widehat u^W_{ij},\widehat u^\mu_{ij}\rangle.
\]
Lower \(\widehat D_{\mathrm{avg\text{-}bdry}}\) and higher \(\widehat C_{\mathrm{bdry}}\) indicate closer alignment between learned classifier boundaries and empirical class-mean differences.

\paragraph{Results.}
Consistent with the final validation-perplexity ordering in Figure~\ref{fig:experiments} (right), RowSGDM achieves lower held-out NLL and perplexity than both Muon and Adam on each of the five paired seeds. RowSGDM also has the smallest boundary discrepancy and the largest boundary cosine on every seed, so its performance advantage is accompanied by closer alignment between learned classifier boundaries and empirical class-mean differences. Table~\ref{tab:language-model-last-layer-results} summarizes the five-seed held-out results.

\begin{table}[t]
\centering
\footnotesize
\setlength{\tabcolsep}{4pt}
\begin{tabular}{@{}lcccc@{}}
\toprule
Optimizer & NLL \(\downarrow\) & PPL \(\downarrow\) & \(\widehat D_{\mathrm{avg\text{-}bdry}}\downarrow\) & \(\widehat C_{\mathrm{bdry}}\uparrow\) \\
\midrule
RowSGDM & \(\mathbf{3.35503\pm0.00005}\) & \(\mathbf{28.6466\pm0.0013}\) & \(\mathbf{1.13194\pm0.00004}\) & \(\mathbf{0.35652\pm0.00005}\) \\
Muon & \(3.35699\pm0.00004\) & \(28.7027\pm0.0012\) & \(1.13807\pm0.00006\) & \(0.35018\pm0.00007\) \\
Adam & \(3.35759\pm0.00002\) & \(28.7199\pm0.0006\) & \(1.13350\pm0.00008\) & \(0.35481\pm0.00009\) \\
\bottomrule
\end{tabular}
\caption{Held-out performance and empirical boundary alignment for last-layer-only LLaMA training. Values are mean \(\pm\) sample standard deviation across five paired training seeds.}
\label{tab:language-model-last-layer-results}
\end{table}

The mean paired RowSGDM-minus-Muon NLL difference is \(-0.001957\), with a descriptive \(95\%\) Student-\(t\) interval \([-0.002058,-0.001857]\), and the RowSGDM-minus-Adam difference is \(-0.002555\), with interval \([-0.002621,-0.002489]\). The mean reductions in \(\widehat D_{\mathrm{avg\text{-}bdry}}\) are \(0.006137\) relative to Muon and \(0.001558\) relative to Adam. Together, the validation, held-out, and boundary-alignment results are consistent with the qualitative ordering and geometric mechanism predicted by our theory.

\paragraph{LLaMA-130M experiment.} We additionally repeat the last-layer comparison on a separately C4-pretrained LLaMA-130M backbone, whose actual parameter count is \(134.1\)M, with \(12\) transformer layers and hidden size \(768\).
The tokenizer, sequence length, global batch size, evaluation sets, optimizer definitions, and five paired seeds are the same as above. We train the zero-initialized \(32{,}000\times768\) output head for \(55{,}000\) updates, with each selected run processing approximately \(5.47\)B valid next-token targets, and a common full-horizon learning-rate sweep selects \(0.002828\) for all three optimizers. RowSGDM attains the lowest held-out perplexity on every seed, with five-seed mean \(\pm\) sample standard deviation \(22.3797\pm0.0014\), compared with \(22.4096\pm0.0030\) for Muon and \(22.4153\pm0.0010\) for Adam. RowSGDM also yields the smallest boundary discrepancy on every seed, with mean \(\widehat D_{\mathrm{avg\text{-}bdry}}=1.19051\), compared with \(1.19664\) for Adam and \(1.20028\) for Muon, providing a larger-model check of the same qualitative advantage.

\FloatBarrier

\subsection{Spectral Exponents for Language Model Representation}
\label{app:language-model-represen}

We empirically examine whether a fitted class-mean spectral exponent below one can occur in language-model representations, as in the regime \(s_\mu<1\) of Theorem~\ref{thm:aniso-dd-ordering}. We also measure class-conditional within-class covariance spectral exponents as descriptive statistics of within-class anisotropy. These measurements are finite-dimensional spectral summaries and are not intended as a direct test of the full probabilistic model underlying the theorem.

\paragraph{Experimental setup.} We analyze final pre-head representations from OLMoE-1B-7B-0125 \citep{muennighoff2024olmoeopenmixtureofexpertslanguage}, Qwen3-30B-A3B \citep{qwen3technicalreport}, and AMUSE-720M \citep{kim2026amuse}, all with hidden dimension \(d=2048\). For each model, we collect exactly \(20{,}000{,}000\) valid same-document next-token positions from the English validation split of C4. At a valid position, let \(h_{it}\in\mathbb R^d\) denote the final normalized hidden representation when the next-token class is \(i\), and let \(n_i\) be the number of observed positions from class \(i\). We define the empirical class mean by
\begin{equation}
\widehat\mu_i:=\frac{1}{n_i}\sum_{t=1}^{n_i}h_{it}.
\label{eq:app-lm-observed-class-mean}
\end{equation}

\paragraph{Class-mean spectral exponent.} Our primary mean-spectrum analysis assigns equal weight to every observed next-token class. Let
\[
\mathcal I_\mu:=\{i:n_i\ge1\},\qquad K:=|\mathcal I_\mu|,
\]
and define
\begin{equation}
\widehat{\bar\mu}:=\frac{1}{K}\sum_{i\in\mathcal I_\mu}\widehat\mu_i,\qquad \widehat\Gamma:=\frac{1}{K}\sum_{i\in\mathcal I_\mu}(\widehat\mu_i-\widehat{\bar\mu})(\widehat\mu_i-\widehat{\bar\mu})^\top.
\label{eq:app-lm-all-observed-mean-cov}
\end{equation}
Thus \(\widehat\Gamma\) is class balanced rather than token-frequency weighted; in particular, every observed class contributes one empirical class mean. If \(\widehat\lambda_1^\mu\ge\cdots\ge\widehat\lambda_{2048}^\mu>0\) are the eigenvalues of \(\widehat\Gamma\), we estimate the mean exponent by the unweighted pointwise OLS fit
\begin{equation}
\log\widehat\lambda_r^\mu=a-\widehat{s}_\mu\log r+\varepsilon_r,\qquad r=1,\ldots,2048.
\label{eq:app-lm-mean-spectral-fit}
\end{equation}
Each integer rank therefore contributes one regression observation; the fit is not log binned or rank weighted.

\paragraph{Within-class covariance spectral exponents.} For classes with sufficiently many observations, we form the within-class sample covariance
\begin{equation}
\widehat\Sigma_i:=\frac{1}{n_i-1}\sum_{t=1}^{n_i}(h_{it}-\widehat\mu_i)(h_{it}-\widehat\mu_i)^\top
\label{eq:app-lm-residual-cov}
\end{equation}
and fit
\begin{equation}
\log\widehat\lambda_{i,r}^{\Sigma}=a_i-\widehat{s}_{\Sigma,i}\log r+e_{i,r},\qquad r=1,\ldots,2048.
\label{eq:app-lm-residual-spectral-fit}
\end{equation}
The reliable-class threshold is \(n_i\ge8192\) for OLMoE and AMUSE and \(n_i\ge16384\) for Qwen3. This leaves \(218\), \(211\), and \(120\) classes with reliable within-class covariance estimates, respectively.

\begin{table}[t]
\centering
\caption{Full-rank pointwise-OLS spectral exponents at the three final checkpoints. \(K\) is the number of observed classes used for the class-mean covariance, and \(K_\Sigma\) is the number of reliable classes used for the within-class covariance summary.}
\label{tab:app-lm-spectral-exponents}
\small
\begin{tabular}{@{}lrrrrrr@{}}
\toprule
Model & \(K\) & \(\widehat{s}_\mu\) & \(K_\Sigma\) & Mean \(\widehat{s}_{\Sigma,i}\) & Median & Max \\
\midrule
OLMoE-1B-7B-0125 & \(47{,}190\) & \(0.957735\) & \(218\) & \(0.987339\) & \(0.973537\) & \(1.374366\) \\
Qwen3-30B-A3B & \(72{,}577\) & \(0.932780\) & \(120\) & \(1.057924\) & \(1.045483\) & \(1.258163\) \\
AMUSE-720M & \(49{,}348\) & \(0.985366\) & \(211\) & \(1.054793\) & \(1.045965\) & \(1.402045\) \\
\bottomrule
\end{tabular}
\end{table}

\paragraph{Results.} For the three final checkpoints examined here, the primary class-mean point estimates in Table~\ref{tab:app-lm-spectral-exponents} all satisfy \(\widehat{s}_\mu<1\): \(0.957735\) for OLMoE, \(0.932780\) for Qwen3, and \(0.985366\) for AMUSE. The within-class covariance exponents are more heterogeneous and can exceed one, as reflected by the maxima in the table.

\paragraph{OLMoE across annealing.} We additionally apply the \(\widehat{s}_\mu\) estimator to six published OLMoE snapshots spanning \(0\)--\(101\)B added annealing tokens. The resulting estimates are \(0.956727\), \(0.957792\), \(0.956569\), \(0.956530\), \(0.958359\), and \(0.957735\). Thus, every sampled estimate is below one, the total range is only \(0.001830\), and there is no monotone trend across the sampled snapshots.

\end{document}